\pdfoutput=1
\documentclass[twoside]{article}
\usepackage[accepted]{aistats2016}

\usepackage{bm}
\usepackage{times} 
\usepackage{color}
\usepackage{algorithm, algorithmic}

\input{header-std.tex}

\newcommand{\tA}{\tilde{A}}
\newcommand{\cT}{\mathcal{T}}
\newcommand{\MP}{\mathbb{M}}
\newcommand{\LP}{\mathbb{L}}
\newcommand{\FMF}{\tilde{\F}_M}
\newcommand{\pt}{(\theta)}
\newcommand{\mh}{\hat{\mu}}
\newcommand{\si}{^{(i)}}

\begin{document}

\twocolumn[
\aistatstitle{Clamping Improves TRW and Mean Field Approximations} 
\aistatsauthor{Adrian Weller \And Justin Domke}
\aistatsaddress{University of Cambridge \And NICTA, The Australian National University}%{DRAFT: Please do not distribute}
]

\begin{abstract}
We examine the effect of clamping variables for approximate inference in undirected graphical models with pairwise relationships and discrete variables. For any number of variable labels, we demonstrate that clamping and summing approximate sub-partition functions can lead only to a decrease in the partition function estimate for TRW, and an increase for the naive mean field method, in each case guaranteeing an improvement in the approximation and bound. We next focus on binary variables, add the Bethe approximation to consideration and examine ways to choose good variables to clamp, introducing new methods. We show the importance of identifying highly frustrated cycles, and of checking the singleton entropy of a variable. 
We  
explore the value of our methods 
by empirical analysis %using a range of popular inference methods, 
and draw lessons to guide practitioners. 
\end{abstract}

\section{INTRODUCTION}

Undirected graphical models, also called Markov random fields (MRFs), are a powerful and compact way to represent dependencies between variables, and have become a central tool in machine learning. 
A key challenge is to estimate the normalizing partition function. For example, this may be used to compute the probability of evidence, and is often a critical component of learning a model. 
An exact solution may be obtained via the junction tree method but unless the treewidth is bounded, this can take exponential time \citep{LauSpi88short}. %,WaiJor08}. 
Hence, many approximate methods have been developed.

We  focus on three popular approaches: 
the \emph{tree-reweighted} approximation (TRW, \citealp{WJW05});  
the \emph{na\"{i}ve mean field} approximation (MF);  
and the \emph{Bethe} approximation, often implemented via belief propagation (BP, \citealp{Pearl,Yedidia}).  
In each case, we shall examine the effect on the respective partition function estimate of \emph{clamping} one or more variables to each possible setting then combining the approximate results obtained on the clamped sub-models. 
See \S \ref{sec:prelim} for all definitions. 
If all variables are clamped, then the exact solution is obtained but with time exponential in the number of variables. Intuitively, as more variables are clamped, one would hope for better results, but this is not always the case and demonstrating guarantees has been challenging. 

\cite{zbn} recently proved that for an \emph{attractive} binary pairwise model (where it is known that the Bethe partition function yields a lower bound), the optimum Bethe partition function approximation can only increase (hence improve) for each variable clamped. They also provided an example of a non-attractive model (their Figure 5c)  %a non-attractive model 
where clamping any variable leads to a \emph{worse} approximation. Nevertheless, 
they introduced two heuristics for identifying a good variable to clamp, and for both attractive and mixed models, demonstrated empirically that approximation error can sometimes be significantly reduced by clamping even one variable. 

We make the following contributions. For both TRW (which yields an upper bound) and MF (which provides a lower bound), we show that for pairwise models with any number of labels, and with any types of potentials, clamping can only improve the partition function estimate by decreasing and increasing the bounds respectively. Our proofs also yield insight into the approximate marginals returned. We next examine how to select a good \emph{sequence} of variables to clamp. Although the methods of \cite{zbn} can perform  well for choosing one variable, we show that, for some models, their methods perform poorly, particularly for selecting multiple variables. We introduce methods that strip a model to its \emph{core}, search for strongly \emph{frustrated cycles}, and make use of approximate singleton entropy.  %\emph{clumps}. 
We provide an empirical analysis of all approaches, including a comparison against the `greedy' choice of the best variable to clamp in hindsight after an exhaustive exploration.  
We conclude with observations to help guide practitioners. 

\subsection{Related Work}
%\paragraph{Related work.}
The technique of 
branching or conditioning on variables, and approximating over the remaining variables has been explored in algorithms such as branch-and-cut \citep{padberg1991branch,mitchell2002branch}, work on resolution versus search \citep{rish2000resolution} and %the various conditioning approaches of 
in 
\cite[Chapter 8]{DarBook}. 
Cutset conditioning was discussed by \cite{Pearl} and refined by \cite{Peot91} as a method to render the remaining topology acyclic before using belief propagation. \cite{Eat09} developed this further, introducing  \emph{conditioned belief propagation}. %  together with \emph{back-belief-propagation}  to help identify which variables to clamp. 
\cite{Liu12} explored feedback message passing for inference in Gaussian (not discrete) models, deriving strong results for attractive models. 
\cite{BouZoe09} discuss soft-binning to split configurations into subsets then apply the mean field approximation on each but without guarantees.   
\cite{ChoiDar08} examined methods to approximate the partition function by deleting edges.

\section{PRELIMINARIES AND NOTATION}\label{sec:prelim}

We consider pairwise models with $n$ variables $X_1,\dots,X_n$ and graph topology $(\V,\ce)$: $\V$ contains nodes $\{1,\dots,n\}$ where $i$ corresponds to $X_i$, and $\ce \subseteq \V \times \V$ contains an edge for each pairwise  relationship. Sometimes we consider multi-label models where each  $X_i \in \{0, \dots, L_i-1\}$, and sometimes we restrict attention to binary models where $X_i \in \mathbb{B}=\{0,1\} \; \forall i$. Let $x=(x_1,\dots , x_n)$ be a configuration of all variables, and $\N(i)$ be the neighbors of $i$.  

We consider probability distribution $p(x)=e^{-E(x)}/Z(\theta)$, where $E(x)$ is the energy of configuration $x$. The \emph{partition function} $Z(\theta)$ is a quantity of fundamental interest. It requires summing over all states to yield the normalizing constant $Z(\theta)=\sum_x \exp \left( -E(x) \right)$ which ensures that  $\sum_x p(x)=1$. We denote the log-partition function, sometimes called the cumulant function, by $A(\theta)=\log Z(\theta)$. 

For binary models, we assume a reparameterization such that $E(x)=-\sum_{i \in \V} \h x_i -\sum_{(i,j)\in \ce} \frac{\W}{2} \left[ x_i x_j + (1-x_i)(1-x_j) \right]$, with singleton potentials $\h$ and edge weights $\W$. If $W_{ij} \geq 0$ then %we say that 
the edge $(i,j)$ is \emph{attractive} (in which case, the edge tends to pull $X_i$ and $X_j$ toward the same value). If $W_{ij} < 0$ then the edge is \emph{repulsive}. If all edges of a model are attractive, then the model is called attractive, else it is \emph{mixed}.

For any (possibly non-binary) model, we write $\theta$ for the vector of all potentials, and  $\mu$ for a vector of marginals, both using the standard overcomplete exponential family representation with ${E(x)~=~-\theta ~\cdot~ \phi(x)}$, where $\phi$ is the vector of sufficient statistics corresponding to the model \citep{WaiJor08}, so that $\mu = \mathbb{E}_\theta [\phi(X)]$.

By considering KL divergence, standard variational methods \citep{WaiJor08} show that $A(\theta) = \max_{\mu \in \MP} \theta \cdot \mu + H(\mu)$, where $\MP$, termed the \emph{marginal polytope}, is the space of all %singleton and pairwise 
marginal vectors $\mu$ that are consistent with a globally valid probability distribution over all configurations, and $H(\mu)$ is the entropy of the corresponding global distribution. 
We shall examine the %approximation error obtained using the 
following popular approximate inference methods, each of which may be defined as maximizing a negative free energy approximation over a space of marginals given by a particular polytope. %In each case, 
We use a tilde above a symbol to indicate an approximate value, and show the method as a subscript.

\paragraph{Naive mean field MF} 
$\tA_{M}(\theta) = \max_{\mu \in \MP'} \theta \cdot \mu + H(\mu)$, where $\MP'$ %\subset \MP$ 
denotes the subspace of distributions where each variable is independent, i.e. distributions are restricted to the fully-factorized form %$\mu(x)=
$\prod_{i \in \V} \mu_i (x_i)$. 

\paragraph{Bethe approximation} 
$\tA_B(\theta) = \max_{\mu \in \LP} \theta \cdot \mu + \tilde{H}_B(\mu)$, where $\LP$ % \supseteq \MP$ 
denotes the standard \emph{local polytope} relaxation which enforces only pairwise consistency, i.e. it is required that $\mu_i (x_i) = \sum_{x_j} \mu_{ij}(x_i,x_j) \; \forall i \in \V, j \in \N(i)$, and $\tilde{H}_B(\mu)$ is the \emph{Bethe entropy} approximation given by $\tilde{H}_B(\mu) = \sum_{i \in \V} H(\mu_i) - \sum_{(i,j) \in \ce} I_{ij} (\mu_{ij})$, with the pairwise mutual information ${I_{ij} ~=~ H(\mu_i) + H(\mu_j) - H(\mu_{ij}) \geq 0}$. 

\paragraph{Tree-reweighted approximation TRW}
$\tA_T(\theta) = \max_{\mu \in \LP} \theta \cdot \mu + \tilde{H}_T(\mu)$, where the TRW entropy approximation  \citep{WJW05} %WaiJaaWil02} 
is specified by a convex sum of entropies of spanning trees, $\tilde{H}_T(\mu) ~=~ \sum_{\cT} \rho_{\cT} H(\mu_{\cT})$, where, for any tree $\cT$ and any $\mu \in \LP$,  %set of pseudomarginals $\mu$ in the local polytope, 
$\mu_{\cT}$ is the distribution that results from taking the tree-decomposition of marginals over $\cT$.
It is known that $H(\mu) \leq \tilde{H}_T(\mu)$ hence $A(\theta) \leq \tA_T(\theta)$. Further, 
%by expanding  $H(\mu_{\cT})$ terms, 
it is easily shown that $\tilde{H}_T(\mu) ~=~\sum_{i \in \V} H(\mu_i) - \sum_{(i,j) \in \ce} c_{ij} I_{ij} (\mu_{ij})$ for edge counting numbers $\{c_{ij} \leq 1\}$. Thus, also $\tilde{H}_B(\mu) \leq \tilde{H}_T(\mu)$ and $\tA_B(\theta) \leq \tA_T(\theta)$.

Since $\MP' \subseteq \MP$, $\tA_M(\theta) \leq A(\theta)$. Also MF marginals are a subset of Bethe on which $\tilde H_B$ is exact, hence $\tA_M(\theta) \leq \tA_B(\theta)$. Collecting relationships, %the following inequalities always hold 
we have the following sandwich results (for any number of labels):
\begin{equation}\label{eq:sandwich}
\tA_{M}(\theta) \leq \tA_B(\theta) \leq \tA_T(\theta),  
%\qquad 
\text{ and } %\qquad
\tA_{M}(\theta) \leq A(\theta) \leq \tA_T(\theta).
\end{equation}
For binary models with supermodular potentials (of any arity; in the case of pairwise models, supermodular potentials equate to an attractive model), \cite{Ruo12} proved that $\tA_B(\theta) ~\leq~ A(\theta)$, but in general, $\tA_B(\theta)$ can be above or below $A(\theta)$.  $\tA_B(\theta)$ is often strikingly accurate, 
though there are settings where other methods are significantly better.

\subsection{Clamping a Variable and Related Definitions}\label{sec:clampdefns}

We are interested in sub-partition functions obtained by \emph{clamping} some %one particular 
variable $X_i$, that is 
let $Z(x_i;\theta) = Z(\theta)|_{X_i=x_i}$ be the sub-partition function %on the model 
on the model on $n-1$ variables $X_1, \dots, X_{i-1}, X_{i+1}, \dots, X_n$
obtained by setting $X_i=x_i$, with $x_i \in \{0, \dots, L_i-1\}$, 
with corresponding definitions for approximate $\tilde{Z}(x_i;\theta) = \tilde{Z}(\theta)|_{X_i=x_i}$. 
Observe that true 
values 
satisfy: 
\begin{equation*}
 Z\pt = \sum_{x_i=0}^{L_i-1} Z(x_i;\theta) \; \text{for any }  X_i; \; %\qquad
 p(X_i=x_i) = \frac{ Z(x_i;\theta) } {Z\pt}.
\end{equation*}
These do not hold in general for approximate values 
but motivate the following definitions, that are achieved by clamping variable $X_i$ and summing approximate sub-partition functions:
\begin{equation}\label{eq:Zi_and_p}
\tilde{Z}^{(i)}(\theta) = \sum_{x_i=0}^{L_i-1} \tilde{Z}(x_i;\theta); \quad %  \qquad \qquad \qquad
\tilde{p}(x_i) = \frac{ \tilde{Z}(x_i;\theta) } {\tilde{Z}\si\pt}.
\end{equation}
Correspondingly, we define 
%\; \text{and correspondingly}, \; \
$\tA^{(i)}(\theta)=\log \sum_{x_i=0}^{L_i-1} \exp \tilde{A}(x_i;\theta) $, 
where $\tilde{A}(x_i;\theta) = \log \tilde{Z}(x_i;\theta)$. %\tilde{Z}|_{X_i=x_i}$. 

Considering the variational perspective above, 
note that %we may write 
$\tilde{A}(x_i;\theta) ~=~ \max_{\mu \in \mathbb{P}(x_i)} \tilde{\F}(\mu)$, i.e. the constrained optimization where $\mathbb{P}$ is the standard space over which the method optimizes ($\mathbb{P}=\MP'$ for MF, $\mathbb{P}=\LP$ for Bethe and TRW), $\mathbb{P}(x_i)$ is the sub-space constrained to %$\mu_i(x_i)=1$, 
$\mu(X_i=x_i)=1$, and $\tilde{\F}(\mu)$ is the respective negative free energy approximation being maximized.
 
For all our approximate methods, %it is clear that 
if \emph{all} variables are clamped, leading to a sum over all possible configurations, 
then the exact partition function will be obtained but in exponential time. 
\cite{zbn} showed for the Bethe approximation that: for an \emph{attractive} binary pairwise model and any variable $X_i$, clamping helps in that $\tilde{A}_B(\theta) \leq \tilde{A}^{(i)}_B(\theta) \leq A(\theta)$; for a \emph{mixed} %general (non-attractive) 
model, however, clamping a variable \emph{can lead to a worse} approximation, %does not always lead to improvement, 
yet empirically it was shown often to help significantly. 

We next show that for TRW and MF, clamping and summing always reduces error and improves bounds for any model type.

\section{NEW RESULTS FOR THE TREE-REWEIGHTED APPROXIMATION TRW}

\begin{theorem}
\label{clamp-trw}
Using %the definition of  $\LP(x_i)$ 
definitions from \S \ref{sec:prelim}, %clampdefns}, 
for any model and any variable $X_i \in \{0, \dots, L_i-1\}$, %with labels in 
\begin{align*}
%\[ 
\tA_T^{(i)}\pt &= \sum_{x_i=0}^{L_i-1} \tilde{p}_T(x_i) \max_{\mu^{x_i} \in \LP(x_i)} \left( \theta \cdot \mu^{x_i} + \tilde{H}_T(\mu^{x_i})\right) \\
	& \- \qquad  + H(\tilde{p}_T) \\
	&= \max_{\nu \in \LP^{(i)}} \left( \theta \cdot \nu + \tilde{H}_T(\nu) \right), 
	%\]
\end{align*}
where $\LP^{(i)}$ denotes all convex combinations of the polytopes $\LP(X_i=0), \dots, \LP(X_i=L_i-1)$, i.e. $\mu \in \LP^{(i)}$ if any only if $\mu = \sum_{x_i=0}^{L_i} r(x_i) \mu^{x_i}$ where $r$ is a distribution and $\mu^{x_i} \in \LP(X_i=x_i)$.
\end{theorem}

\begin{proof}
Since log-sum-exp is the convex conjugate of the negative entropy, for any $a$, $\log \sum_i \exp a_i = \sup_{w \in \Delta} \sum_i a_i w_i +H(w)$, where $\Delta$ is the probability simplex.  Moreover, the maximizing $w^*$ is $w^*_i = e^{a_i}/\sum_j e^{a_j}$.  Applying this to $\tilde A_T^{(i)}$, % gives that
\begin{align*}
&\tilde A_T^{(i)}(\theta) = \log \sum_{x_i} \exp \tA_T(x_i; \theta) \\
	&=
\!\max_{w \in \Delta} \sum_{x_i} w(x_i) \left( \max_{\mu^{x_i} \in \LP(x_i)} \theta \cdot \mu^{x_i} + \tilde H_T(\mu^{x_i}) \right) \! + \!H(w)  \\
                              &= \!\max_{w \in \Delta} \max_{\{\mu^{x_i} \in \LP(x_i)\}} \!\sum_{x_i} w(x_i) \! \left( \theta \cdot \mu^{x_i} + \tilde H_T(\mu^{x_i}) \right) \!\! + \!\! H(w),
\end{align*}
where $\Delta$ denotes the probability simplex for labels of $X_i$, and $\mu^{x_i}$ are pseudomarginal vectors that result from clamping $X_i=x_i$.  The final equality above follows because a sum of maximizations over independent sets of variables is equivalent to the joint maximization of all variables over a sum of the objectives.  The above form for $w^*$ gives the first equality of the theorem  (since $w^*=\tilde{p}_T$ from \eqref{eq:Zi_and_p}). 

Next we shall consider $\sum_{x_i} w(x_i) \tilde H_T(\mu^{x_i})$. For any fixed $w$ and $\{\mu^{x_i}\}$, let $\nu=\sum_{x_i} w(x_i) \mu^{x_i}$ and observe that for any tree $\cT$, $\nu_{\cT}(x)=\sum_{x_i} w(x_i)\mu_{\cT}^{x_i}(x_{-i})$. 
We shall also need that  $\sum_{x_i} w(x_i) H(\mu_{\cT}^{x_i}) = H(\nu_{\cT})-H(w).$  This follows, essentially, from $H_{\nu_{\cT}}[X_{-i}\vert X_i]=H_{\nu_{\cT}}[X]-H_{\nu_{\cT}}[X_i]$ \citep{Cover}. 
Now given the TRW distribution $\rho$ over spanning trees of the model,  
\begin{align*}
\sum_{x_i} w(x_i) \tilde H_T(\mu^{x_i}) &= 
\sum_{x_i}  w(x_i) \sum_{\cT} \rho_{\cT} H(\mu_{\cT}^{x_i}) \\ 
%&= \sum_{\cT} \sum_{x_i}  w(x_i) \rho_{\cT} H(\mu_{\cT}^{x_i})
&=
\sum_{\cT} \rho_{\cT} \left( H(\nu_{\cT})-H(w) \right) \\ &=  \tilde H_T(\nu)-H(w), \quad \text{ and hence,}
\end{align*}
%\]
%
% hence, 
$\tilde A_T^{(i)}(\theta) = \max_{w \in \Delta} \max_{\{\mu^{x_i} \in \LP(x_i)\}} %\mathcal{\tilde M}(x_i)\}} 
\left( \theta \cdot \nu + \tilde H_T(\nu) \right)$. This is equivalent to the stated result, using the dependence of $\nu$ on $w$ and ${\mu^{x_i}}$.
\end{proof}

Observe that the result of clamping a single variable with TRW is precisely to tighten the local polytope from $\LP$ to $\LP^{(i)}$. 
As an immediate corollary, %we have 
$\tA_T^{(i)}(\theta) \leq \tA_T(\theta)$. Further, $\tA_T^{(i)}(\theta) = \log \sum_{x_i} \tilde{Z}_T(x_i;\theta) \geq  \log \sum_{x_i} Z(x_i;\theta) = \log Z\pt = A\pt$. Hence, we have shown the following. % result.

\begin{theorem}
For any discrete model (any number of labels, any types of potentials) and any variable $X_i$, $A(\theta) ~\leq~  \tA^{(i)}_{T}(\theta) ~\leq~ \tA_{T}(\theta)$.
\end{theorem}

Note that all the analysis above assumes that the distribution $\rho$ over trees $\cT$ used by TRW is constant. % as variables are clamped. 
However, when a variable $X_i$ is clamped, its edges may be removed from the graph, which can only further decrease the bound. 
To see this, construct a new distribution $\rho'$ over subgraphs $\{\cal U\}$ as follows: for each original tree $\cT$, let $\cal{R}(\cT)$ be $\cT$ less its edge(s) to $X_i$, and let $\rho'_{\cal U} = \sum_{\cT : \cal U = \cal R(\cT)} \rho_{\cT}$. In the clamped model, %it is easy to see that, 
if $\cal U = \cal R(\cT)$, then $H(\mu^{x_i}_{\cal U})=H(\mu^{x_i}_\cT) \; \forall \mu^{x_i} \in \LP(x_i)$.  If $\cal U$ is disconnected, new edges (not incident to $X_i$) may be added to it to make a tree, and this can only reduce $H(\mu^{x_i}_{\cal U})$ \citep{WaiJor08}. 
In addition, tree weights may be reoptimized to reduce the bound still further. 

\section{NEW RESULTS FOR THE NAIVE MEAN FIELD APPROXIMATION MF}\label{sec:MF}

In this Section, we consider the mean field negative free energy approximation,
\begin{align}\label{eq:mf}
\FMF(\mu) &=  \sum_{i \in \V} \sum_{x_i} \theta_i(x_i) \mu_i(x_i)  + \\
	&\sum_{(i,j) \in \ce} \sum_{x_i,x_j} \theta_{ij}(x_i,x_j) \mu_i(x_i) \mu_j(x_j) + \sum_{i \in \V} H(\mu_i), \notag
\end{align}
where $\mu \in \MP'$ is %a set of univariate marginals determining 
a fully-factorized distribution. Given $\theta$, let $\mh$ be a maximizer of $\FMF(\mu)$ over $\MP'$.  Define $\mh^{x_i}(X) ~=~ \mh(X_{-i}) \I[X_i~=~x_i]$, where $\I$ is the standard indicator function. 
We first show a key Lemma, from which the main result in Theorem \ref{clamp-mf} easily follows.

\begin{lemma}\label{lem:mf}
$\FMF(\mh)=\log \sum_{x_i} \exp \FMF(\mh^{x_i})$.
\end{lemma}
\begin{proof}
Write that $\log \sum_{x_i} \exp \FMF(\mh^{x_i}) = \sum_{x_i} w(x_i) \FMF(\mh^{x_i}) + H(w)$, where $w(x_i) = \exp \FMF(\mh^{x_i}) / \sum_{x_i'} \exp \FMF(\mh^{x_i'})$.  
%Then, one can easily see that 
Hence, $w(x_i) ~\propto~ \exp \left[ \theta_i(x_i) + \sum_j \sum_{x_j} \theta_{ij}(x_i,x_j) \mh_j(x_j) \right]$.  This is precisely the mean-field update if $\FMF$ is maximized with respect to $\mu_i$ while holding all other marginals fixed \cite[\S 5.3]{WaiJor08}.  Hence, $w(x_i)=\mh_i(x_i)$, since $\mh$ %was assumed to be 
is 
a maximizer, %and the maximizer 
which 
is unique when $\{\hat{\mu}_j, j \not = i\}$ are fixed.  It thus follows that
\[ \log \sum_{x_i} \exp \FMF(\mh^{x_i}) = \sum_{x_i} \mh_i(x_i) \FMF(\mh^{x_i}) + H(\mh_i). \]
%
%It is not hard to show that 
Now consider \eqref{eq:mf} and observe that 
%Now use 
$\sum_{x_i} \mh_i(x_i) \FMF(\mh^{x_i}) = \FMF(\mh) - H(\mh_i)$, hence the result follows. 
\end{proof}

\begin{theorem}
\label{clamp-mf}
For any model (any number of labels) and any variable $X_i$, $\tA_{M}(\theta) ~\leq~ \tA^{(i)}_{M}(\theta) ~\leq~ A(\theta)$.
\end{theorem}

\begin{proof} 
Using Lemma \ref{lem:mf}, we have 
\begin{align*}
%\[ 
\tA_{M}(\theta) &= \FMF(\mh) = \log \sum_{x_i} \exp \FMF(\mh^{x_i}) \\
	&\leq \log \sum_{x_i} \exp \max_{\mu \in \MP'(x_i)} \FMF(\mu) =  \tA_{M}^{(i)}(\theta), 
\end{align*}
where $\MP'(x_i)$ is the constrained sub-space defined in \S \ref{sec:clampdefns}.

For the other inequality, $\tA_{M}^{(i)}(\theta) = \log \sum_{x_i} \tilde{Z}_M(x_i;\theta) \leq  \log \sum_{x_i} Z(x_i;\theta) = \log Z\pt = A\pt$.
\end{proof}
In practice, %given a local optimum for the parent problem, 
there may be locally optimum solutions for the clamped problems that show worse performance than the parent. However, the analysis above shows that this concern is guaranteed to be avoided %may be removed by 
if the clamped optimizations are initialized at the solution of the parent problem.

\section{METHODS TO SELECT WHICH VARIABLES TO CLAMP}\label{sec:select}

%In this Section and 
Henceforth, we focus on binary pairwise models. 
As shown above, clamping a variable and summing over approximate sub-partition functions will always reduce error 
and improve bounds 
for both MF and TRW. Further, \cite{zbn} proved that for the Bethe approximation, this is also true for attractive models, and empirically it is often helpful for mixed models.

This leads to the 
%natural 
question of how to choose which variable, or sequence of variables, to clamp. \cite{zbn} 
introduced two selection heuristics, motivated by trying to break \emph{strong} cycles (we say a subgraph is strong if all its edges have weights with high absolute value; since Bethe and TRW are exact on trees, this goal is reasonable), and demonstrated that these heuristics were effective in several contexts. 
We first describe these earlier approaches.
 
\emph{maxW} is a simple $O(|\ce|)$ method which  picks a variable $X_i$ with $\max_{i \in \V} \sum_{j \in \N(i)} |W_{ij}|$.  
One way in which maxW can make a poor selection is to choose a variable at the centre of a large star configuration but far from any cycle. \emph{Mpower} was introduced as a more complex approach to attempt to avoid this problem (we introduce a simpler method in \S \ref{sec:core}) by considering the convergent series of powers of a modified $|W|$ matrix, which approximates a weighted count around all cycles. It was shown that Mpower outperforms maxW in some cases, though for many examples, their performance was similar. 

Note that both maxW and Mpower rely exclusively on the absolute value of edge weights $|W_{ij}|$, while ignoring their signs, and also ignoring singleton potentials. In the remainder of this Section, we demonstrate how these earlier methods can perform poorly in certain circumstances, and introduce new approaches. Details of all selection methods are provided in the Appendix. 

\setlength{\textfloatsep}{15pt}

\begin{figure}%[2][h]
\begin{center}
\begin{tikzpicture}[scale=0.55]
	%% draw the vertices
		\foreach \pos/ \name in { {(1,4.0)/1}, {(2,5.0)/2}, {(3,4.0)/3}, {(2,3.0)/4}, {(2,1.75)/5},  
												  {(2,0.5)/6}, {(4,1)/7}, {(4,0)/8}, {(0,1)/9}, {(0,0)/10}}
		\node[svertex] (\name) at \pos {$x_{\name}$};
	\foreach \source/ \dest in {1/2, 1/3, 1/4, 2/3, 2/4, 3/4, 4/5, 5/6, 6/7, 6/8, 6/9, 6/10}
		\path[plain edge] (\source) -- (\dest);
\end{tikzpicture}
\qquad \quad
\begin{tikzpicture}[scale=0.55,baseline=-7.62pt]
	%% draw the vertices
		\foreach \pos/ \name in { {(1,4.0)/1}, {(2,5.0)/2}, {(3,4.0)/3}, {(2,3.0)/4}}%, {(2,1.75)/5},  
	%											  {(2,0.5)/6}, {(4,1)/7}, {(4,0)/8}, {(0,1)/9}, {(0,0)/10}}
		\node[svertex] (\name) at \pos {$x_{\name}$};
	\foreach \source/ \dest in {1/2, 1/3, 1/4, 2/3, 2/4, 3/4}%, 4/5, 5/6, 6/7, 6/8, 6/9, 6/10}
		\path[plain edge] (\source) -- (\dest);
\end{tikzpicture}
\end{center}
\caption{\small On the left, a model with the `lamp' topology from \cite{zbn}. On the right, the \emph{core} of the same model. This is obtained by iteratively removing variables with degree 1. If maxW is applied to the original model, it often chooses to clamp $X_6$ (which has highest degree) whereas any of $X_1,\dots,X_4$ would be better, and will be selected if the model is first stripped to its core. 
See \S \ref{sec:core}.
}
\label{fig:core}
\end{figure}
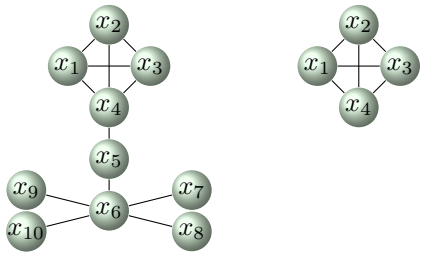

\subsection{Stripping to the Core}\label{sec:core}

Following \cite{Sud07}, we define the \emph{core} of a graph (model), to be what remains after iteratively pruning nodes (variables) with degree 1. Equivalently, it is the subgraph (submodel) induced on the nodes (variables) which either belong to some cycle, or lie on a path between cycles. An example is shown in Figure \ref{fig:core}. % in the Appendix. 
For pairwise entropy approximations, we should expect it will be better to 
strip a model to its core before applying any method to select a variable 
(then clamping it in the original model). %In particular, i
In many cases that were previously challenging for maxW, this quick pre-processing step enables maxW to perform as well as the more expensive Mpower method.

\subsection{Balanced Models, Frustrated Cycles and Strong Cycles}\label{sec:bal}

%In this Section and henceforth, we focus on binary pairwise models. For these, 
A \emph{frustrated cycle} is a cycle with an odd number of repulsive edges (see \S \ref{sec:prelim} for definitions). 
These cause difficulties for many methods of inference. 
A \emph{balanced} model (or sub-model) is one that does not contain any frustrated cycles. It is easily shown \citep{Har53,bb} that a  model may be  mapped into an equivalent attractive model by flipping an appropriately chosen subset of variables iff it is balanced (in which case such a set can be identified in linear time). Hence, results for attractive models readily extend to the broader class of balanced models.

Pairwise approximations such as Bethe and TRW are exact on models without any cycles. Further, it is known that for Bethe and TRW, frustrated cycles can lead to more trouble than balanced cycles \citep[\S 6.3]{bb}.  
This is illustrated in the top row of Figure \ref{fig:example2all}, which shows approximation error for symmetric models (i.e. no singleton potentials) with uniform edge weights. As edge weights rise, both MF and Bethe underestimate with error that is bounded and tends to $\log 2$,\footnote{For high positive weights: with MF, all singleton marginals are pulled toward 0 or 1; with Bethe, for $K_4$ and $K_5$ the same happens, for cycles all edge marginals approach $(1/2 \; 0; 0 \; 1/2)$; in all cases, $\tilde H(\mu') \tend 0$, whereas the true distribution has two dominating states (all 0s or all 1s), hence $H(\mu) \tend \log 2$.} while TRW tends to the correct solution. For strong negative weights, however, which lead to frustrated cycles, Bethe and TRW show rapidly increasing error without bound. Note that the model on the $C_4$ cycle is balanced even with negative edge weights (since there are an even number of edges), with symmetric error either side of 0 edge weights.\footnote{With $W<0$, the models with fully connected topologies $K_n$ for $n>3$ contain both balanced and frustrated cycles but the number and strength of the frustrated cycles dominate.}  The observation that Bethe and TRW can perform arbitrarily badly for strong frustrated cycles, whereas the error for MF is bounded, explains the later experimental results where MF outperforms Bethe, see \S \ref{sec:disc}.

This motivates trying to identify strong frustrated cycles. %, i.e. ones where all edge weights have high absolute value. 
Both the maxW and Mpower earlier methods of \cite{zbn} consider only $|W_{ij}|$, hence are unable to differentiate between balanced and frustrated cycles. 
To find strong frustrated cycles is NP-hard but we introduce heuristics that build on a recent algorithm by \cite{Son12}, which was used in a cutting plane approach to tighten the local polytope for MAP inference. %, where it was also shown that frustrated cycles are particularly problematic. 
We combine ideas from their algorithm with cycle scores based on the loop series method \citep{CheChe06,Sud07,abc} and present two new heuristics for identifying a good variable to clamp: \emph{frustCycles}, which seeks to identify a variable lying on %many 
strong frustrated cycles, which if clamped, would remove those cycles; and \emph{strongCycles}, which attempts also to take into due consideration the (lower) value of removing strong balanced cycles. Details are provided in the Appendix.

\begin{figure*}
\begin{center}
\setlength\tabcolsep{4pt}
\begin{tabular}{cccc}
\begin{sideways} {\small \- \; \quad Original error} \end{sideways} &
\includegraphics[width=.30\linewidth]{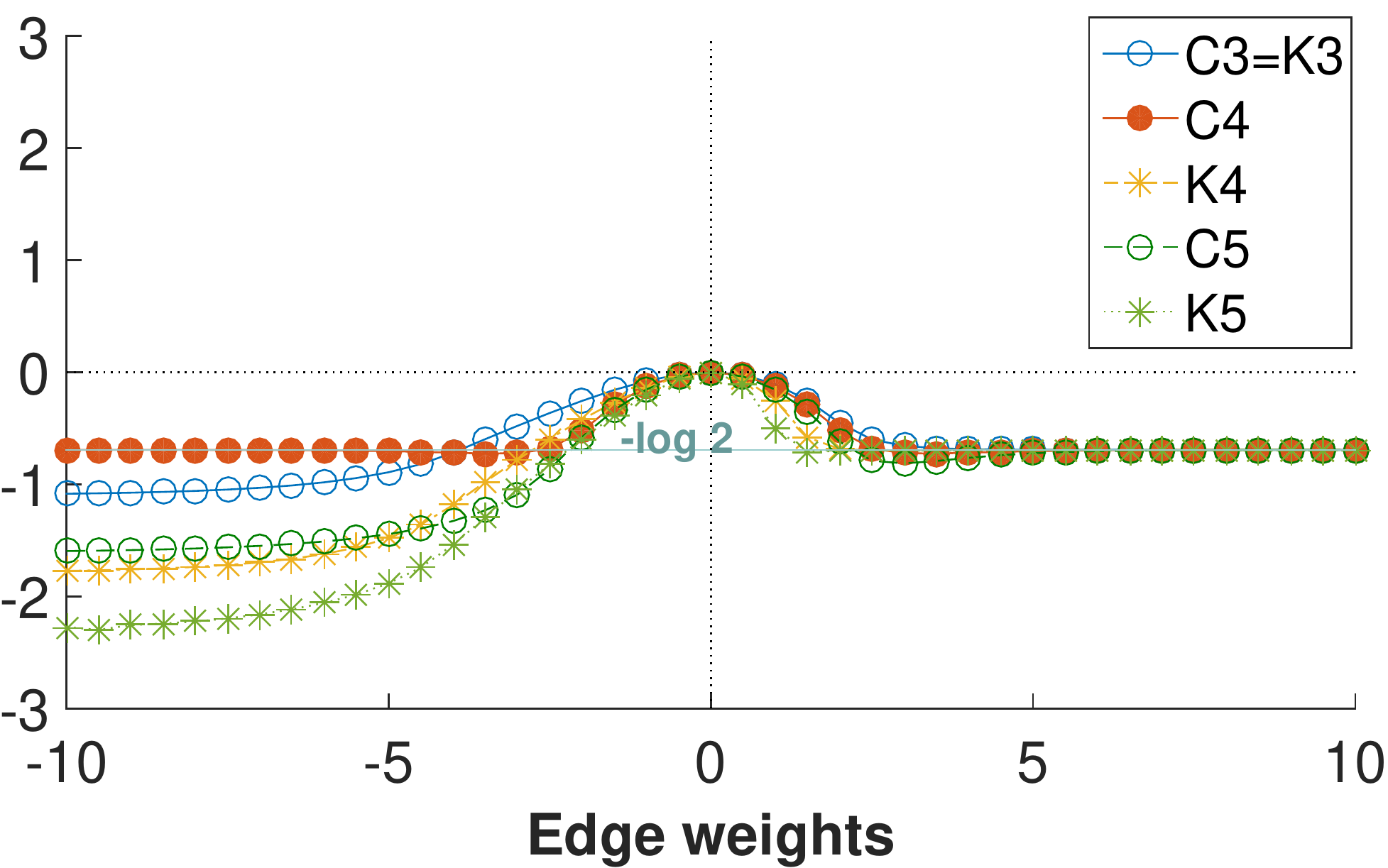} & %{Example2all_BP-eps-converted-to.pdf} &
\includegraphics[width=.30\linewidth]{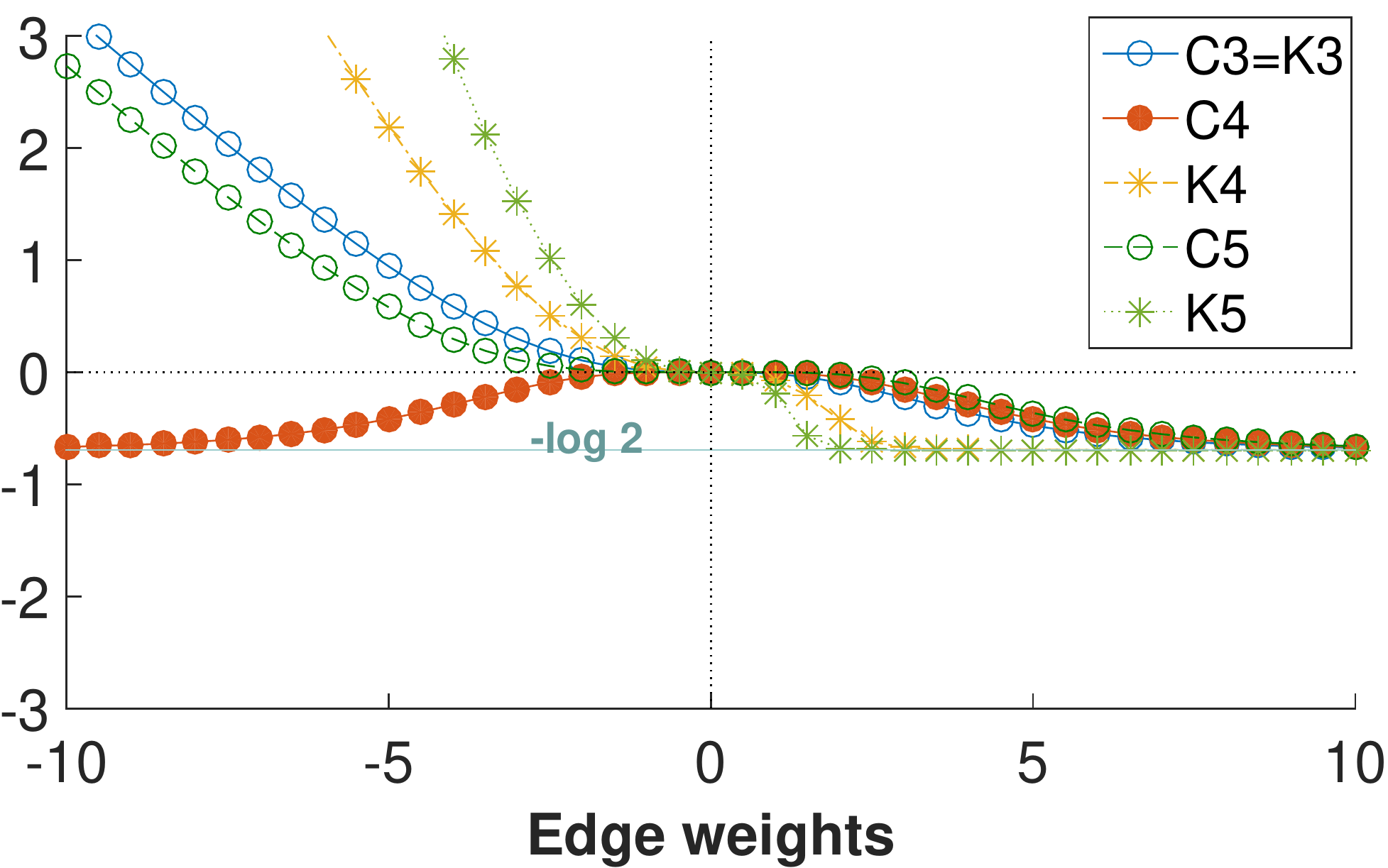} &
\includegraphics[width=.30\linewidth]{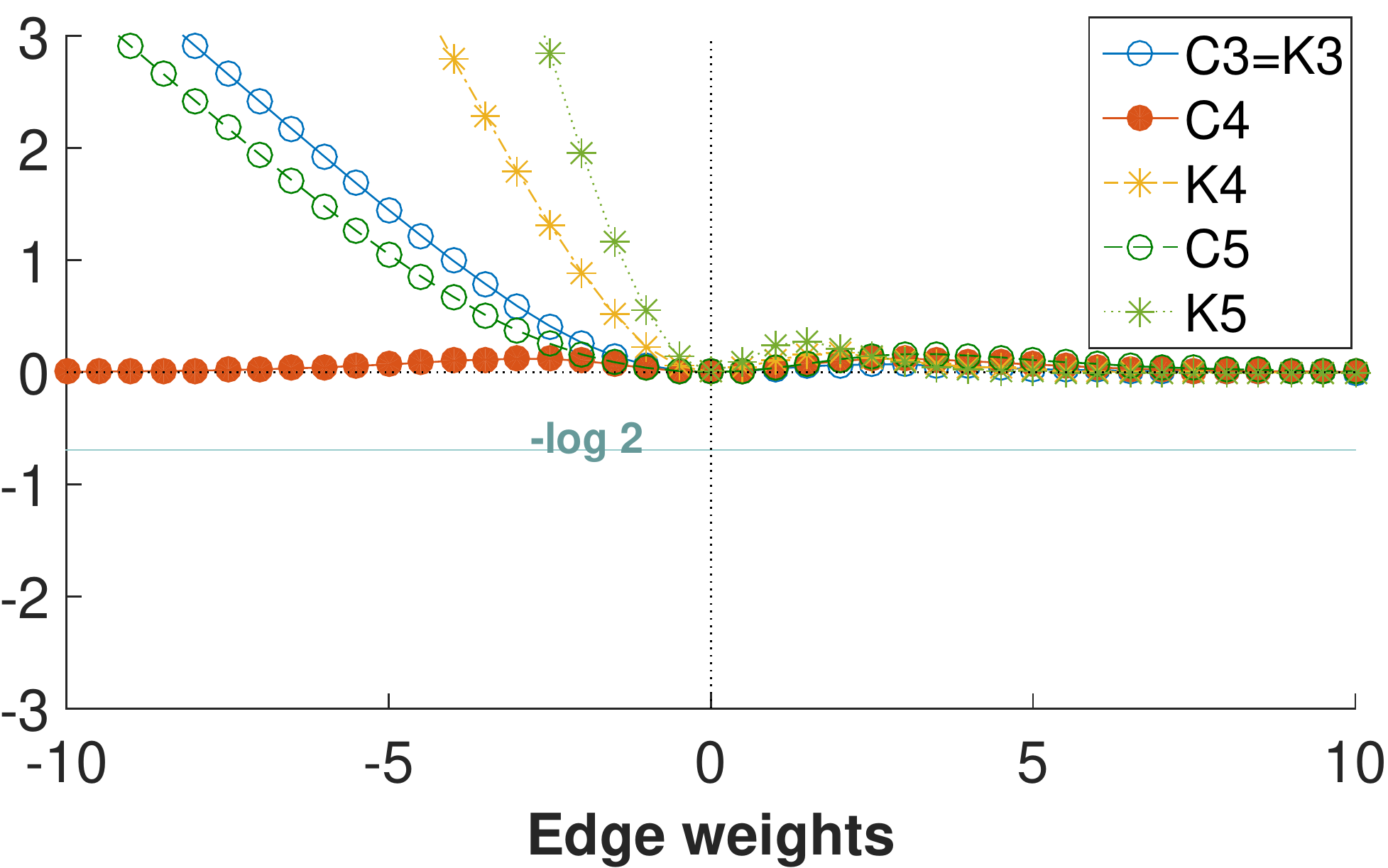} \\
& {\small (a) MF $\tilde{A}_{M}\pt-A\pt$} & {\small (b)  Bethe $\tilde{A}_B\pt-A\pt$}  & {\small (c) TRW $\tilde{A}_T\pt-A\pt$} \\
\begin{sideways} {\small \- \; \quad After 1 clamp } \end{sideways} &
\includegraphics[width=.30\linewidth]{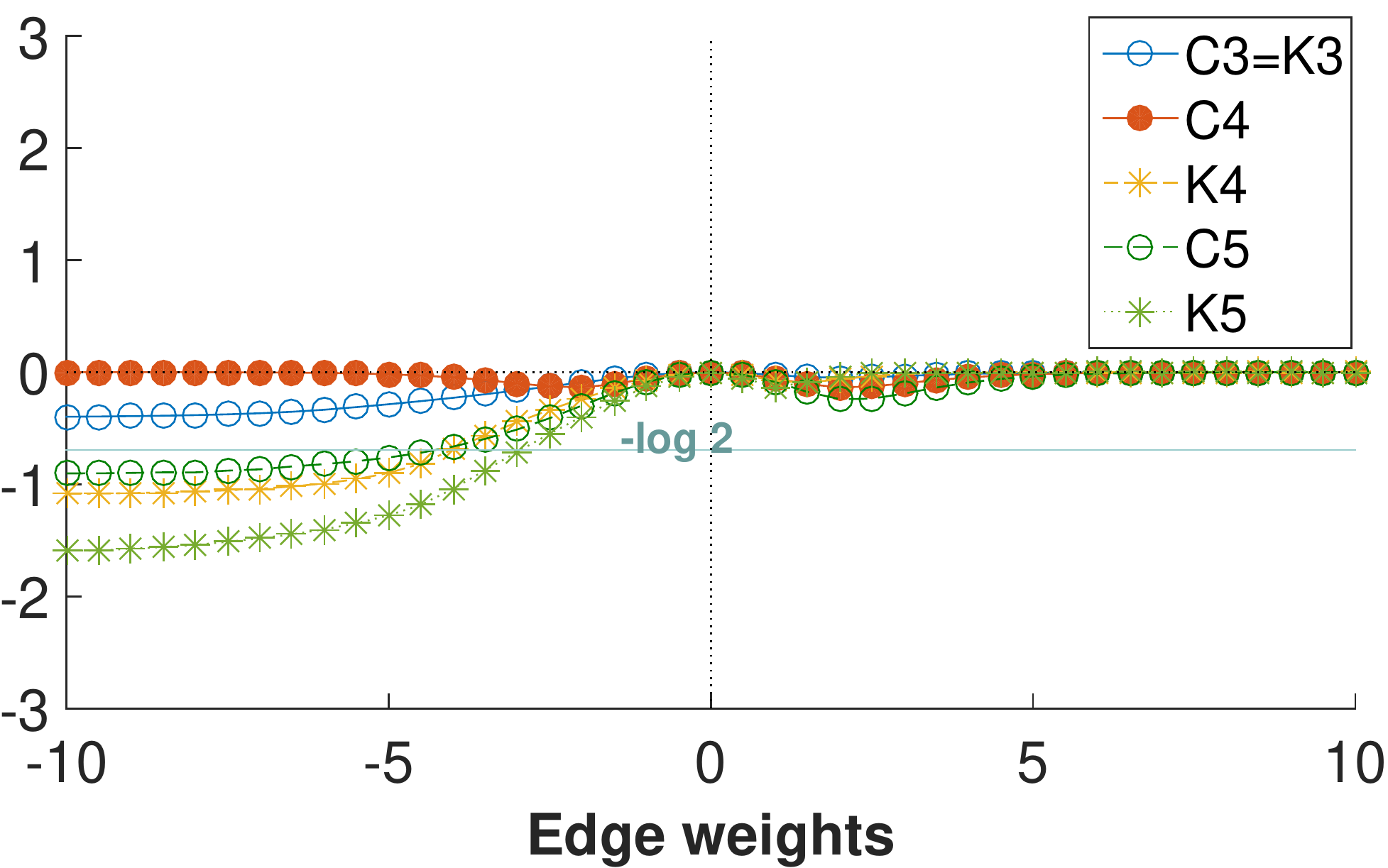} & 
\includegraphics[width=.30\linewidth]{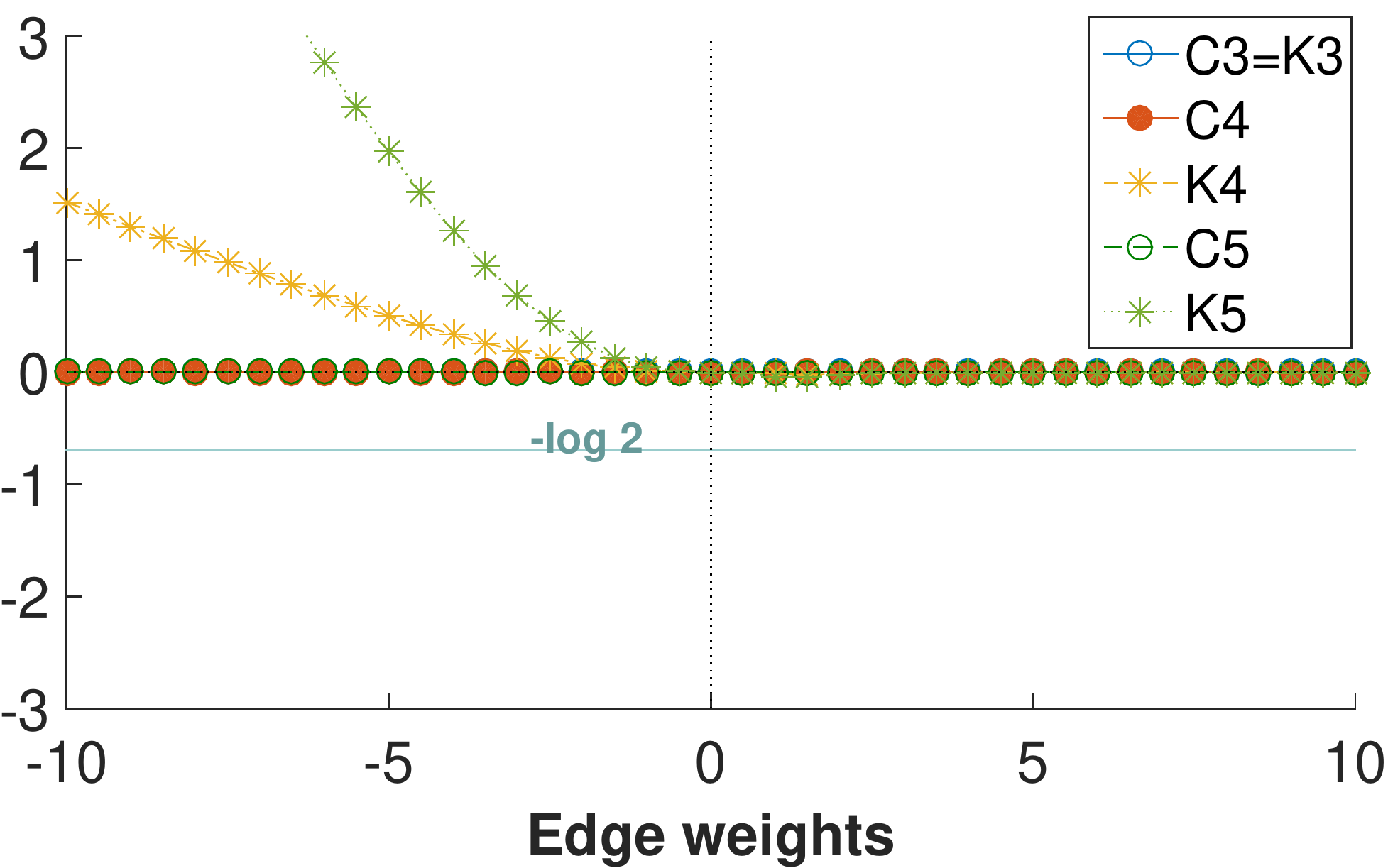} &
\includegraphics[width=.30\linewidth]{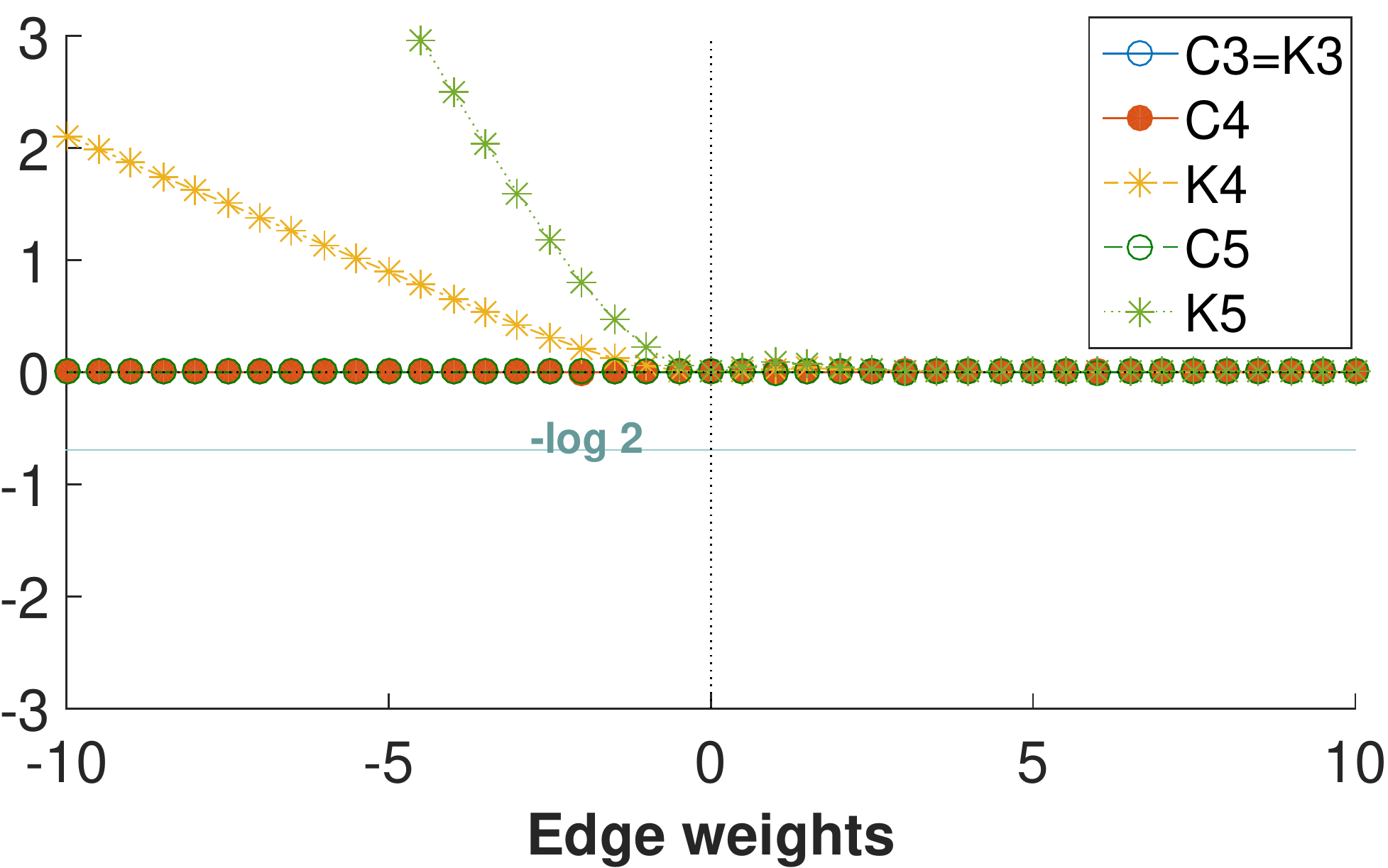}
\end{tabular}
\end{center}
\caption{\small Top: Approximate log-partition function minus true value, i.e. $\tilde{A}\pt-A\pt$, for symmetric models (no singleton potentials) on 3, 4 and 5 variables with cycle $C_n$ and complete graph $K_n$ topologies, with uniform edge weights that are varied, for (a) MF, (b) Bethe, and (c) TRW. 
Bottom: Error $\tilde{A}\si\pt-A\pt$ for the same models and methods after clamping (any) one variable and summing approximate sub-partition functions. Observe, before clamping: Even strong positive weights lead to underestimates with bounded error; while strong negative weights with frustrated cycles lead to unbounded overestimates for Bethe and TRW. After clamping: All methods are significantly improved; if there are no frustrated cycles remaining, all methods are almost exact. See discussion in \S \ref{sec:bal}.}
\label{fig:example2all}
\end{figure*}

\subsection{Using Singleton Entropies}\label{sec:clumps}

All previous clamping selection methods examine only edge weights. However, if a variable already has very low singleton entropy (typically this would be due to a strong singleton potential: this might be present in the original model, or could have arisen %been achieved 
as a result of earlier clamping rounds), then it has effectively already been held fixed to one value, and there is little to be gained from clamping it. On the other hand, if a variable has high entropy and is strongly connected to many others, without a frustrated cycle, then clamping it can effectively lead to a cascade where many other variables will also be `effectively clamped', yielding a significant improvement in approximation error. Afterward, there is little residual value in actually clamping those other variables.

This effect is illustrated by comparing the %top and bottom 
rows of Figure \ref{fig:example2all}. Observe that even in the $K_5$ fully connected model, when no frustrated cycle is present (i.e. when edge weights are positive), just one clamping is sufficient to obtain almost zero error.
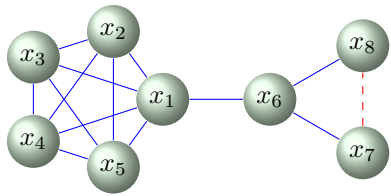
\begin{figure}%[50][h]
\centering
\begin{tikzpicture}[scale=0.95]%, auto,swap]
	% draw the vertices
	\foreach \pos/ \name in { {(1,0)/1}, {(0.31,0.95)/2}, {(-0.81,0.59)/3}, 
												  {(-0.81,-0.59)/4}, {(0.31,-0.95)/5}, {(2.5,0)/6}, {(3.8,-0.75)/7}, {(3.8,0.75)/8}}
		\node[vertex] (\name) at \pos {$x_{\name}$};
	\foreach \source/ \dest in {1/2, 1/3, 1/4, 1/5, 2/3, 2/4, 2/5, 3/4, 3/5, 4/5, 1/6, 6/7, 8/6}
		\path[edge] (\source) -- (\dest);
	%% rep edges
	\foreach \source/\dest in {7/8}
		\path[rep edge] (\source) -- (\dest);
\end{tikzpicture}
\caption{\small Example of a model where the earlier maxW and Mpower heuristics of \cite{zbn} perform poorly to select variables to clamp but our new methods perform well. Solid blue (dashed red) edges are strongly attractive (repulsive). 
maxW and Mpower both repeatedly select variables in %the clump 
$X_1,\dots,X_5$: the first clamp is good but less so than picking from the frustrated cycle $X_6, X_7, X_8$; then repeat clampings in $X_1, \dots,X_5$ reap little benefit. %The best sequence is to pick $X_6$ then $X_1$.
It is best to pick $X_6$ then $X_1$.
}
\label{fig:barbell}
\end{figure}
A further illustration is provided by considering 
%As illustrations, consider 
the model in Figure \ref{fig:barbell}. % in the Appendix. 
Recognizing 
%To recognize 
this effect, 
ideally we would compute singleton entropies by exact inference, but that would clearly be too costly, % to compute the % 
hence, we use approximate inference. % to estimate the singleton entropy. 
Specifically, we introduce TRE versions of each earlier method:  
for each variable, multiply its respective earlier heuristic clamp score %metric 
by its TRw Entropy (we want both %to be 
high) and choose the best. 
We use the TRW approximate entropy %rather than Bethe 
for two reasons \citep{abc}: (i)~TRW   %demonstrated %has shown 
%that 
singleton marginals typically have similarly good accuracy to Bethe, while often being easier to estimate (since the TRW free energy is convex); and (ii)~we are particularly interested in cases where edge potentials are high around a variable, 
and in this setting, Bethe marginals can be %particularly 
poor, being pulled toward 0 or 1 even if the true marginal is close to 1/2. 
TRE versions of all heuristics perform well for multiple clampings on models such as the one in Figure \ref{fig:barbell}.

\section{EXPERIMENTS}\label{sec:exp}

\begin{figure*}
\begin{center}
\setlength\tabcolsep{1pt}
\begin{tabular}{cccccc}
\begin{sideways} {\small \- \; \qquad large (81)} \end{sideways} &
\includegraphics[width=.24\linewidth]{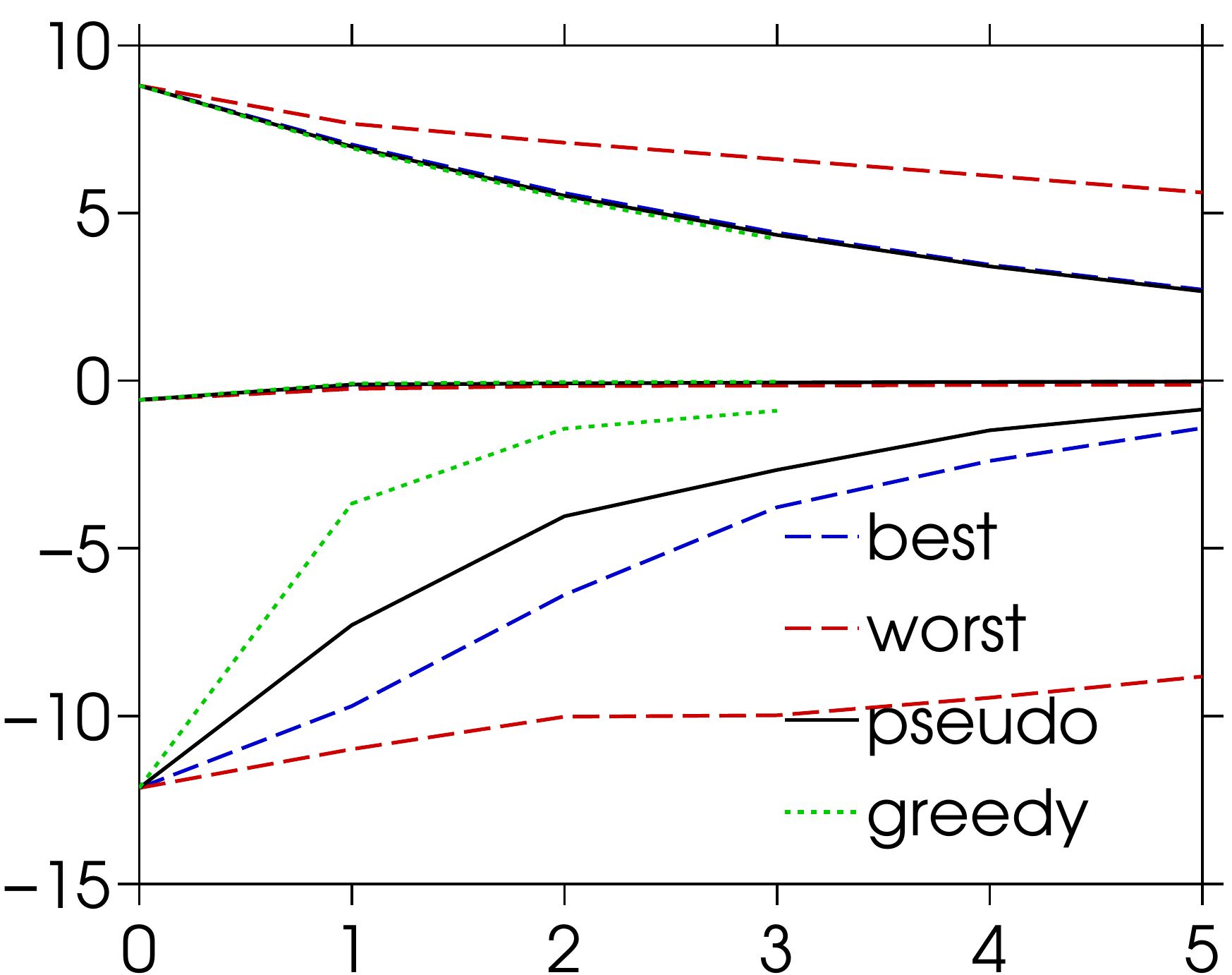} & 
\includegraphics[width=.24\linewidth]{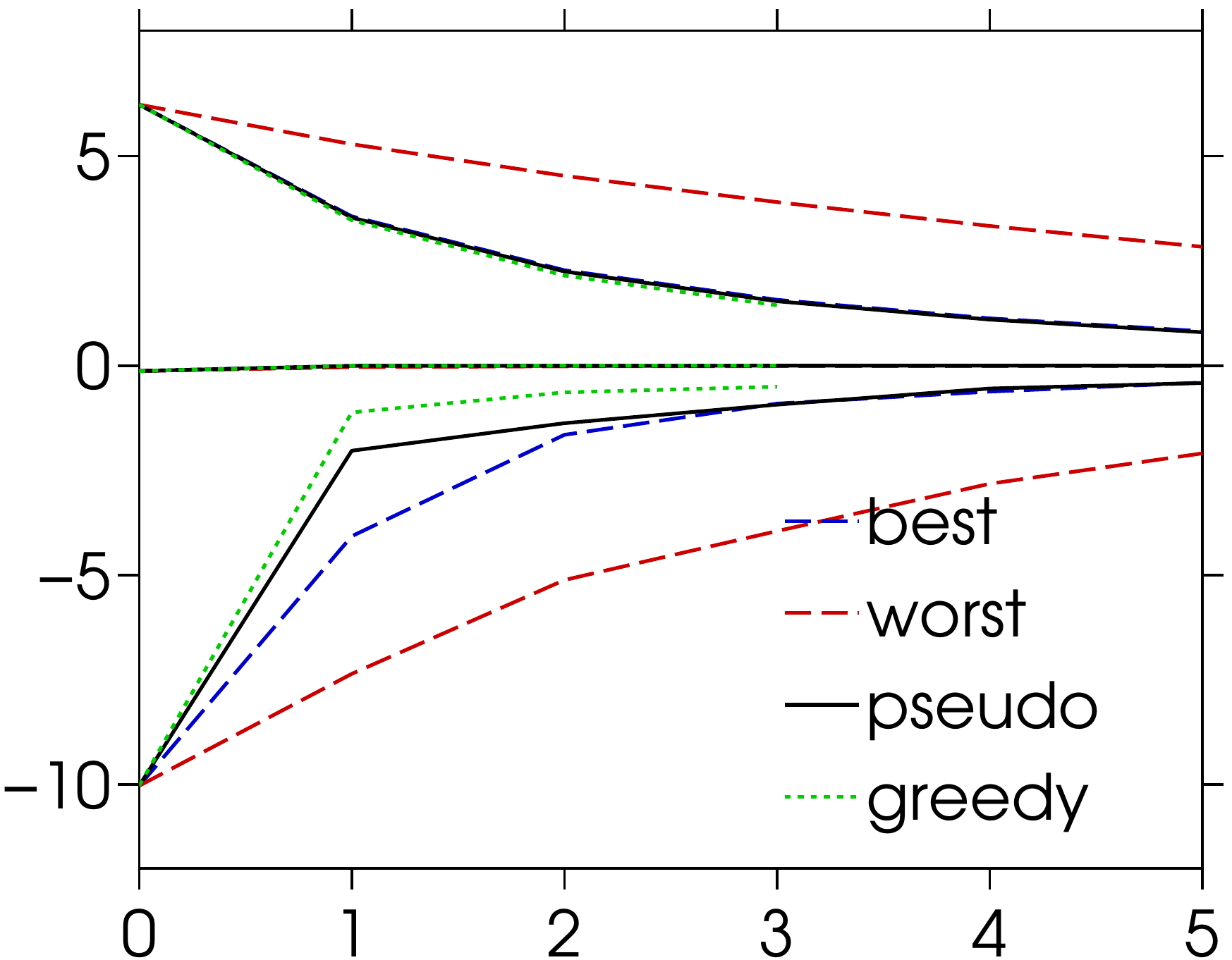} &
\- \; &
\includegraphics[width=.24\linewidth]{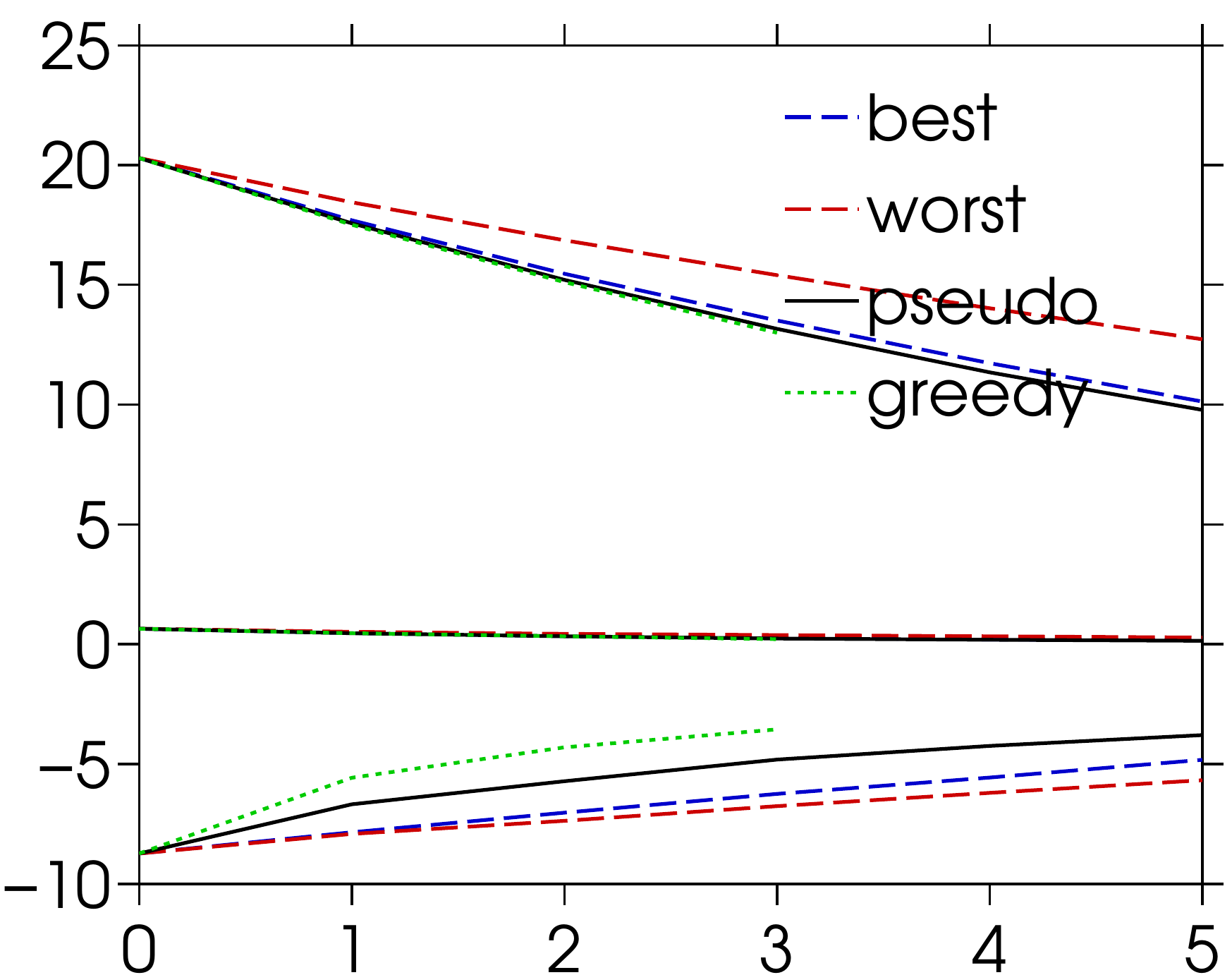} & 
\includegraphics[width=.24\linewidth]{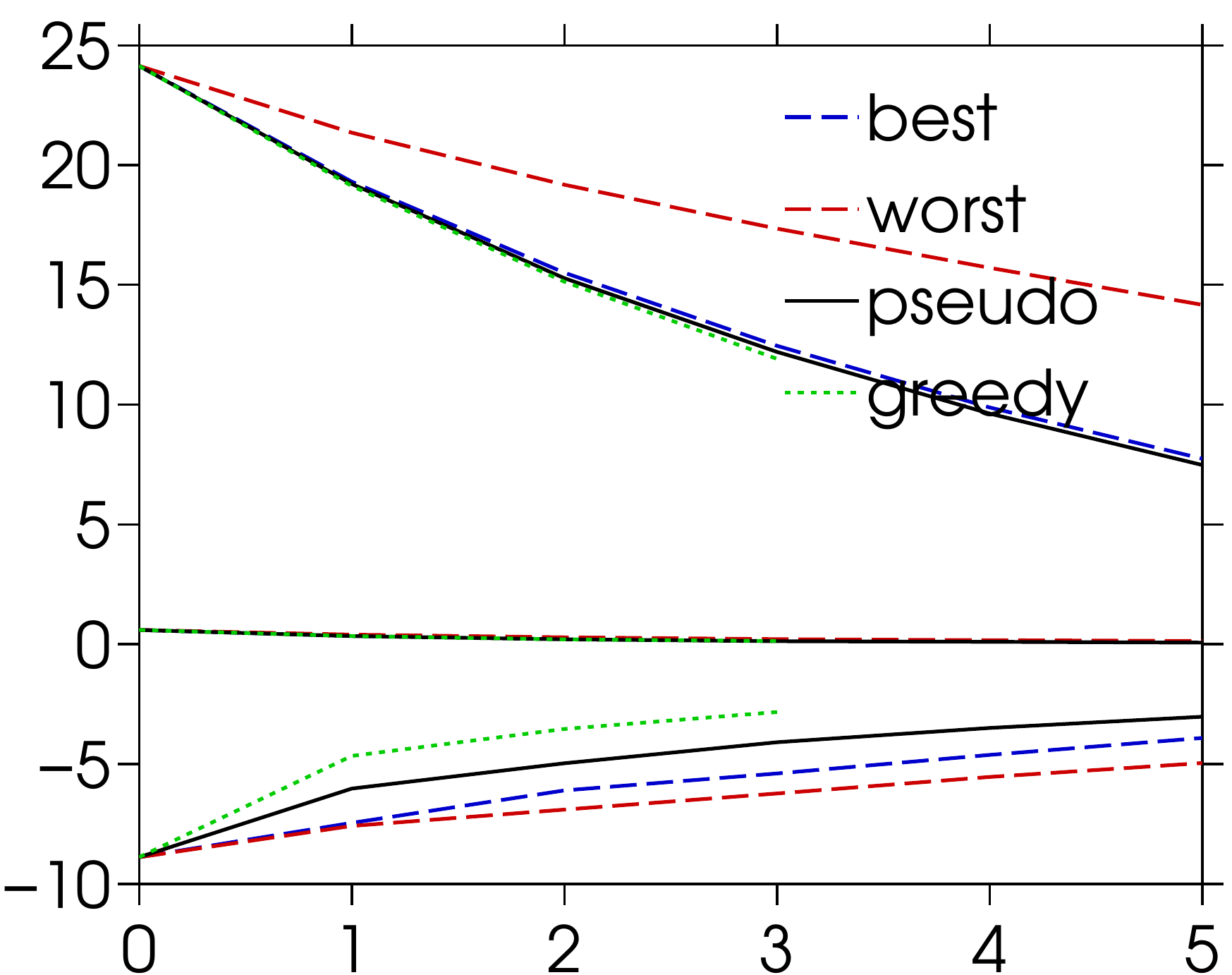} \\
\begin{sideways} {\small \- \; \; \quad medium (49)} \end{sideways} &
\includegraphics[width=.24\linewidth]{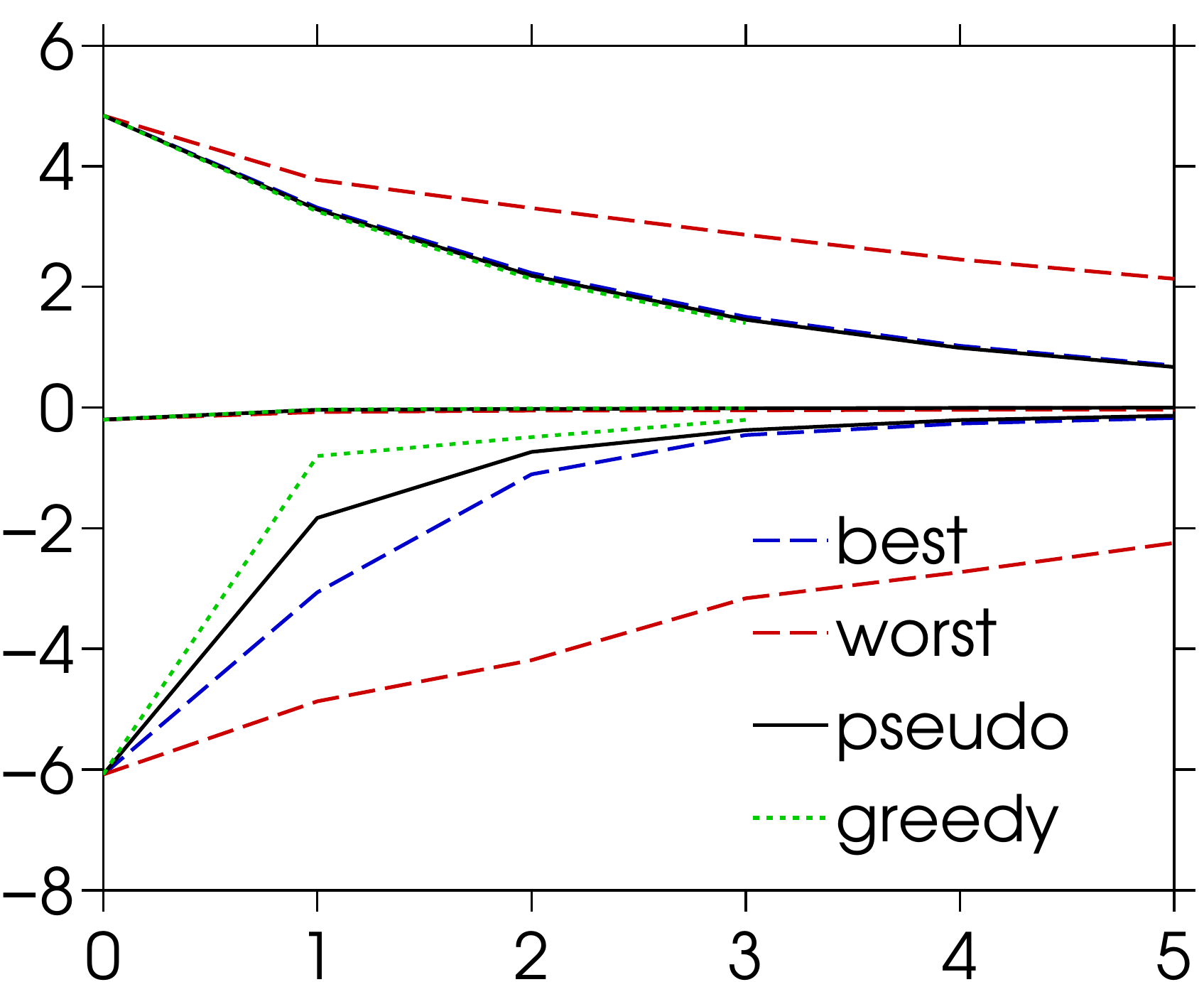} & 
\includegraphics[width=.24\linewidth]{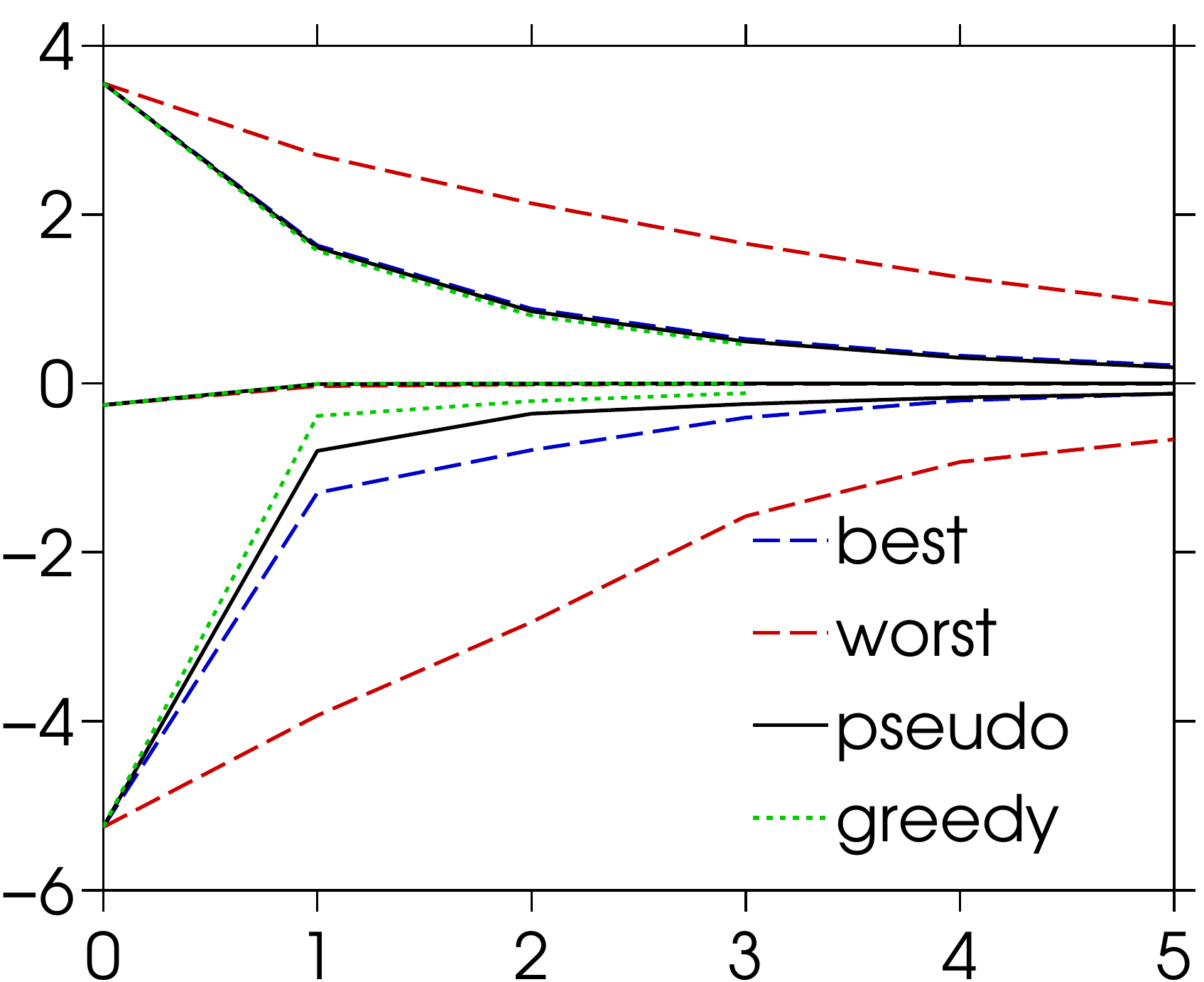} &
\- \; &
\includegraphics[width=.24\linewidth]{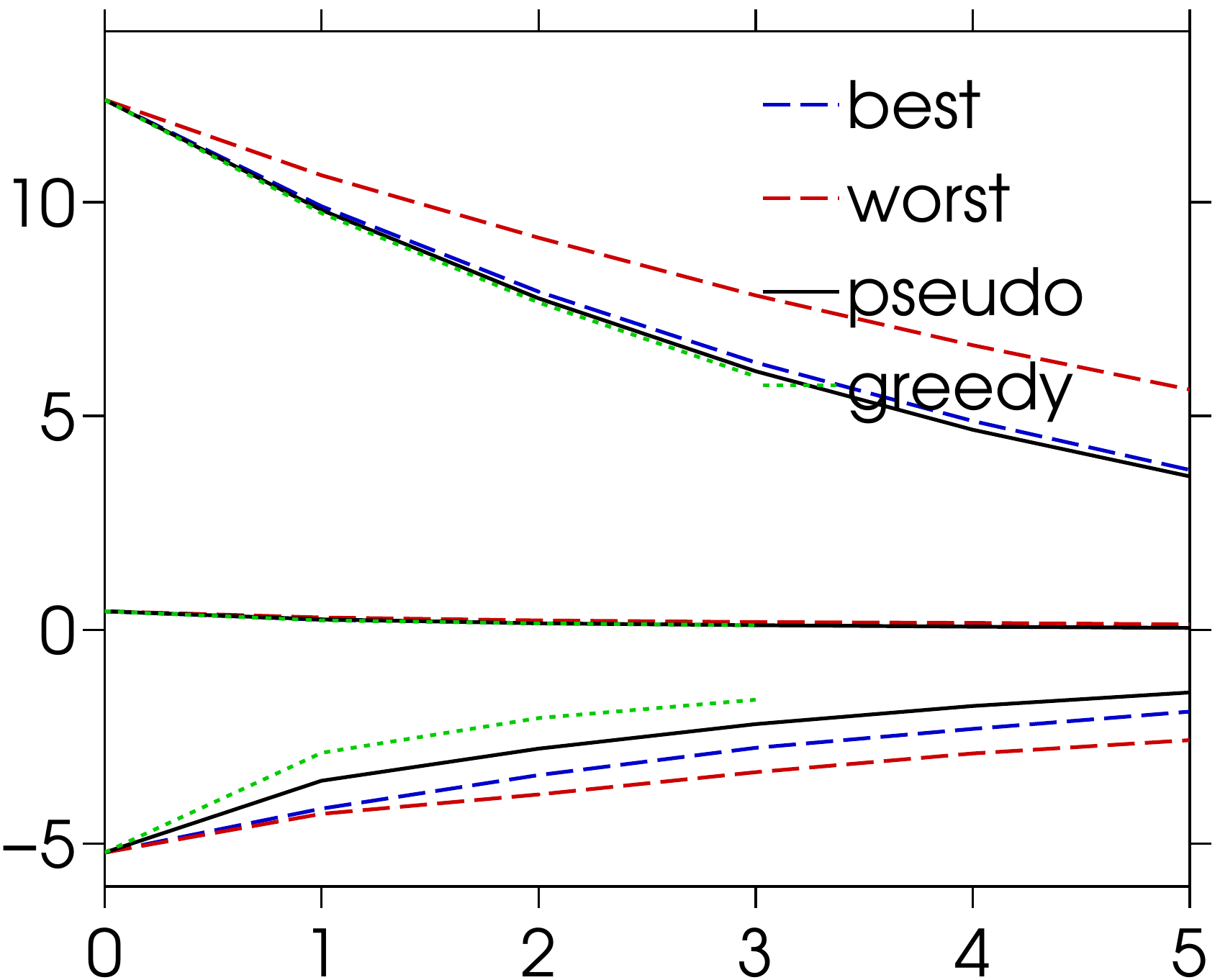} & 
\includegraphics[width=.24\linewidth]{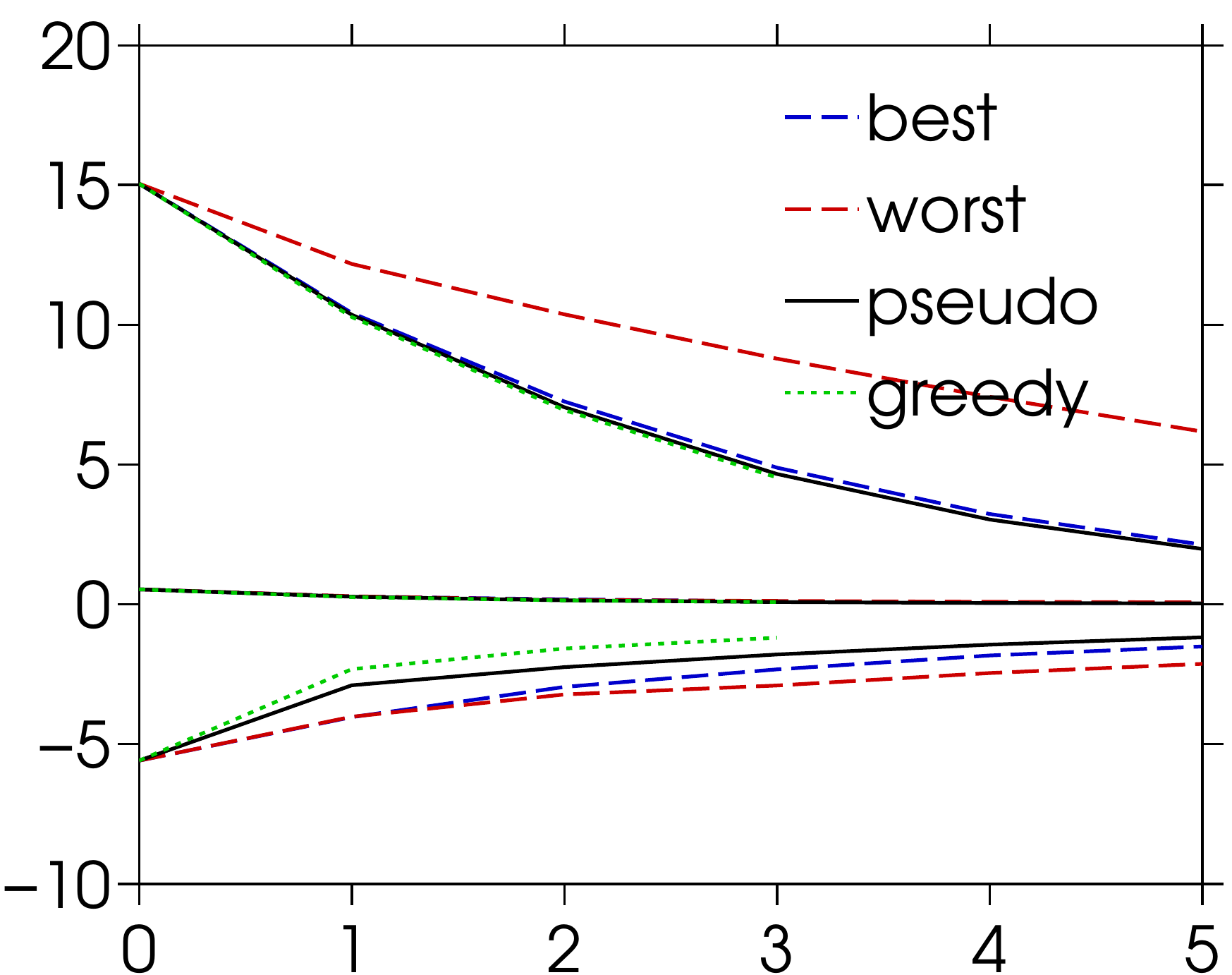} \\
\begin{sideways} {\small \- \; \qquad small (25)} \end{sideways} &
\includegraphics[width=.24\linewidth]{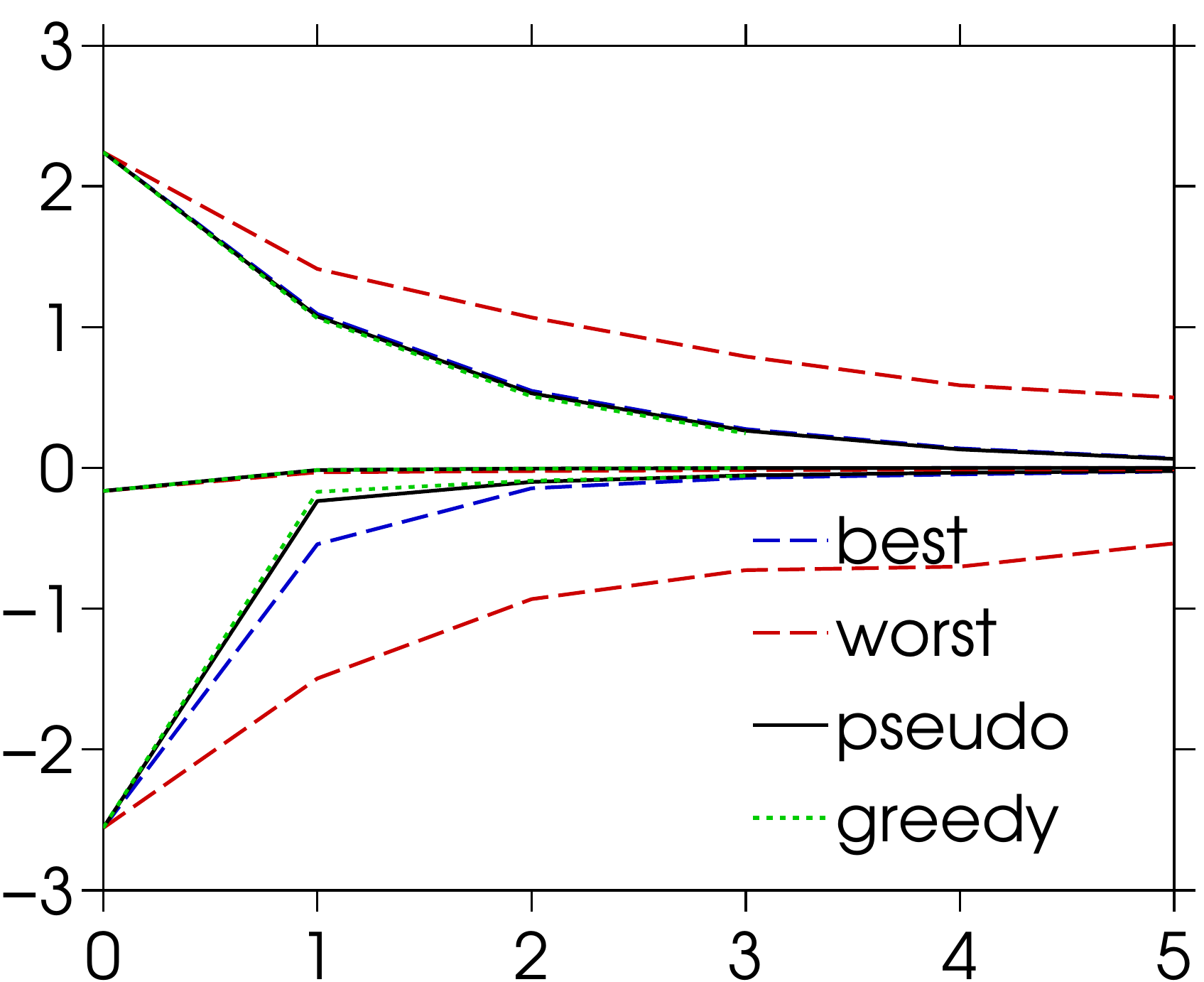} & 
\includegraphics[width=.24\linewidth]{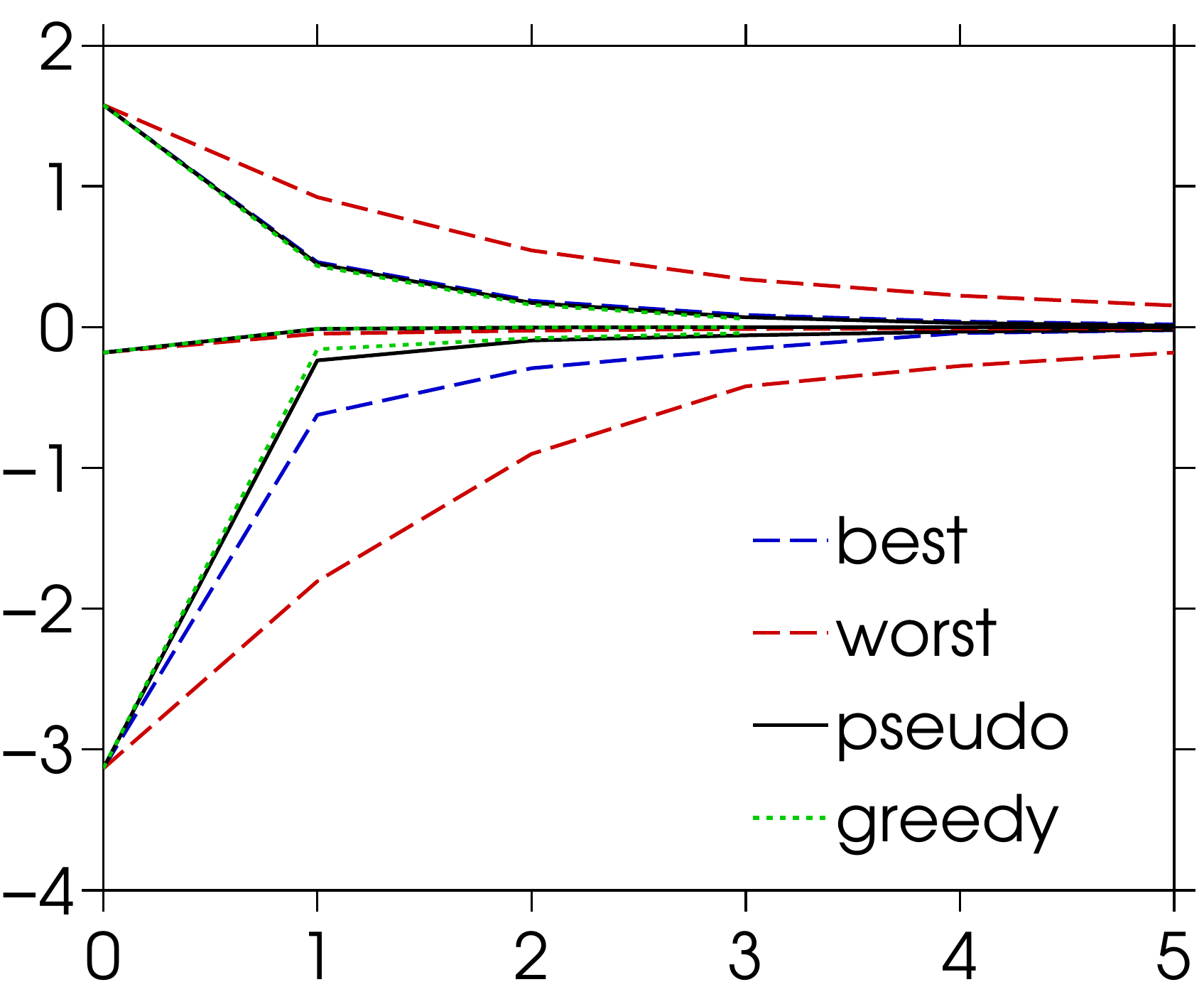} &
\- \; &
\includegraphics[width=.24\linewidth]{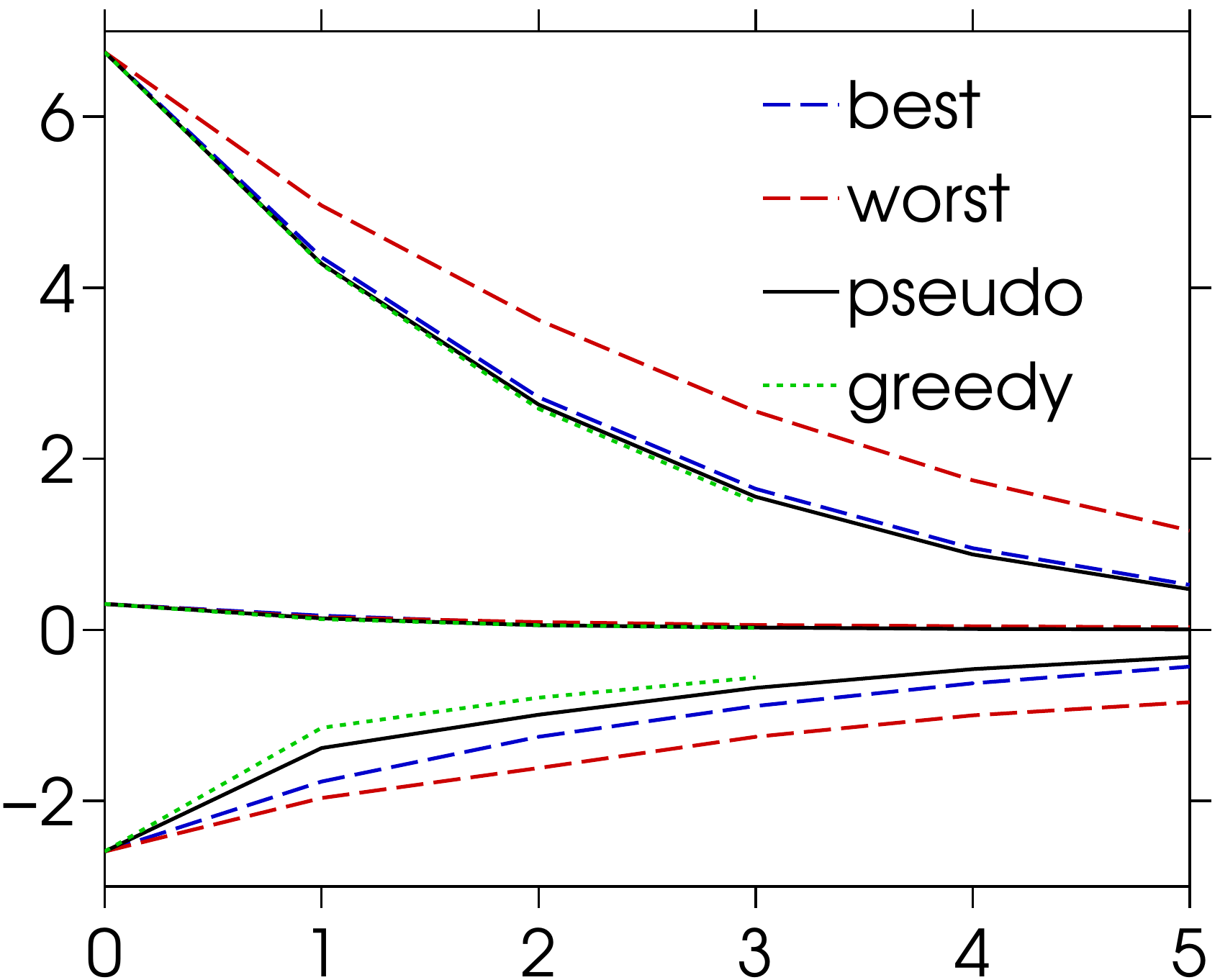} & 
\includegraphics[width=.24\linewidth]{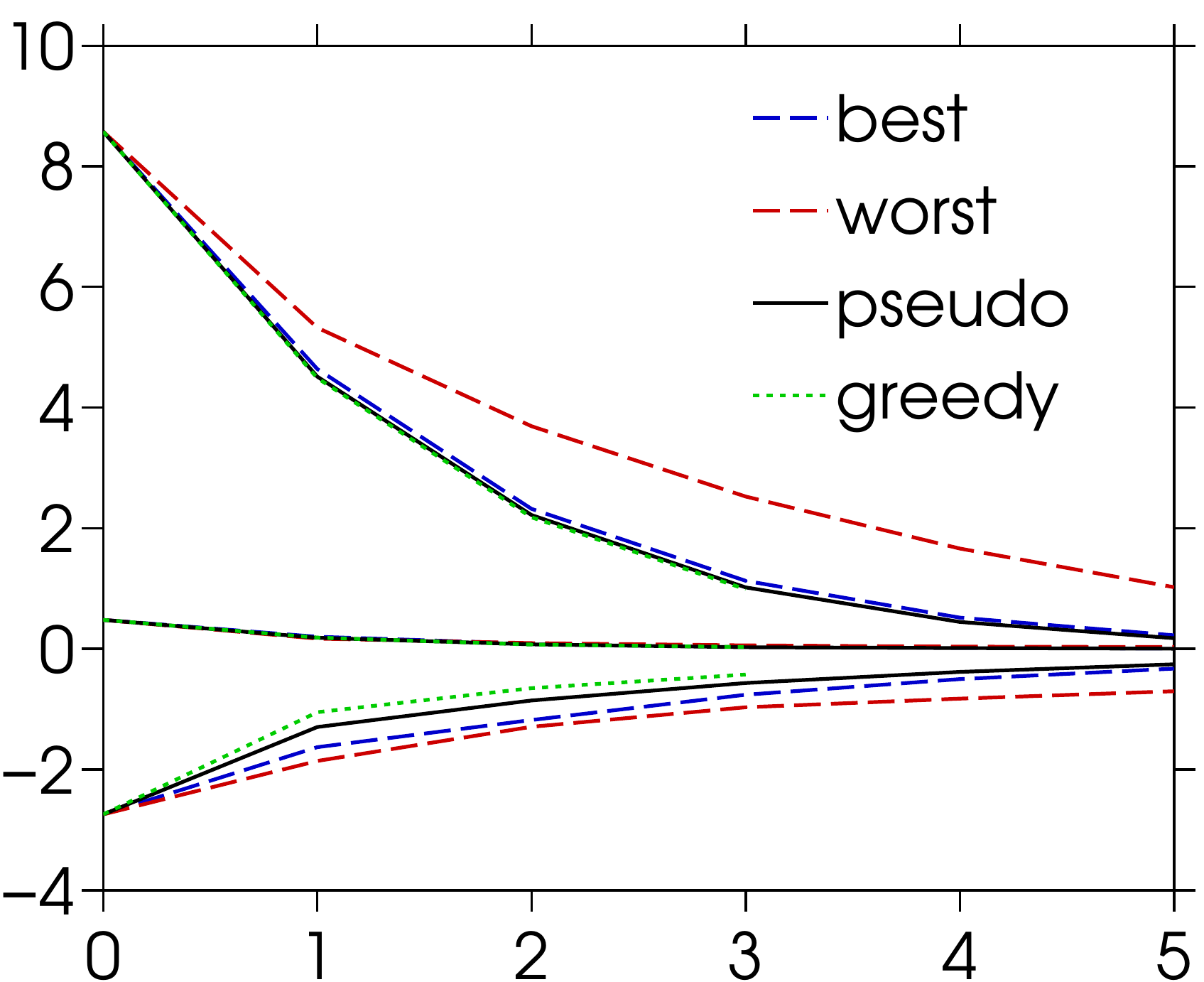} \\
& {\small attractive grids $[0,6]$} & {\small attractive random $[0,6]$} & & {\small mixed grids $[-6,6]$} & {\small mixed random $[-6,6]$}  
\end{tabular}
\end{center}
\caption{\small Average error plots $\tilde A(\theta)-A(\theta)$  %over 15 runs 
for TRW, Bethe and MF over 100 runs. Each plot shows a particular model type and size, arranged by  
%(within each plot, these are top, middle and bottom) 
model type (horizontally) and model size (vertically). 
Within each plot, to aid comparison, we show: top curves for TRW, middle curves for Bethe, and bottom curves for MF, as justified by equation \eqref{eq:sandwich}; in each case, for 0 to 5 clampings. 
%This is justified by equation \eqref{eq:sandwich}. 
Grids are toroidal 5x5, 7x7 and 9x9. Random models are Erd\"{o}s-Renyi with the same number of variables, edge probability s.t. avg degree is 4 to match grids. All models shown have $\theta_i \sim U[-2,2]$, with: attractive $W_{ij} \sim [0,6]$, or mixed $W_{ij} \sim [-6,6]$. Error shown for Bethe on mixed models is $|\tilde A(\theta)-A(\theta)|$. 
\emph{best} and \emph{worst} curves indicate the best and worst of our 10 selection heuristics, run from the start up to that clamp point.
}
\label{fig:expfinal}
\end{figure*}
\begin{figure*}
\begin{center}
\setlength\tabcolsep{1pt}
\begin{tabular}{cccc}
\includegraphics[width=.24\linewidth]{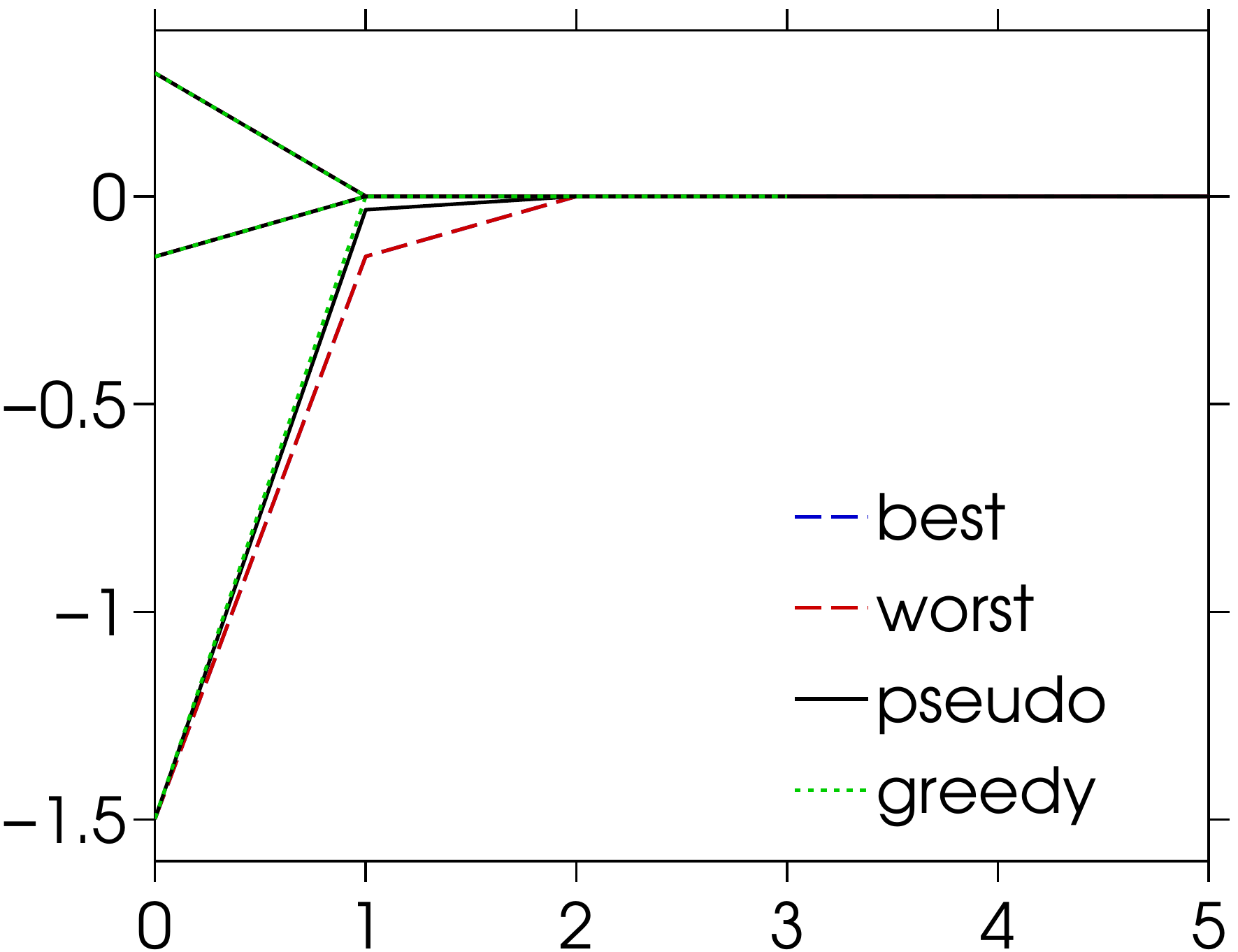} & % attractive 15
\includegraphics[width=.24\linewidth]{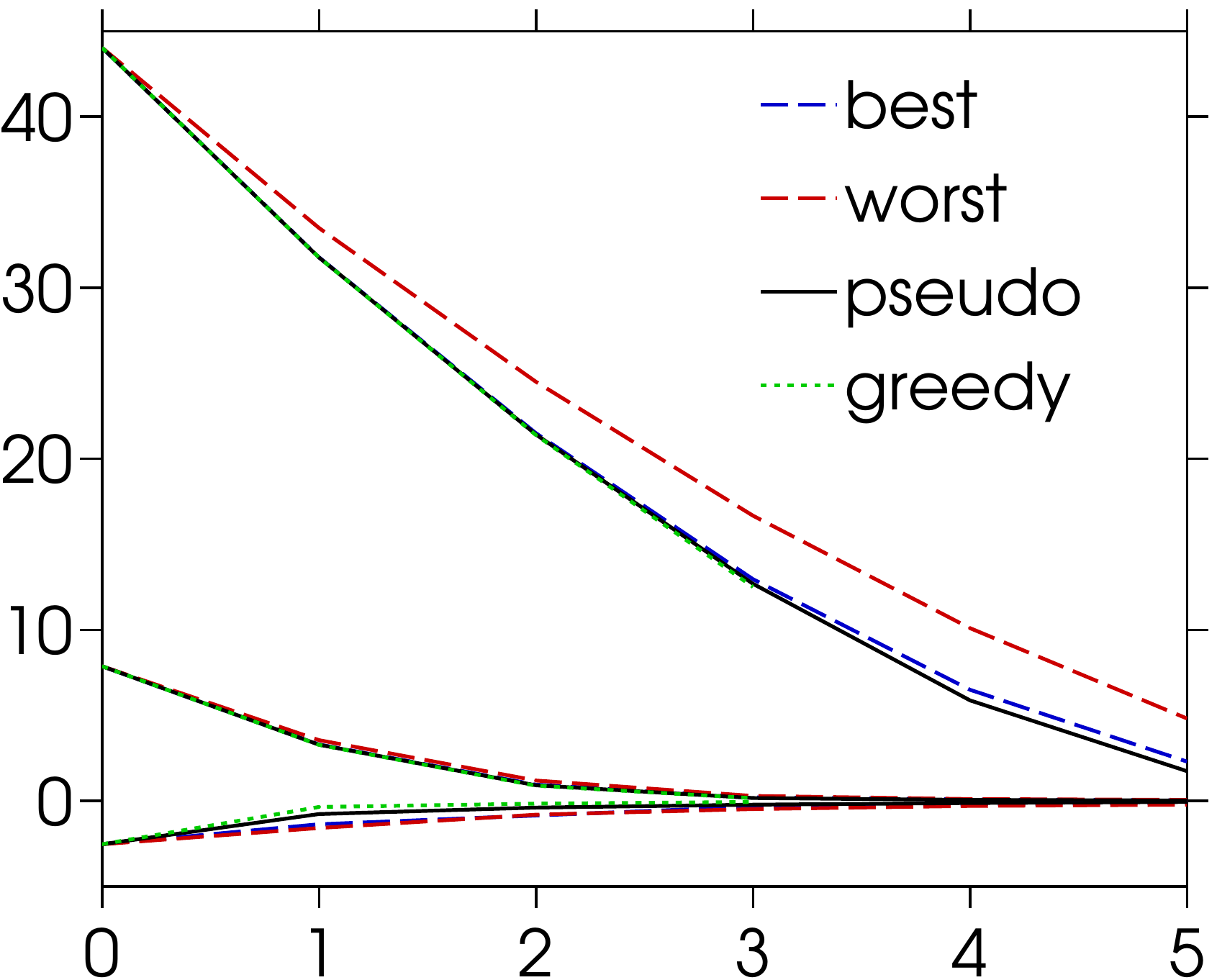} &
\includegraphics[width=.24\linewidth]{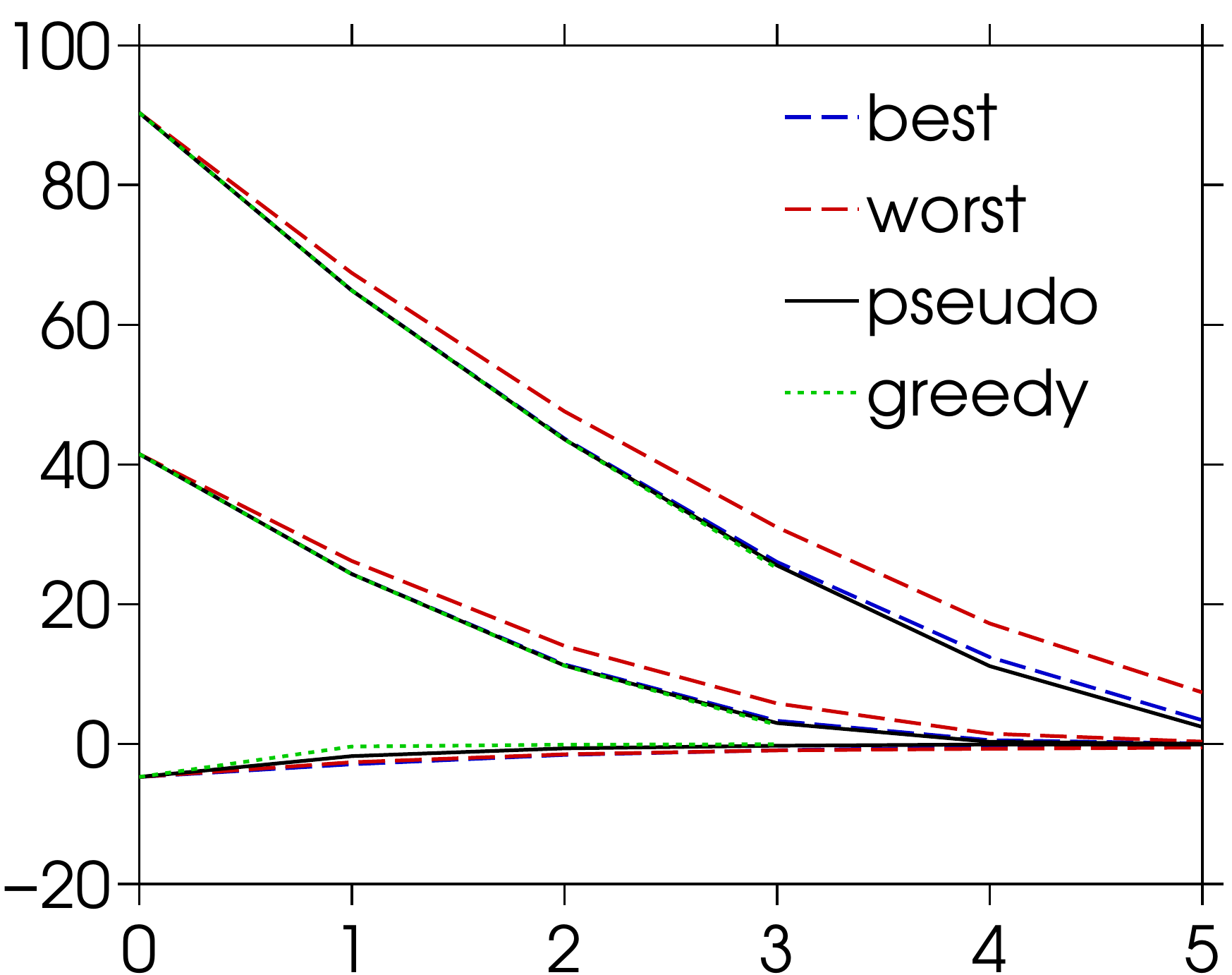} &
\includegraphics[width=.24\linewidth]{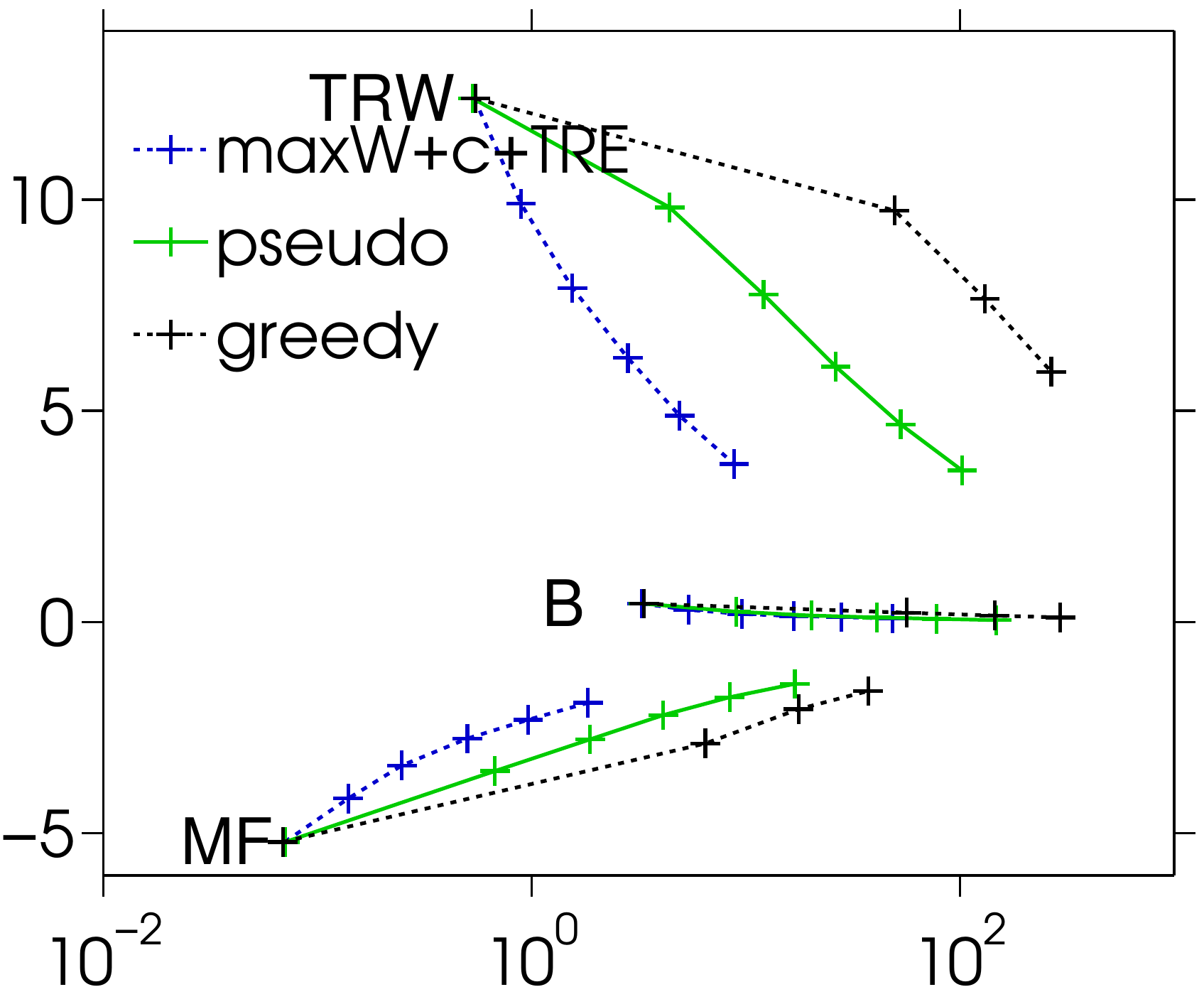} \\ % -time2 are the new v2 time plots; earlier -Copy means time
{\small (a) attractive $K_{15}, [0,6]$} & {\small (b) mixed $K_{15}, [-6,6]$} & {\small (c) mixed $K_{15}, [-12,12]$} & {\small (d) time ($\log$ axis) vs. error}\\
&&& {\small 7x7 grid, $W_{ij} \sim U[-6,6]$}
\end{tabular}
\end{center}
\caption{\small 
Left three plots (a)-(c) show average error $\tilde A(\theta)-A(\theta)$ for TRW, Bethe and MF over 100 runs; each plot shows a model on a complete graph topology on 15 variables using $\theta_i \sim U[-2,2]$ and edge weights drawn uniformly from the respective range shown.  
Within each plot, to aid comparison, we show: top curves for TRW, middle curves for Bethe, and bottom curves for MF, as justified by equation \eqref{eq:sandwich}; in each case, for 0 to 5 clampings. 
Error shown for Bethe on mixed models is $|\tilde A(\theta)-A(\theta)|$. 
\emph{best} and \emph{worst} curves indicate the best and worst of our 10 selection heuristics, run from the start up to that clamp point.  
\\ \- As shown in (c), when a mixed model is highly connected with strong edge weights, \emph{MF can be much more accurate than Bethe}. 
 We believe this is because Bethe (and TRW) can return arbitrarily high error for strong frustrated cycles,  
see \S \ref{sec:bal}.
\\ \- Right (d) shows a typical plot of runtime (secs) vs. error for our sparse models, here showing average results for a 7x7 grid with mixed edge weights $W_{ij} \sim U[-6,6]$. `B' indicates Bethe. Time plots for models with dense edges, such as those on a complete graph, look quite different, with MF performing significantly better. 
\\ \- Full details and results are provided in the Appendix.
}
\label{fig:complete}
\end{figure*}
%complete
%
%
We tested all approaches on 100 runs each of various randomly generated binary pairwise models, exploring up to 5 clampings. 
We used all topologies shown in Figure \ref{fig:expfinal}, with the following parameters: all used random singleton potentials $\theta_i \sim U[-2,2]$; attractive models had $W_{ij} \sim U[0,2]$ (typically a difficult lower intermediate value for Bethe) or $W_{ij} \sim U[0,6]$; mixed models had $W_{ij} \sim U[-6,6]$ or $W_{ij} \sim U[-12,12]$. Exact values were computed using the junction tree algorithm. All inference methods were implemented using the standard open source libDAI library \citep{ libdai}. 
In addition, we performed experiments on complete graphs with fewer variables but a similar number of edges, see Figure \ref{fig:complete}, and on random 4-regular graphs.  
Figure \ref{fig:complete}(d) shows typical timings vs. performance. 
Full details and results are provided in the Appendix.

We implemented all variable selection methods, specifically: maxW, Mpower, frustCycles and strongCycles, always first stripping to the core. We also tried the original maxW without stripping (as a comparison, which was previously shown to perform well).  
In addition, we used TRE versions of all these, for a total of 10 heuristics. We also implemented a greedy search over all possible clampings up to 3, to see how this would perform compared to our heuristics, and we implemented pseudo-greedy, which tried only the 10 heurstics in our basket and picked the best performer. This best performer was determined for MF (TRW) by the highest (lowest) solution. Similarly for Bethe for attractive models, best performer meant the variable which led to the greatest increase. For pseudo-greedy for Bethe on mixed models, where there is no way to be sure which is best, we tried various options, settling on picking the variable that gave the best improvement (i.e. the biggest fall) in TRW.

\subsection{Discussion}\label{sec:disc}

Looking across all results (Figures \ref{fig:expfinal} and \ref{fig:complete}, with more details in the Appendix), we make the following observations. 

Bethe typically dominates for accuracy as an inference method, as has been previously observed. However, as mixed models become more densely interconnected with strong edges, %larger and denser with strong edge weights, 
MF becomes competitive, and can even be far superior to %become much more accurate than 
Bethe, e.g. see Figure \ref{fig:complete}(c). We believe this is because Bethe (and TRW) can return arbitrarily high error for strong frustrated cycles,  
see \S \ref{sec:bal}. Figure \ref{fig:histB} in the Appendix shows a histogram of Bethe errors on mixed models.

Clamping improves accuracy significantly, particularly when models have many strong edge weights. The improvement is greater for random models than for those with fixed degree; this is likely because some high degree variables will be present which will be good to clamp. Our heuristics perform well. Pseudo-greedy (which takes the best of our set of heuristics at each stage) performs almost identically to true greedy (which tries all possible clampings), except for MF. There, we believe that part of the effect is due to the highly non-convex optimization, so that true greedy effectively gets the benefit of many random initializations. There is clear value to probing with our portfolio of heuristics and picking the best, since %. We examined how often each individual heuristic performed best, and 
no one method dominated. The best single performer was maxW after being augmented with our core and TRE updates (these augmentations were particularly helpful after the first clamping), which was best 
only about half the time (Figure \ref{fig:succ} in the Appendix). %In addition, Mpower with TRE did well for MF. 
See the Appendix for more details, including error plots zoomed in around the Bethe results. 

Runtime varies significantly, see Figure \ref{fig:complete}(d) for a typical example. Considering clamping approaches, greedy takes the longest time though yields little benefit over pseudo-greedy. maxW, augmented with our core and TRE updates, is fast and yields the best time-adjusted results.  
See Appendix for all timings, and note that sometimes, clamping makes the subsequent optimization problems easier to solve, hence the total time with clamping is occasionally \emph{lower} than without, while also being significantly more accurate.  All branches over multiple clampings can be parallelized, as clearly can (pseudo-)greedy approaches.  
As an inference method, MF runs the fastest but this could be influenced by our implementation, with all timings sensitive to parameters. In order to get TRW to converge, we used damping which significantly slows it down, though there may be faster convergent methods. Further, edge weights could be optimized which is not implemented in libDAI. For Bethe, we used the HAK double-loop algorithm \citep{HAK03}, which was needed to ensure convergence.

 Additional discussion on greedily selecting which next variable to clamp is provided in \S \ref{sec:addl} of the Appendix.

\section{CONCLUSION}

TRW and MF are important and widely-used methods of approximate inference, yielding useful upper and lower bounds on the true partition function. 
For both methods, 
we have derived guarantees on the beneficial effect of clamping any variable and summing approximate sub-partition functions, % using the important TRW and MF methods, 
for any model type (attractive or not) and any number of discrete labels. Such guarantees have been difficult to obtain, 
and do not apply in general for Bethe, 
 with the only prior result to our knowledge being that of \cite{zbn} for Bethe, only for the restricted case of attractive binary pairwise models. 
%Indeed, there are models where clamping any variable worsens the resulting Bethe approximation. 
By clamping TRW or MF, this leads directly to useful improved upper or lower bounds on the true partition function, and also helpfully on the optimum Bethe approximation. Further, our derivation provides a surprising interpretation in terms of a tightening of the local polytope relaxation ($\LP^{(i)}$ in Theorem \ref{clamp-trw}).

Earlier approaches to selecting a variable to clamp can perform poorly in some settings. We examined when this is likely to occur, and introduced new methods based on first stripping to the core, looking for heavy (frustrated) cycles, and using singleton entropies. These new methods empirically yielded significant %empirical 
benefits.

Based on an experimental comparison across the different inference approaches and clamping selection methods, including examining the accuracy improvement vs. time tradeoff, we are able to suggest the following practical recommendations:
\bit[leftmargin=*]
\tm As has been previously observed, typically Bethe is the best approach, provided convergence difficulties do not arise. However, perhaps surprisingly, for densely connected mixed models with strong edges, MF can be much more accurate (e.g. Figure \ref{fig:complete}(c)). 
\tm Clamping can be very helpful, more so for denser models with stronger edge weights, a setting where inference on the original model is hard.
\tm For variable selection,  
if speed is critical, use just the updated maxW heuristic (augmented with core and TRE). Otherwise, use our basket of approaches and pick the pseudo-greedy best option. For Bethe on mixed models, use TRW to guide pseudo-greedy selection. 
\tm In many cases, it will be helpful to run MF and TRW in order to obtain guaranteed bounds on the true partition function. If a Bethe method is used, %the MF and TRW 
bounds can also 
be useful 
%helpfully be used 
to check if a poor local optimum was returned (below the MF value).  %Additional discussion is provided in \S \ref{sec:addl} of the Appendix.
\eit

\newpage
\begin{footnotesize}
\bibliography{aiReferences}
\end{footnotesize}

\clearpage
\newpage
\onecolumn
\section*{APPENDIX: SUPPLEMENTARY MATERIAL\\ Clamping Improves TRW and Mean Field Approximations}

In this Appendix, we provide:
\bit
	\tm Details of all methods used for selecting a variable to clamp.
	\tm Additional experimental details and results.
	\tm Additional discussion on greedily selecting a variable to clamp.
\eit

\section{Details of all methods for selecting a variable to clamp}

\subsection{Earlier methods: maxW and Mpower}

The maxW and Mpower heuristics introduced by \cite{zbn} were defined as follows.

\subsubsection{maxW}

This is the simplest method yet it can be very effective. Assign a clamp score $s(i)$ to each variable $i$ by setting $s(i) = \sum_{j \in \N(i)} |W_{ij}|$. Pick the variable with highest score.

\subsubsection{Mpower}

Form matrix $M$ defined by $M_{ij} = \frac{1}{n-1} \tanh \big| \frac{W_{ij}}{4} \big|$. The $\tanh$ term is inspired by the effect of cycles in Lemma 5 from \citealp{abc}, which was derived using loop series methods \citep{Sud07,CheChe06}, and dividing by $n-1$ is in order to ensure that row sums are $<1$ and hence $\sum_{k=1}^\infty M^k$ is convergent. Note that $[M^k]_{ii}$ is the sum over all paths of length $k$ from $i$ to $i$ of the product of the modified edge weights along the cycle. 

To compute the sum over all $k$, evaluate $(I-M)^{-1} - I$ and examine diagonal terms. However, this overcounts all cycles, and in particular it includes relatively high value terms coming from paths simply from $i$ to any neighbor $j$ and back again, along with all powers of these. In order to discard these, compute clamp score $s(i)$ as the $i$th diagonal term of  $(I-M)^{-1} - I$ minus $s_i/(1-s_i)$, where $s_i$ is the $i$th diagonal term of $M^2$. Pick the variable with highest score.

See \citep[Supplement]{zbn} for more details.

\subsection{New methods}

We introduced the following new methods.

\subsubsection{frustCycles and strongCycles}

For frustCycle, the goal is to try to identify at least one frustrated cycle composed of edges $(i,j)$ with high absolute weight $|W_{ij}|$. The method builds on an algorithm introduced by \cite{Son12}. strongCycles works in %almost exactly 
the same way but also takes into consideration balanced cycles. These approaches are the first to examine the sign of edge weights $W_{ij}$ (in order to identify strong frustrated cycles) rather than just their absolute value $|W_{ij}|$.

For both methods, the value of removing a cycle is estimated using a cycScore heuristic from Lemma 5 of \citealp{abc}, which uses the loop series method \citep{CheChe06,Sud07} to attempt to estimate the extent of error, i.e. $A\pt - \tilde A\pt$, caused by the cycle. This estimate will be positive for a balanced cycle and negative for a frustrated cycle, see \S \ref{sec:bal}. Algorithm \ref{alg:frust} provides an outline of the methods.

\begin{algorithm}
\caption{frustCycles / strongCycles methods to select a variable to clamp}
\label{alg:frust}
\textbf{Input:} Edge weight model parameters $\{W_{ij}$\}\\
\textbf{Output:} Clamp scores $s(i)$ for each variable, variable to clamp
\begin{algorithmic}[1]
\STATE Similarly to Mpower, set all $M_{ij} \leftarrow \tanh \big| \frac{W_{ij}}{4} \big|$. % (here no division by $n-1$).
\STATE Construct a maximum weight spanning tree $T$ using weights $\{M_{ij}\}$.
\STATE Initialize all variables to have clamp score $s(i) \leftarrow 0$.
\FORALL {edges of the model not in $T$} 
	\STATE Consider the edge together with $T$, which creates a cycle $C$ including the edge %Add the edge to $T$ thus creating a cycle  $C$ %with edges
	\STATE For the cycle $C$, compute cycScore $\leftarrow \log (1 + \prod_{(i,j) \in C} \tanh \frac{W_{ij}}{4})$  // note the signed $W_{ij}$ here, see \citep[Lemma 5]{abc}
	\STATE For strongCycles, add cycScore to all vertices in $C$
	\STATE For frustCycles, add cycScore to all vertices in $C$ only if cycScore$<0$, i.e. a frustrated cycle 
\ENDFOR
\STATE Set all $s(i) \leftarrow |s(i)|$.
\STATE Pick the variable with highest score.
\end{algorithmic}
\end{algorithm}

\subsubsection{Strip to the core}
To strip a model to its core, simply iteratively remove variables with degree 1 until no more remain, see Figure \ref{fig:core}. This is typically run as a pre-processing step before applying other clamp selection heuristics. When removing variables, care must be taken to keep track of the original variable indices for those that remain. 
For all methods above, we first strip to the core. 

%\subsection{Recap}

We have described 4 heuristics so far: maxW, Mpower, frustCycles and strongCycles, all of which first strip to the core. In addition, recognizing that MF makes the assumption that all variables are independent, which may be poor when edge strengths are strong irrespective of cycles being present, we also use a maxW0 heuristic which does not first strip to the core, for a total of five.

\subsubsection{TRE methods}

All methods so far use only edge weights. We add TRE versions of all the above (see \S \ref{sec:clumps}), hence this gives 10 heuristics.

\subsubsection{Meta-heuristics for clamping}

Now we have described 10 heurstics, each of which performs better or worse in different contexts. For MF, which always yields a lower bound for $A(\theta)$, and TRW, which always yields an upper bound, we may `probe' by trying all these heuristics and then pick the one that yields the best improvement in $\tilde A\pt$. That is, for MF, take the one that yields $\max \tilde A\si_M\pt$; for TRW, take the one that yields $\min \tilde A\si_T\pt$. 
For Bethe, if the model is attractive, then $\tA_B\si\pt$ is a lower bound and we may similarly pick the heuristic that delivers $\max \tA\si_B\pt$. We call these meta-heuristics \emph{pseudo-greedy}. In addition, we can do the same `full greedy' process, where we try to clamp \emph{all} possible variables then pick the best. We call this full meta-approach \emph{greedy}. 

For Bethe on mixed models, we can't know in advance if we have an over- or under-estimate, so we cannot do exactly the same thing. Instead, we explored performance achieved by picking the variable that gave best improvement in: (i) TRW; (ii) TRW-MF, that is the gap between the two; and (iii) MF. Of these, the greedy-TRW heuristic was the most successful, and it is this version of greedy (and correspondingly, pseudo-greedy) that we report for Bethe on mixed models.

\section{Additional experimental details and results}

For all inference methods, we used the open source libDAI library \citep{libdai}, with the following parameters:

For MF, \\
\texttt{ MF[tol=1e-7,maxiter=10000,damping=0.0,init=RANDOM,updates=NAIVE] }

For Bethe, \\
\texttt{ HAK[doubleloop=1,clusters=BETHE,init=UNIFORM,tol=1e-7,maxiter=10000] } 
\\This is guaranteed to converge to a stationary point of the Bethe free energy (whereas BP may not converge).

For TRW, \\
\texttt{ TRWBP[updates=SEQFIX,tol=1e-7,maxiter=10000,logdomain=0,nrtrees=1000,...\\damping=0.25,init=UNIFORM] }

Note that, particularly for MF and Bethe methods, we may obtain local (rather than global) optima. Because of this, we might occasionally observe results that get worse with clamping, even where theory shows that the global optimum can only improve. For MF, initializing the optimization of each clamped problem to the solution of the parent problem removes this concern, though we did not find this necessary in practice: empirically it appeared sufficient to initialize using the same random seed each time. 
For Bethe, we see no easy way to avoid this issue without using expensive methods such as those of \cite{nb}.

\paragraph{Models. }
All grids are toroidal, so all variables have degree 4. Random 4-regular graphs are randomly generated s.t.  all variables still have exactly degree 4, though the structure is random. All random Erd\"{o}s-Renyi models have edge probability s.t. the average degree is 4. Note that the complete graphs have far fewer variables (just 10 or 15) but are much more densely connected (with roughly the same number of edges as the corresponding models with degree 4) and have higher treewidth.

\subsection{Additional results}
%\newpage
We first provide plots of number of clamps vs. error for all runs, then the same plots but zoomed in so that the Bethe results are easier to see; then plots of runtime ($\log$ scale) vs. error for all runs. Note that sometimes clamping makes the subsequent optimization problems easier to solve, hence the total time with clamping is occasionally \emph{lower} than without, while also being significantly more accurate (for example, see TRW performance with mixed $[-6,6]$ for the complete graph on 10 variables in Figure \ref{fig:weird}). 

Next in Figure \ref{fig:histB}, we show the distribution of signed error $\tA_B(\theta)-A(\theta)$ for Bethe on all our mixed models, showing the bias toward overestimation suggested by the discussion in \S \ref{sec:bal}.

Finally, in Figure \ref{fig:succ}, we provide plots showing performance of each heuristic - this indicates how often each one picks the same variable to clamp as pseudo-greedy at each specific clamp step.

\newpage
\begin{figure}
\begin{center}
\setlength\tabcolsep{1pt}
\begin{tabular}{ccccccc}
\begin{sideways} {\small \- \; \qquad large (81)} \end{sideways} &
\includegraphics[width=.24\linewidth]{grid_+9_9=_+-2_2=_+-6_6=-eps-converted-to.pdf} & 
\includegraphics[width=.24\linewidth]{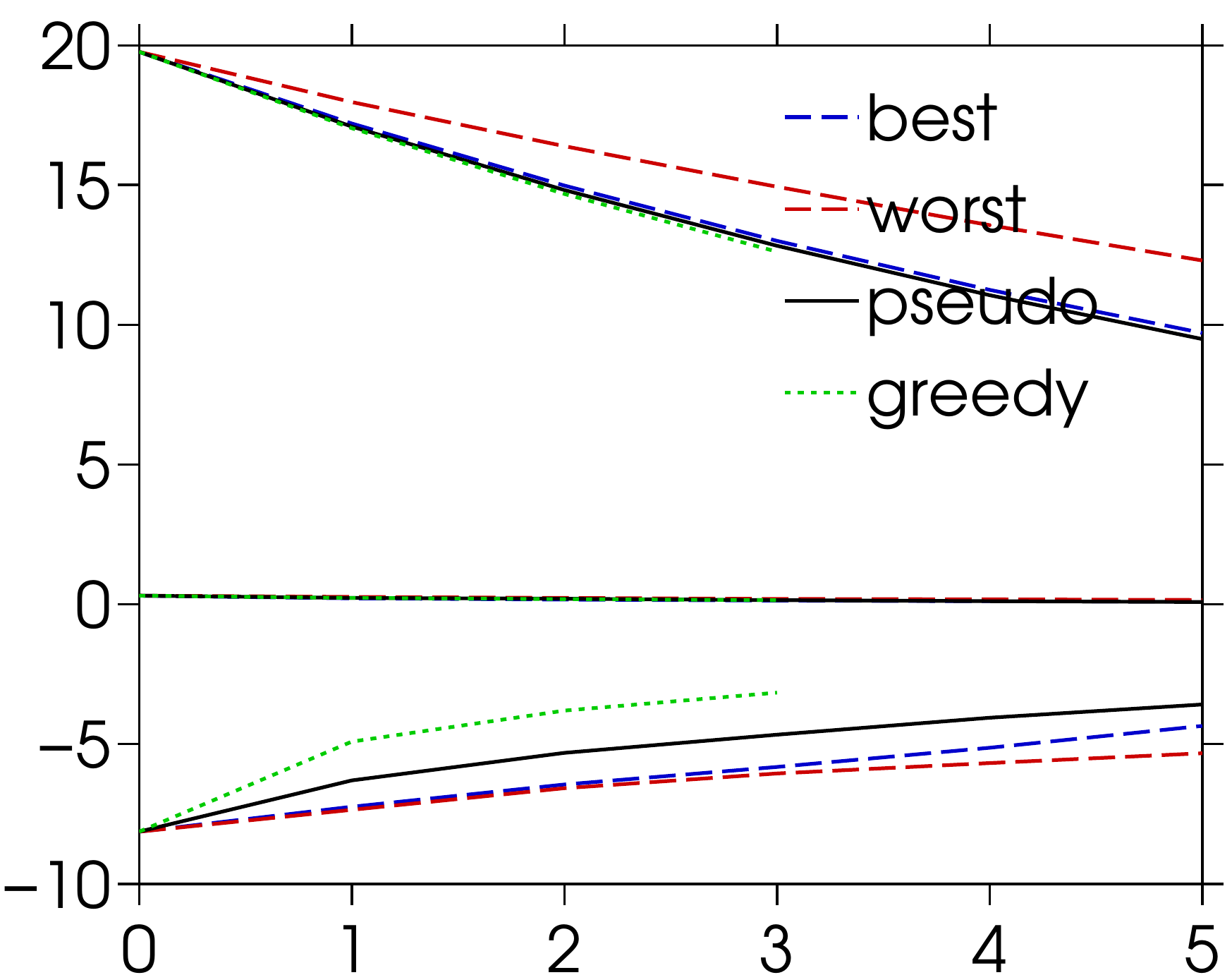} & 
\includegraphics[width=.24\linewidth]{random_graph_+81_05=_+-2_2=_+-6_6=-eps-converted-to.pdf} 
\\
\begin{sideways} {\small \- \; \; \quad medium (49)} \end{sideways} &
\includegraphics[width=.24\linewidth]{grid_+7_7=_+-2_2=_+-6_6=-eps-converted-to.pdf} & 
\includegraphics[width=.24\linewidth]{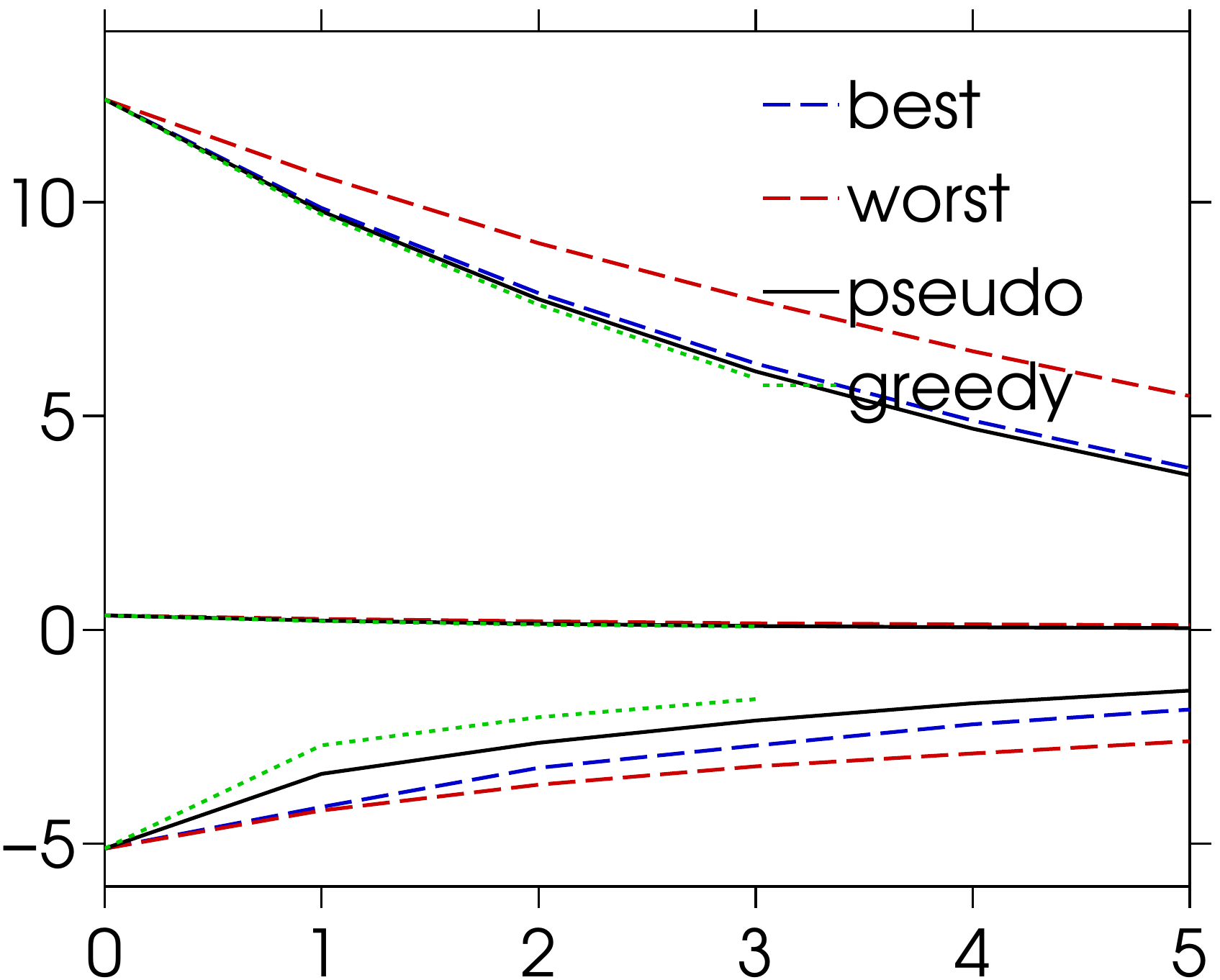} &
\includegraphics[width=.24\linewidth]{random_graph_+49_0833333333333333=_+-2_2=_+-6_6=-eps-converted-to.pdf} &
\- \; &
\includegraphics[width=.24\linewidth]{complete_15_+-2_2=_+-6_6=-eps-converted-to.pdf} & 
\begin{sideways} {\small \- \;  \; \; complete $K_{15}$} \end{sideways} 
\\
\begin{sideways} {\small \- \; \qquad small (25)} \end{sideways} &
\includegraphics[width=.24\linewidth]{grid_+5_5=_+-2_2=_+-6_6=-eps-converted-to.pdf} & 
\includegraphics[width=.24\linewidth]{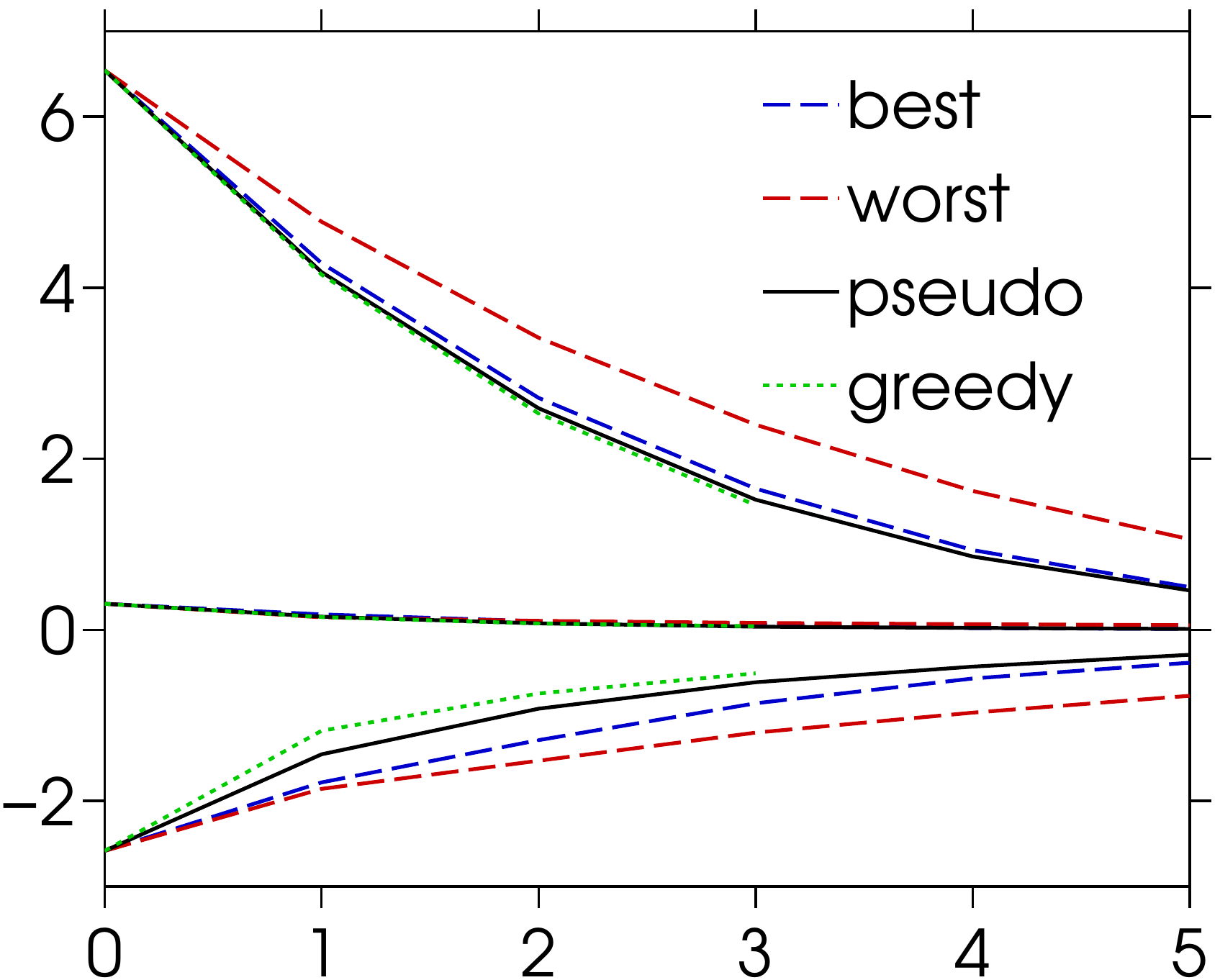} &
\includegraphics[width=.24\linewidth]{random_graph_+25_166666666666667=_+-2_2=_+-6_6=-eps-converted-to.pdf} &
\- \; \; &
\includegraphics[width=.24\linewidth]{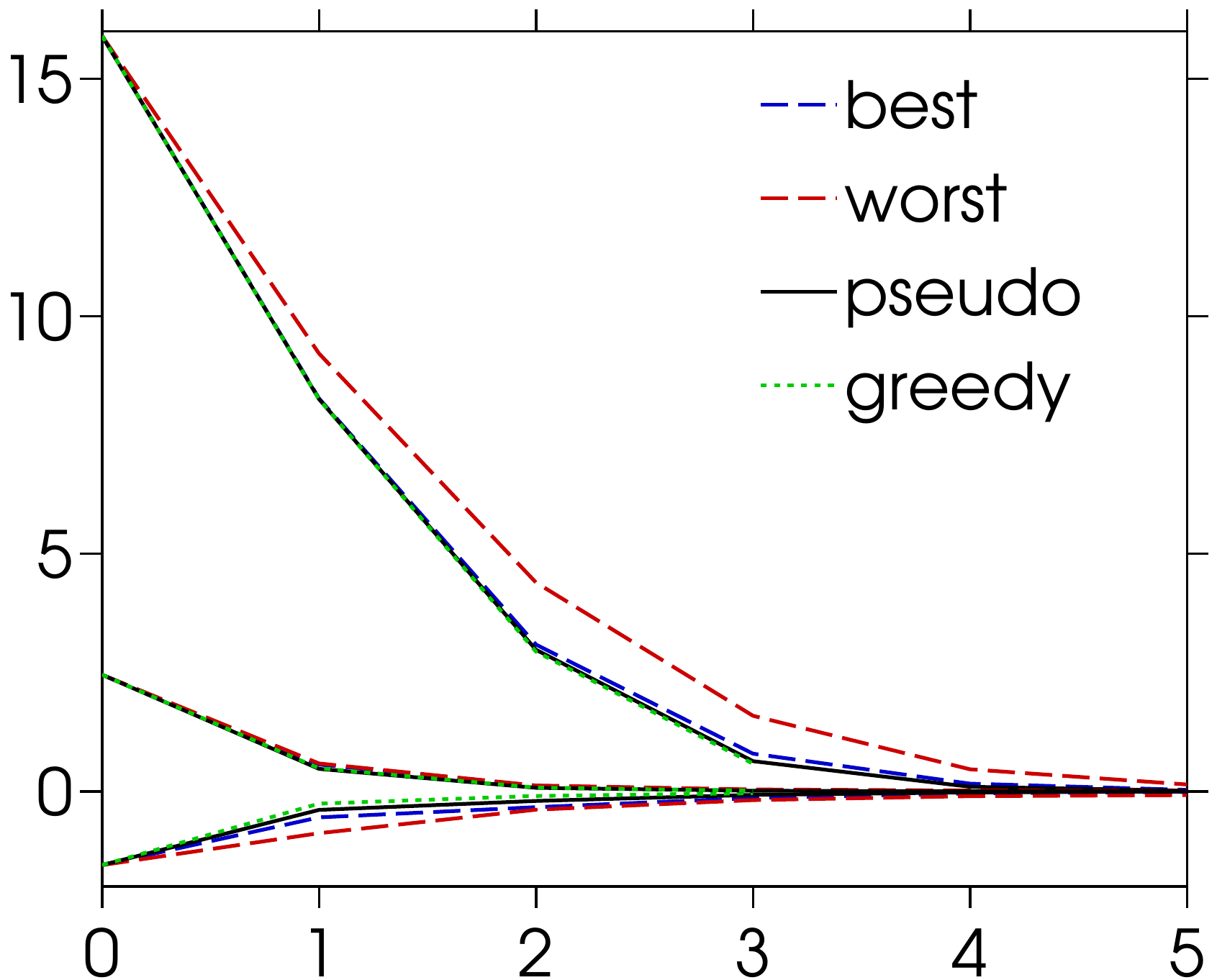} &
\begin{sideways} {\small \- \; \; \;  complete $K_{10}$} \end{sideways} 
\\
& {\small grids} & {\small random 4-regular} & {\small random Erd\"{o}s-Renyi} & & {\small complete graph} 
\end{tabular}
\end{center}
\caption{\small Mixed $[-6,6]$
}
\label{fig:exp1}
\end{figure}

\begin{figure}
\begin{center}
\setlength\tabcolsep{1pt}
\begin{tabular}{ccccccc}
\begin{sideways} {\small \- \; \qquad large (81)} \end{sideways} &
\includegraphics[width=.24\linewidth]{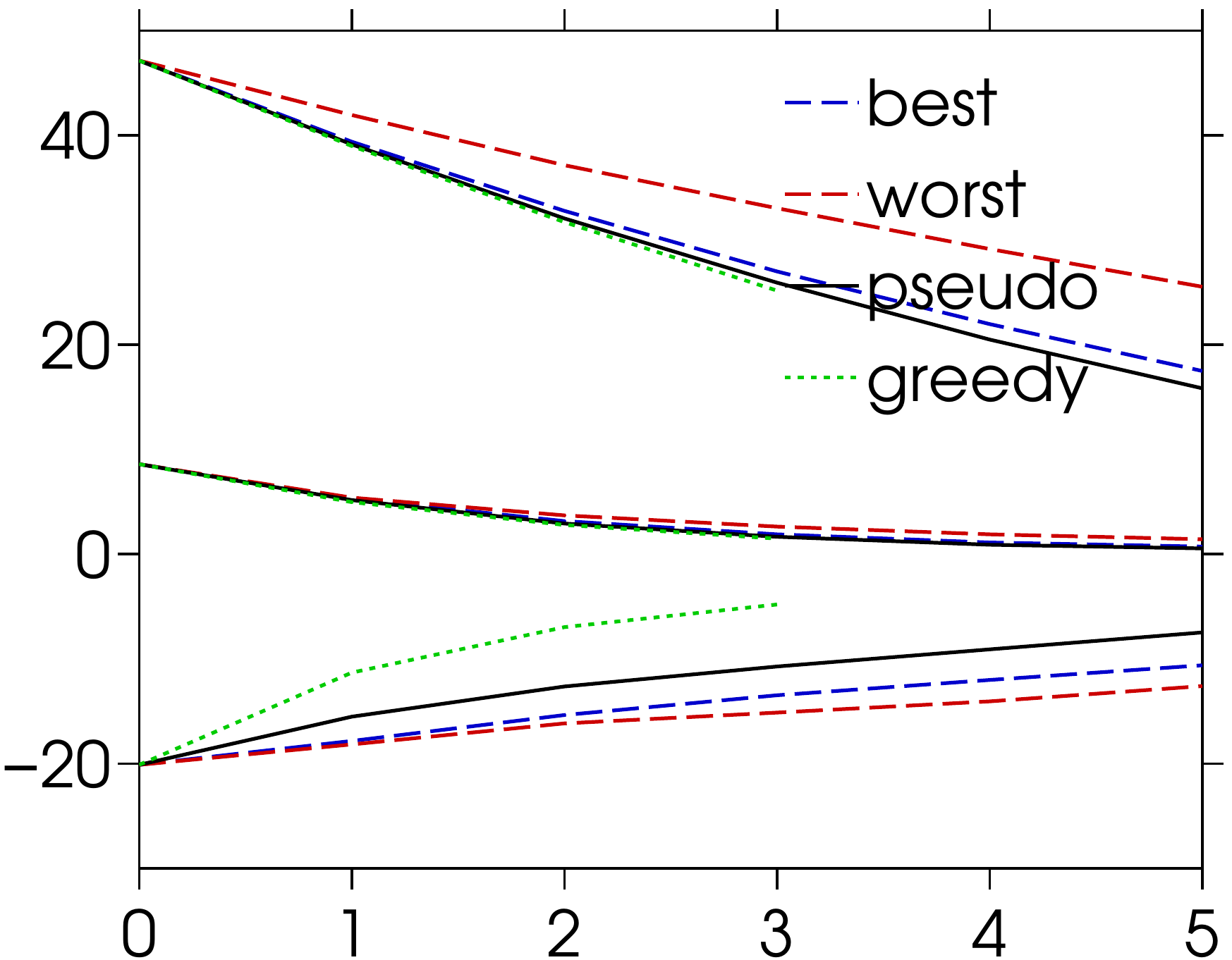} & 
\includegraphics[width=.24\linewidth]{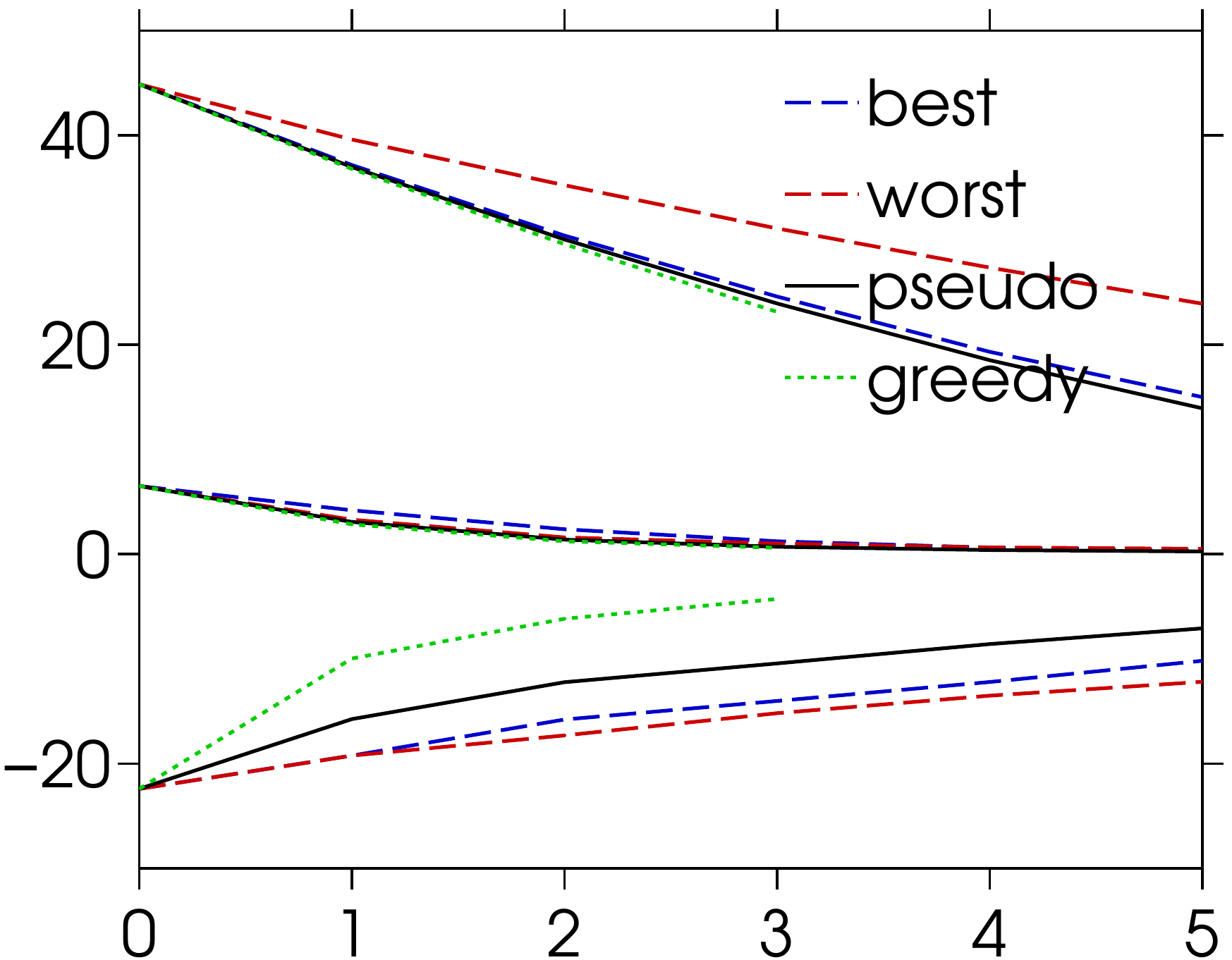} &
\includegraphics[width=.24\linewidth]{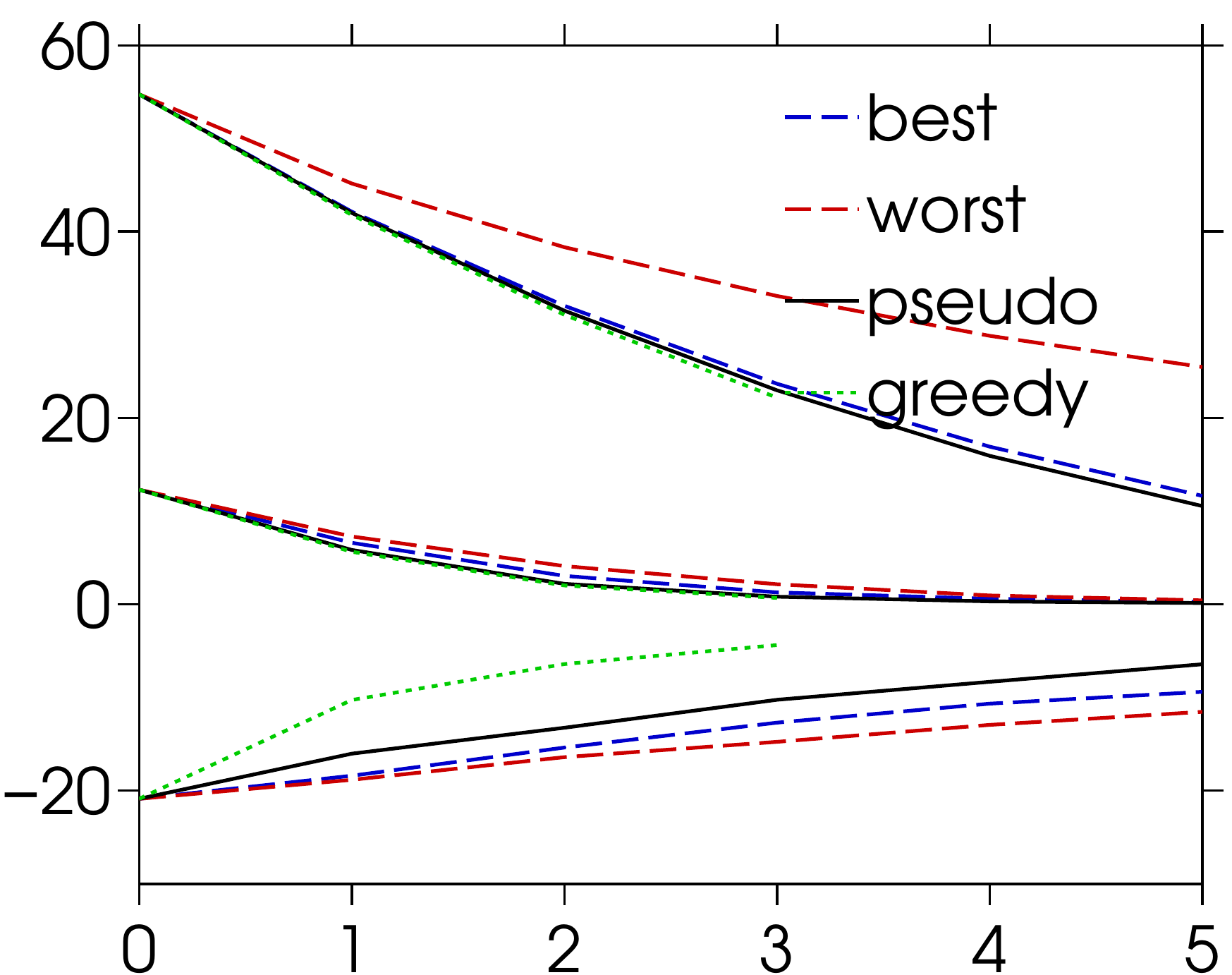} 
\\
\begin{sideways} {\small \- \; \; \quad medium (49)} \end{sideways} &
\includegraphics[width=.24\linewidth]{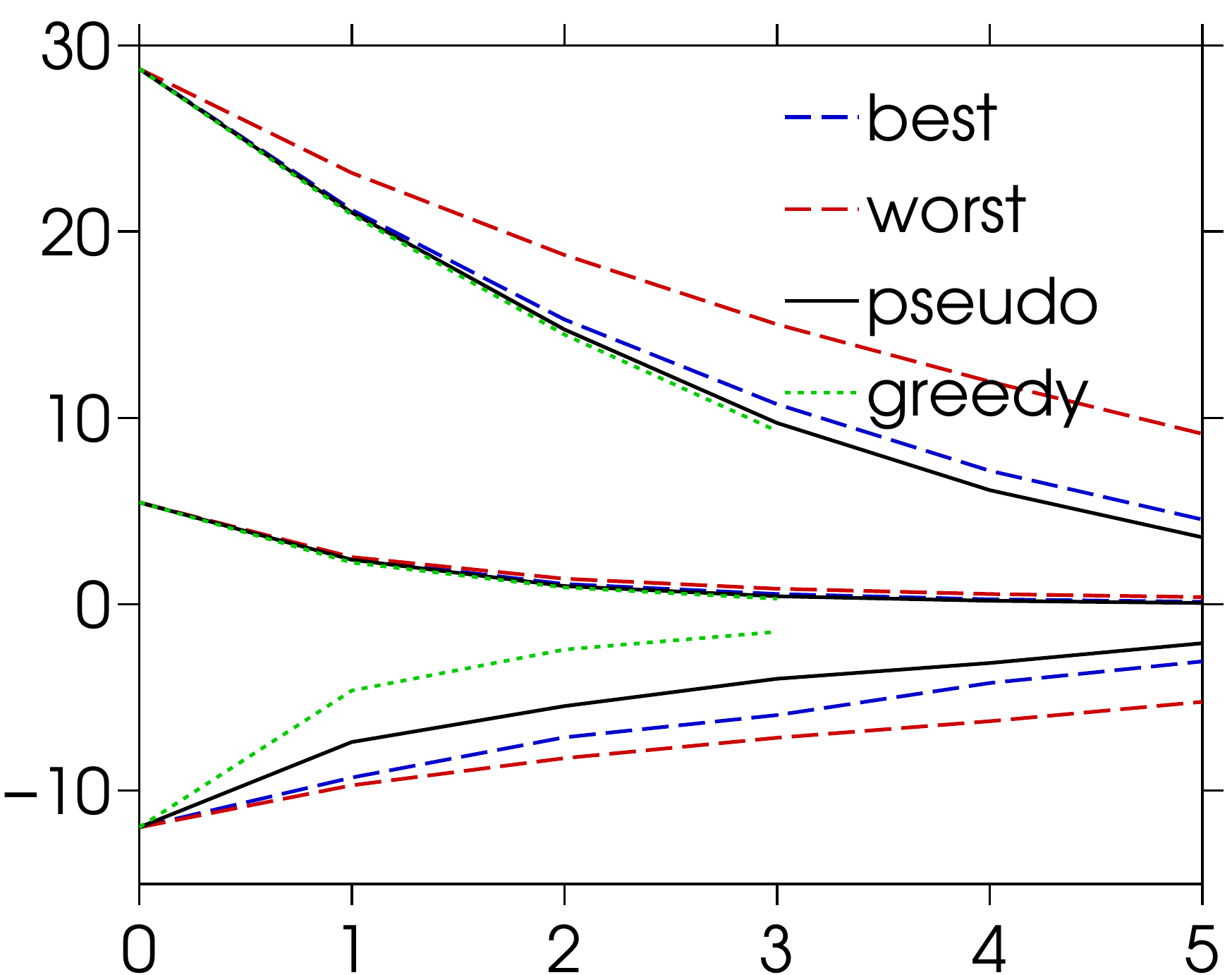} & 
\includegraphics[width=.24\linewidth]{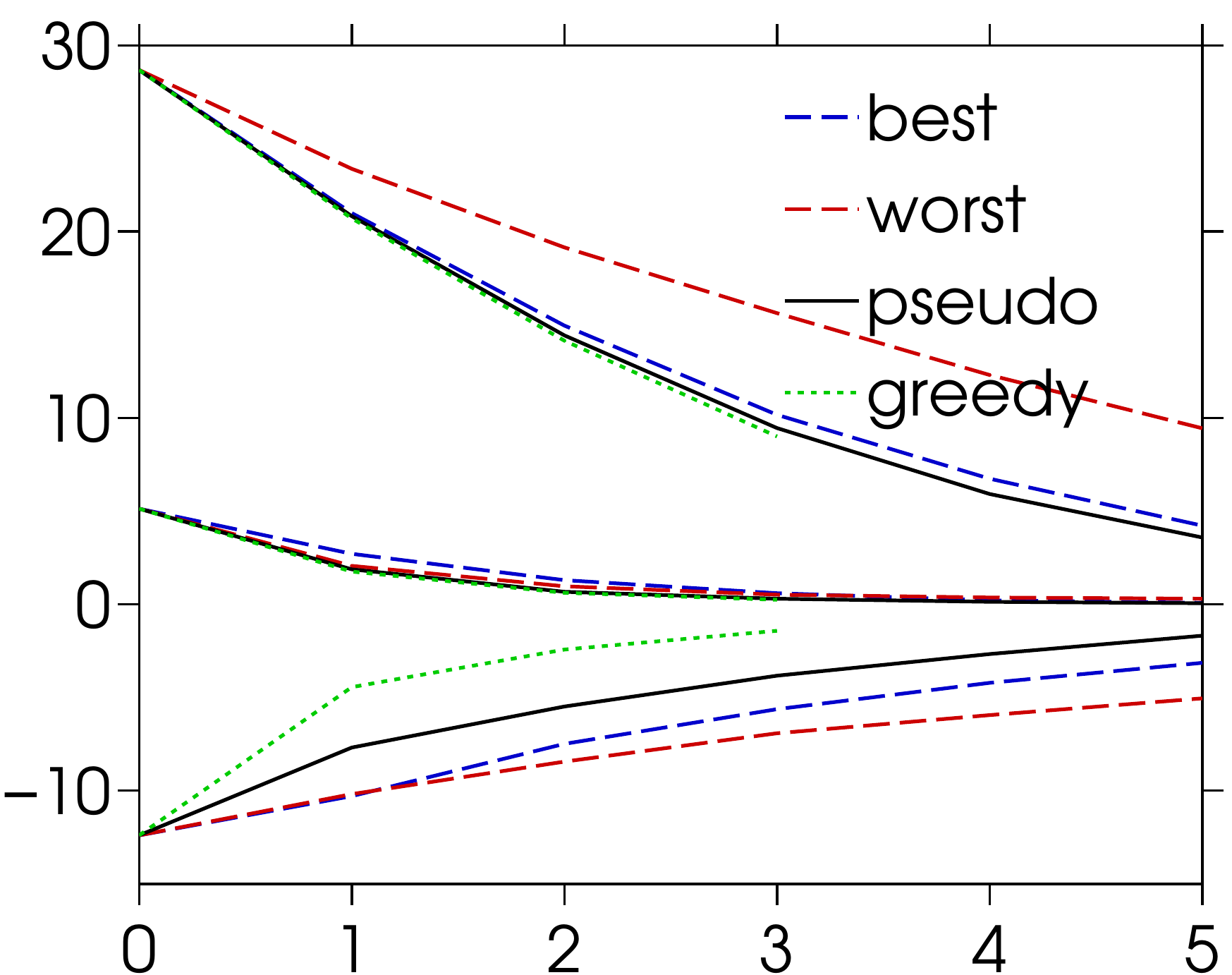} &
\includegraphics[width=.24\linewidth]{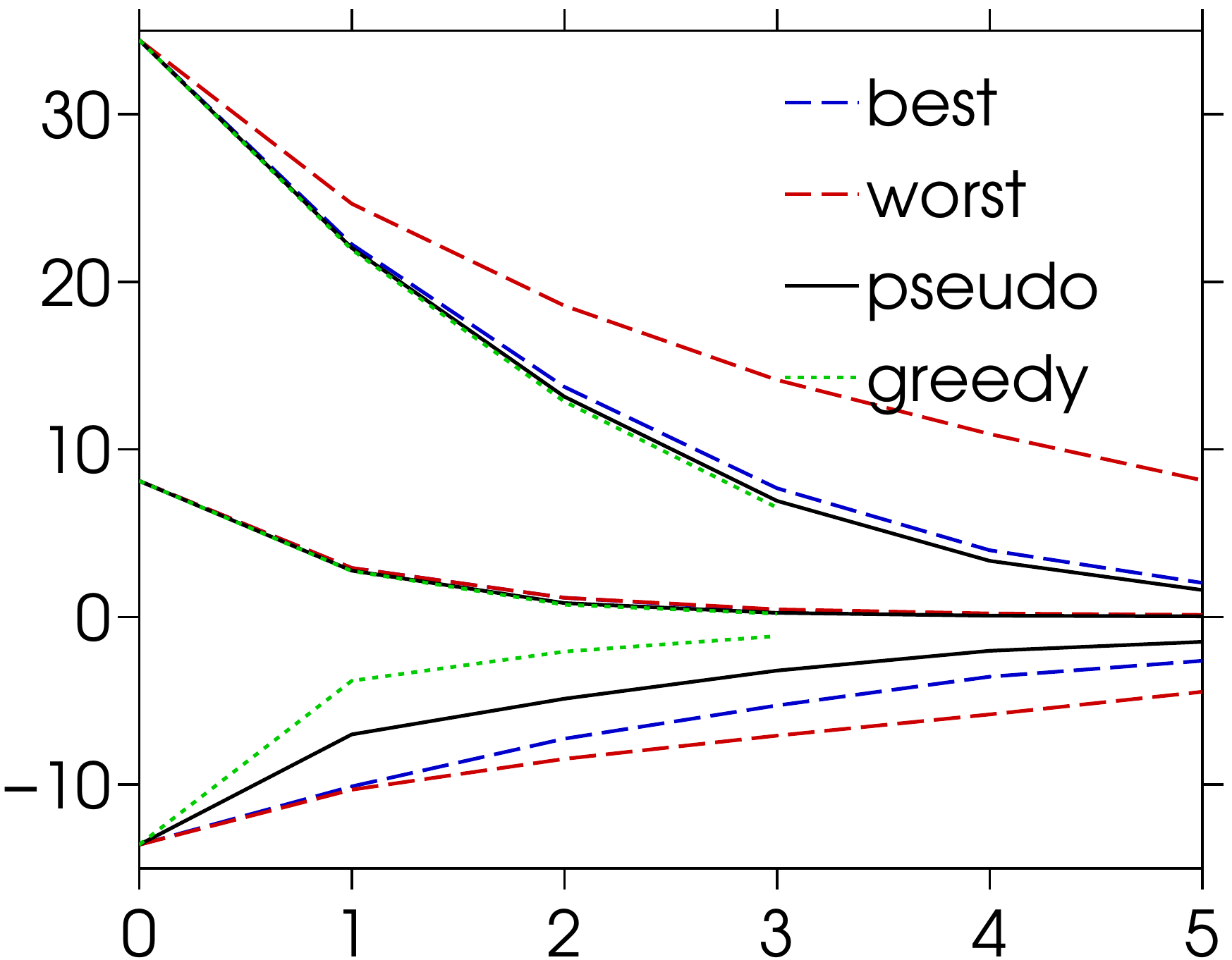} &
\- \; &
\includegraphics[width=.24\linewidth]{complete_15_+-2_2=_+-12_12=-eps-converted-to.pdf} &
\begin{sideways} {\small \- \;  \; \; complete $K_{15}$} \end{sideways} \\
\begin{sideways} {\small \- \; \qquad small (25)} \end{sideways} &
\includegraphics[width=.24\linewidth]{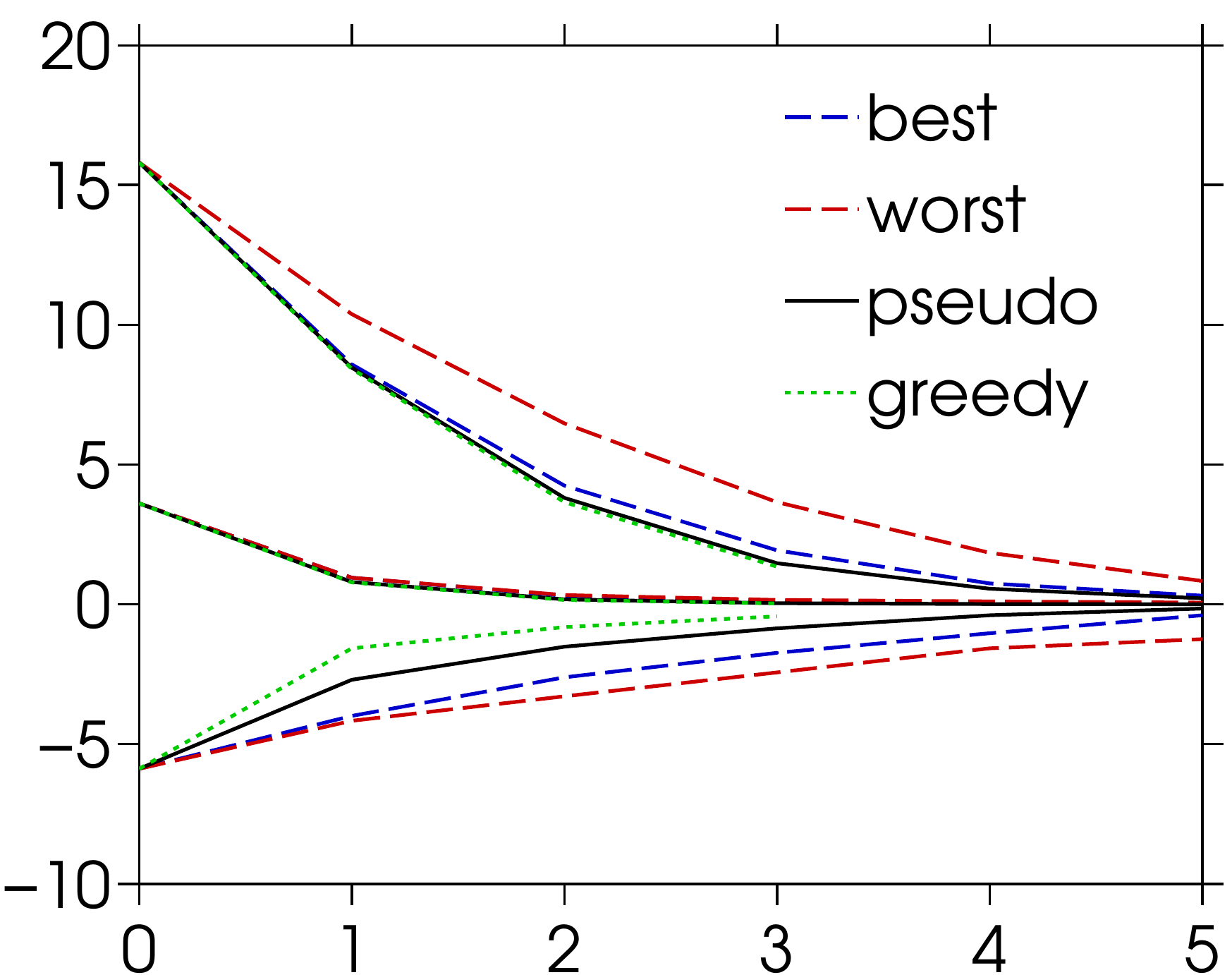} & 
\includegraphics[width=.24\linewidth]{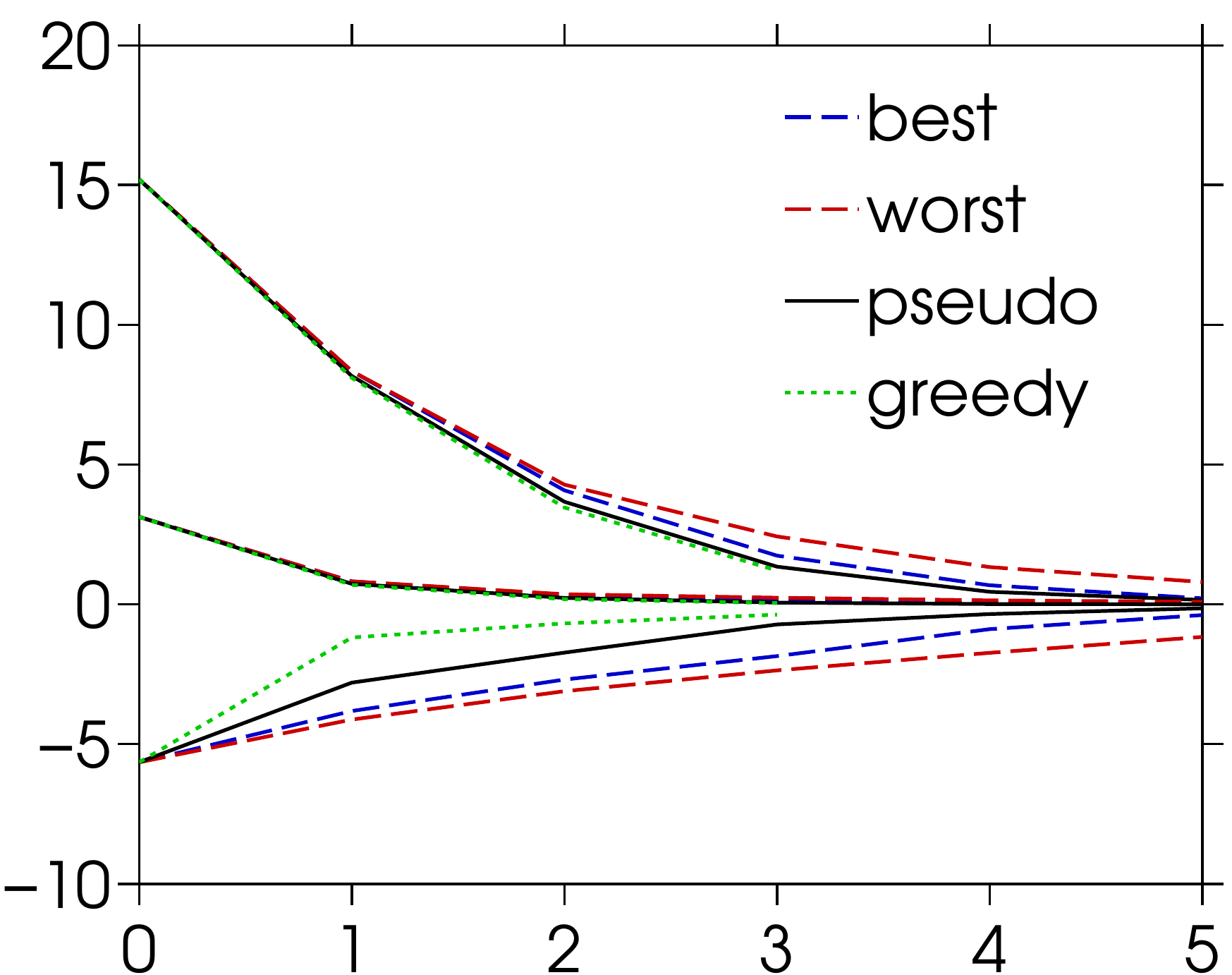} &
\includegraphics[width=.24\linewidth]{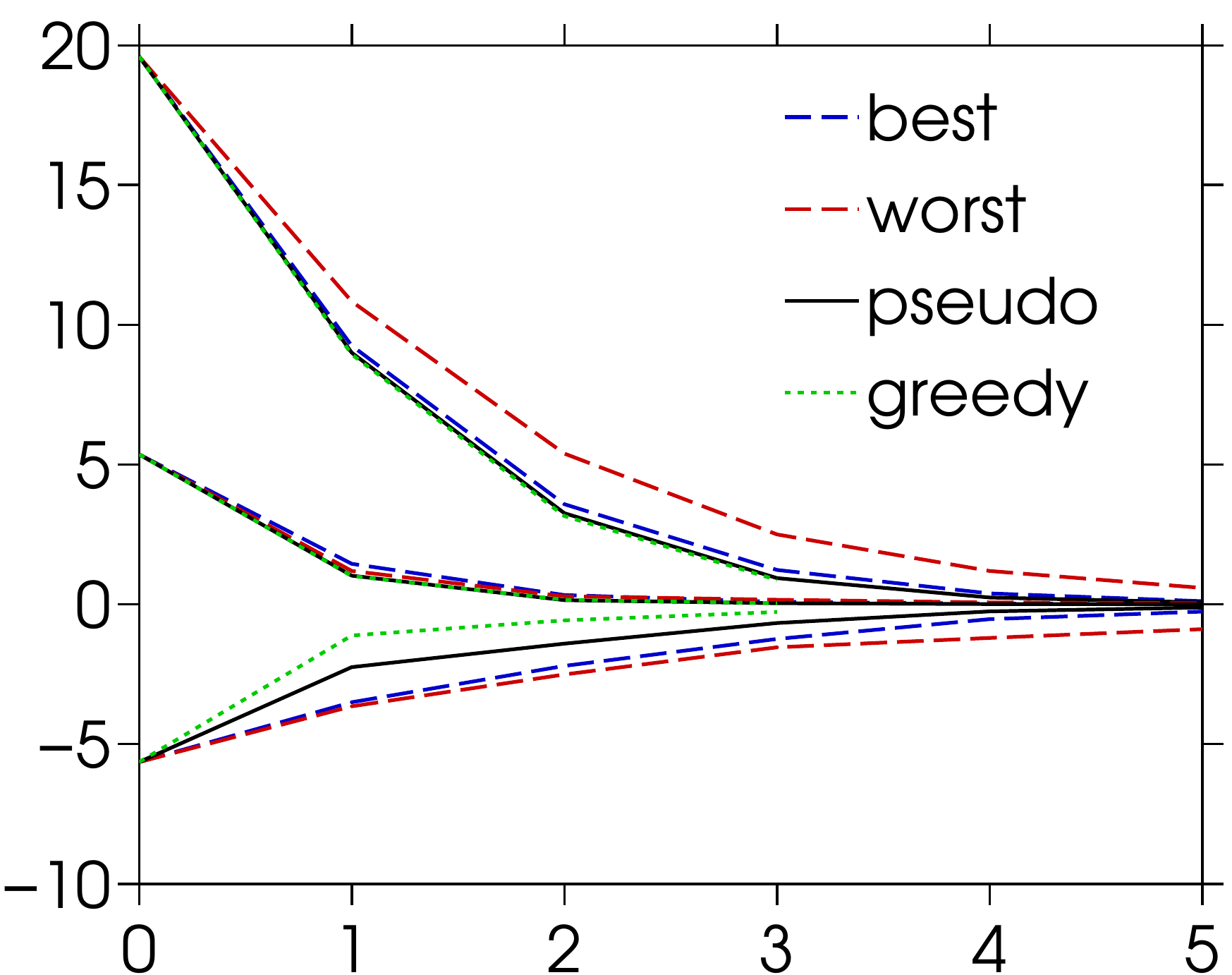} &
\- \; &
\includegraphics[width=.24\linewidth]{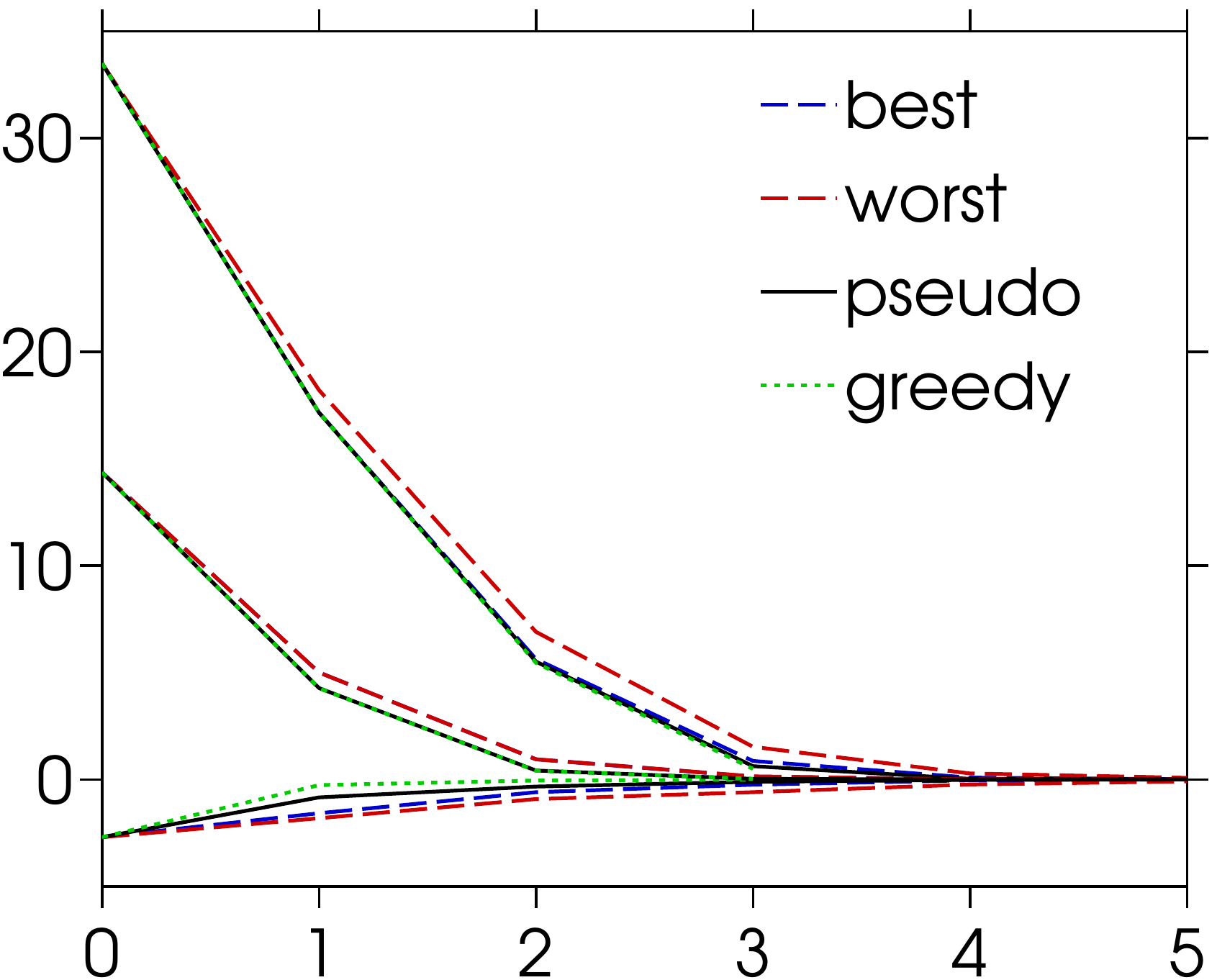} &
\begin{sideways} {\small \- \; \; \;  complete $K_{10}$} \end{sideways} \\
& {\small grids} & {\small random 4-regular} & {\small random Erd\"{o}s-Renyi} & & {\small complete graph} 
\end{tabular}
\end{center}
\caption{\small Mixed $[-12,12]$
}
\label{fig:exp2}
\end{figure}

\begin{figure}
\begin{center}
\setlength\tabcolsep{1pt}
\begin{tabular}{ccccccc}
\begin{sideways} {\small \- \; \qquad large (81)} \end{sideways} &
\includegraphics[width=.24\linewidth]{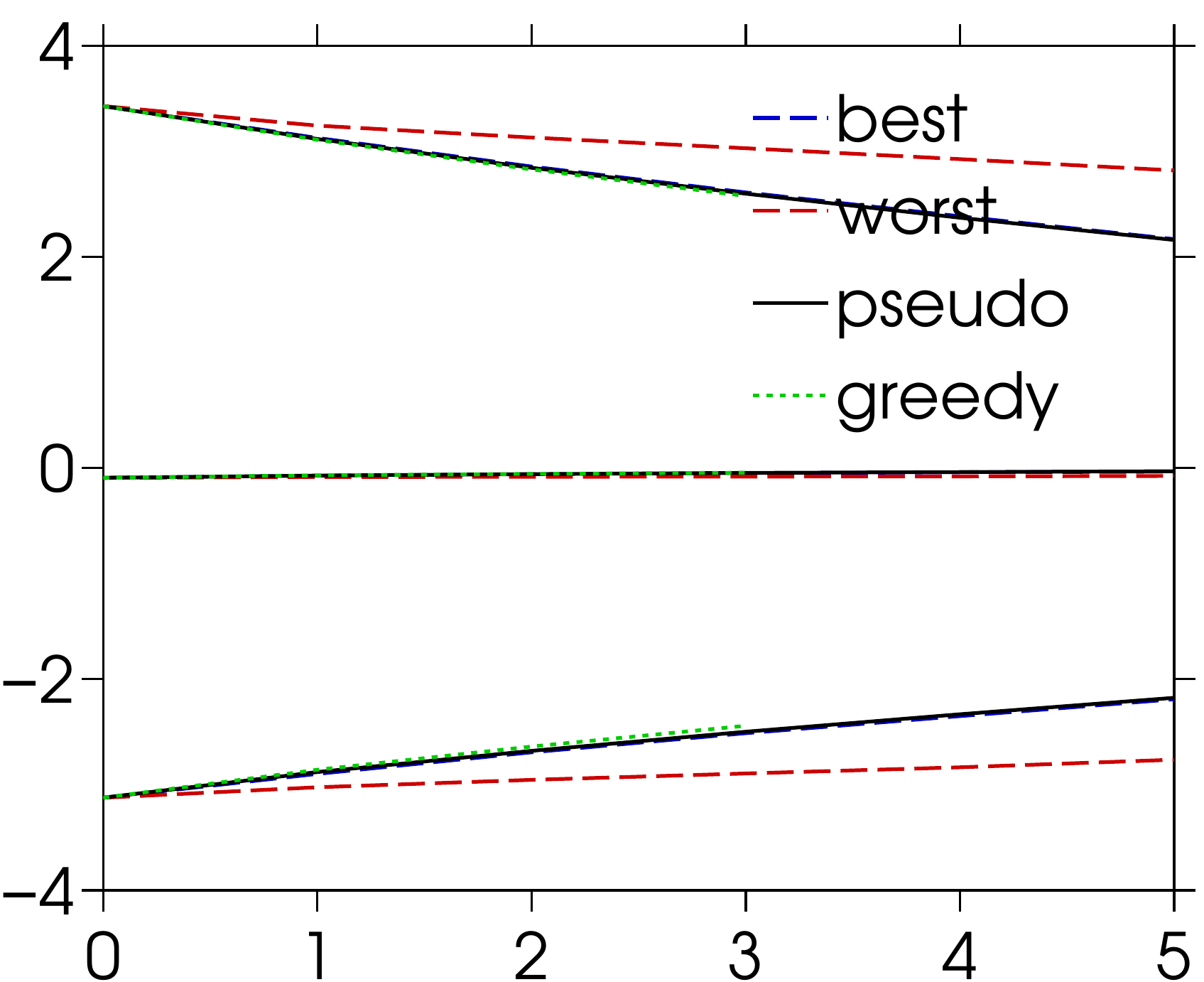} & 
\includegraphics[width=.24\linewidth]{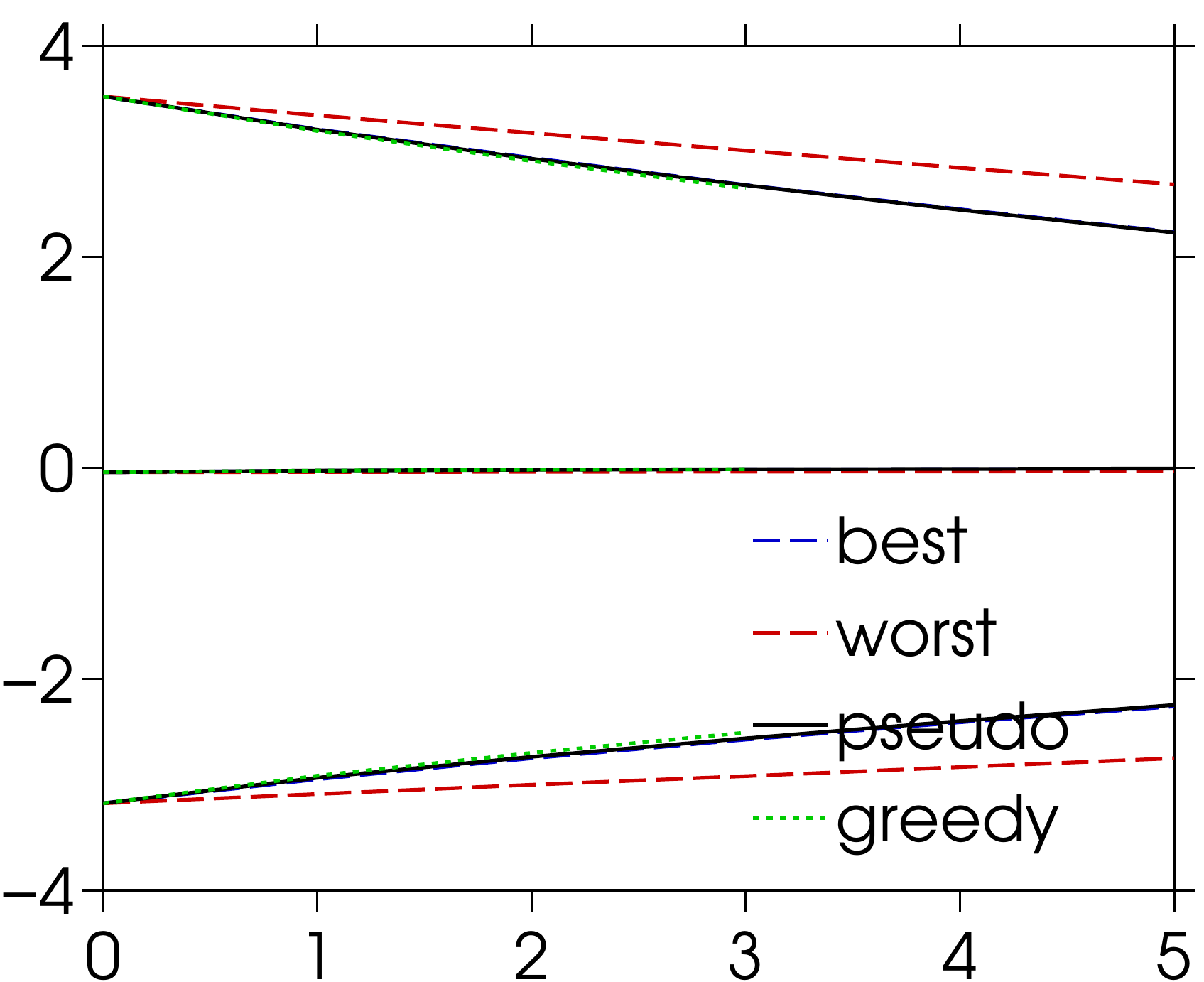} &
\includegraphics[width=.24\linewidth]{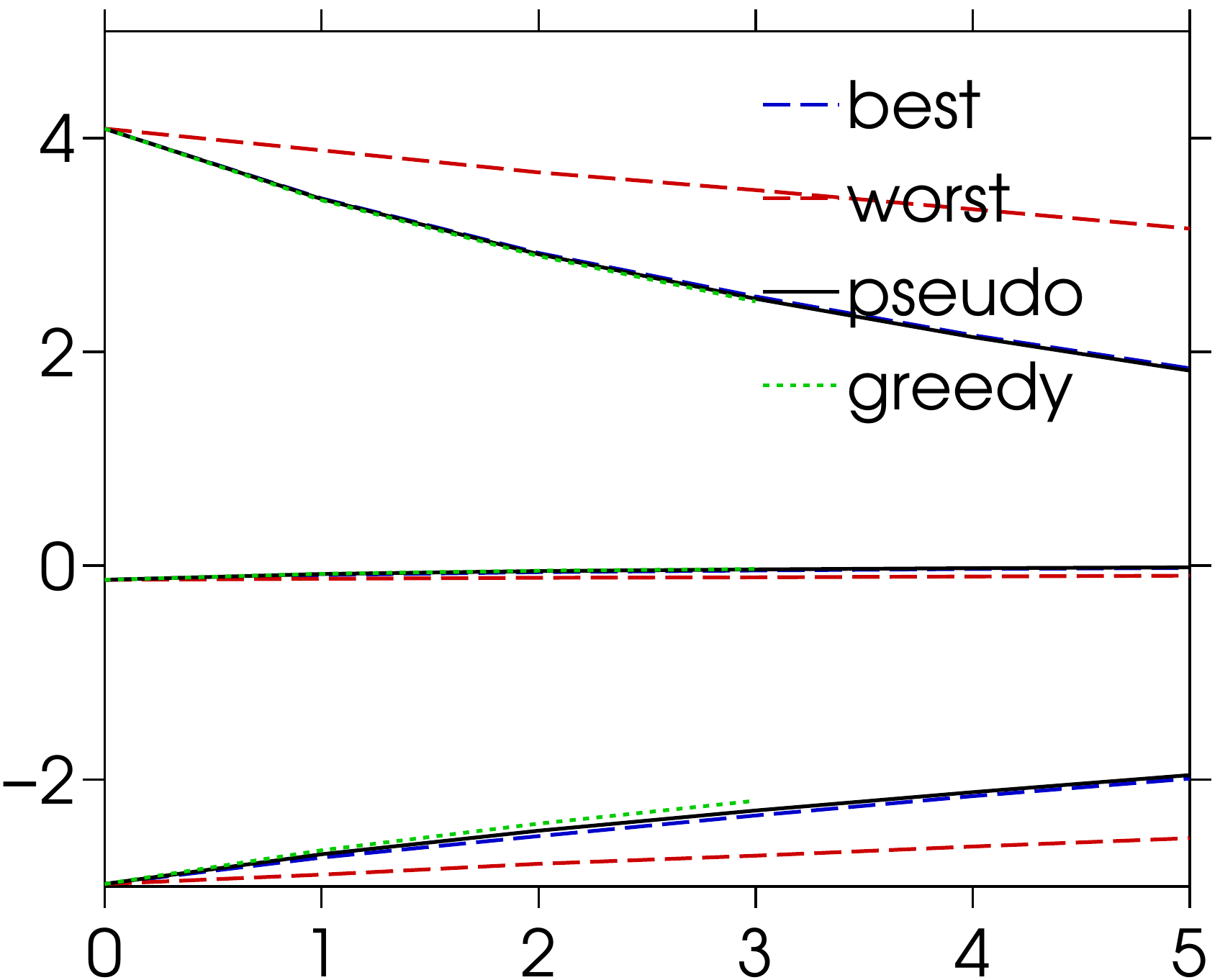} 
\\
\begin{sideways} {\small \- \; \; \quad medium (49)} \end{sideways} &
\includegraphics[width=.24\linewidth]{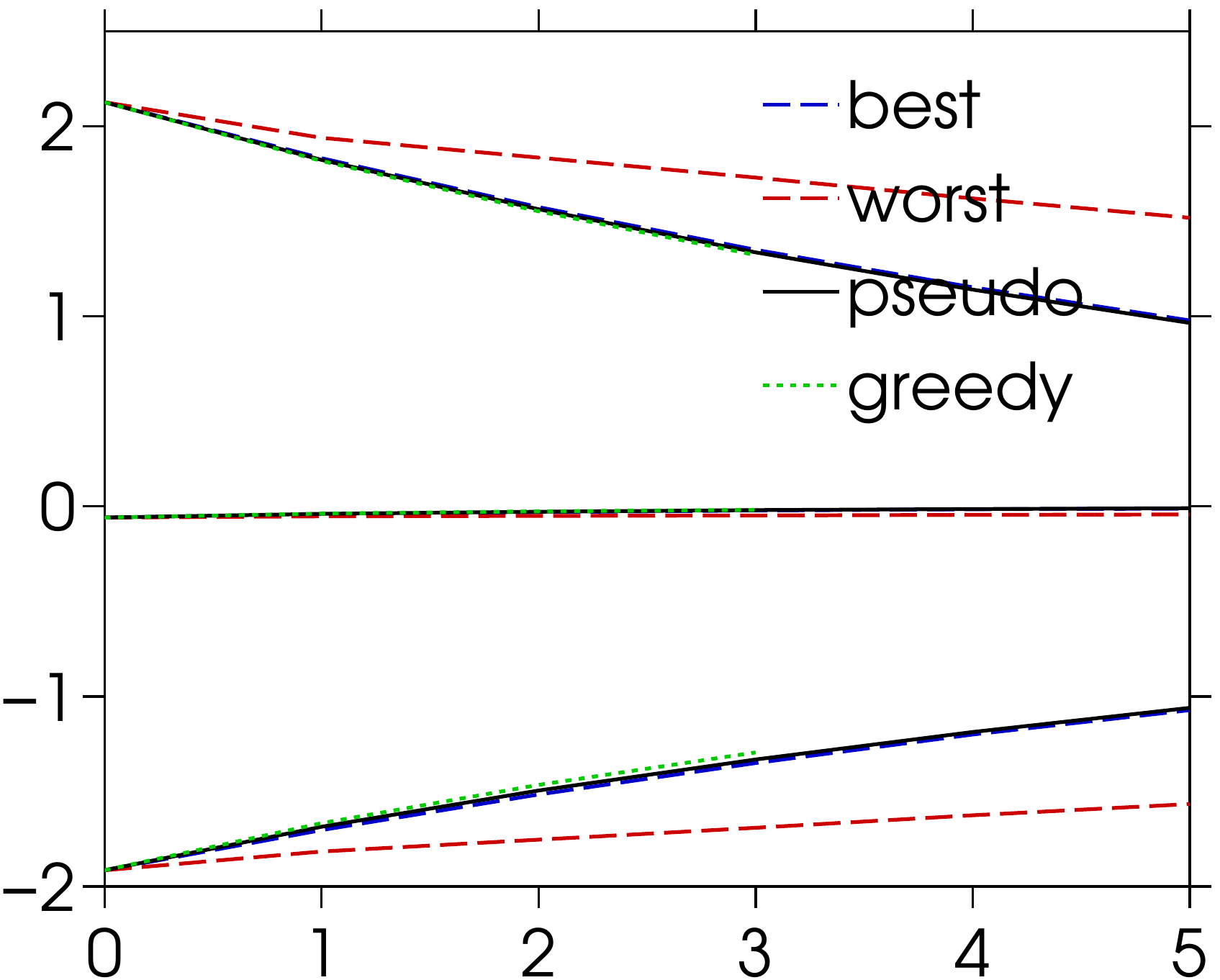} & 
\includegraphics[width=.24\linewidth]{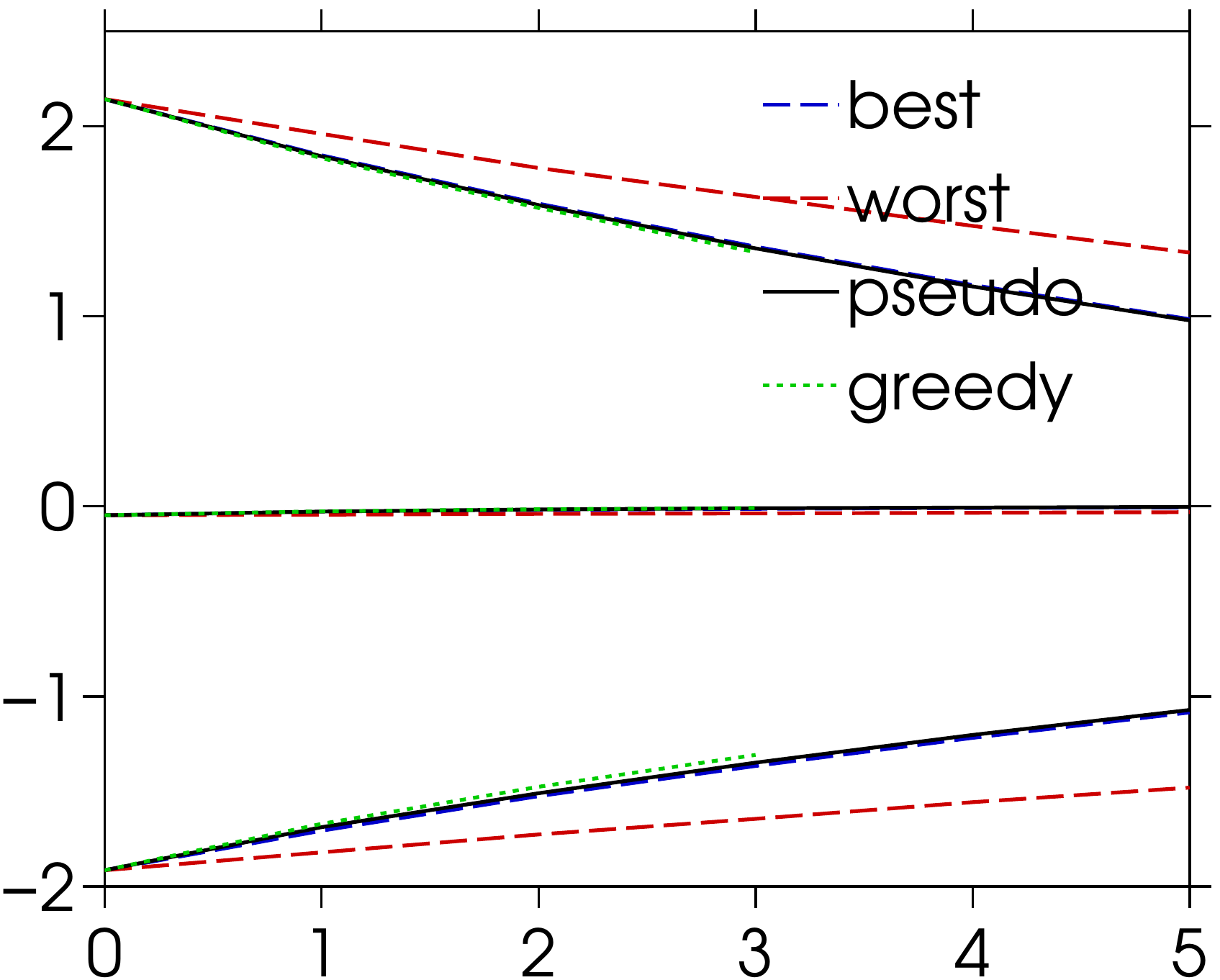} &
\includegraphics[width=.24\linewidth]{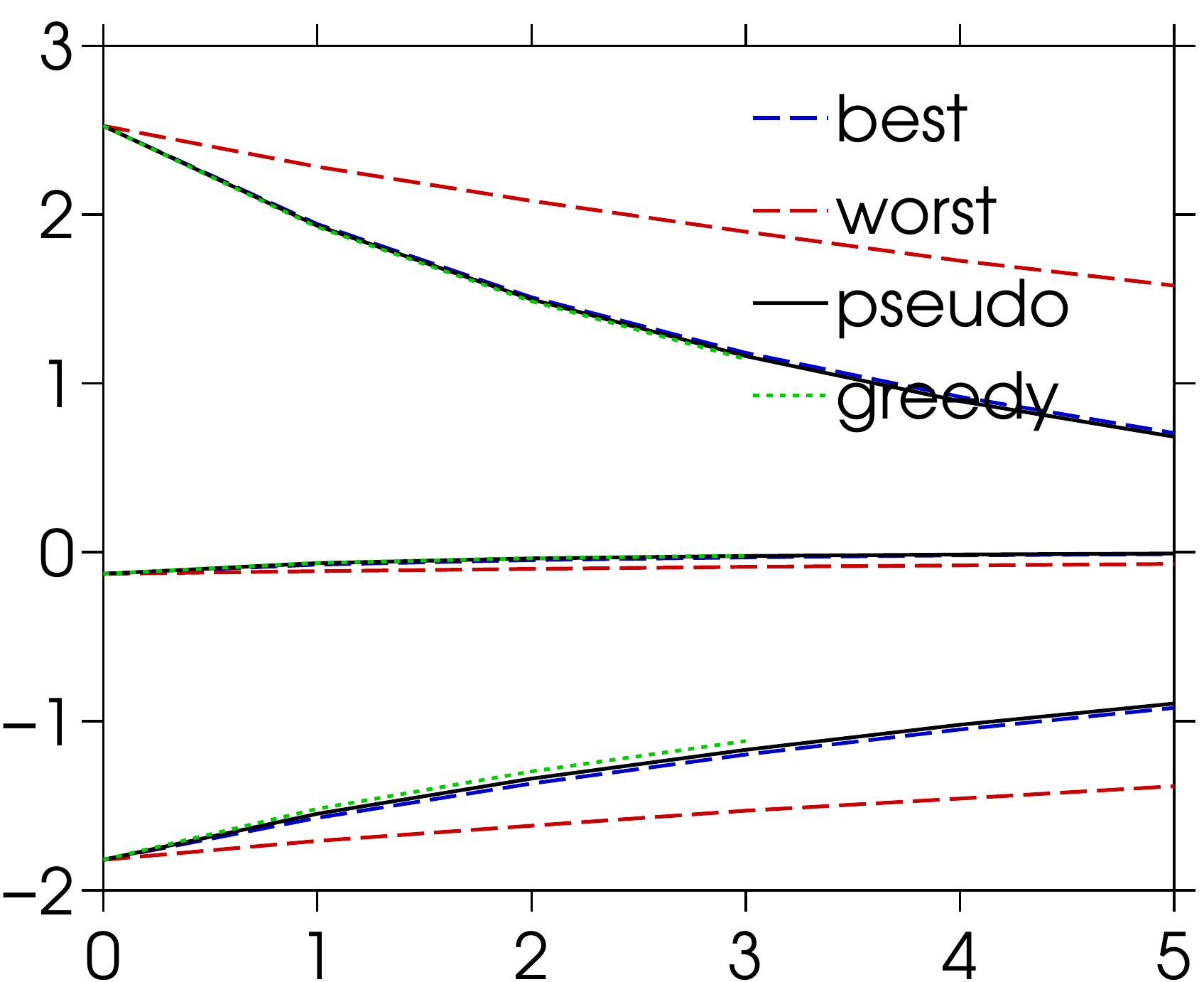} &
\- \; &
\includegraphics[width=.24\linewidth]{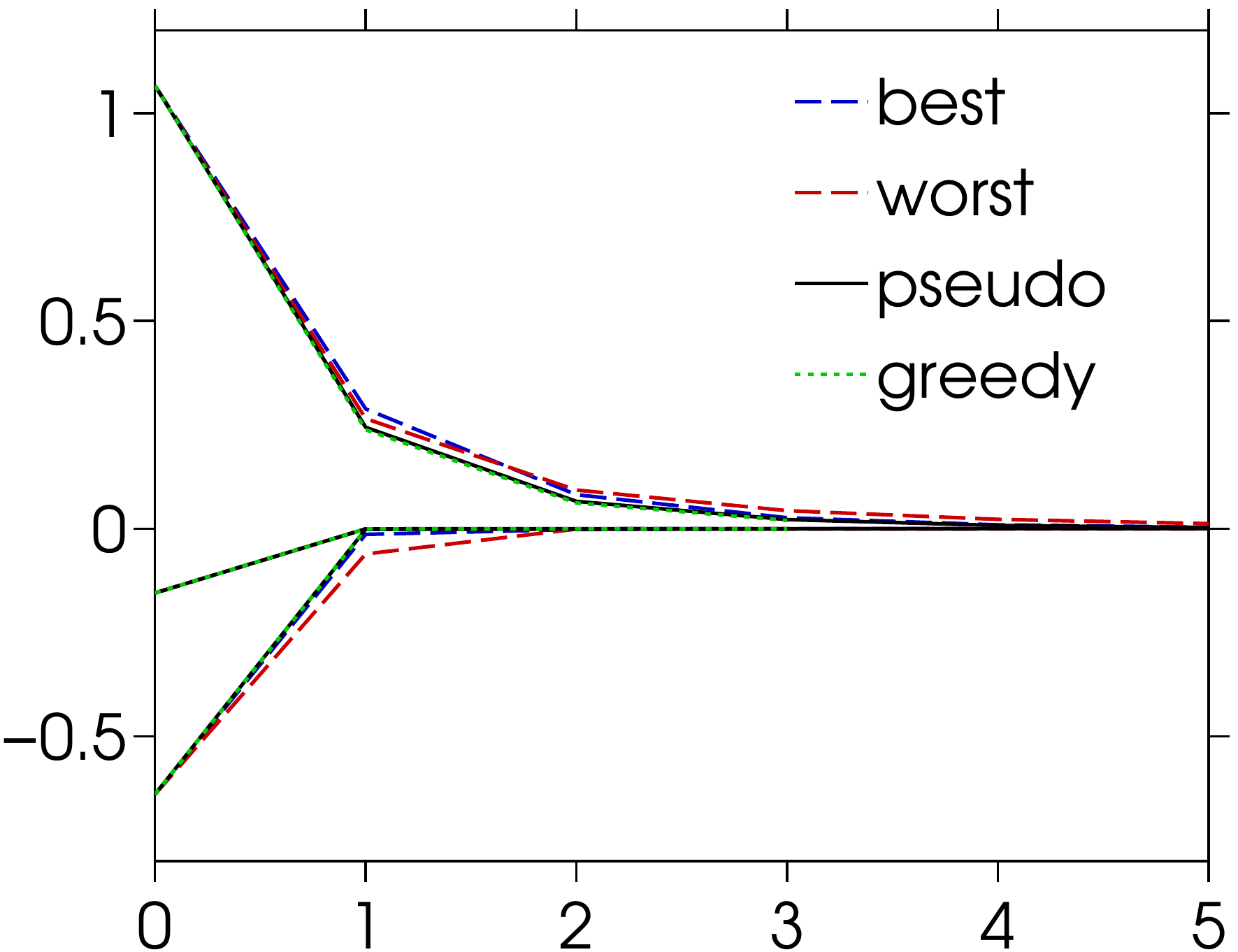} &
\begin{sideways} {\small \- \; \; \;  complete $K_{15}$} \end{sideways} \\
\begin{sideways} {\small \- \; \qquad small (25)} \end{sideways} &
\includegraphics[width=.24\linewidth]{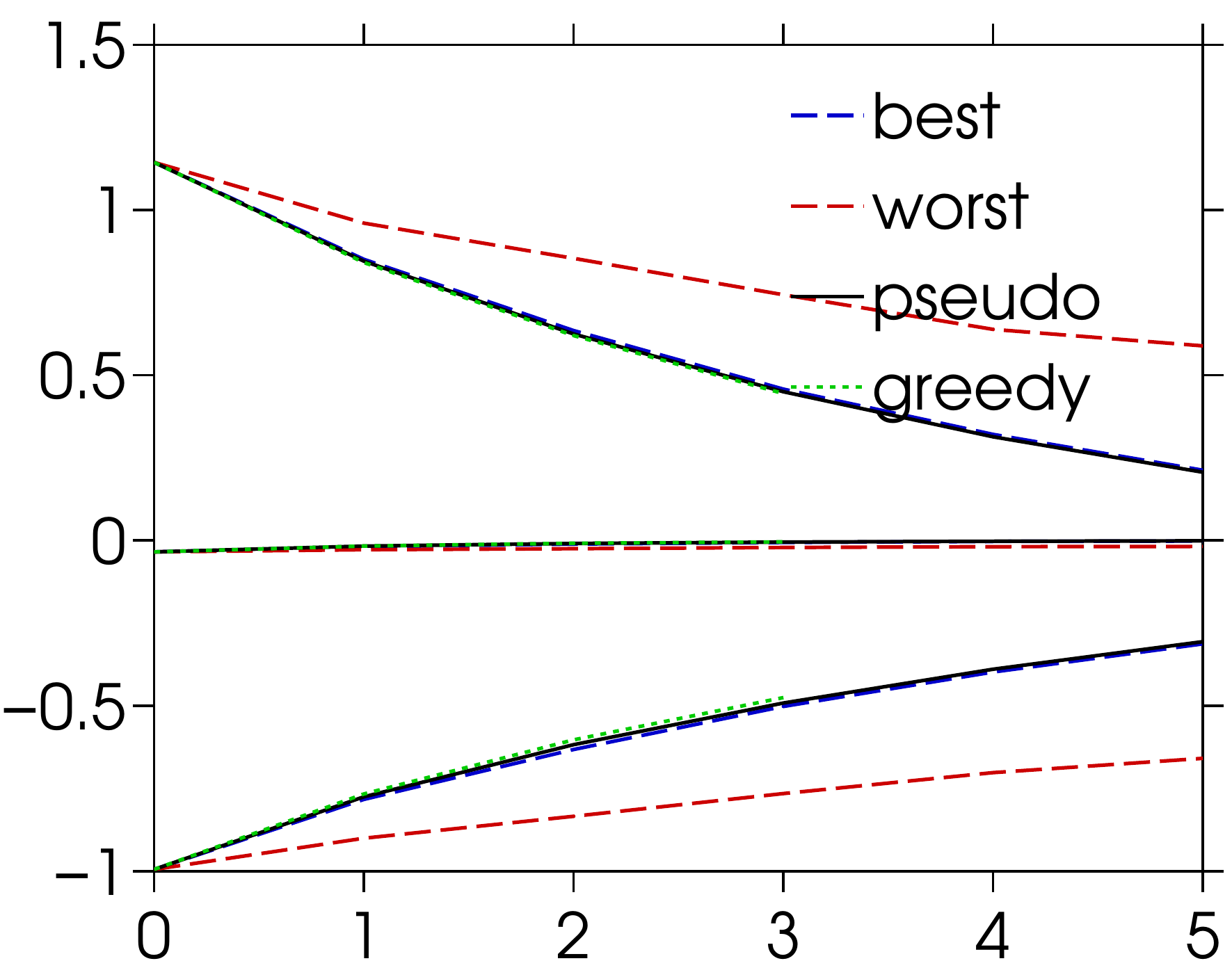} & 
\includegraphics[width=.24\linewidth]{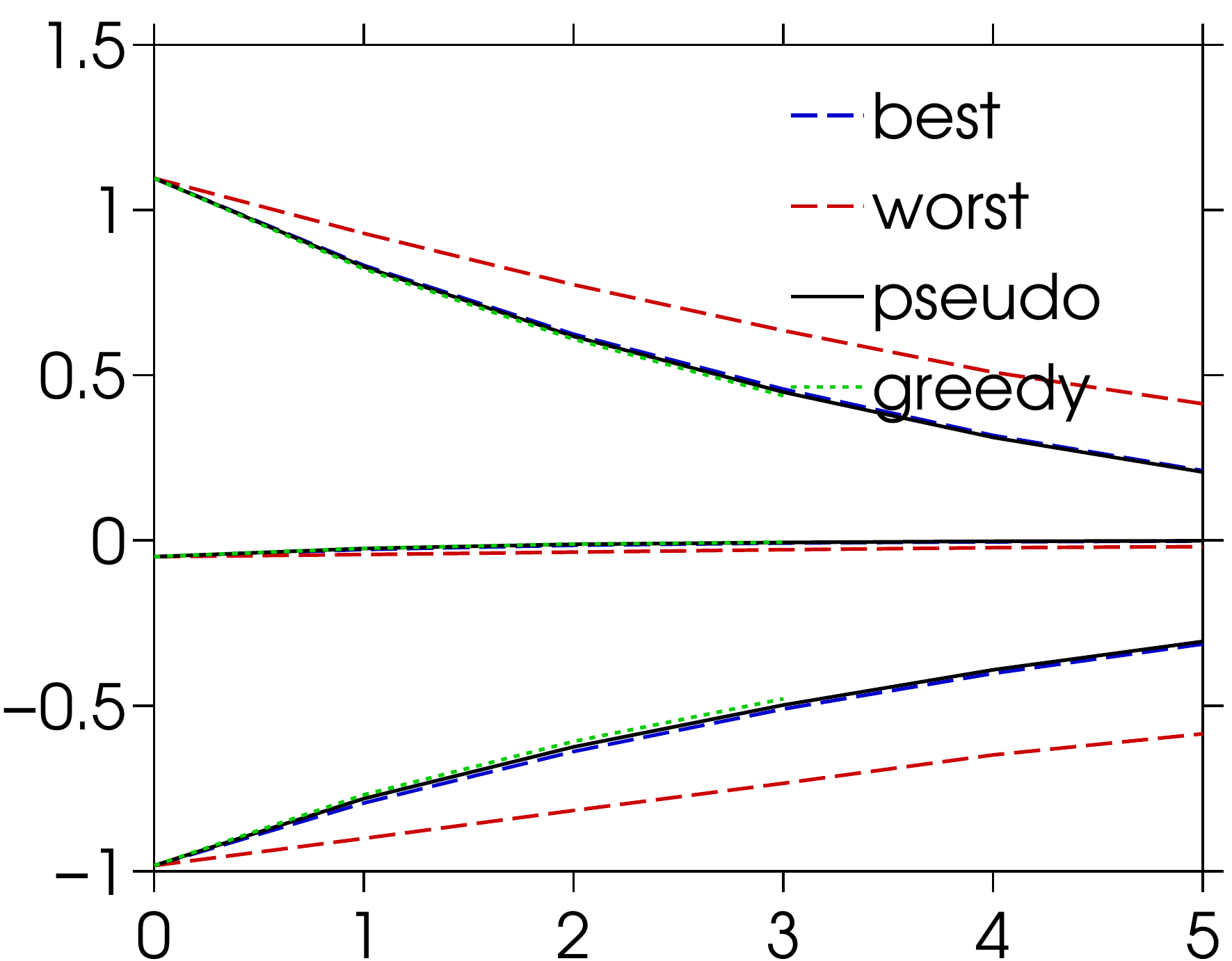} &
\includegraphics[width=.24\linewidth]{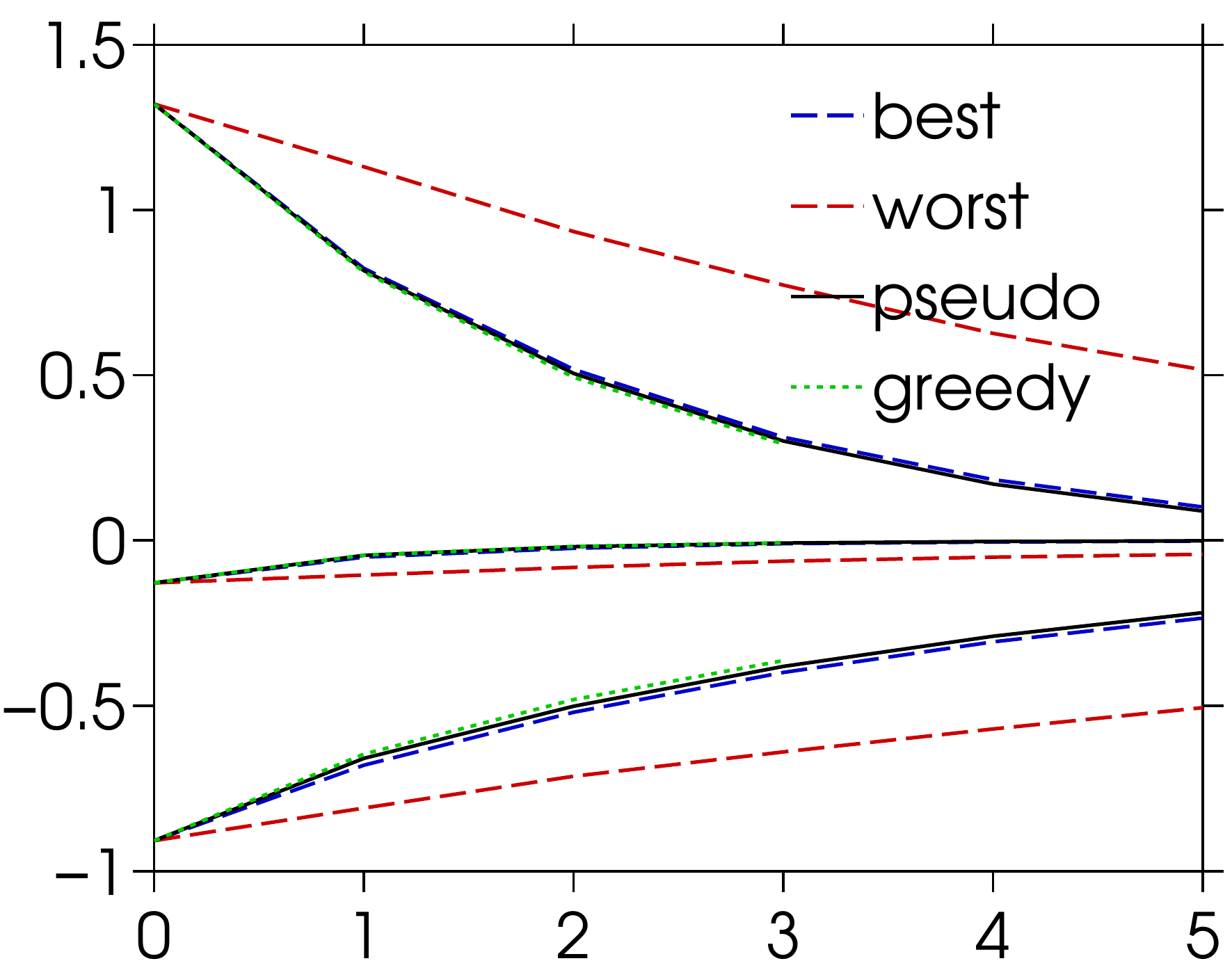} &
\- \; &
\includegraphics[width=.24\linewidth]{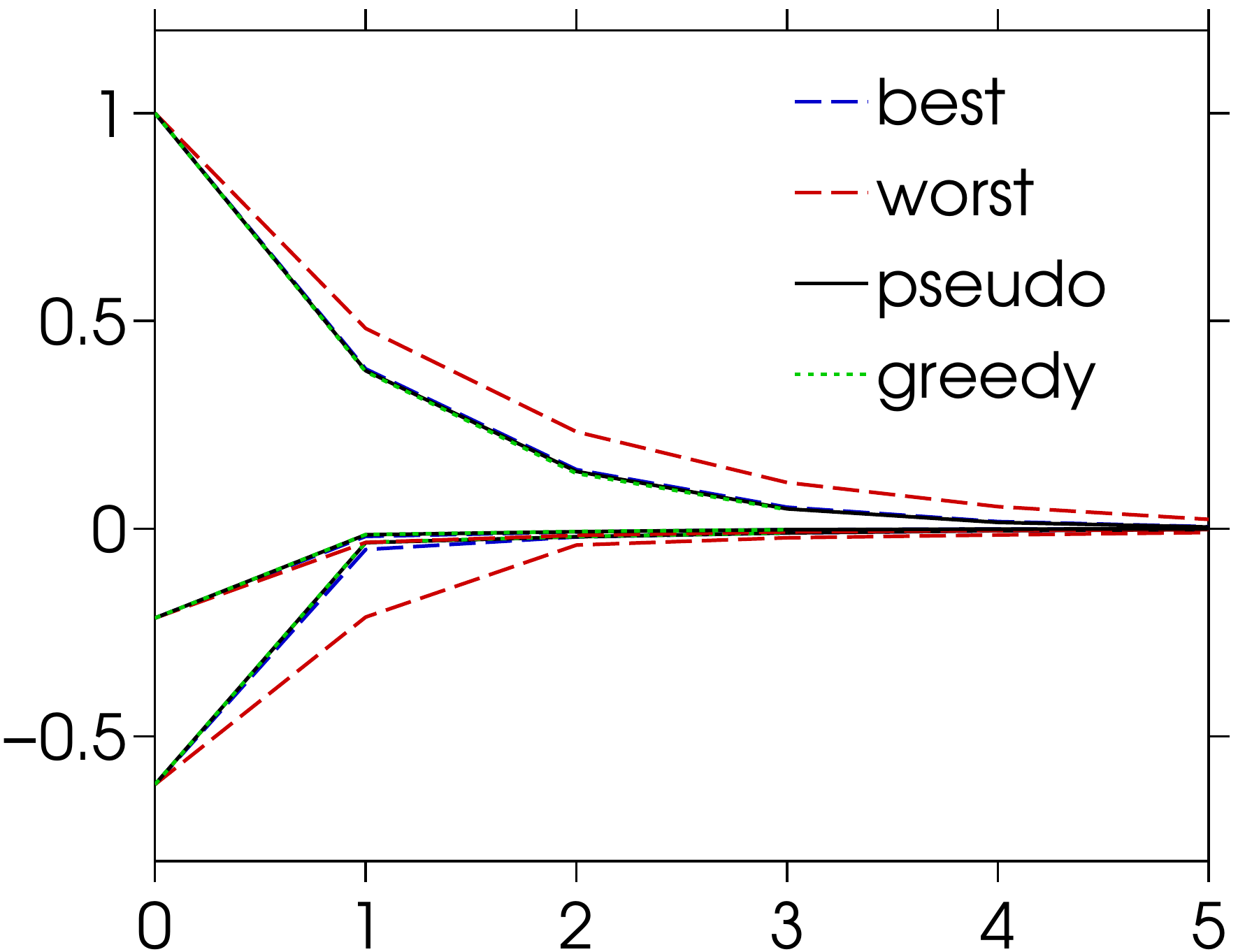} &
\begin{sideways} {\small \- \; \; \;  complete $K_{10}$} \end{sideways} \\
& {\small grids} & {\small random 4-regular} & {\small random Erd\"{o}s-Renyi} & & {\small complete graph} 
\end{tabular}
\end{center}
\caption{\small Attractive $[0,2]$
}
\label{fig:exp3}
\end{figure}

\begin{figure}
\begin{center}
\setlength\tabcolsep{1pt}
\begin{tabular}{ccccccc}
\begin{sideways} {\small \- \; \qquad large (81)} \end{sideways} &
\includegraphics[width=.24\linewidth]{grid_+9_9=_+-2_2=_+0_6=-eps-converted-to.pdf} & 
\includegraphics[width=.24\linewidth]{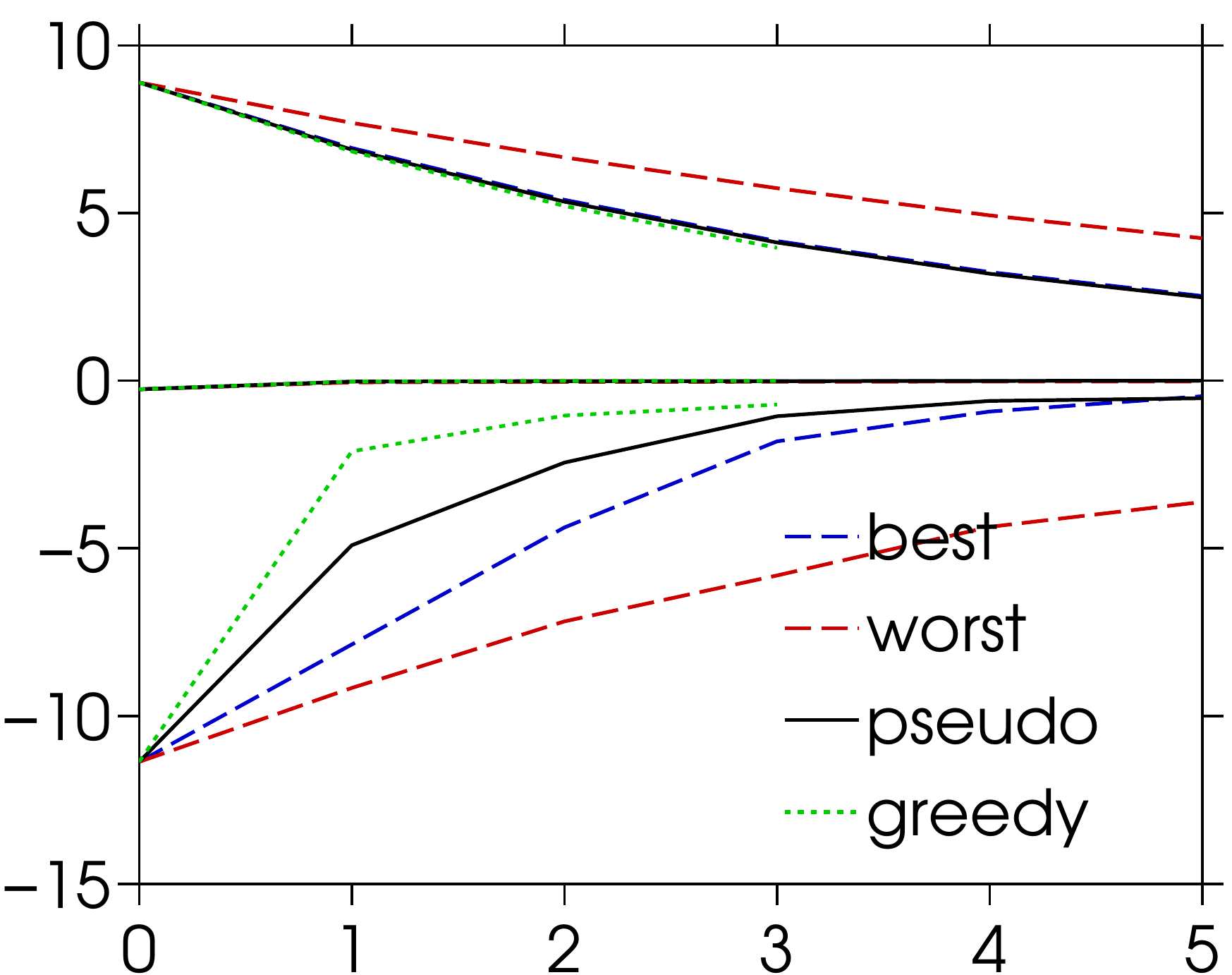} &
\includegraphics[width=.24\linewidth]{random_graph_+81_05=_+-2_2=_+0_6=-eps-converted-to.pdf} 
\\
\begin{sideways} {\small \- \; \; \quad medium (49)} \end{sideways} &
\includegraphics[width=.24\linewidth]{grid_+7_7=_+-2_2=_+0_6=-eps-converted-to.pdf} & 
\includegraphics[width=.24\linewidth]{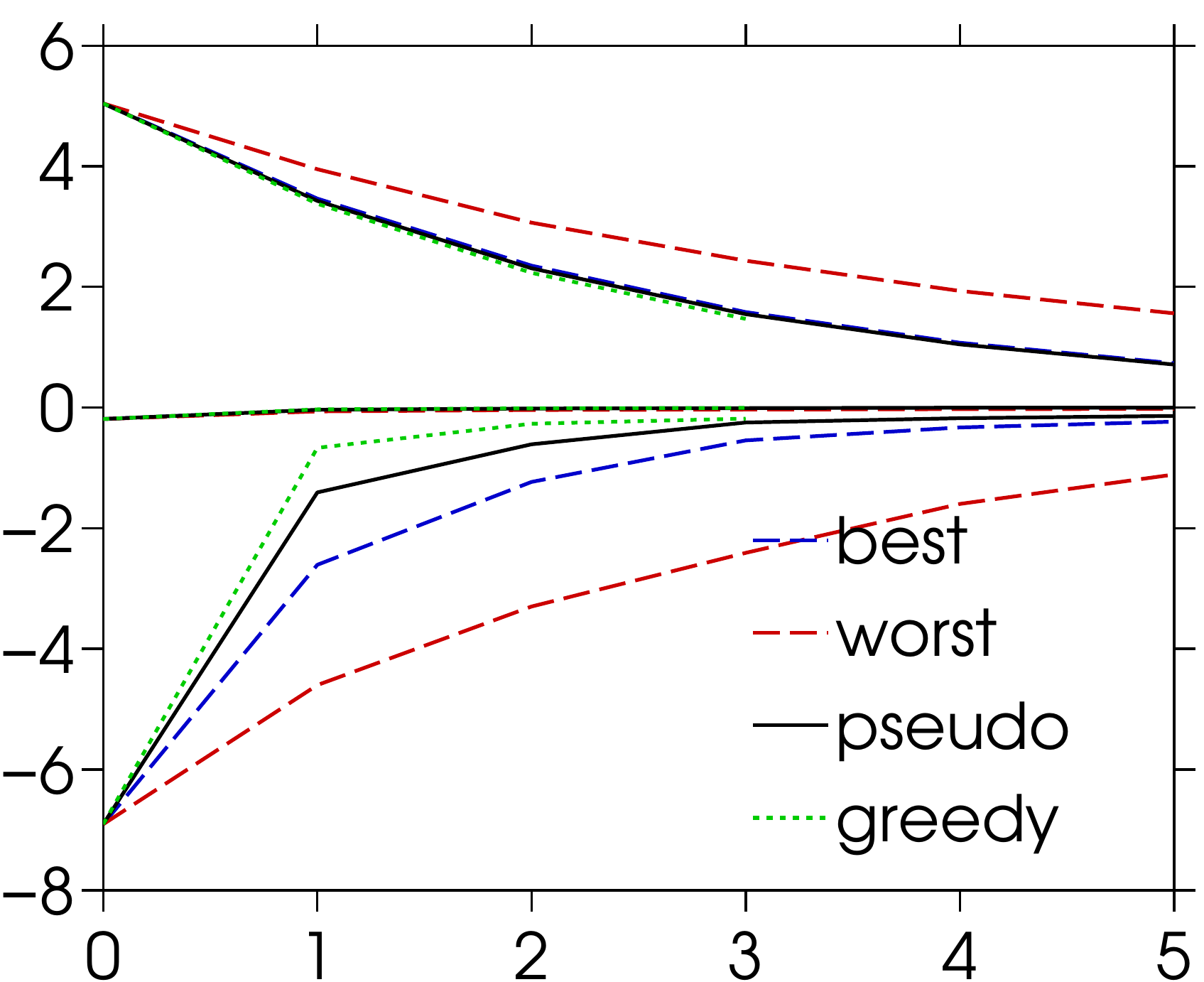} &
\includegraphics[width=.24\linewidth]{random_graph_+49_0833333333333333=_+-2_2=_+0_6=-eps-converted-to.pdf} &
\- \; &
\includegraphics[width=.24\linewidth]{complete_15_+-2_2=_+0_6=-eps-converted-to.pdf} &
\begin{sideways} {\small \- \; \; \;  complete $K_{15}$} \end{sideways} \\
\begin{sideways} {\small \- \; \qquad small (25)} \end{sideways} &
\includegraphics[width=.24\linewidth]{grid_+5_5=_+-2_2=_+0_6=-eps-converted-to.pdf} & 
\includegraphics[width=.24\linewidth]{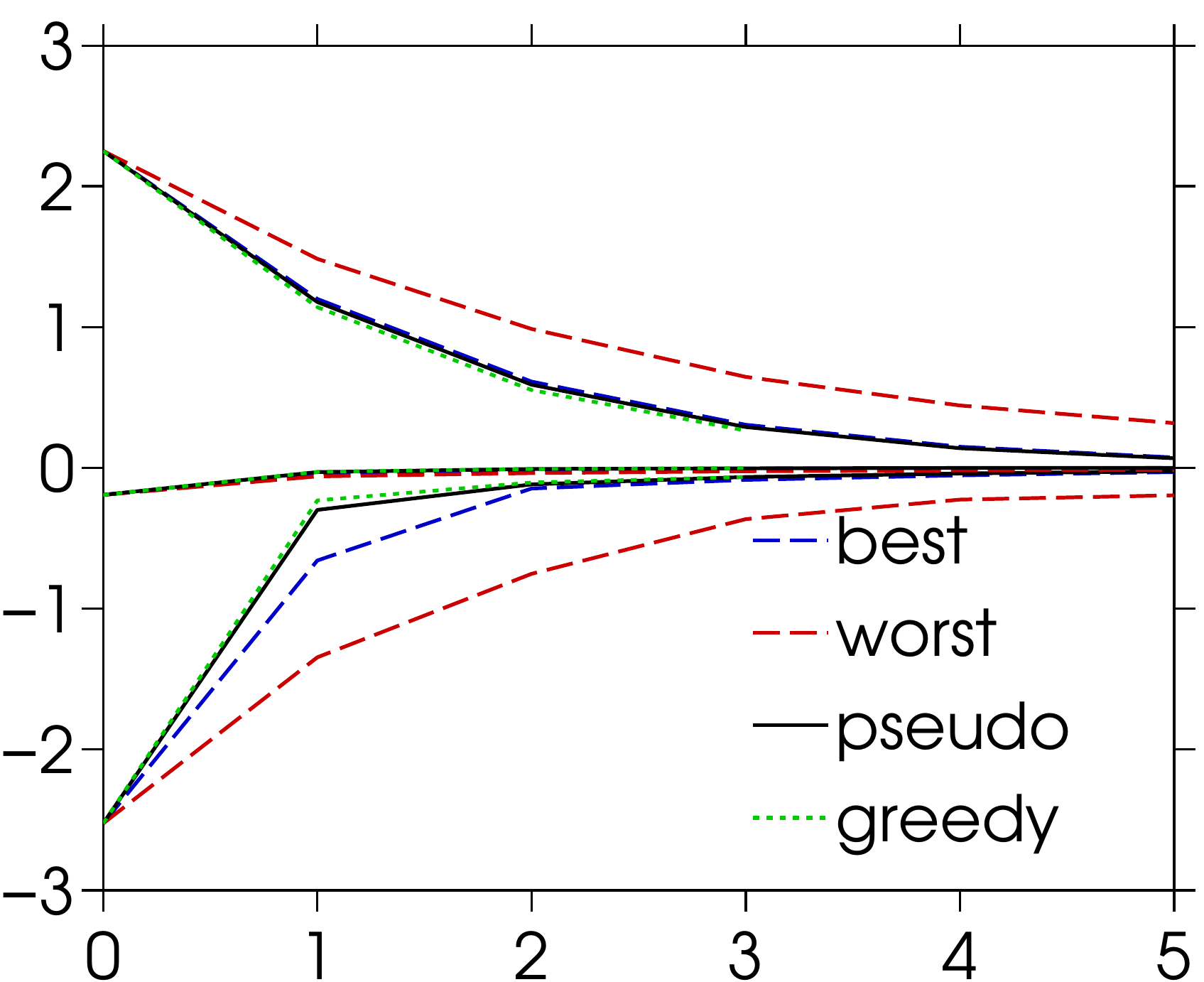} &
\includegraphics[width=.24\linewidth]{random_graph_+25_166666666666667=_+-2_2=_+0_6=-eps-converted-to.pdf} &
\- \; &
\includegraphics[width=.24\linewidth]{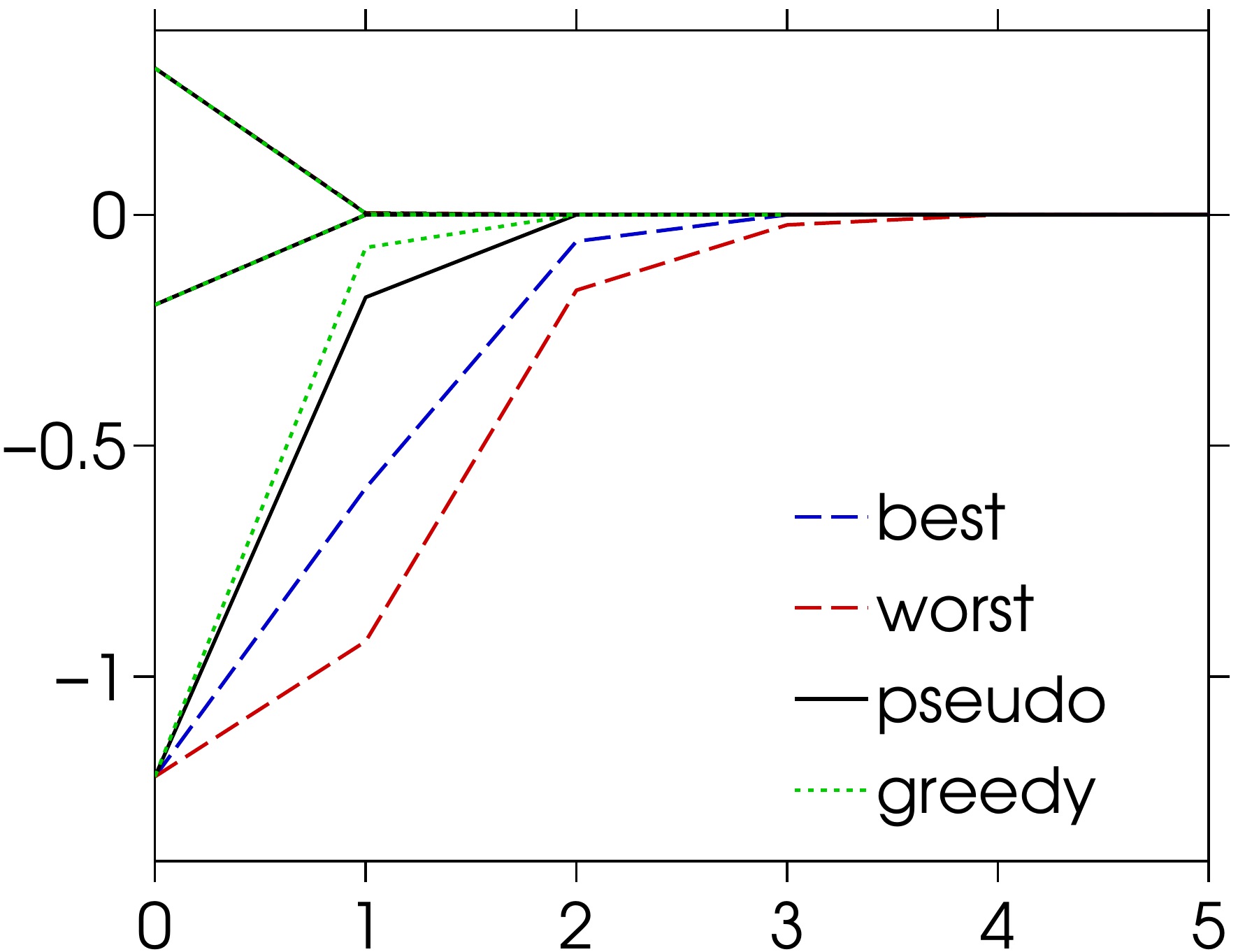} &
\begin{sideways} {\small \- \; \; \;  complete $K_{10}$} \end{sideways} \\
& {\small grids} & {\small random 4-regular} & {\small random Erd\"{o}s-Renyi} & & {\small complete graph} 
\end{tabular}
\end{center}
\caption{\small Attractive $[0,6]$
}
\label{fig:exp4}
\end{figure}

%%%%%%%%%%%%%%%%%%%%%%%%%%%%%%%
%% Bethe zoom
%%%%%%%%%%%%%%%%%%%%%%%%%%%%%%%

\begin{figure}
\begin{center}
\setlength\tabcolsep{1pt}
\begin{tabular}{ccccccc}
\begin{sideways} {\small \- \; \qquad large (81)} \end{sideways} &
\includegraphics[width=.24\linewidth]{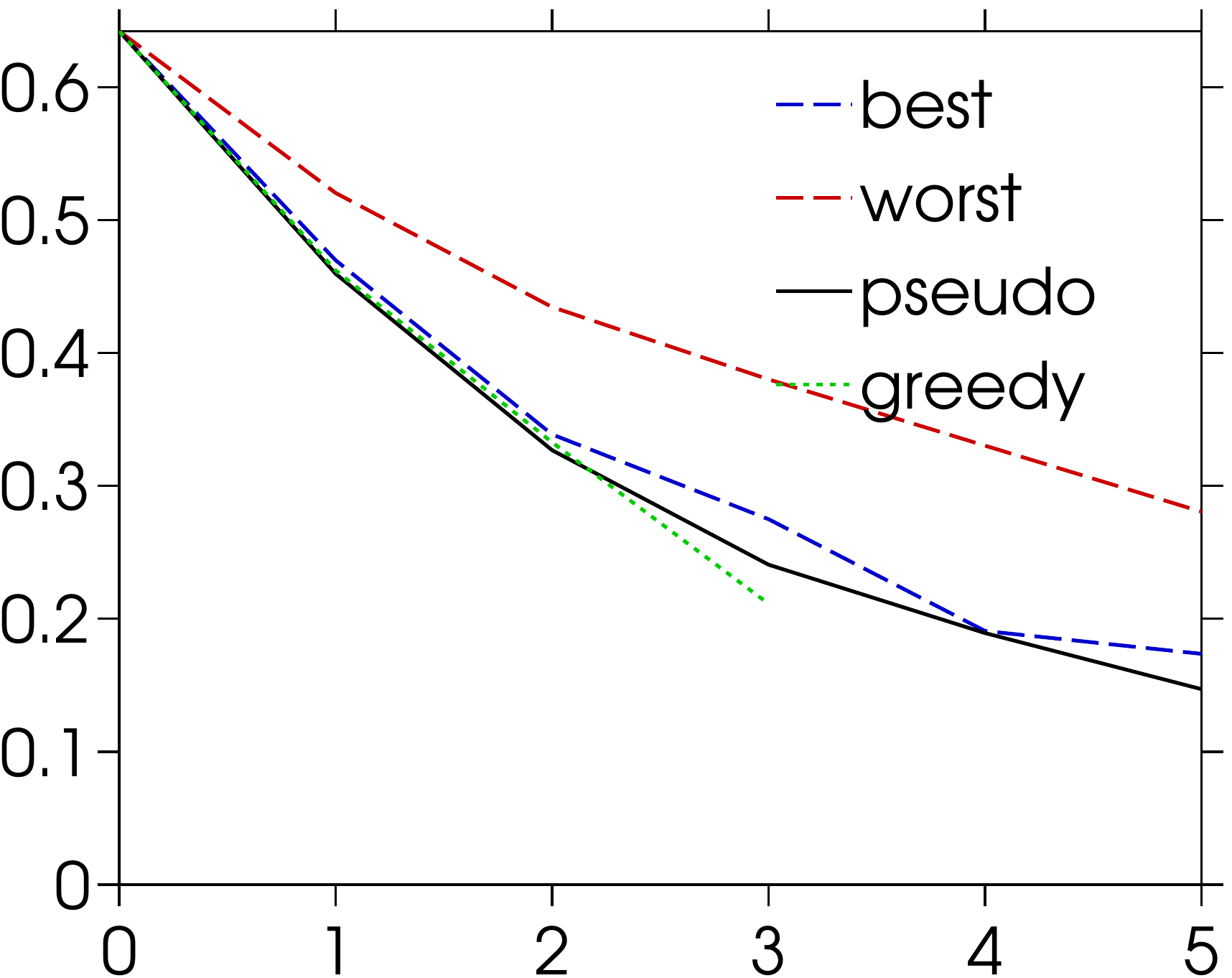} & 
\includegraphics[width=.24\linewidth]{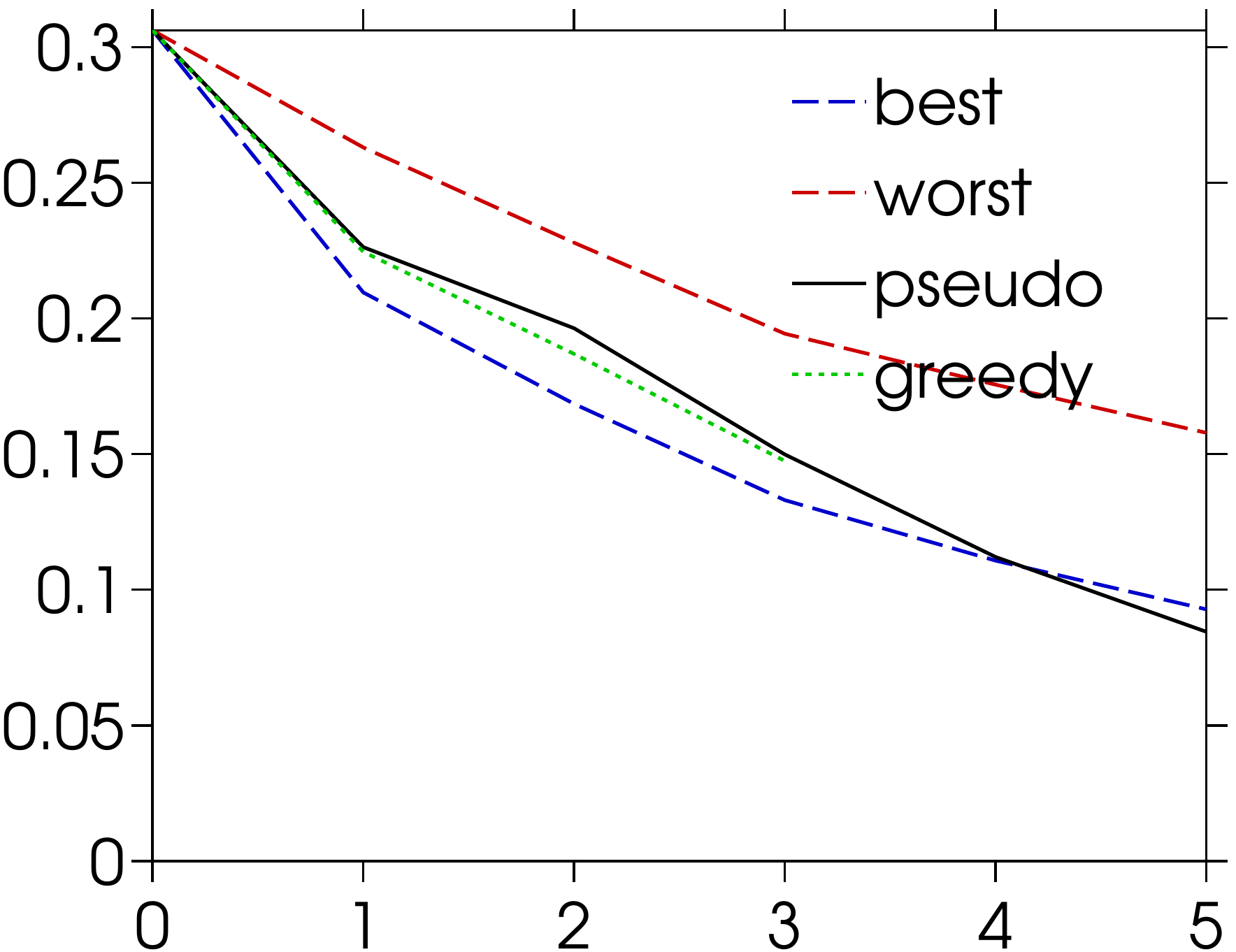} & 
\includegraphics[width=.24\linewidth]{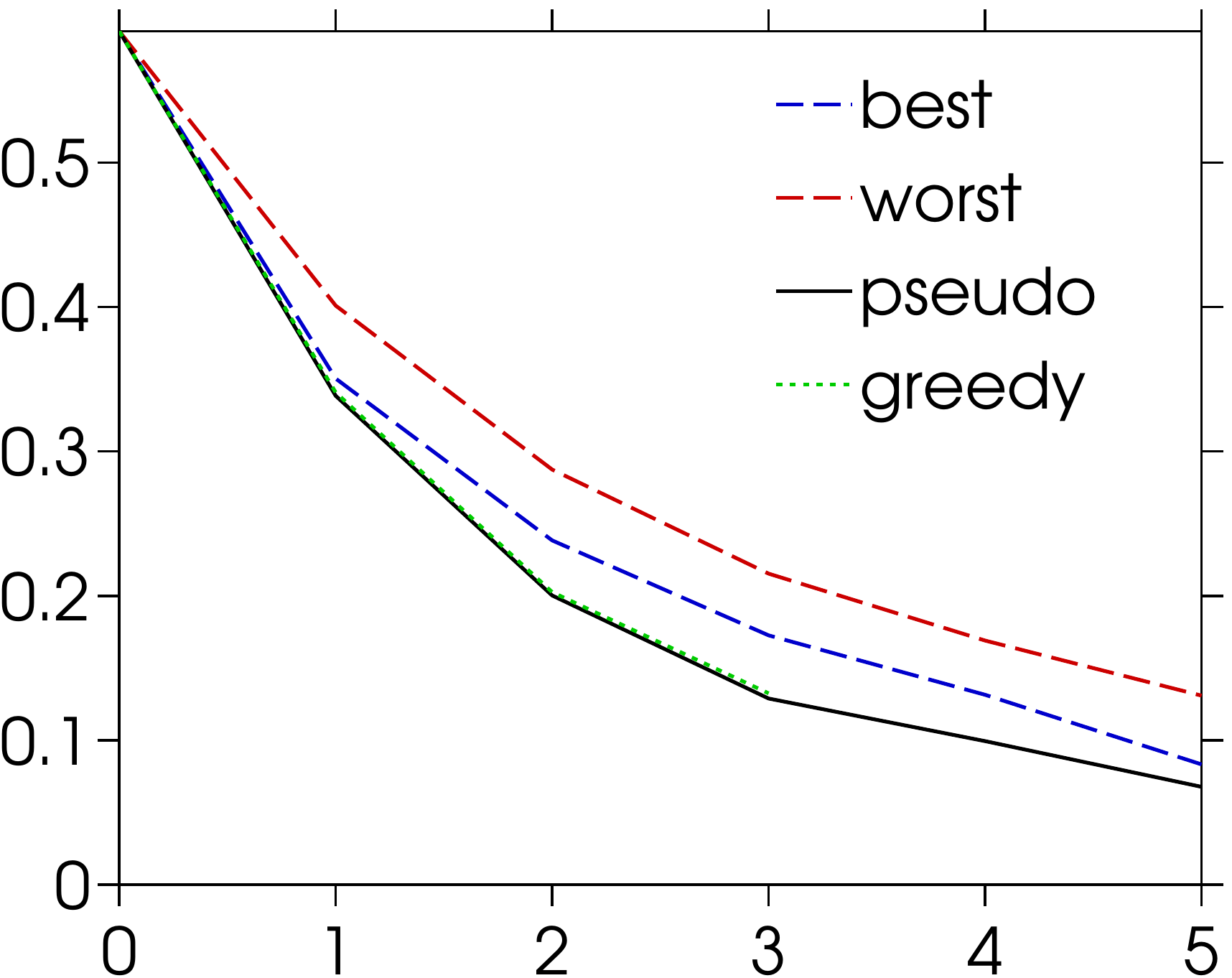} 
\\
\begin{sideways} {\small \- \; \; \quad medium (49)} \end{sideways} &
\includegraphics[width=.24\linewidth]{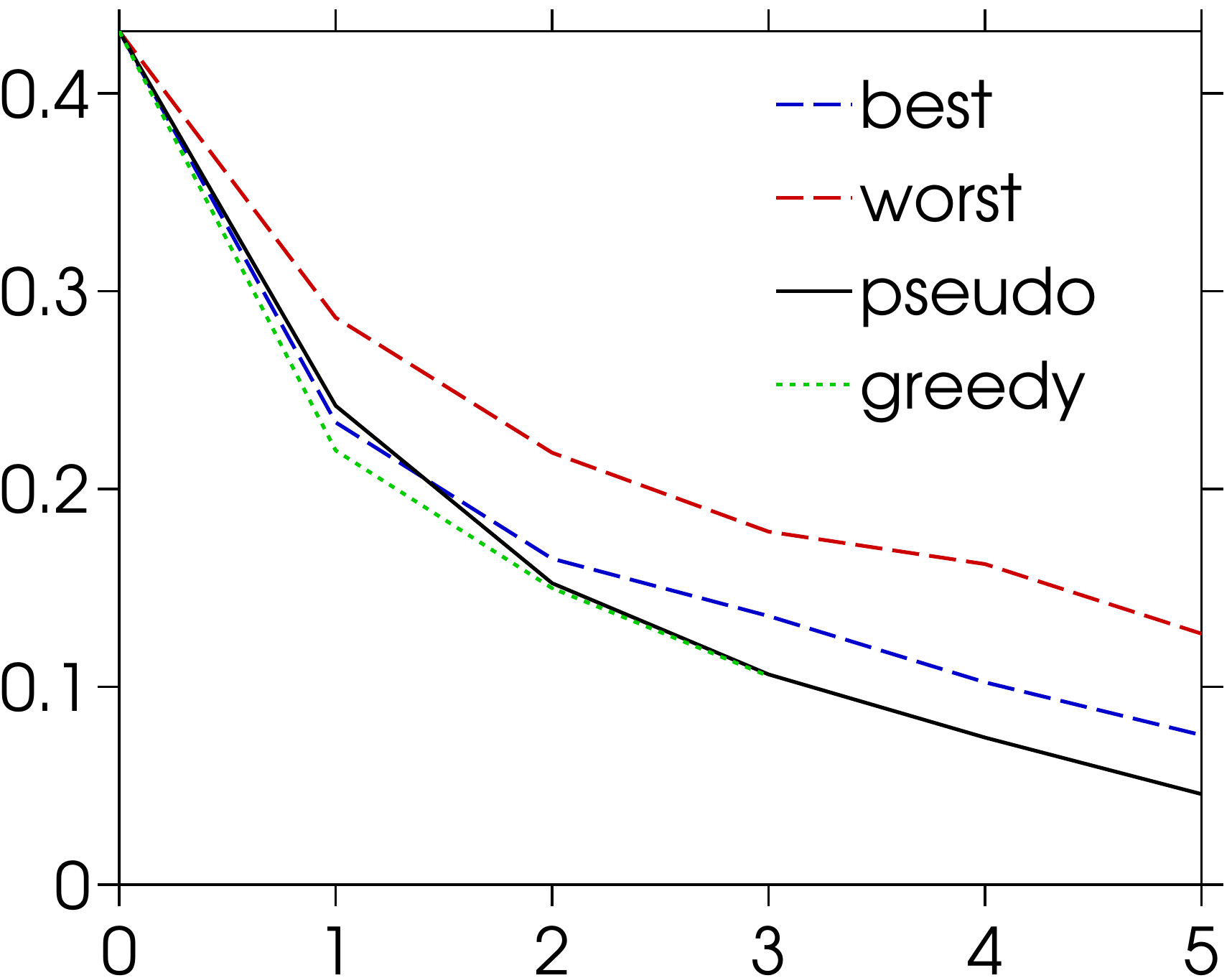} & 
\includegraphics[width=.24\linewidth]{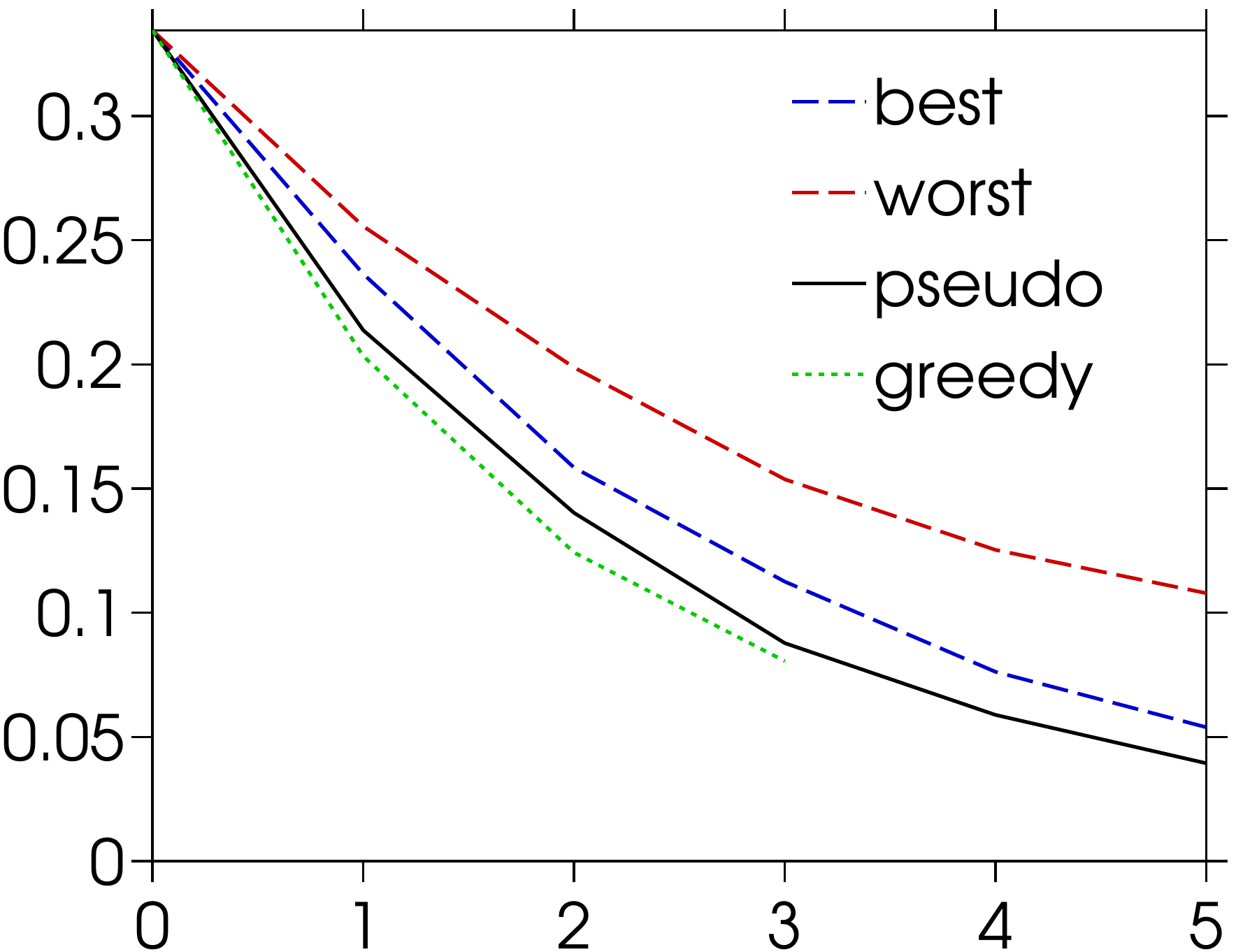} &
\includegraphics[width=.24\linewidth]{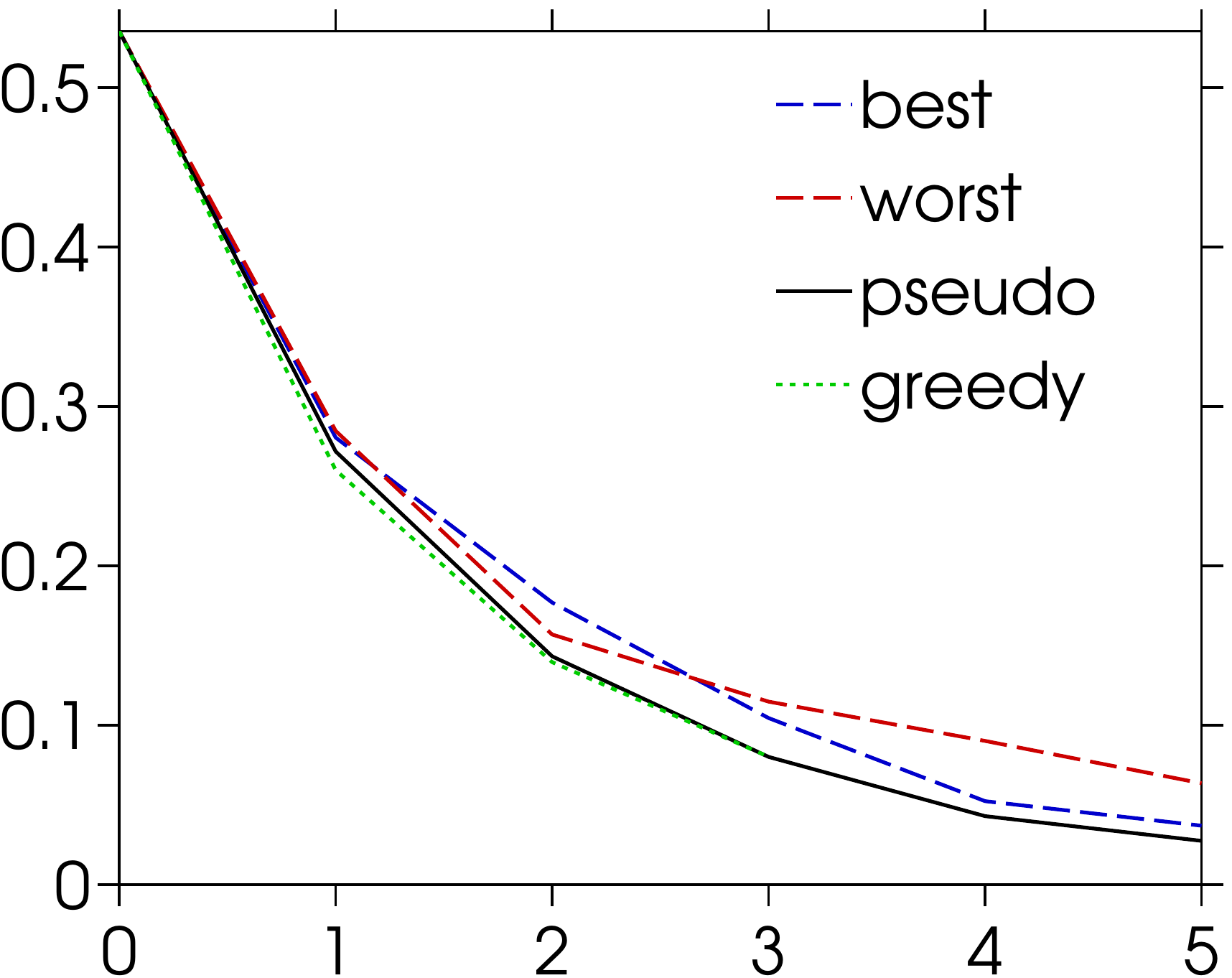} &
\- \; &
\includegraphics[width=.24\linewidth]{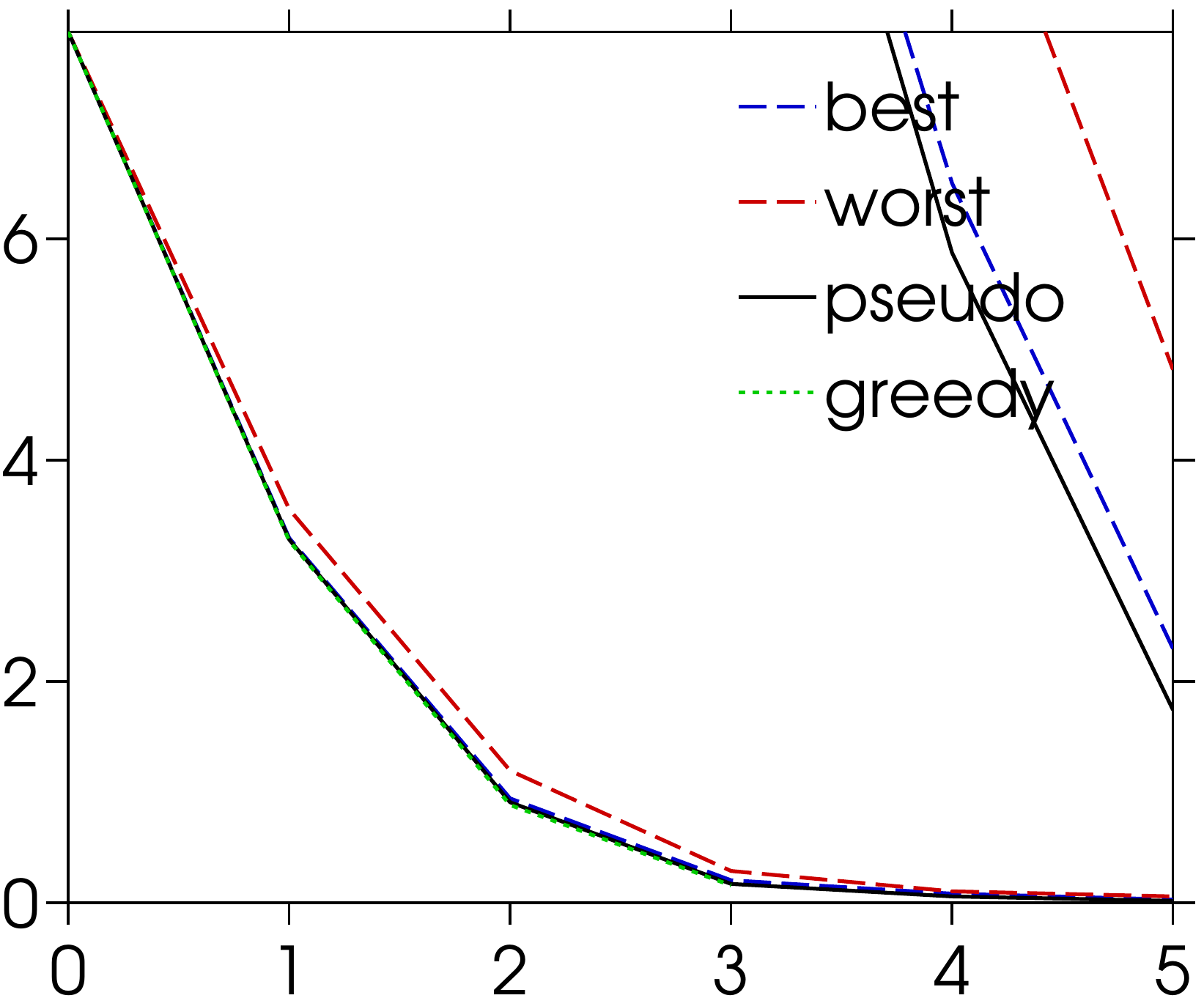} & 
\begin{sideways} {\small \- \;  \; \; complete $K_{15}$} \end{sideways} 
\\
\begin{sideways} {\small \- \; \qquad small (25)} \end{sideways} &
\includegraphics[width=.24\linewidth]{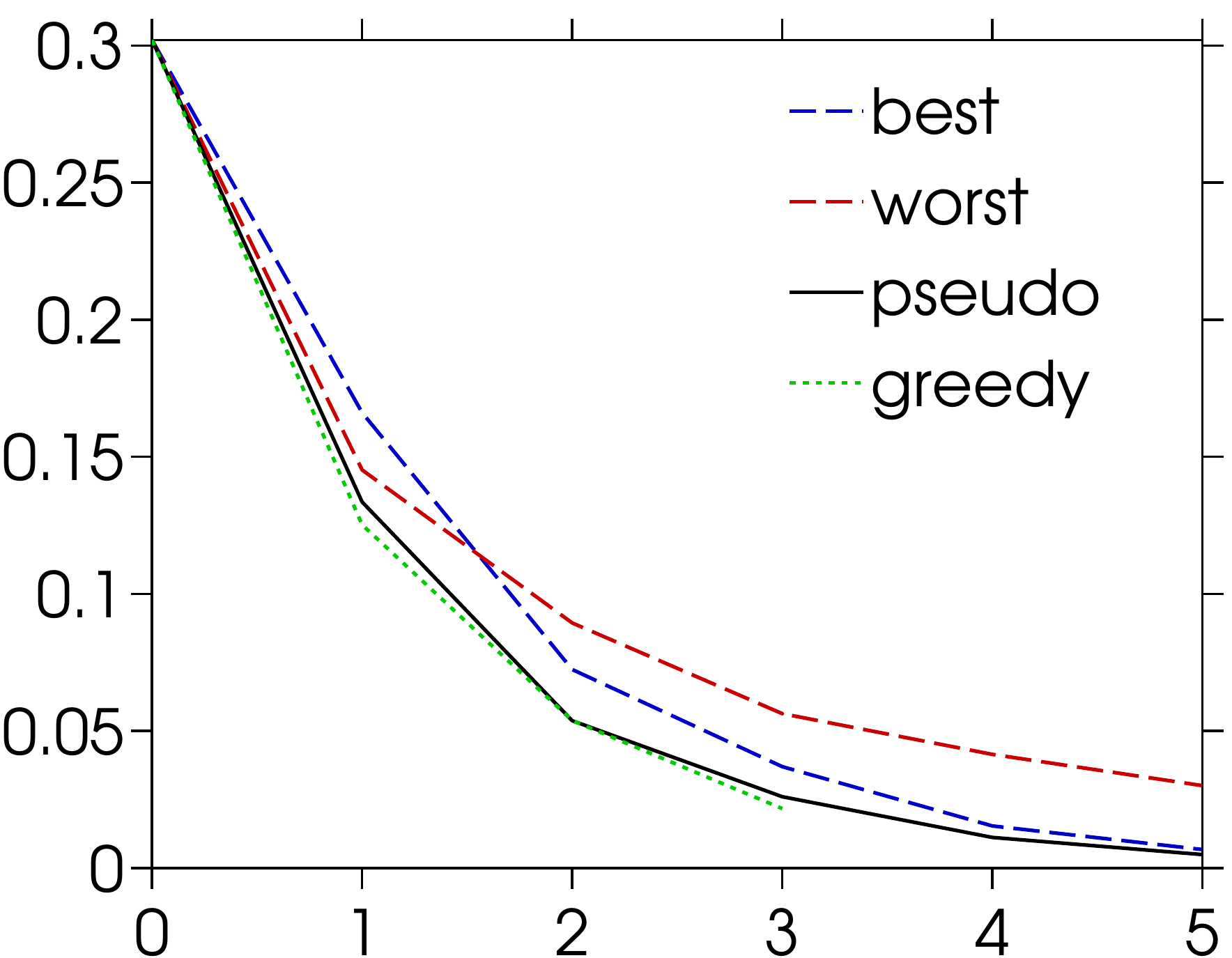} & 
\includegraphics[width=.24\linewidth]{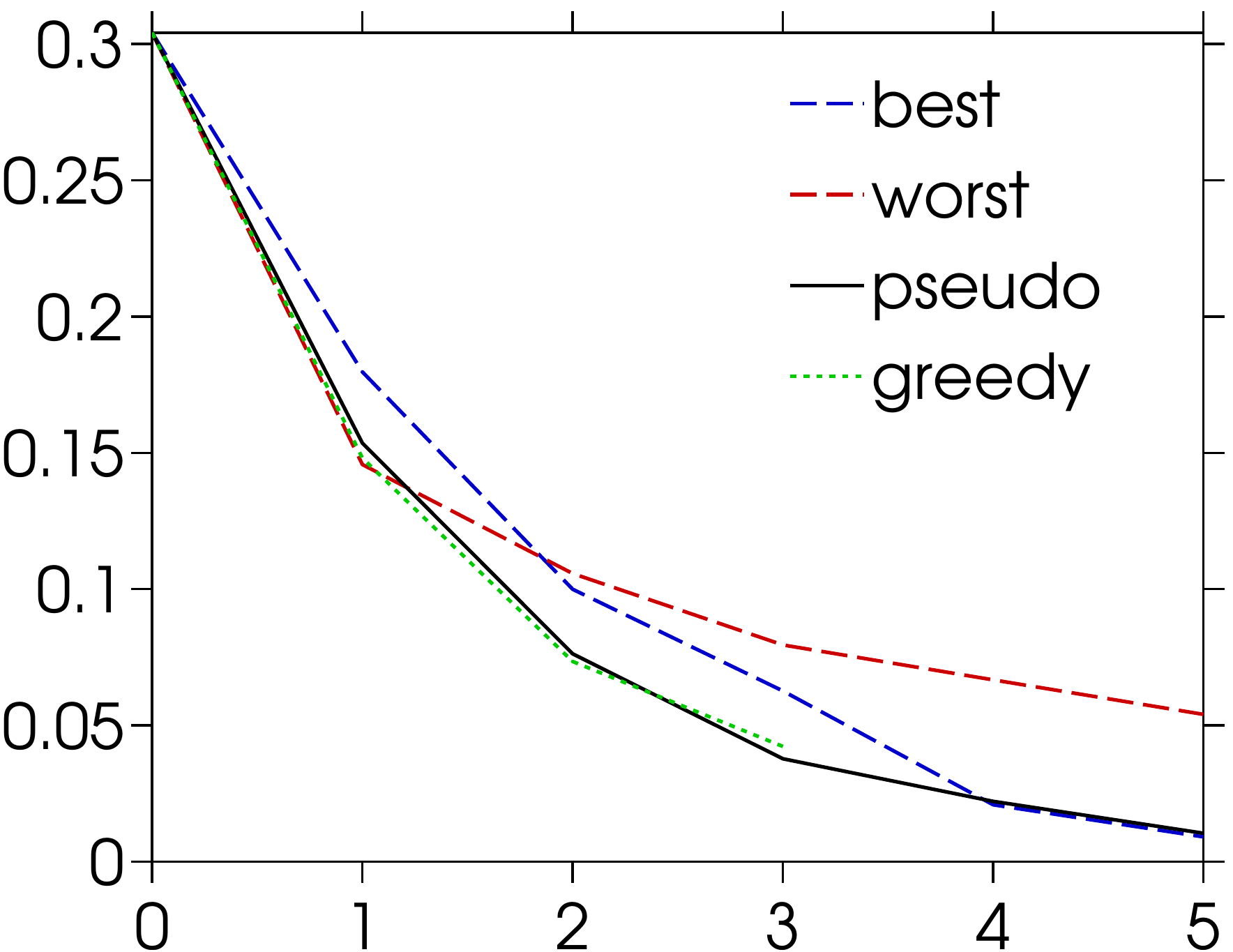} &
\includegraphics[width=.24\linewidth]{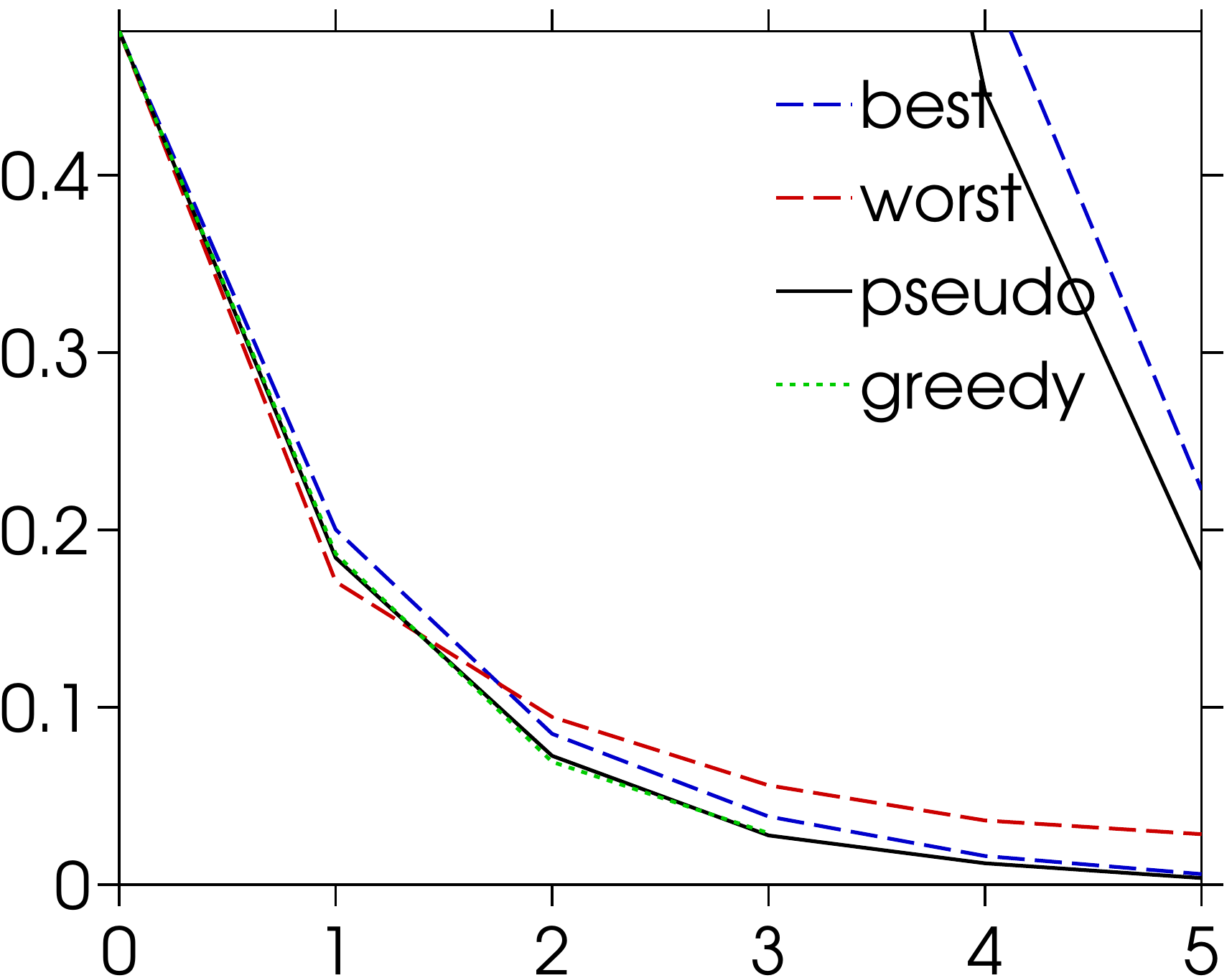} &
\- \; \; &
\includegraphics[width=.24\linewidth]{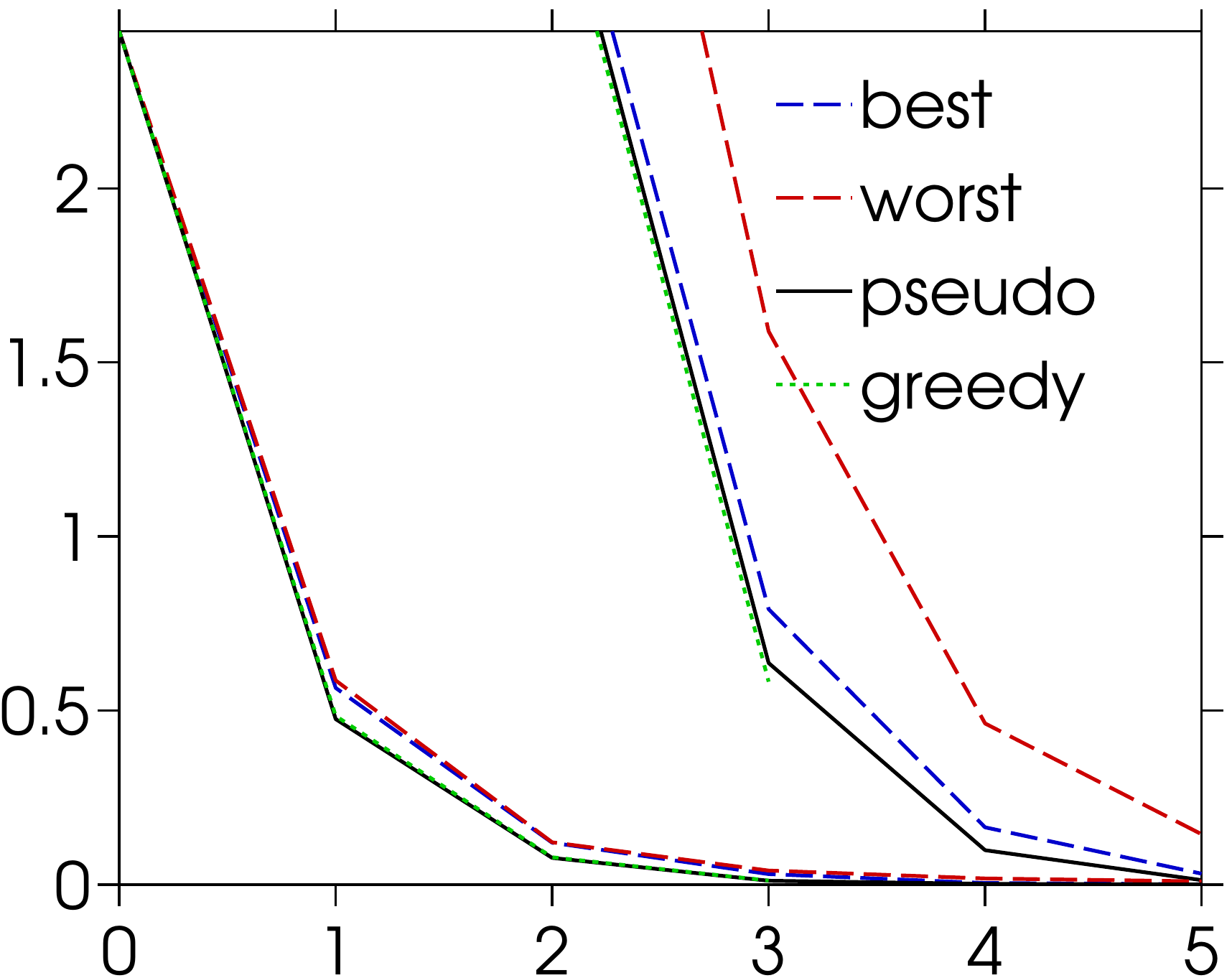} &
\begin{sideways} {\small \- \; \; \;  complete $K_{10}$} \end{sideways} 
\\
& {\small grids} & {\small random 4-regular} & {\small random Erd\"{o}s-Renyi} & & {\small complete graph} 
\end{tabular}
\end{center}
\caption{\small Mixed, zoomed around Bethe $[-6,6]$
}
%\label{fig:weird}
\end{figure}

\begin{figure}
\begin{center}
\setlength\tabcolsep{1pt}
\begin{tabular}{ccccccc}
\begin{sideways} {\small \- \; \qquad large (81)} \end{sideways} &
\includegraphics[width=.24\linewidth]{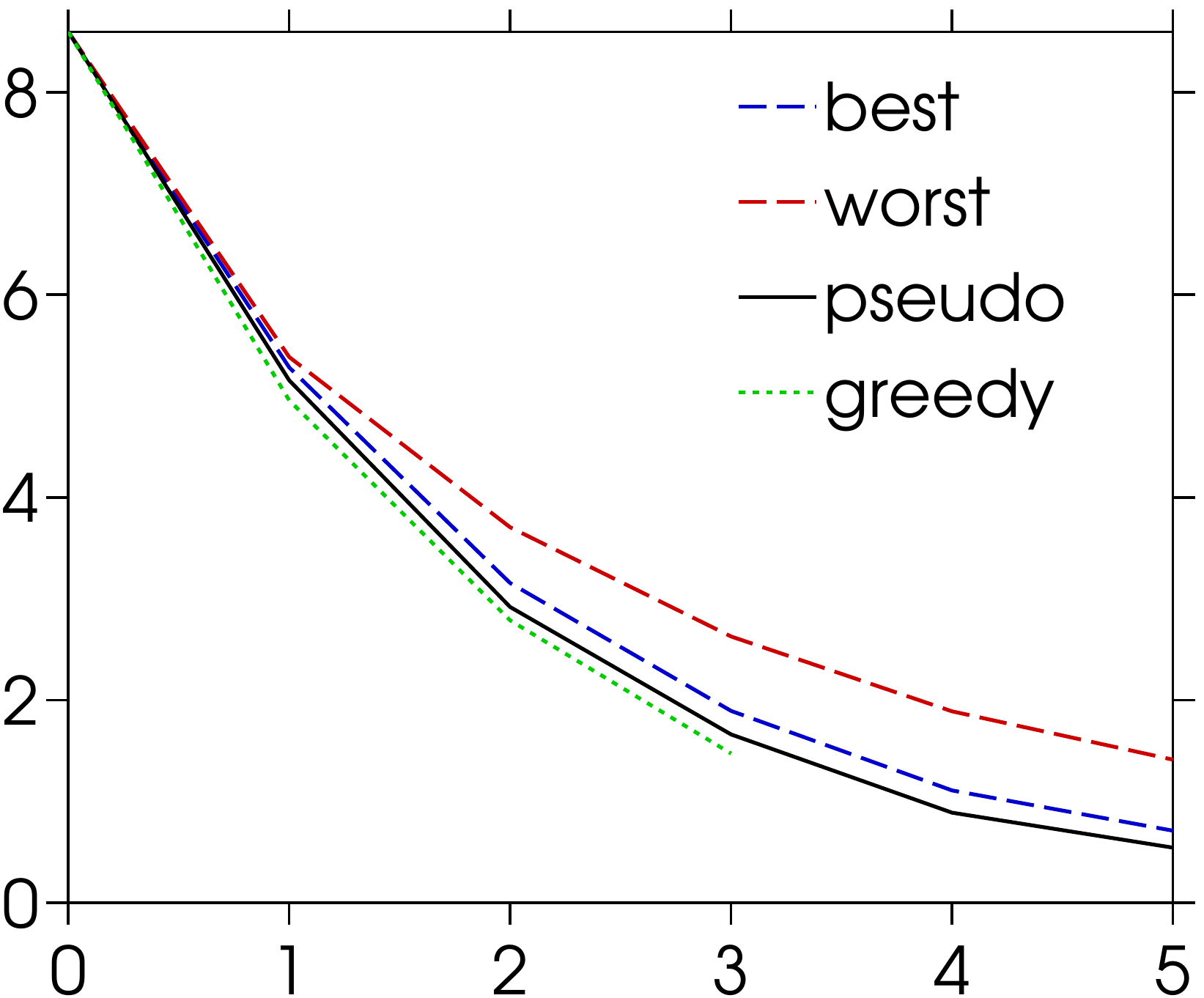} & 
\includegraphics[width=.24\linewidth]{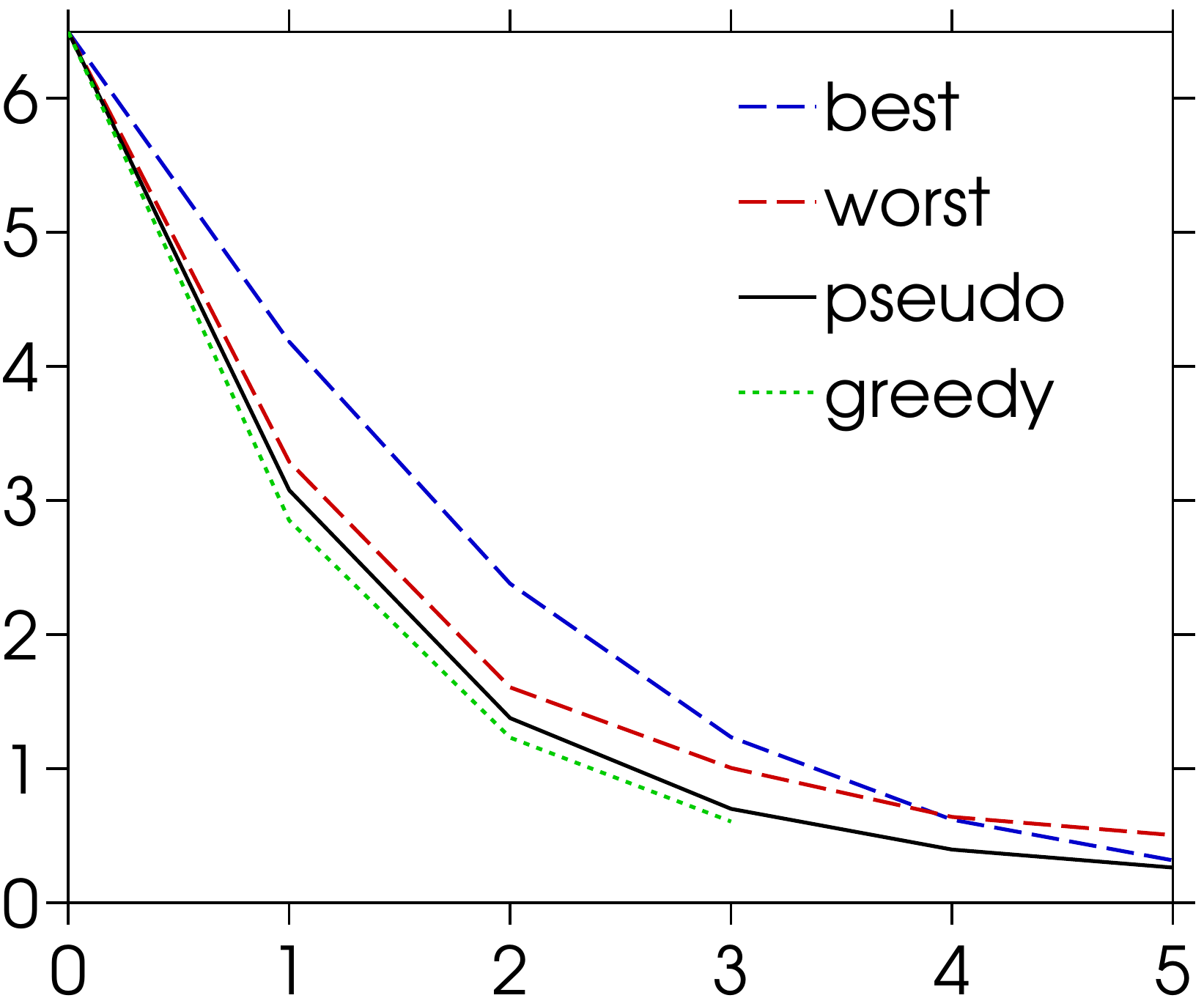} &
\includegraphics[width=.24\linewidth]{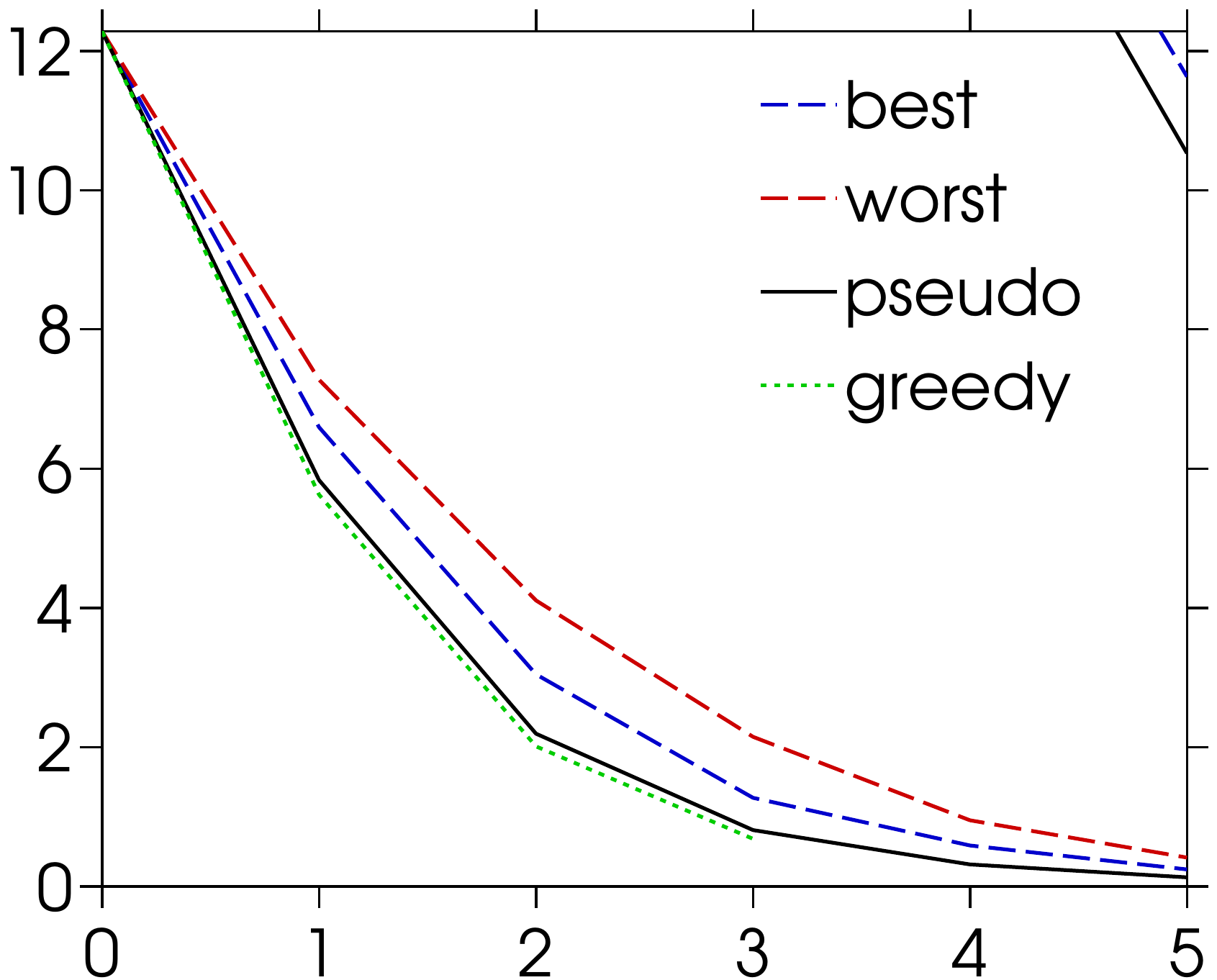} 
\\
\begin{sideways} {\small \- \; \; \quad medium (49)} \end{sideways} &
\includegraphics[width=.24\linewidth]{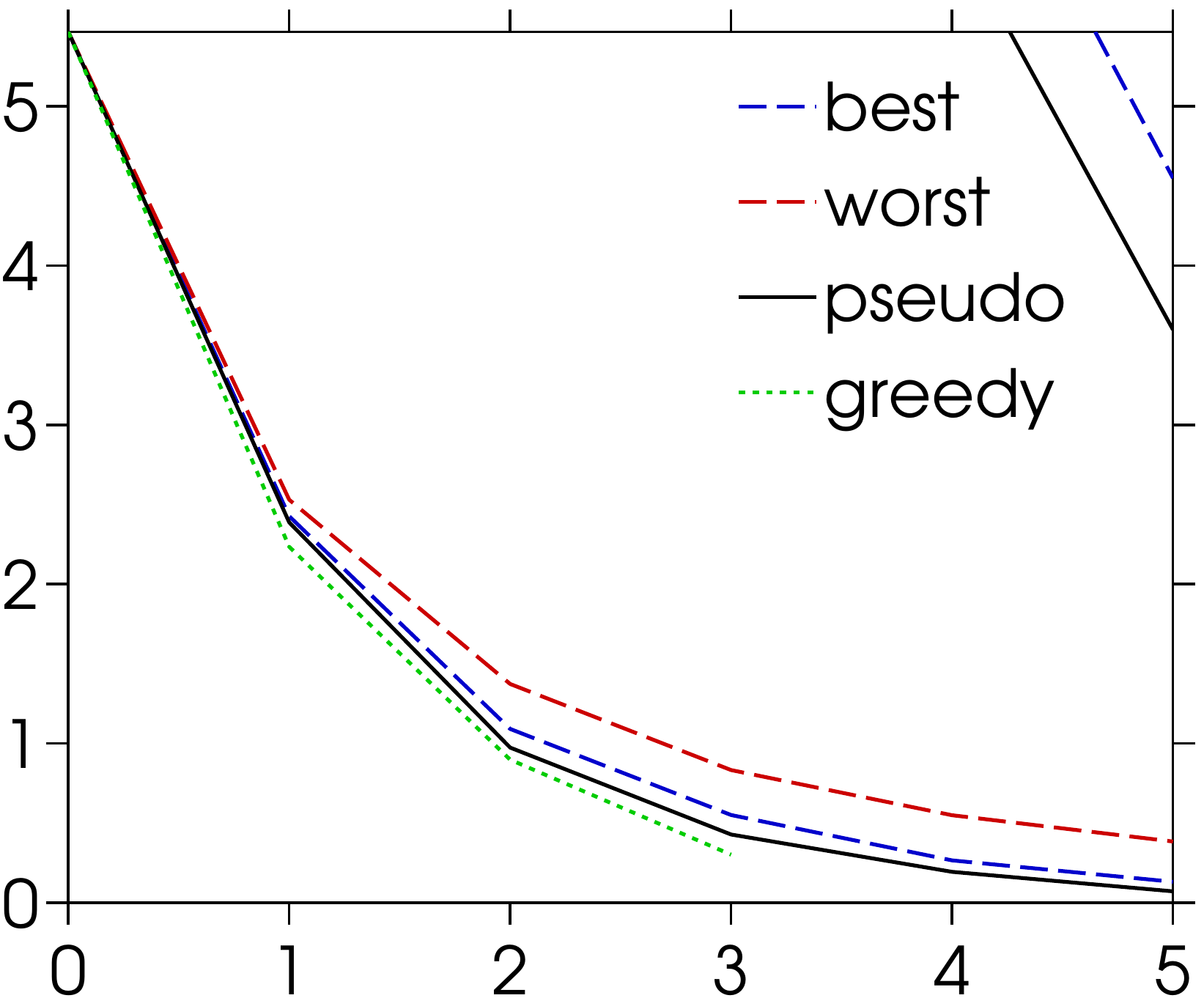} & 
\includegraphics[width=.24\linewidth]{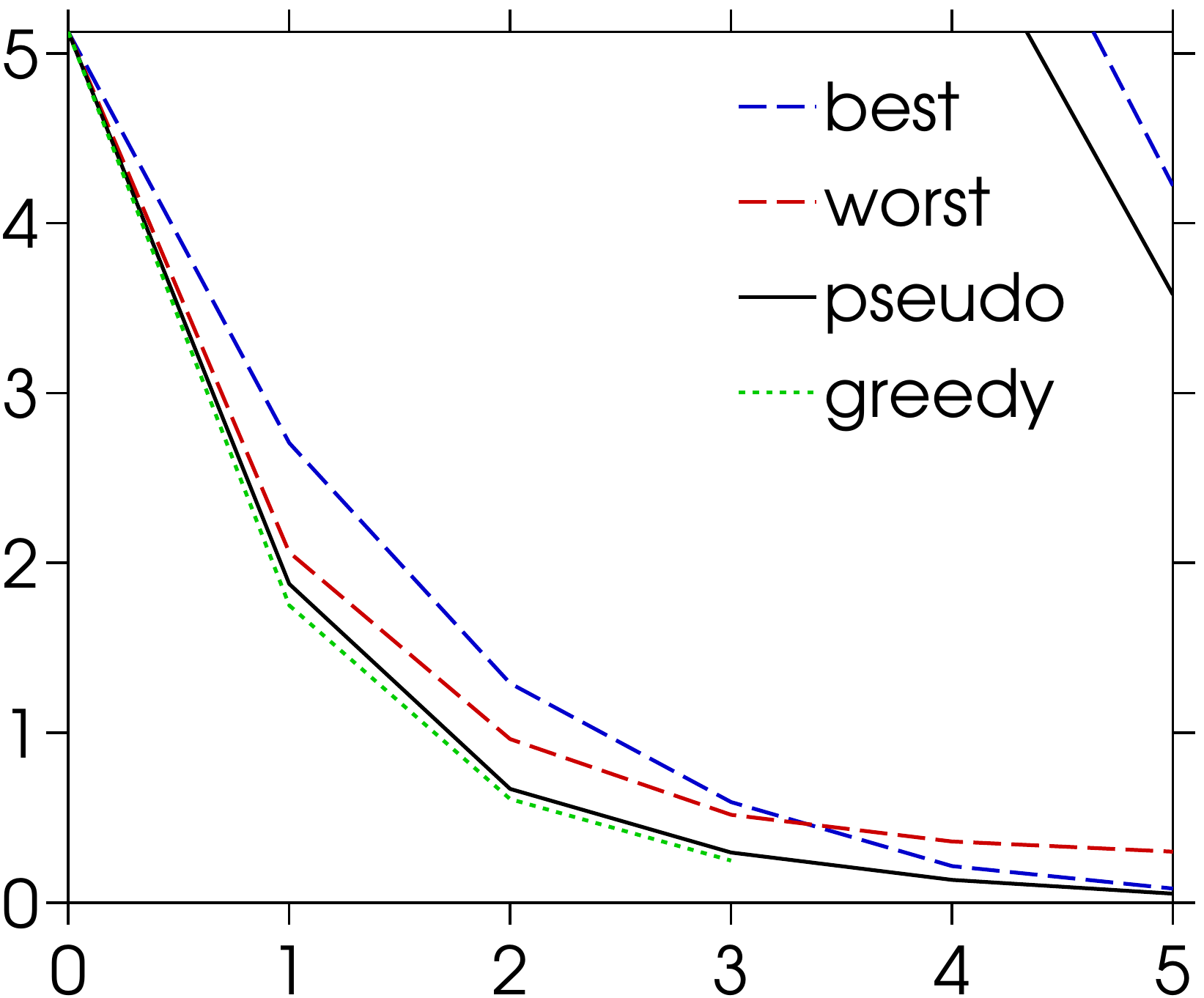} &
\includegraphics[width=.24\linewidth]{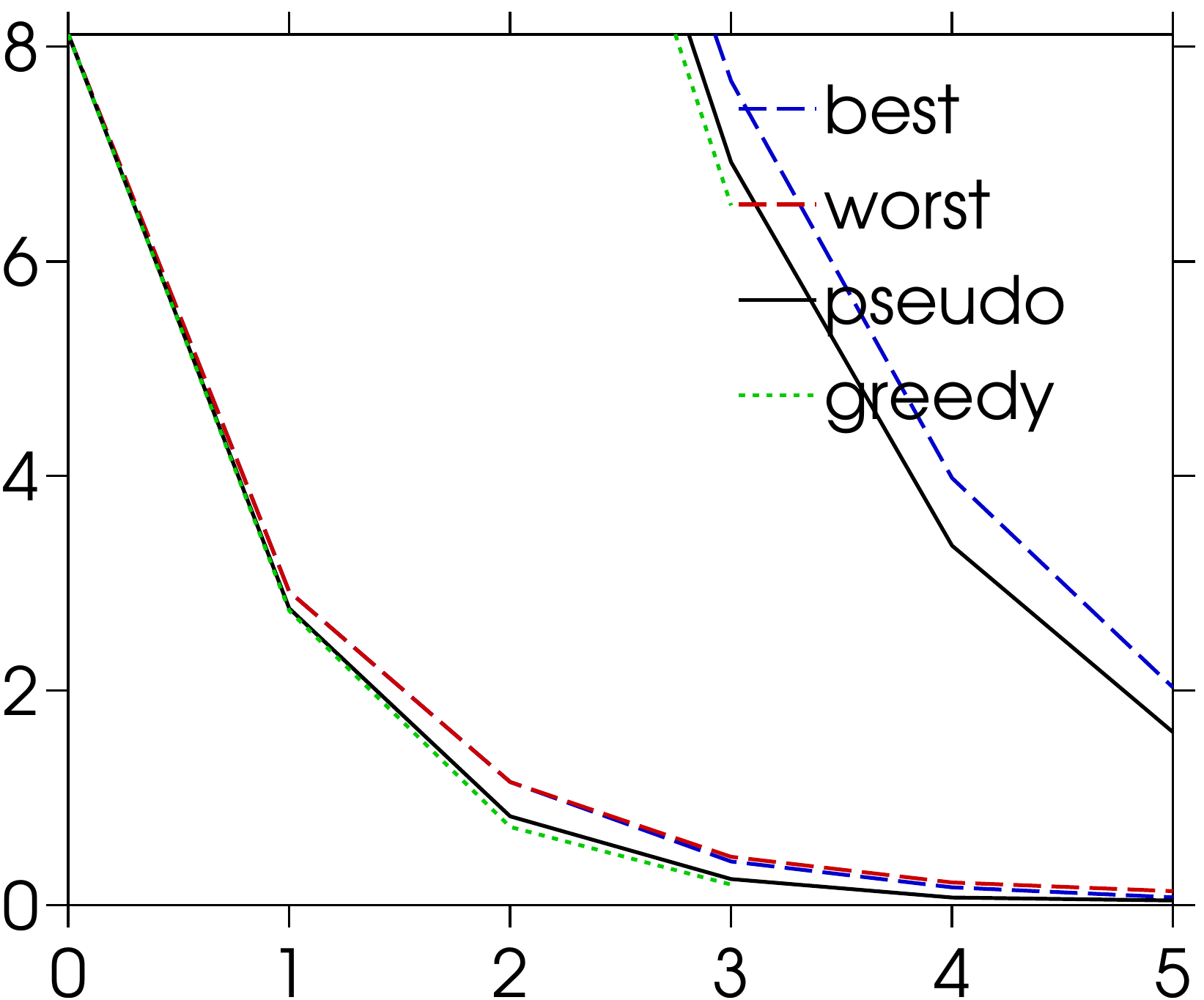} &
\- \; &
\includegraphics[width=.24\linewidth]{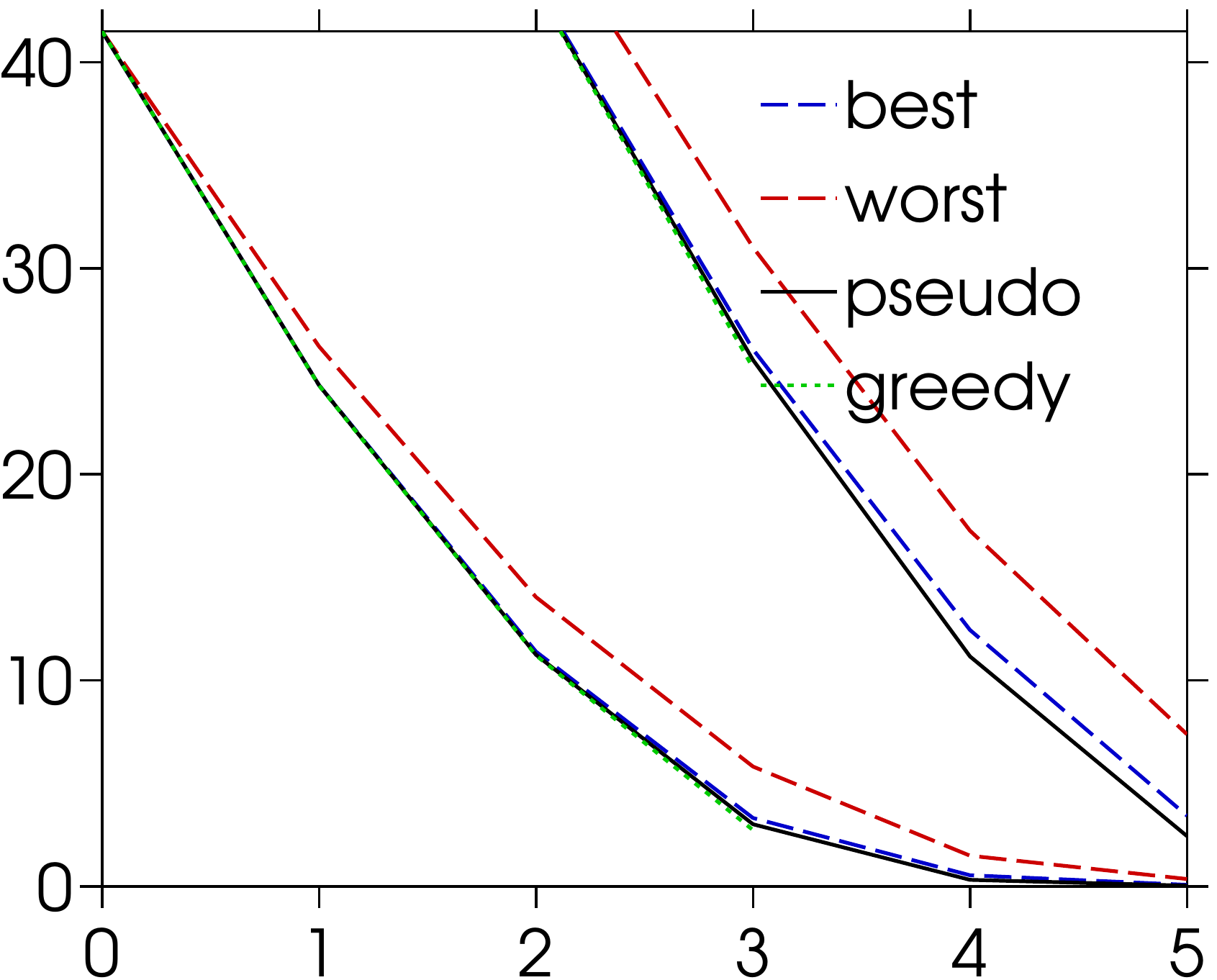} &
\begin{sideways} {\small \- \;  \; \; complete $K_{15}$} \end{sideways} \\
\begin{sideways} {\small \- \; \qquad small (25)} \end{sideways} &
\includegraphics[width=.24\linewidth]{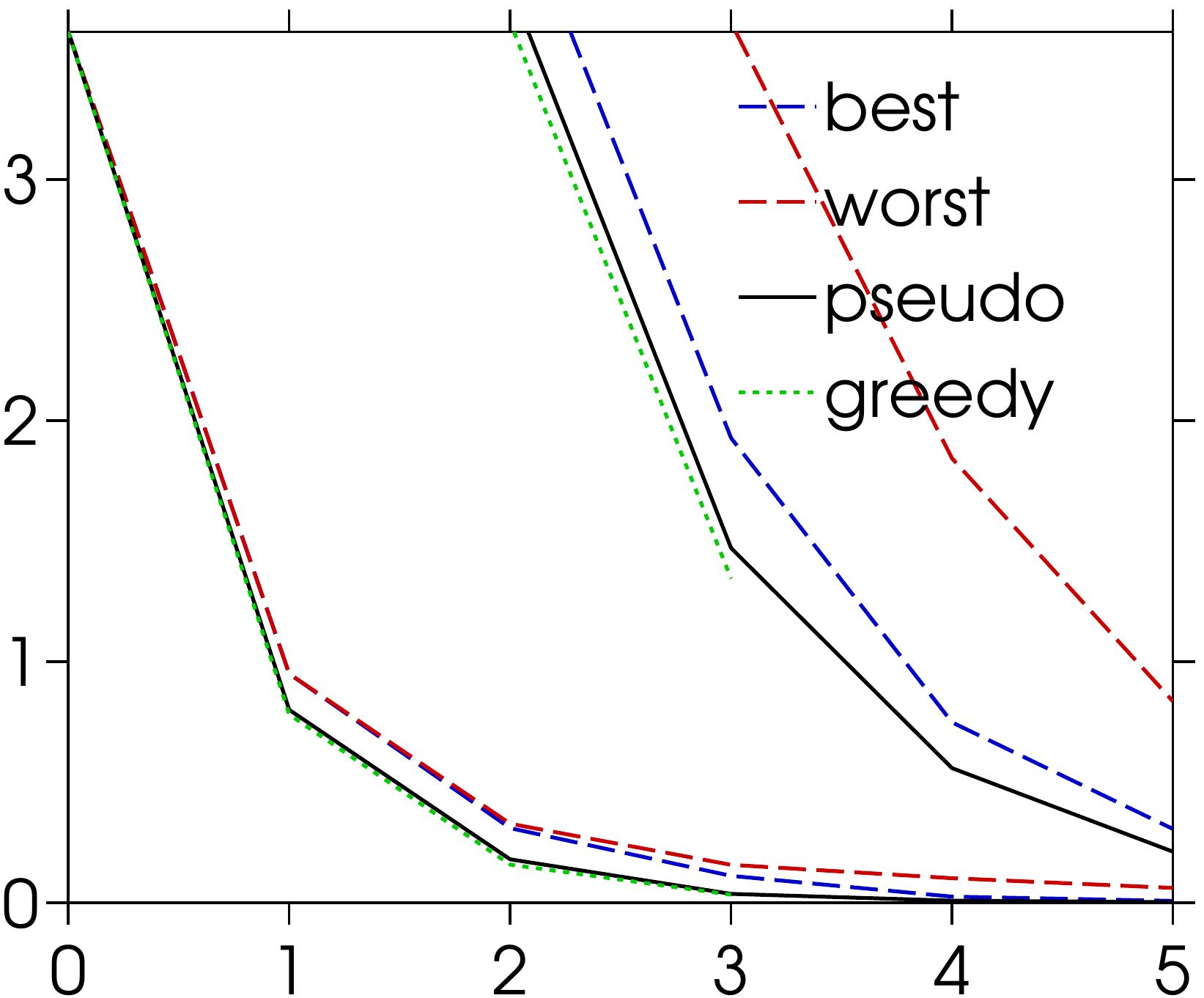} & 
\includegraphics[width=.24\linewidth]{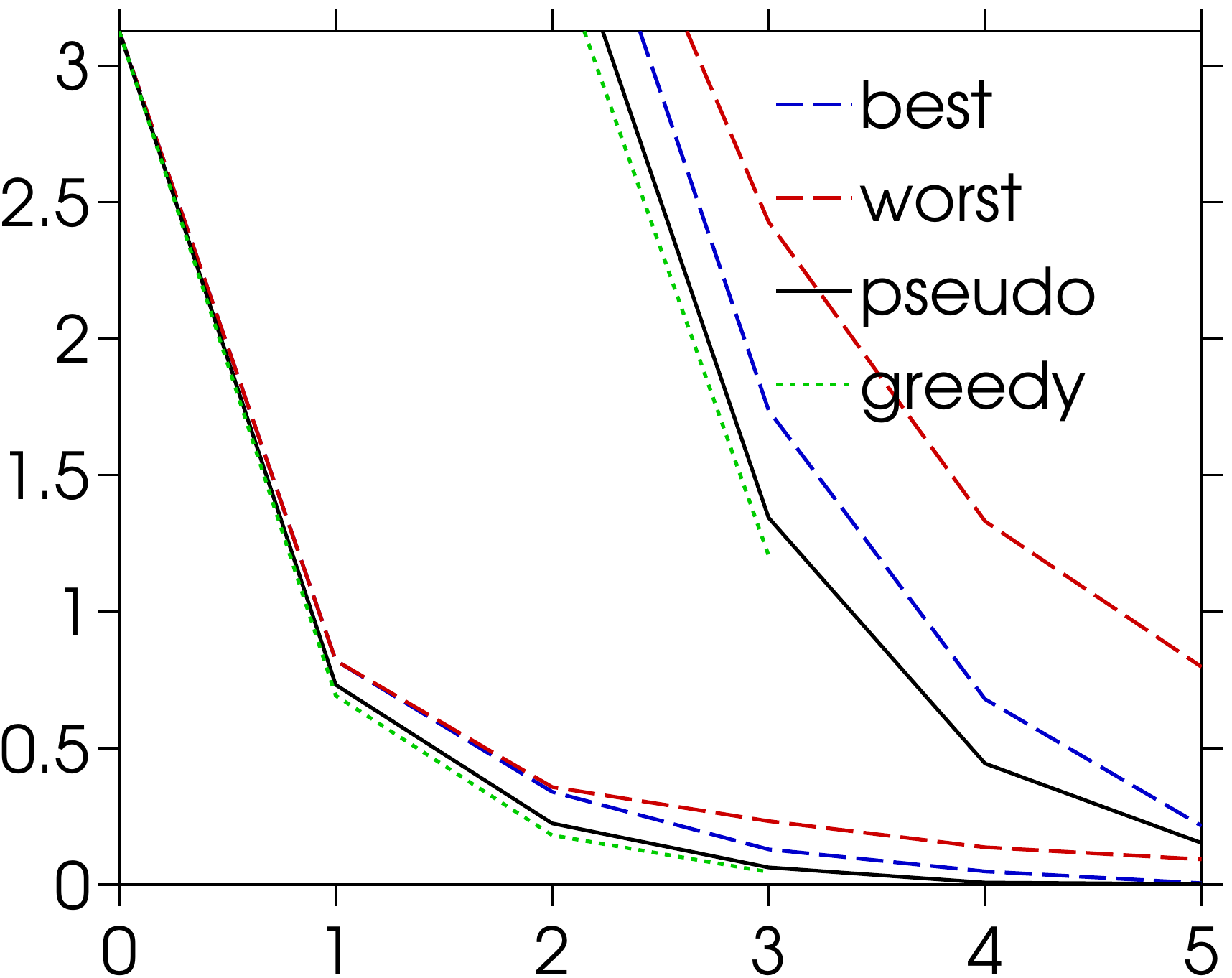} &
\includegraphics[width=.24\linewidth]{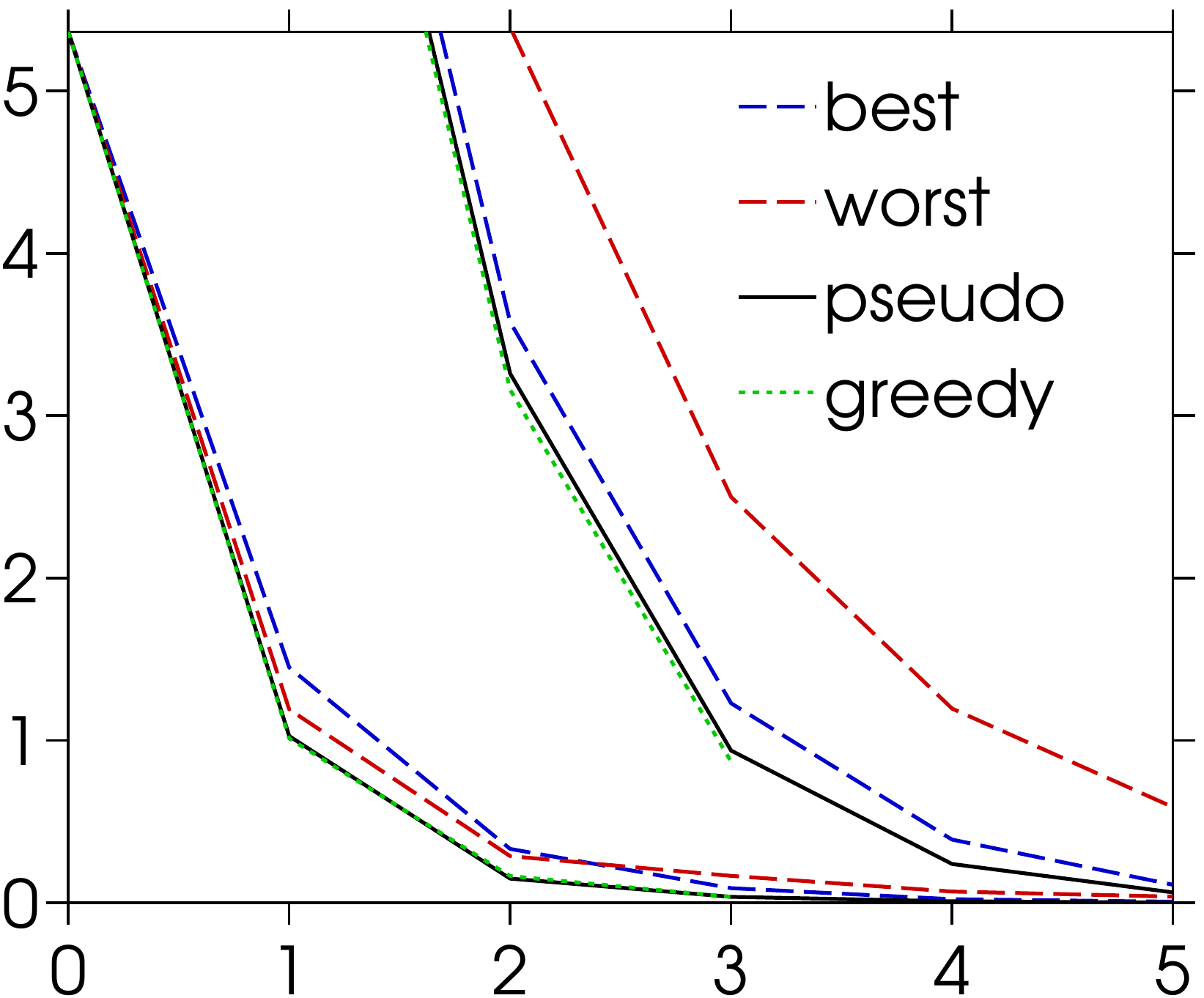} &
\- \; &
\includegraphics[width=.24\linewidth]{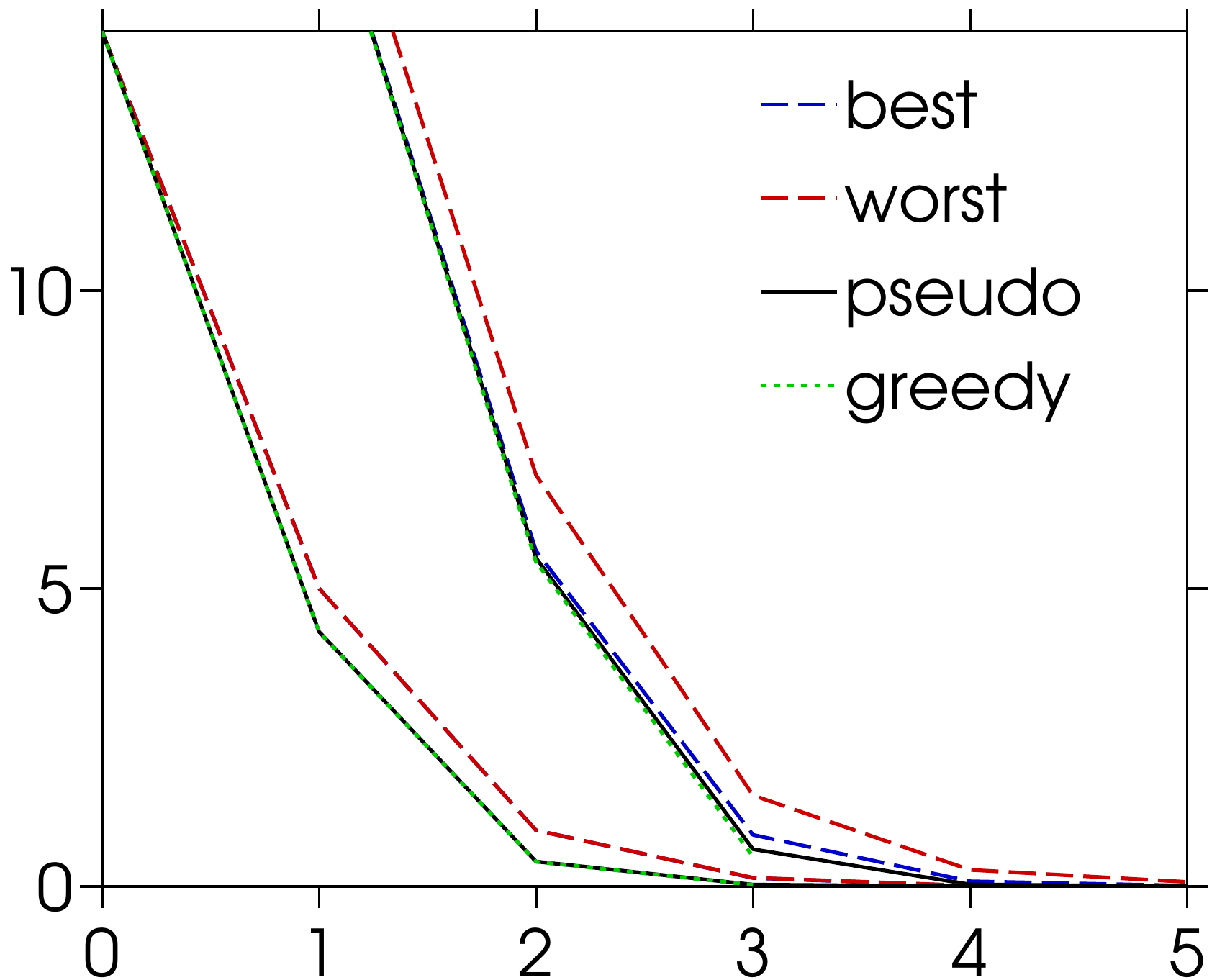} &
\begin{sideways} {\small \- \; \; \;  complete $K_{10}$} \end{sideways} \\
& {\small grids} & {\small random 4-regular} & {\small random Erd\"{o}s-Renyi} & & {\small complete graph} 
\end{tabular}
\end{center}
\caption{\small Mixed, zoomed around Bethe $[-12,12]$
}
%\label{fig:weird2}
\end{figure}

\begin{figure}
\begin{center}
\setlength\tabcolsep{1pt}
\begin{tabular}{ccccccc}
\begin{sideways} {\small \- \; \qquad large (81)} \end{sideways} &
\includegraphics[width=.24\linewidth]{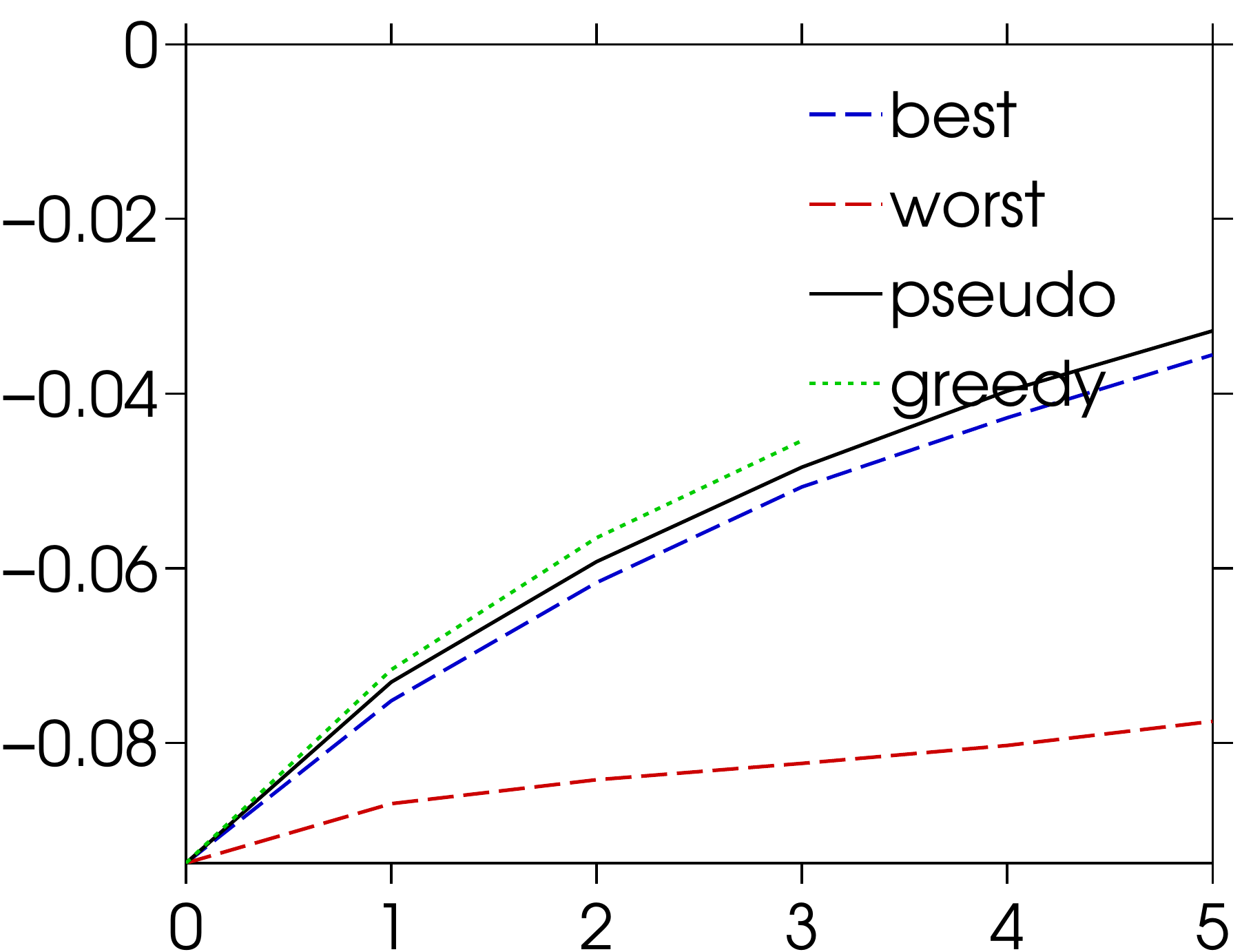} & 
\includegraphics[width=.24\linewidth]{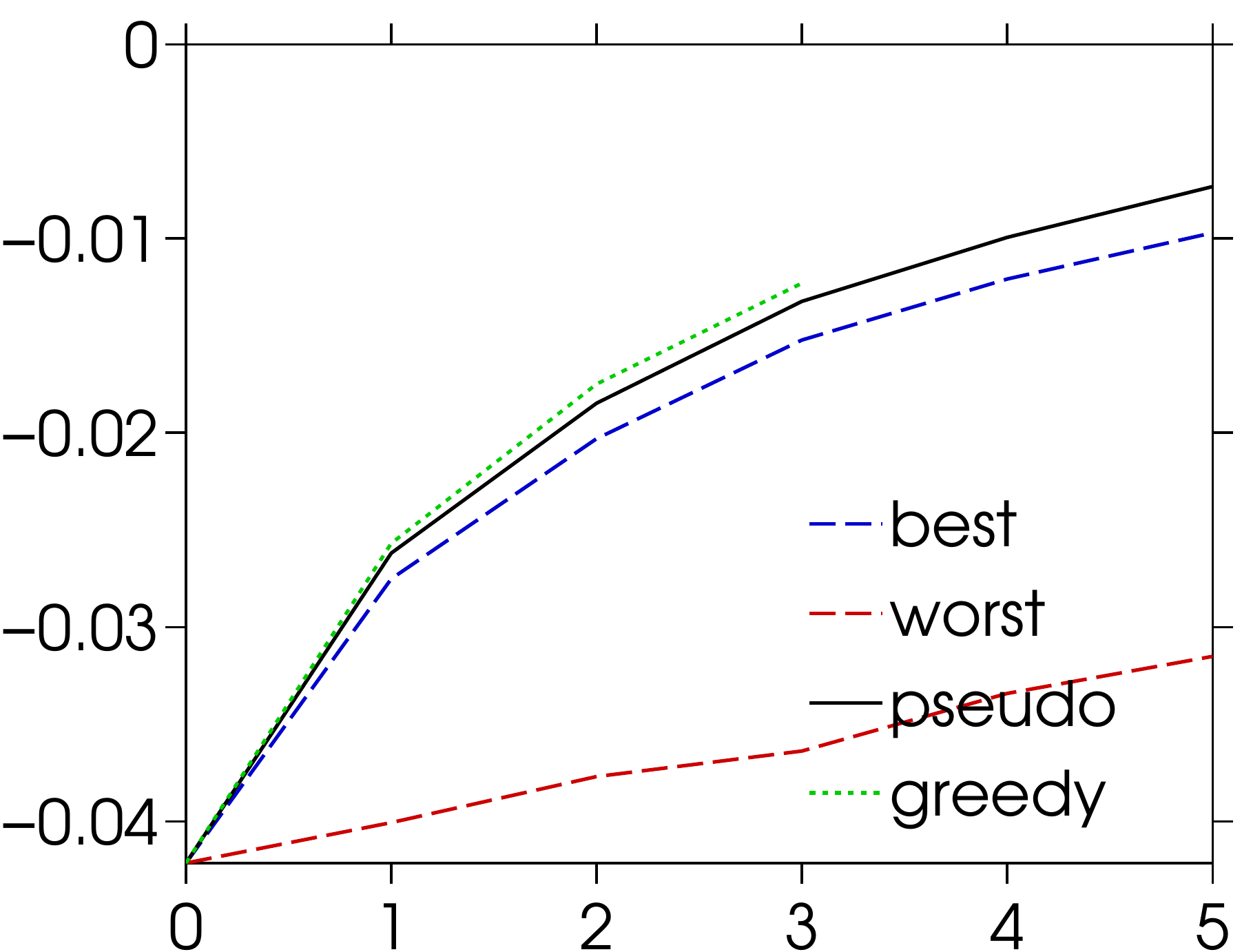} &
\includegraphics[width=.24\linewidth]{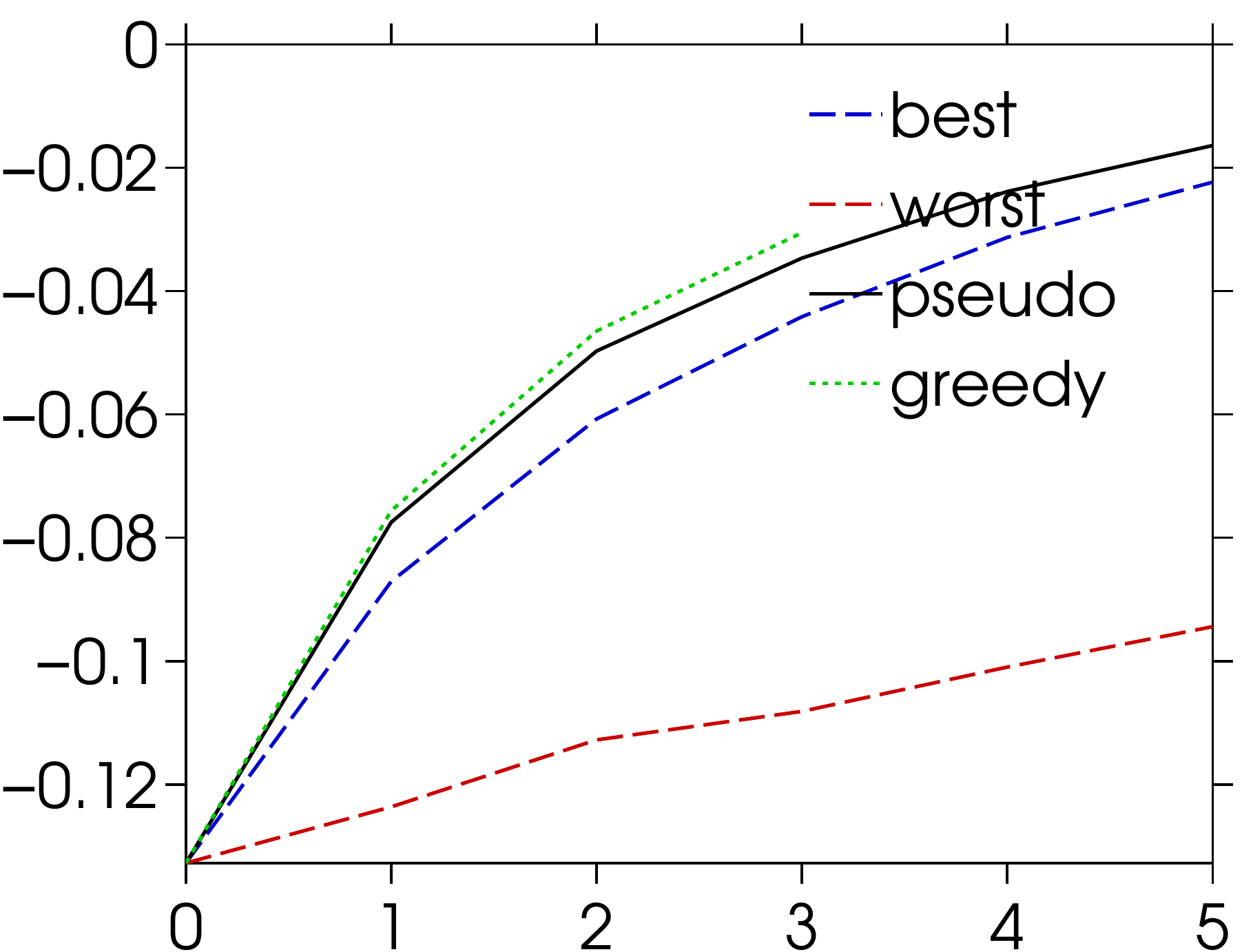} 
\\
\begin{sideways} {\small \- \; \; \quad medium (49)} \end{sideways} &
\includegraphics[width=.24\linewidth]{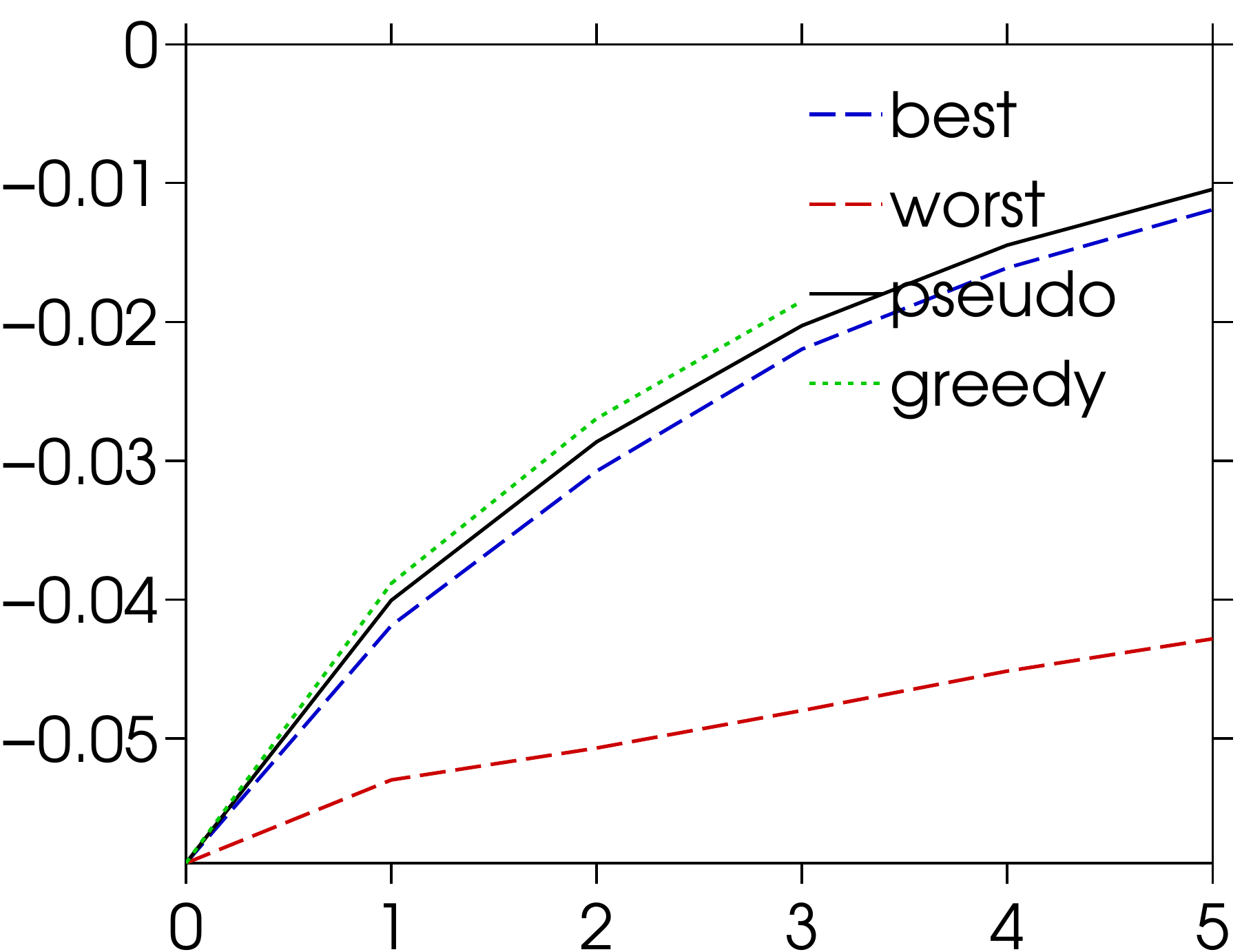} & 
\includegraphics[width=.24\linewidth]{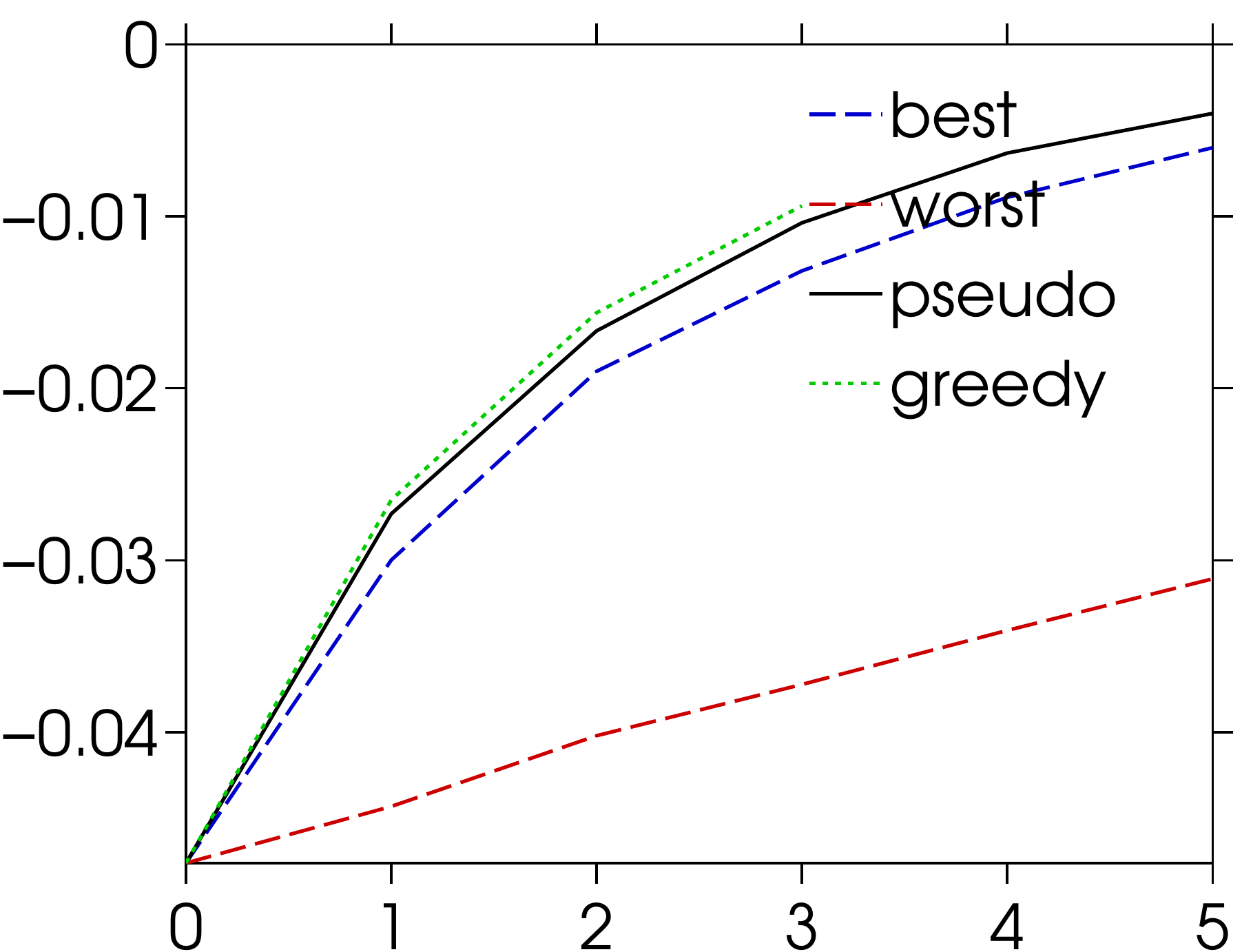} &
\includegraphics[width=.24\linewidth]{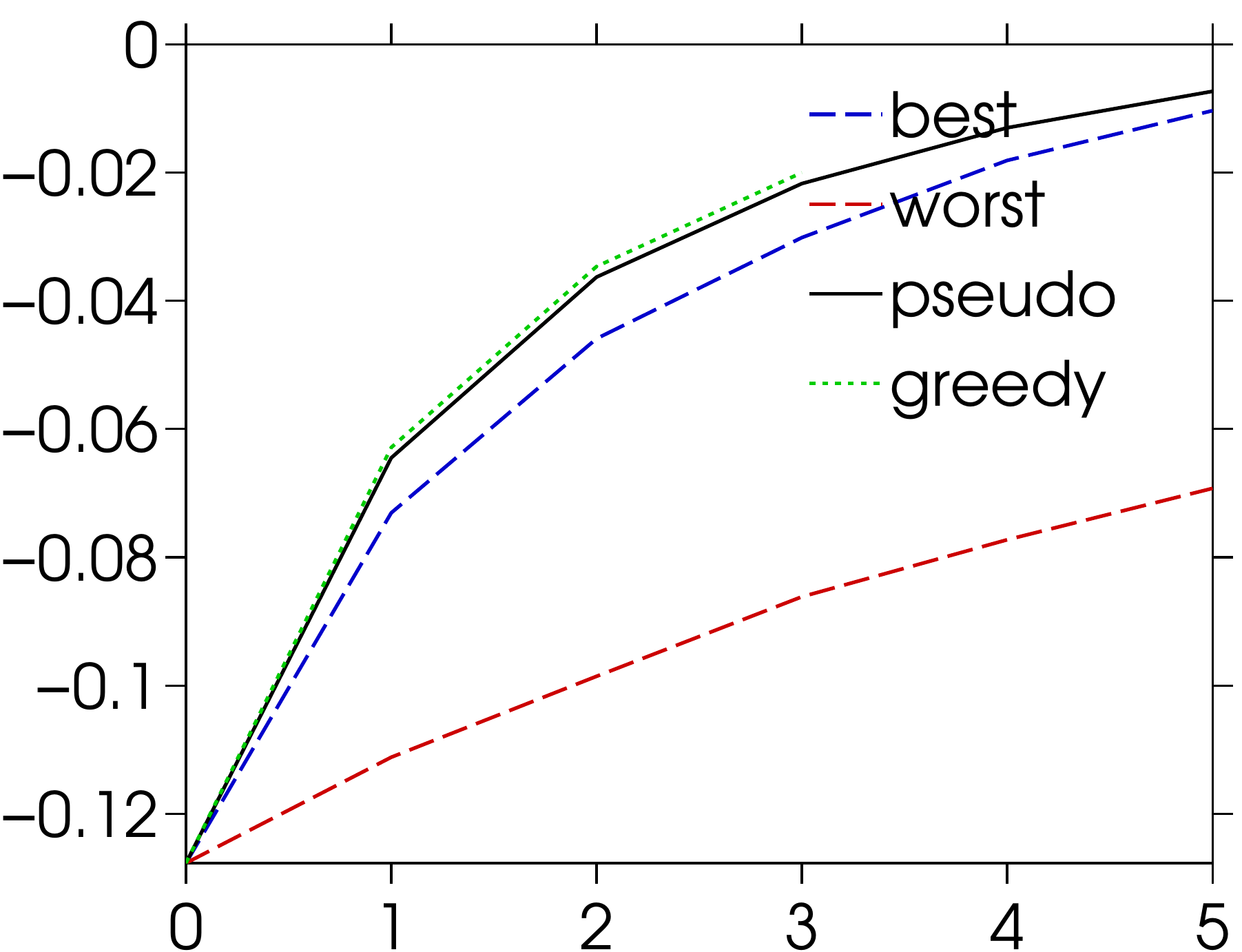} &
\- \; &
\includegraphics[width=.24\linewidth]{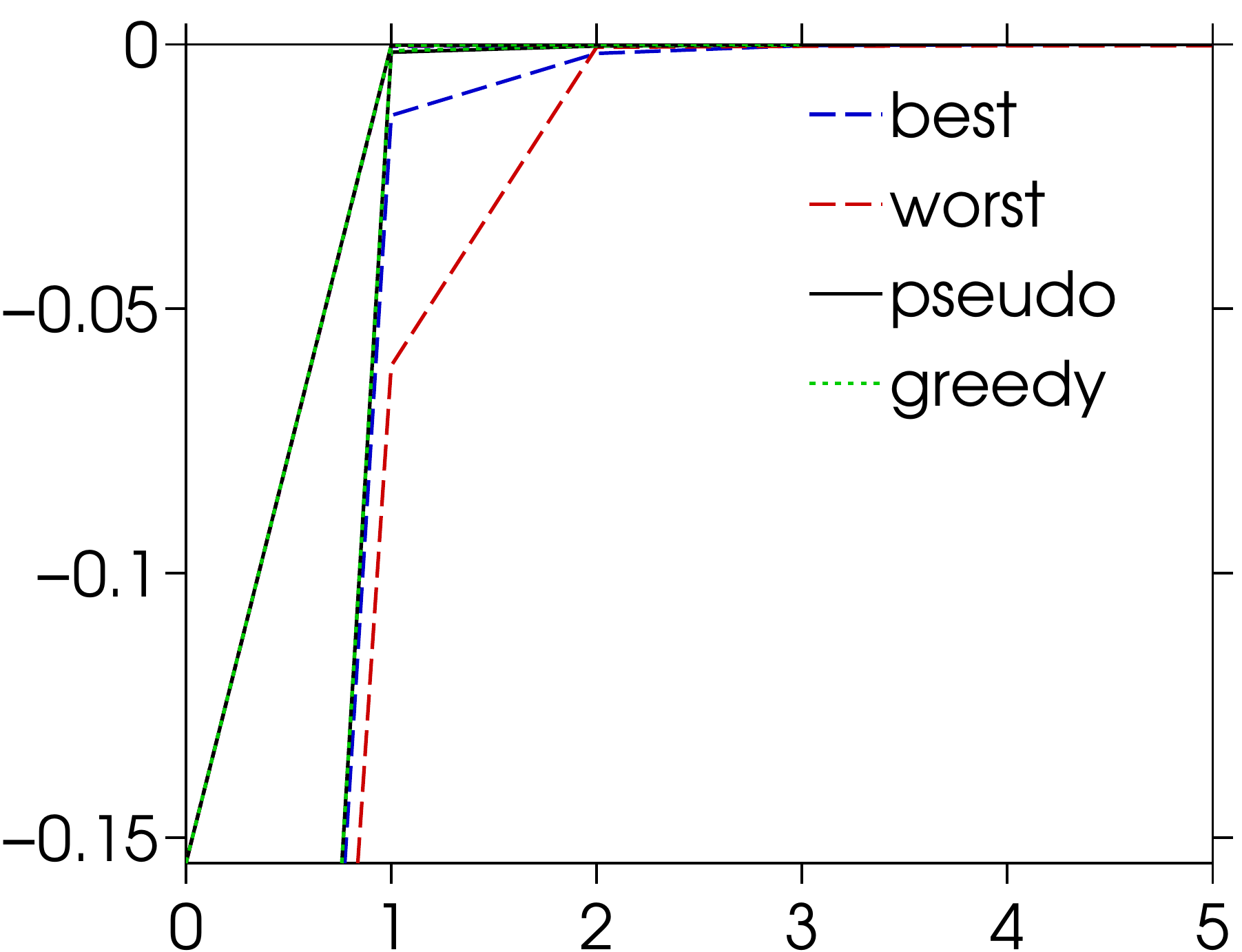} &
\begin{sideways} {\small \- \; \; \;  complete $K_{15}$} \end{sideways} \\
\begin{sideways} {\small \- \; \qquad small (25)} \end{sideways} &
\includegraphics[width=.24\linewidth]{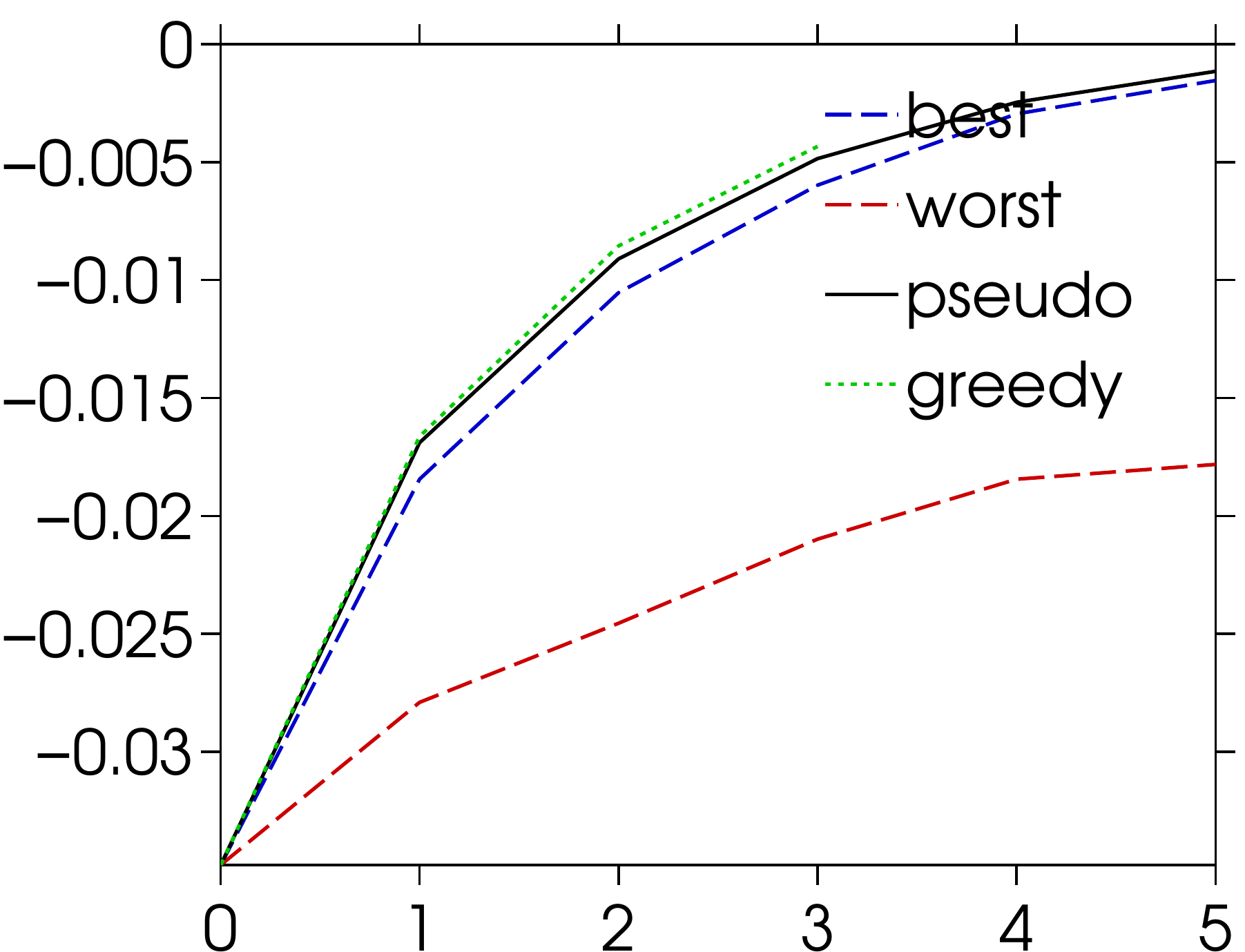} & 
\includegraphics[width=.24\linewidth]{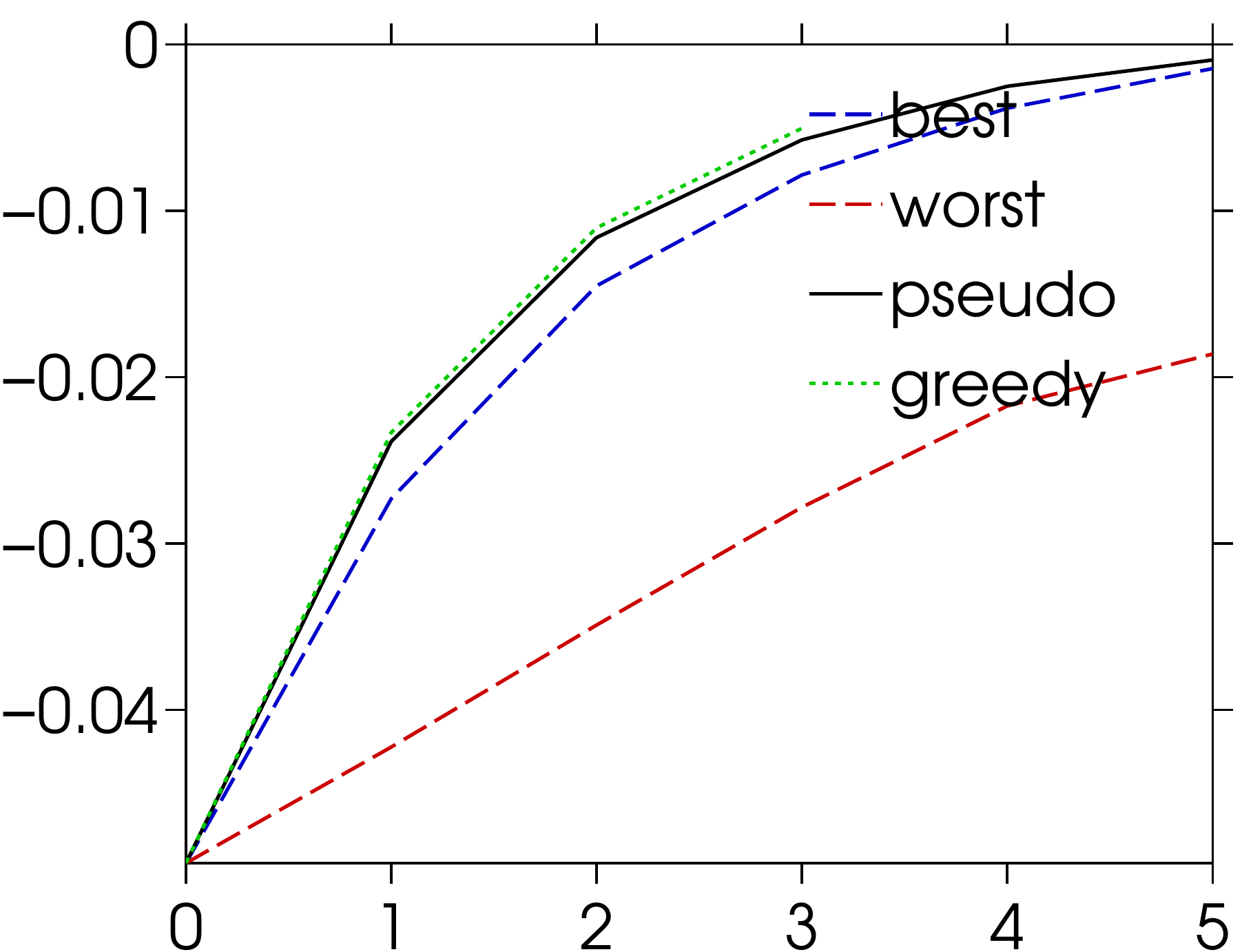} &
\includegraphics[width=.24\linewidth]{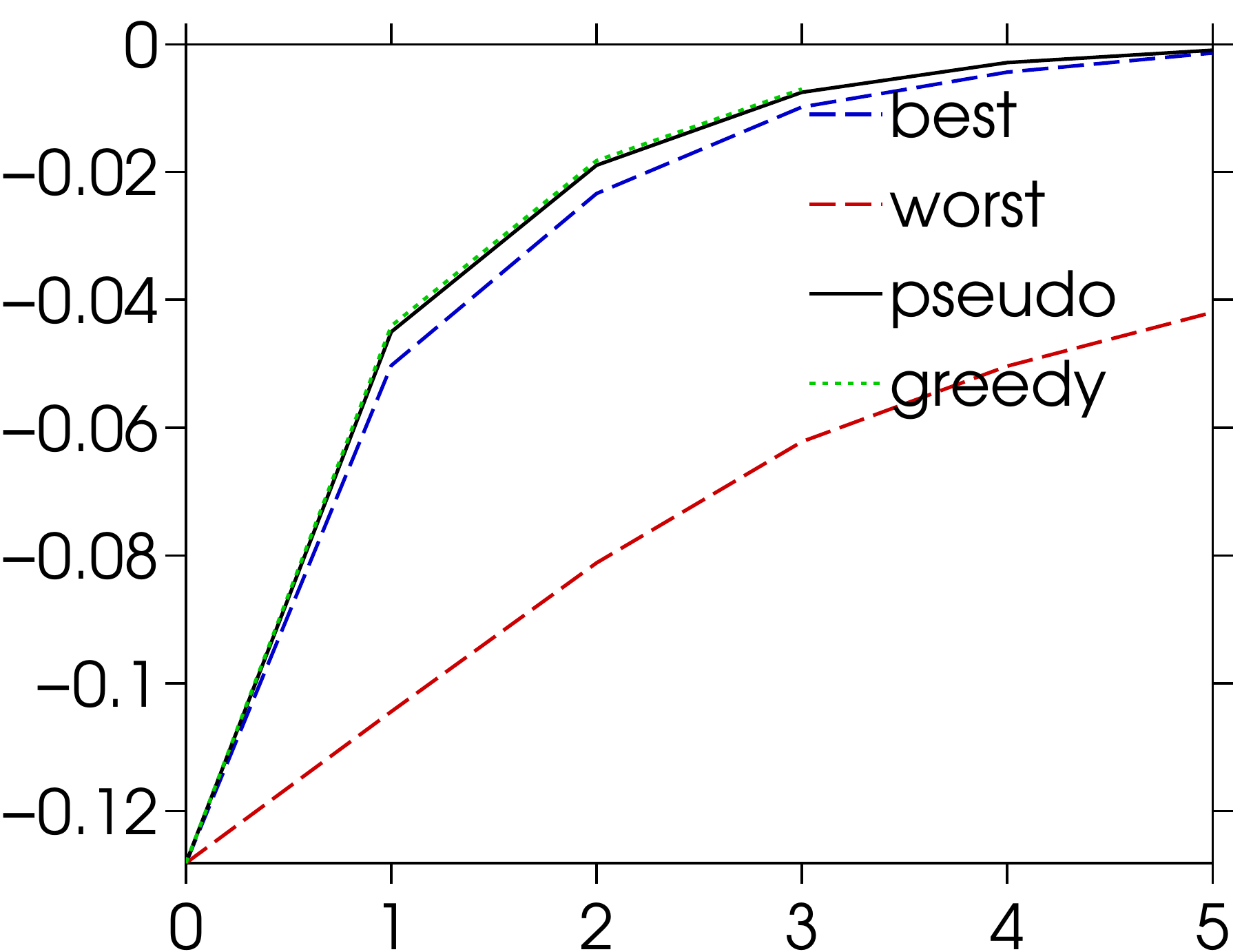} &
\- \; &
\includegraphics[width=.24\linewidth]{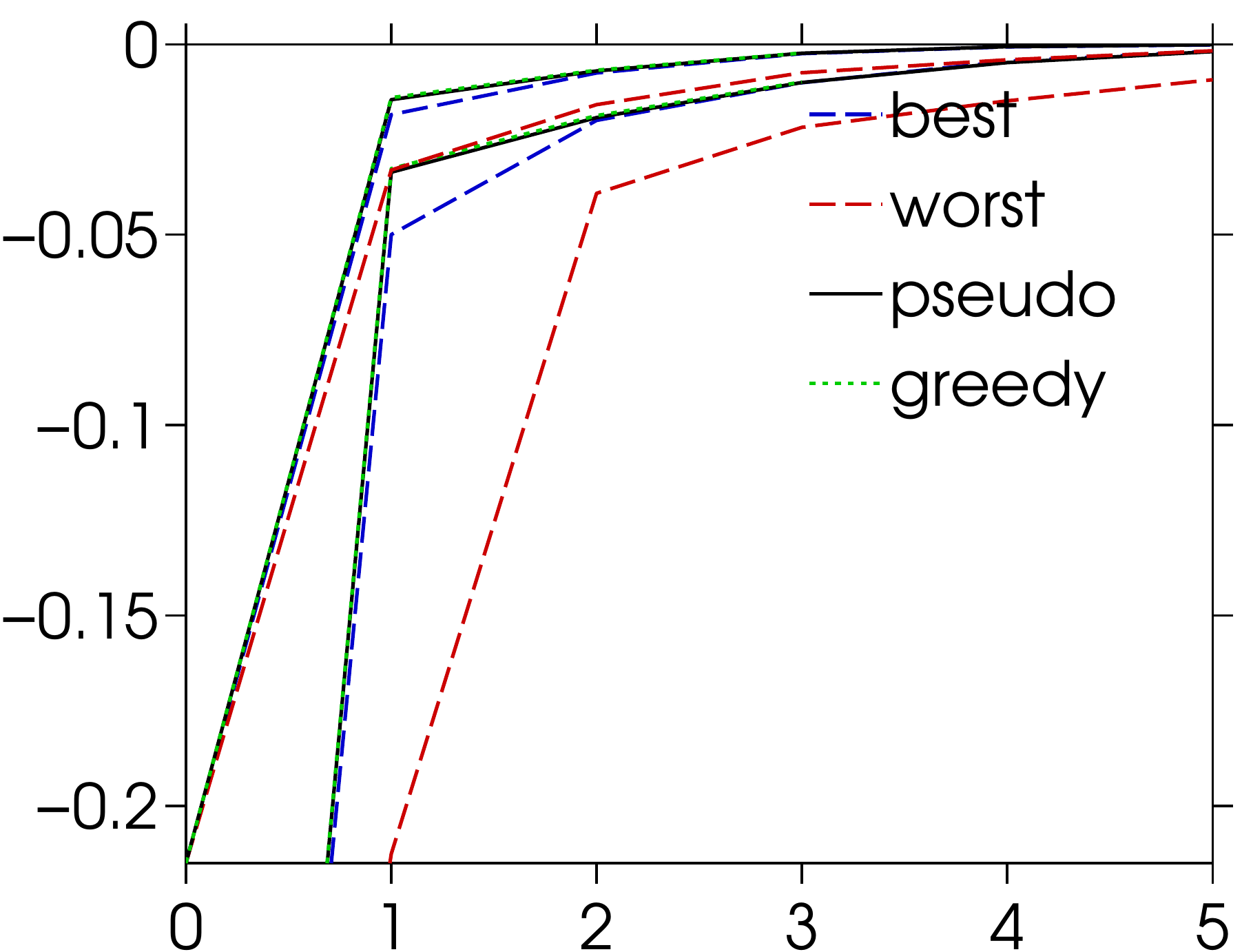} &
\begin{sideways} {\small \- \; \; \;  complete $K_{10}$} \end{sideways} \\
& {\small grids} & {\small random 4-regular} & {\small random Erd\"{o}s-Renyi} & & {\small complete graph} 
\end{tabular}
\end{center}
\caption{\small Attractive, zoomed around Bethe $[0,2]$
}
%\label{fig:exp3}
\end{figure}

\begin{figure}
\begin{center}
\setlength\tabcolsep{1pt}
\begin{tabular}{ccccccc}
\begin{sideways} {\small \- \; \qquad large (81)} \end{sideways} &
\includegraphics[width=.24\linewidth]{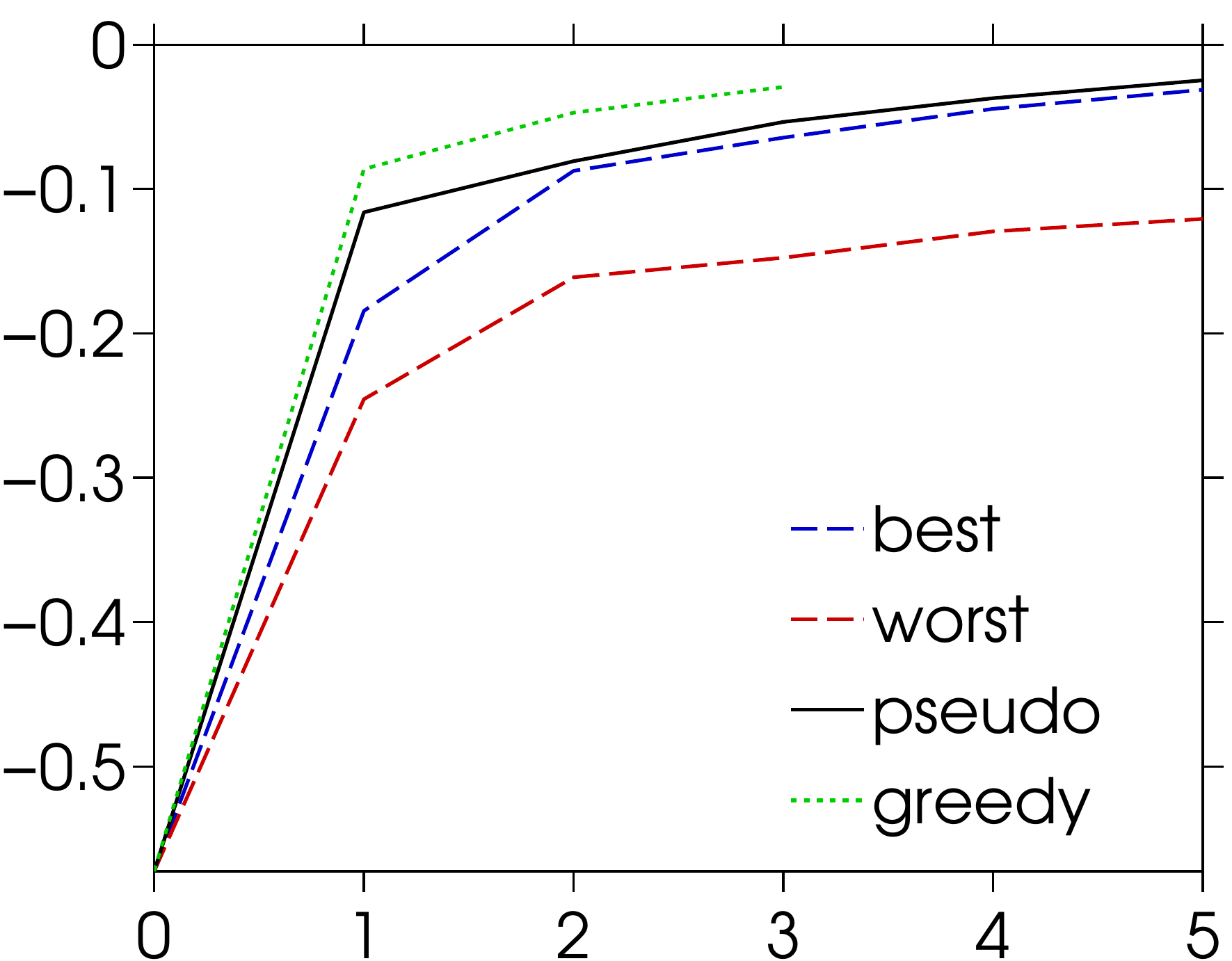} & 
\includegraphics[width=.24\linewidth]{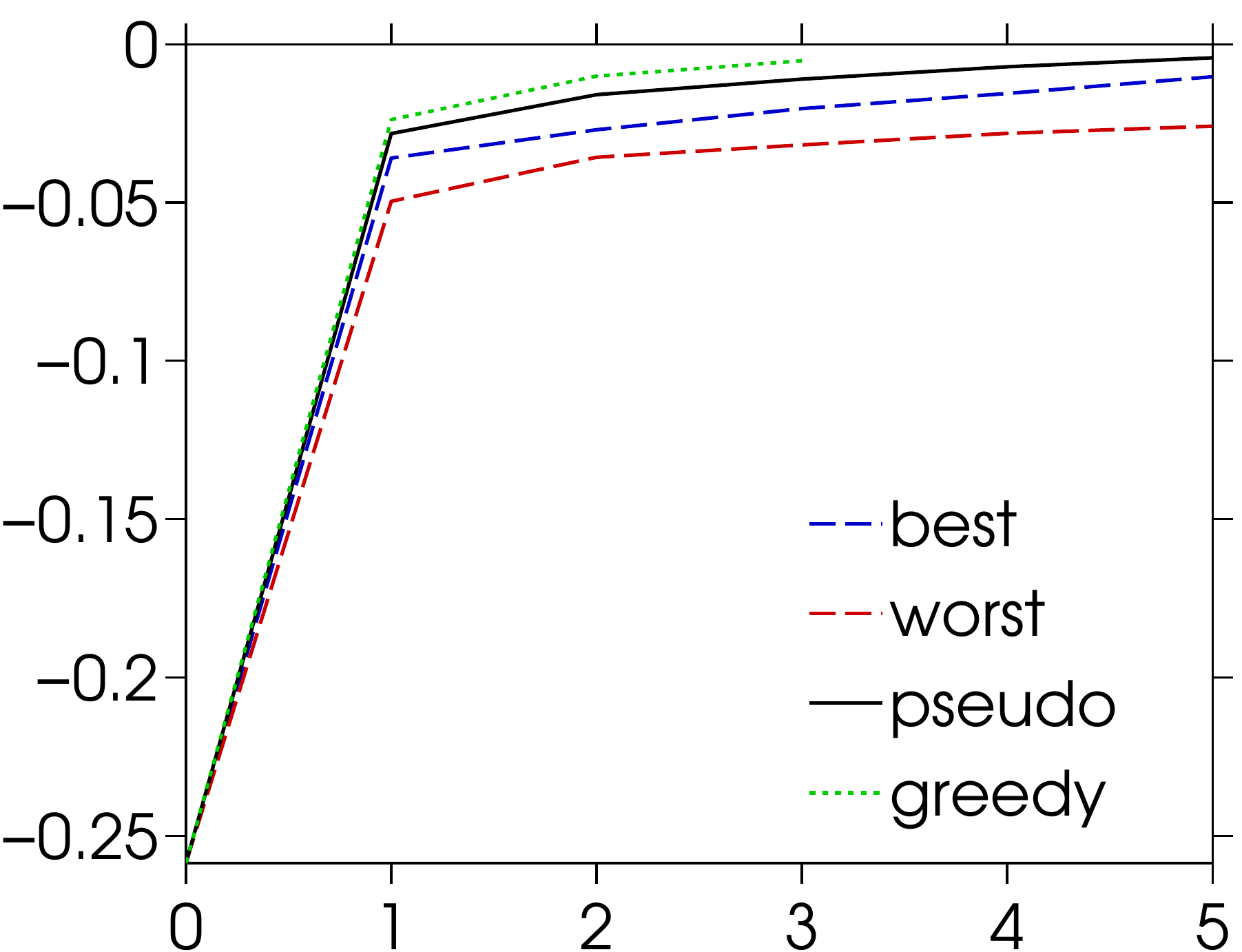} &
\includegraphics[width=.24\linewidth]{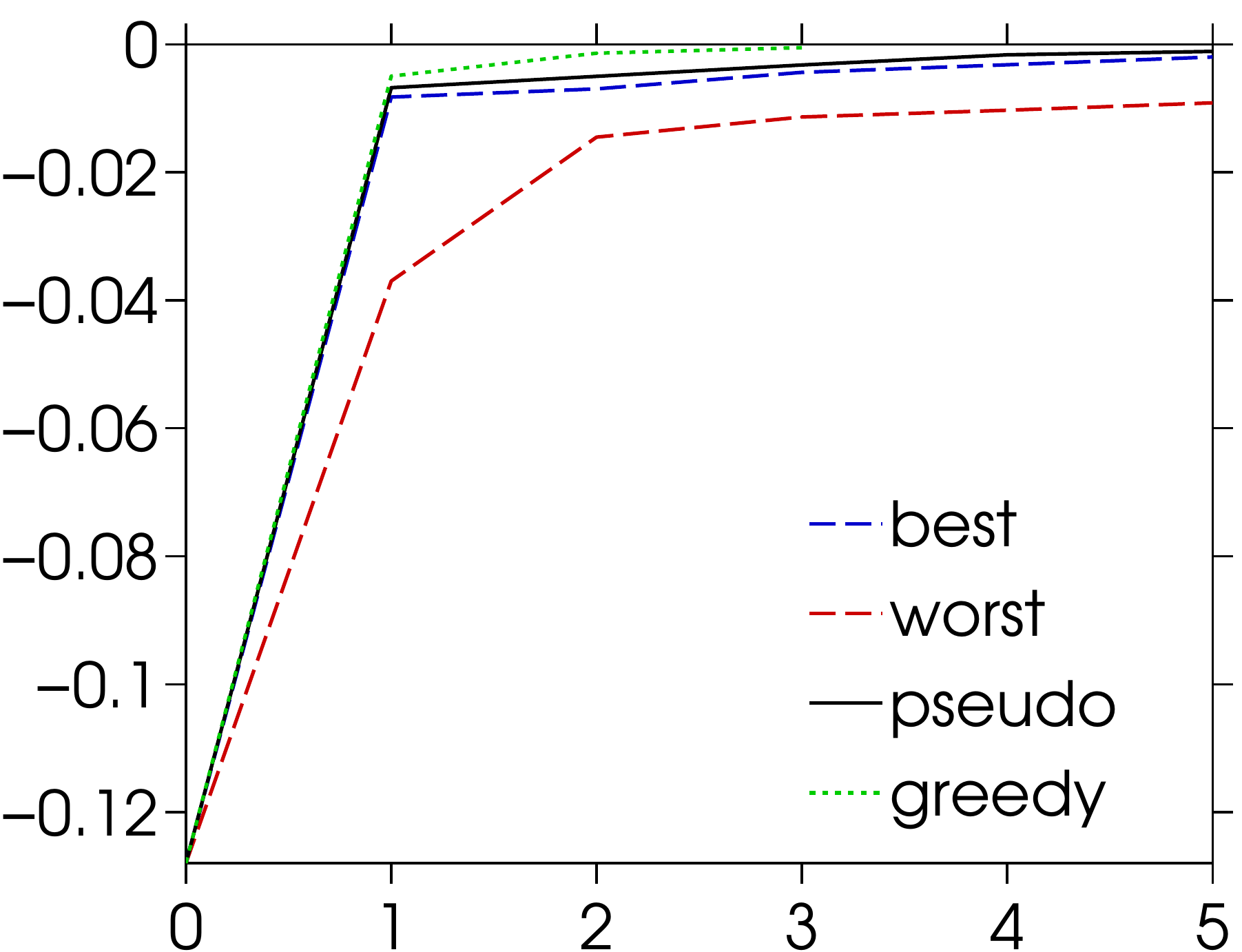} 
\\
\begin{sideways} {\small \- \; \; \quad medium (49)} \end{sideways} &
\includegraphics[width=.24\linewidth]{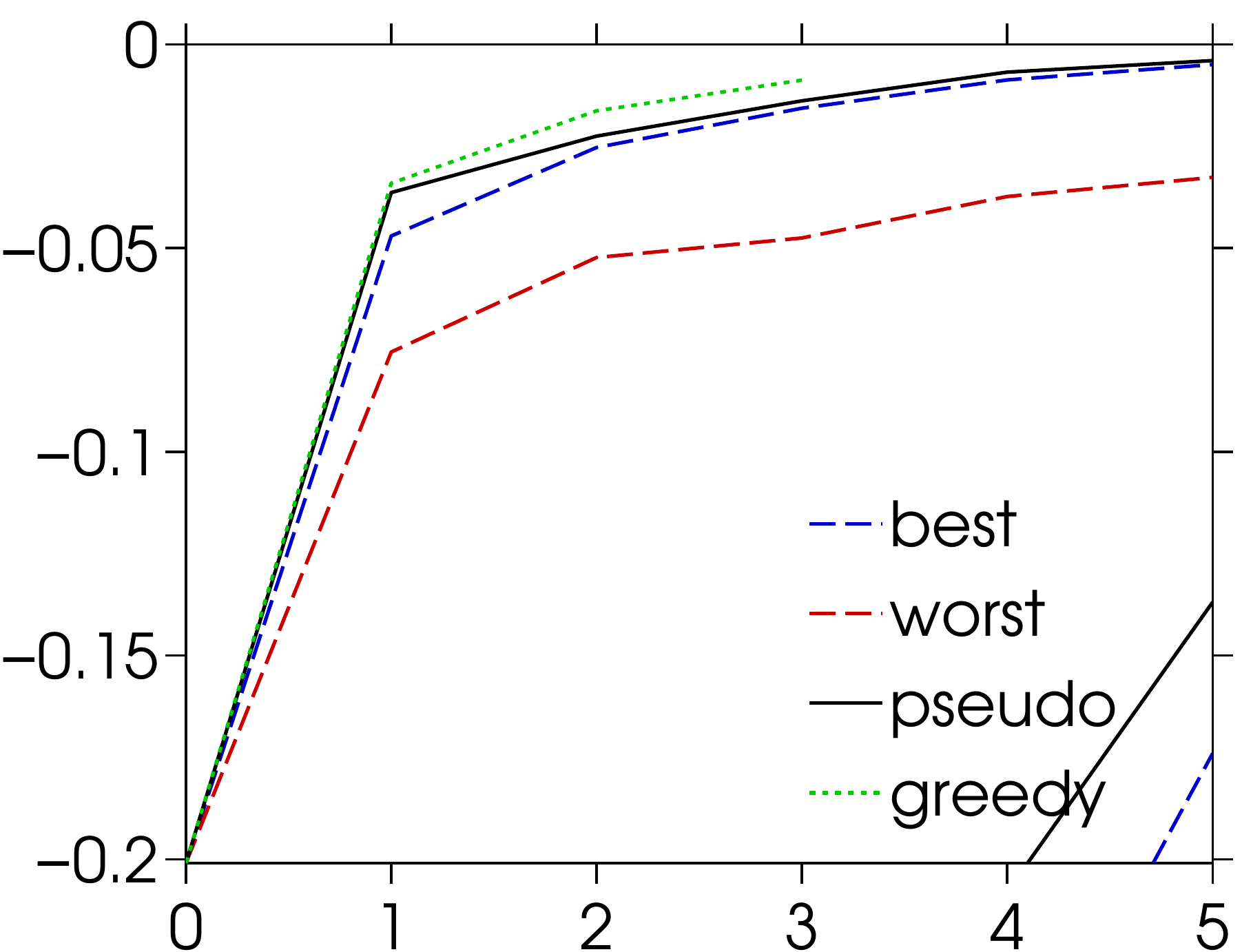} & 
\includegraphics[width=.24\linewidth]{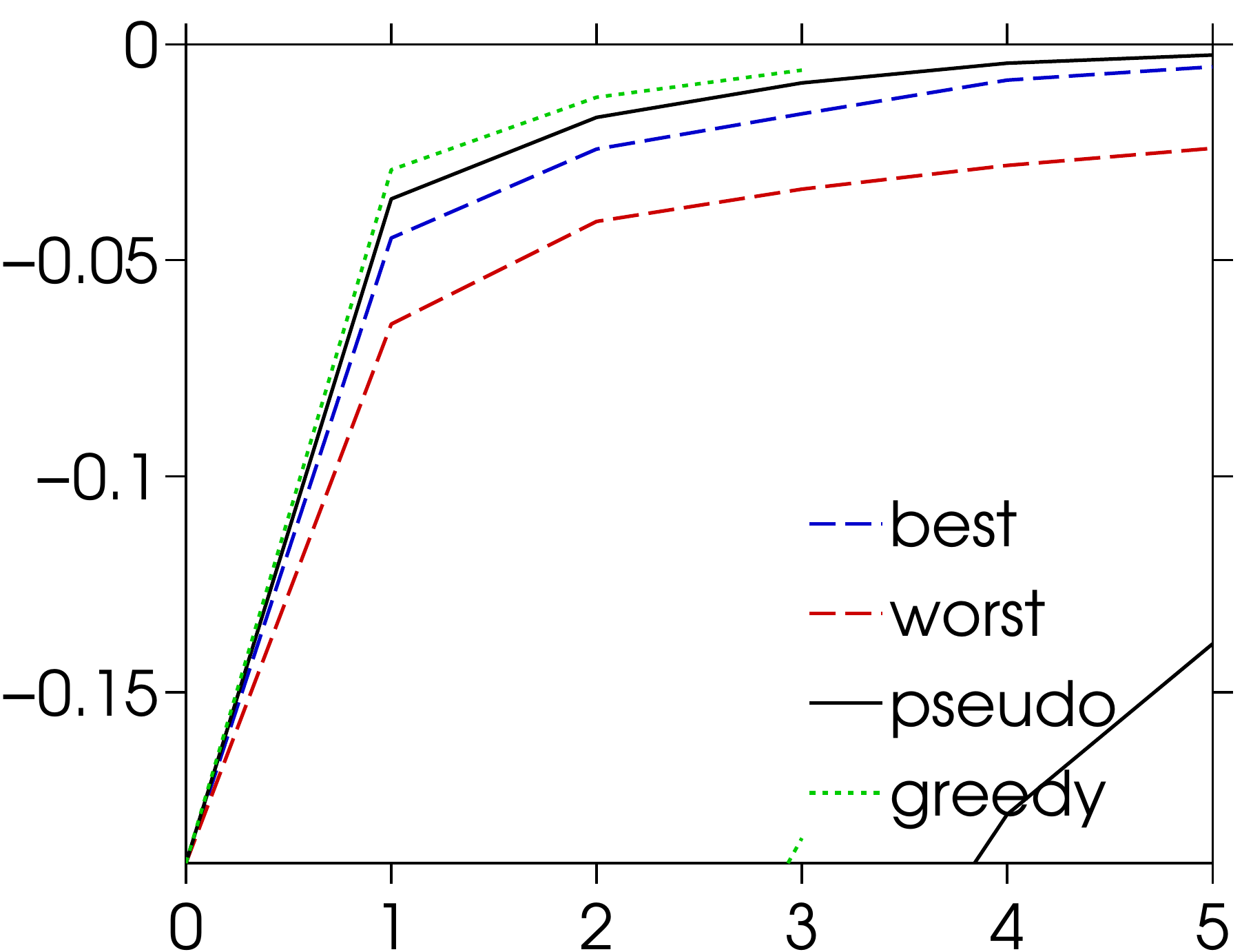} &
\includegraphics[width=.24\linewidth]{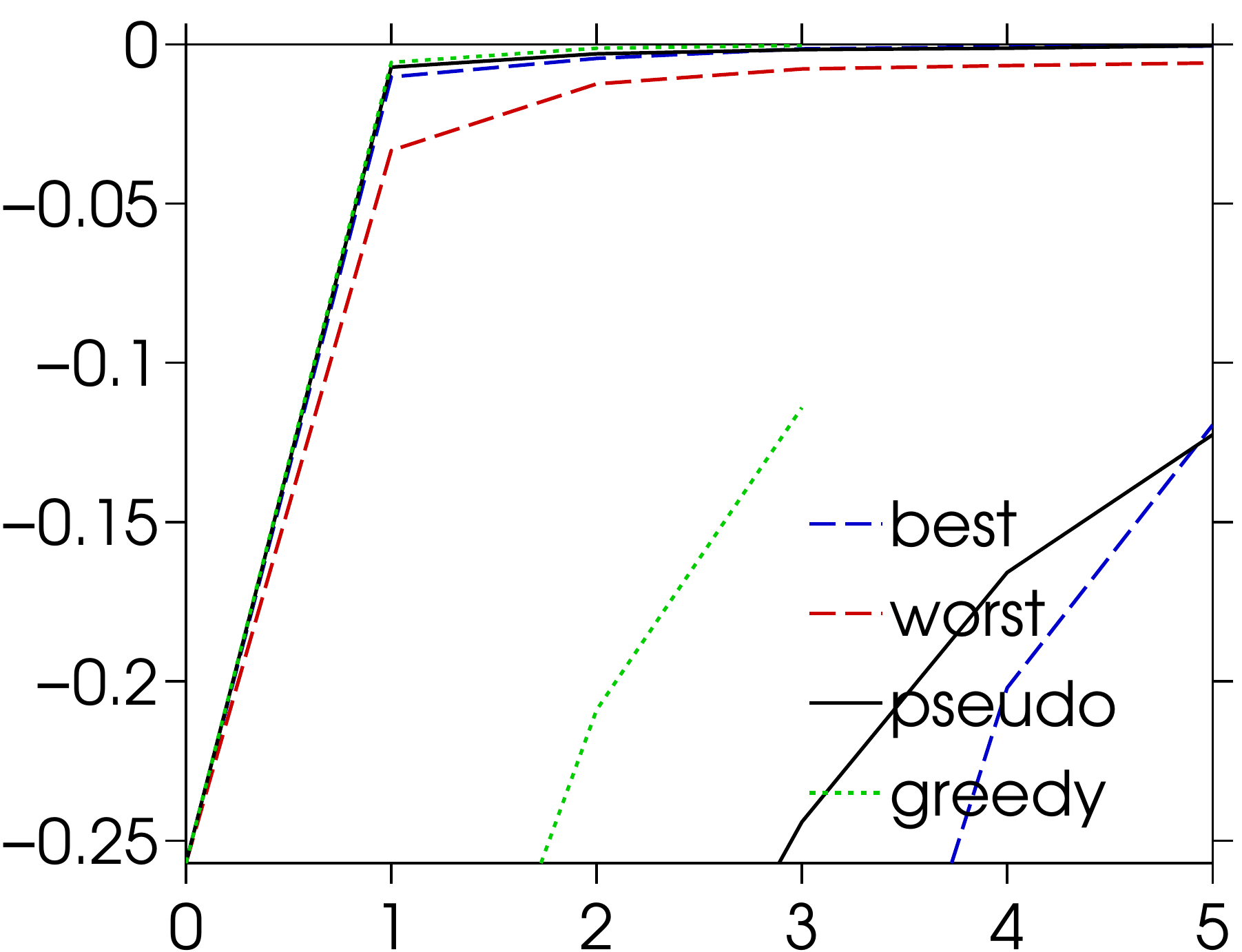} &
\- \; &
\includegraphics[width=.24\linewidth]{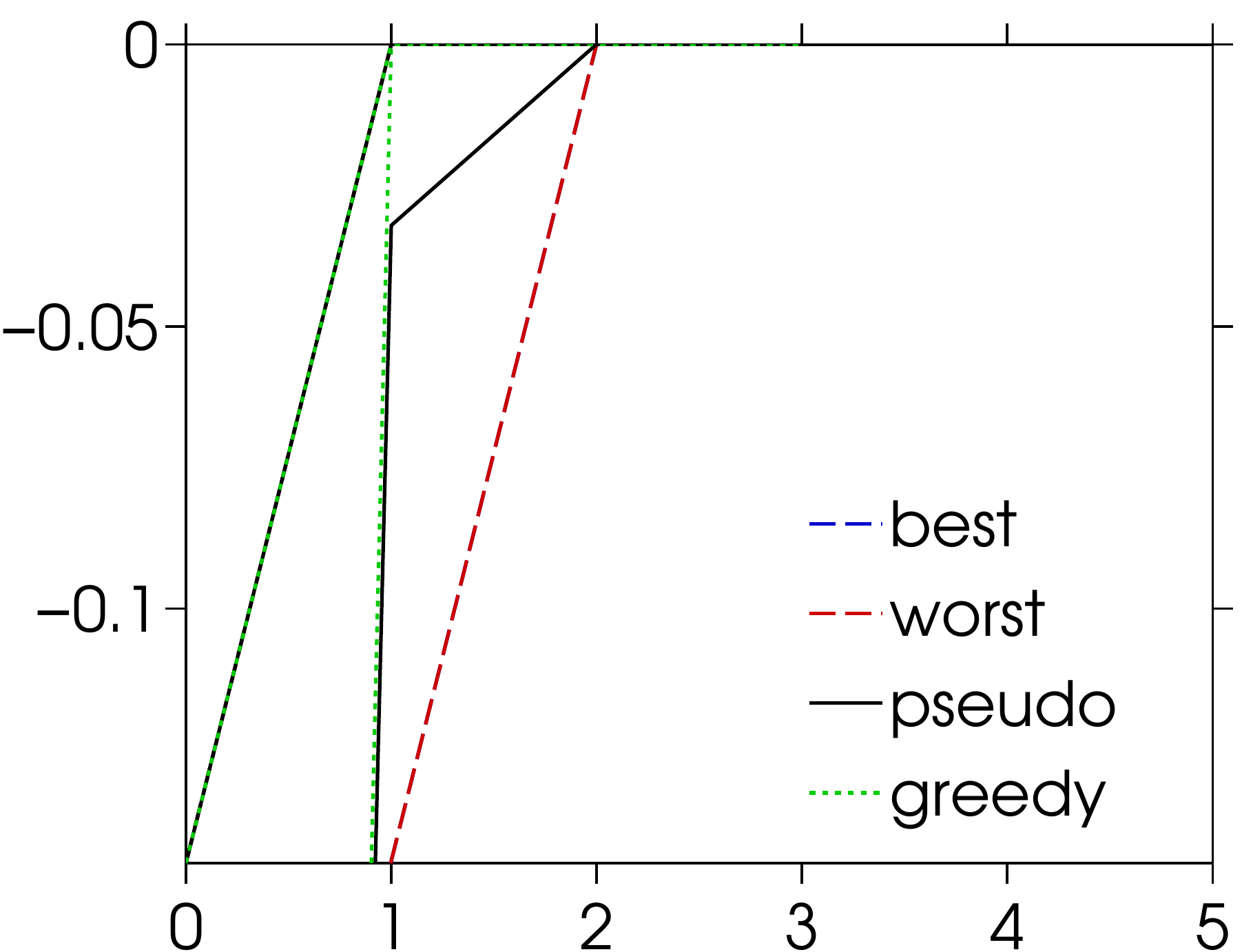} &
\begin{sideways} {\small \- \; \; \;  complete $K_{15}$} \end{sideways} \\
\begin{sideways} {\small \- \; \qquad small (25)} \end{sideways} &
\includegraphics[width=.24\linewidth]{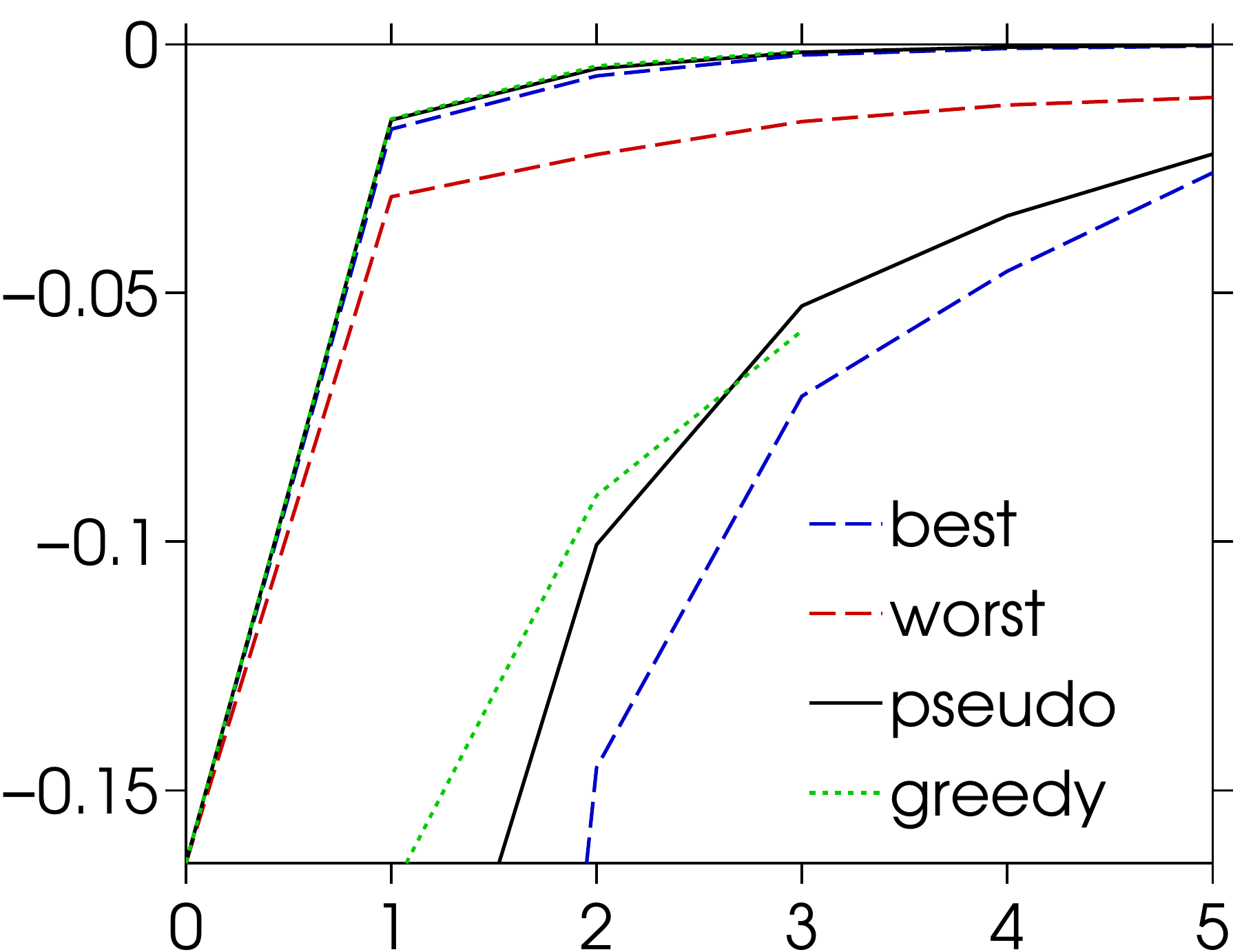} & 
\includegraphics[width=.24\linewidth]{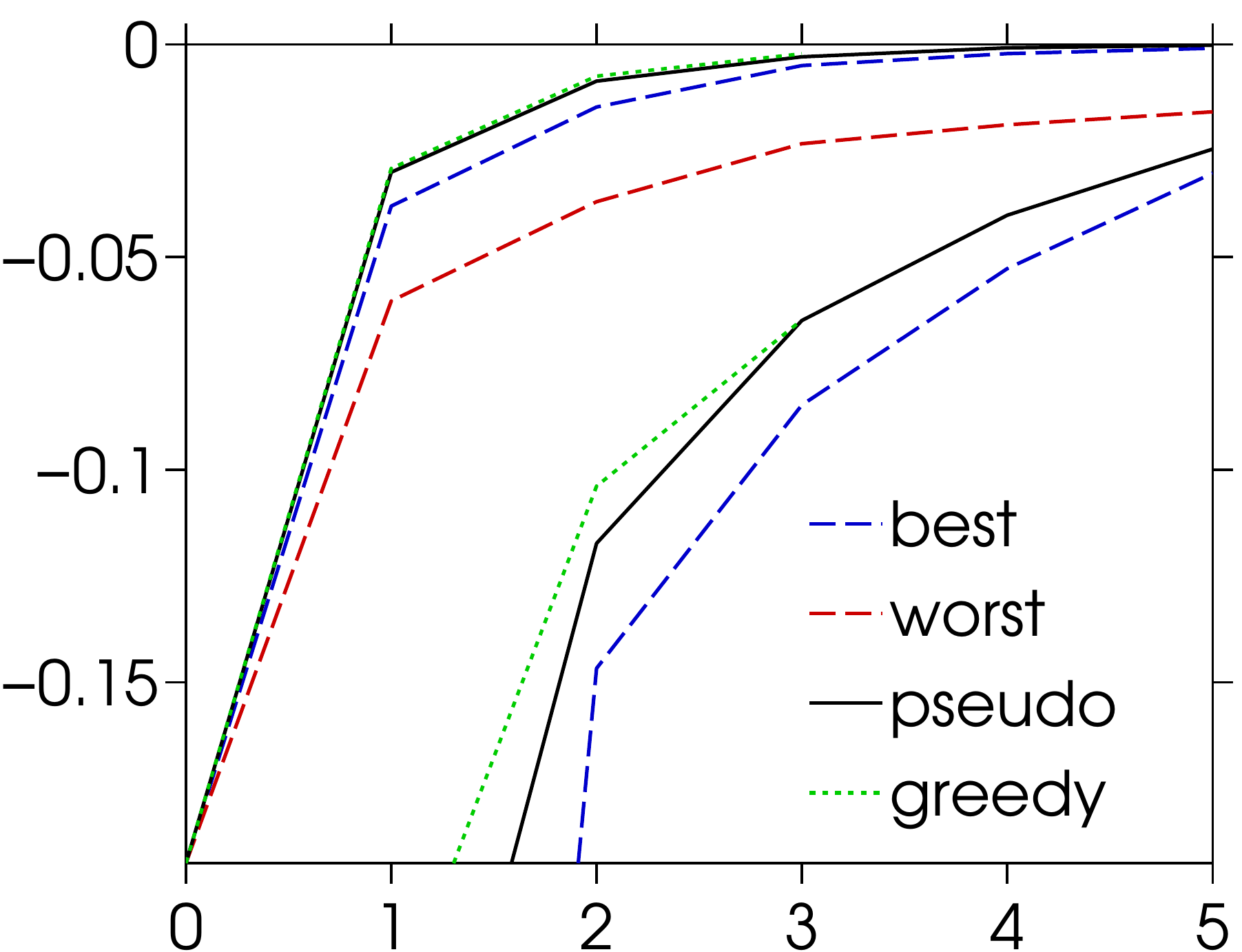} &
\includegraphics[width=.24\linewidth]{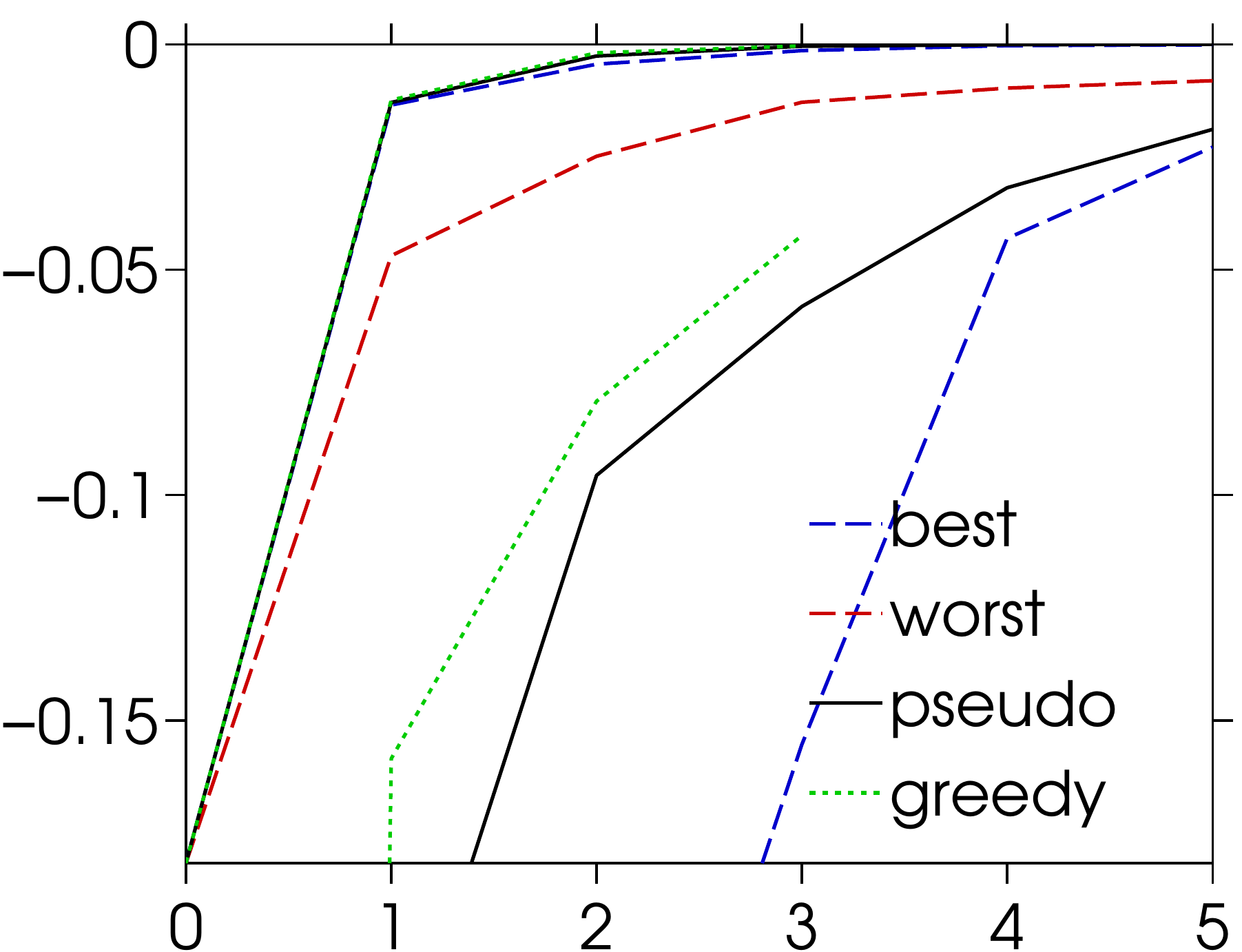} &
\- \; &
\includegraphics[width=.24\linewidth]{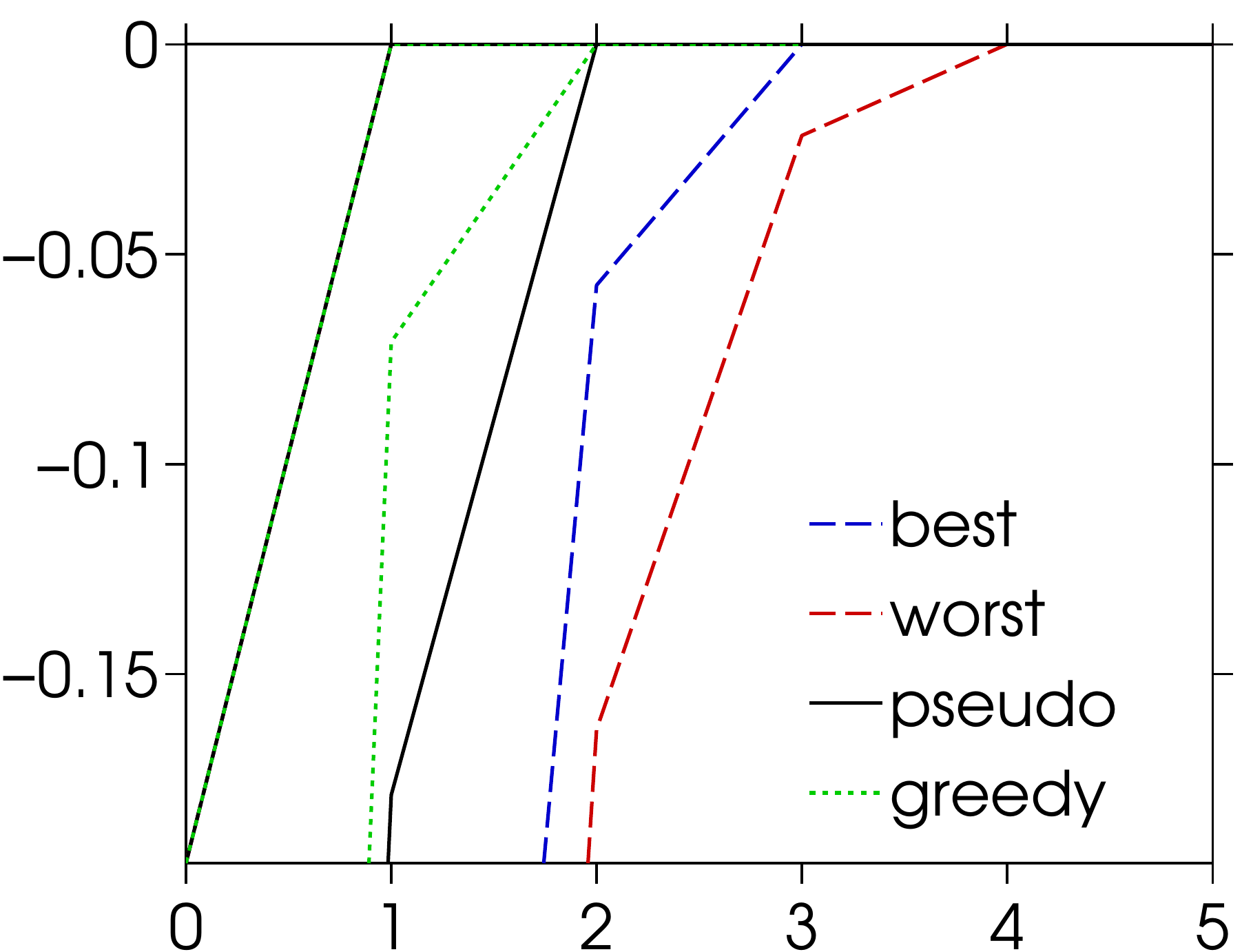} &
\begin{sideways} {\small \- \; \; \;  complete $K_{10}$} \end{sideways} \\
& {\small grids} & {\small random 4-regular} & {\small random Erd\"{o}s-Renyi} & & {\small complete graph} 
\end{tabular}
\end{center}
\caption{\small Attractive, zoomed around Bethe $[0,6]$
}
%\label{fig:exp4}
\end{figure}

%%%%%%%%%%%%%%%%%%%%%%%%%%%%%%%%%%%%%%
%TIMINGS
%%%%%%%%%%%%%%%%%%%%%%%%%%%%%%%%%%%%%%

\begin{figure}
\begin{center}
\setlength\tabcolsep{1pt}
\begin{tabular}{ccccccc}
\begin{sideways} {\small \- \; \qquad large (81)} \end{sideways} &
\includegraphics[width=.24\linewidth]{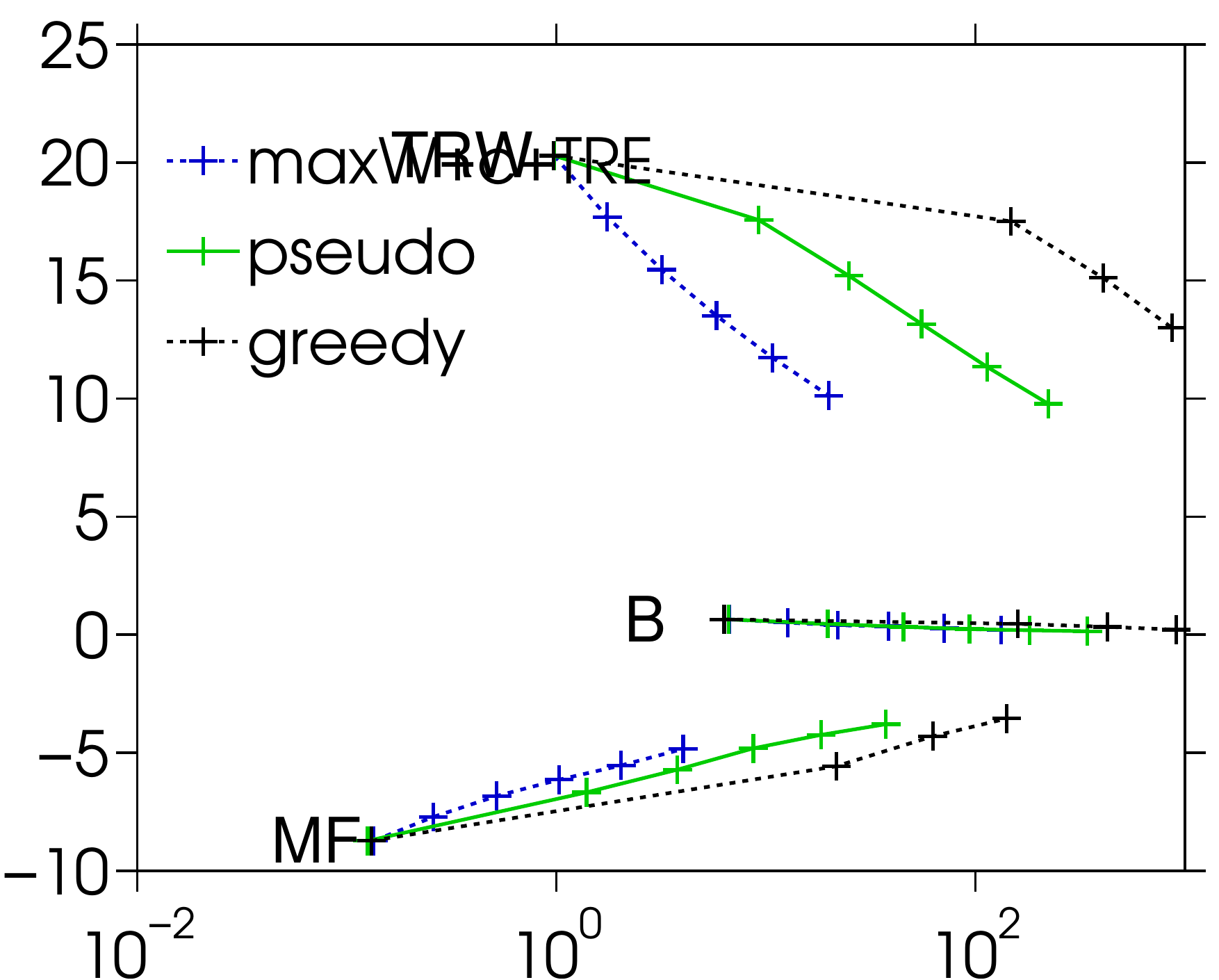} & 
\includegraphics[width=.24\linewidth]{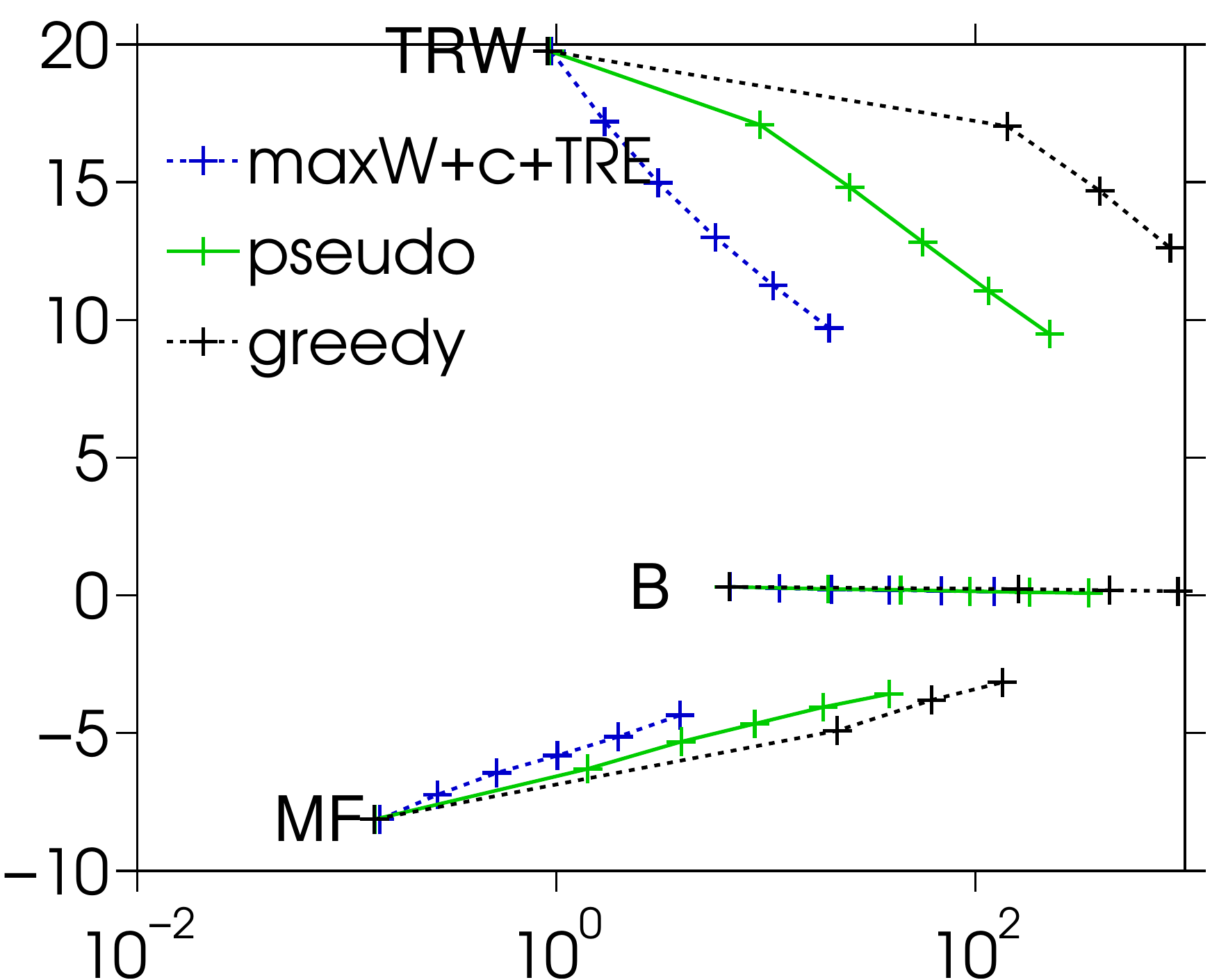} & 
\includegraphics[width=.24\linewidth]{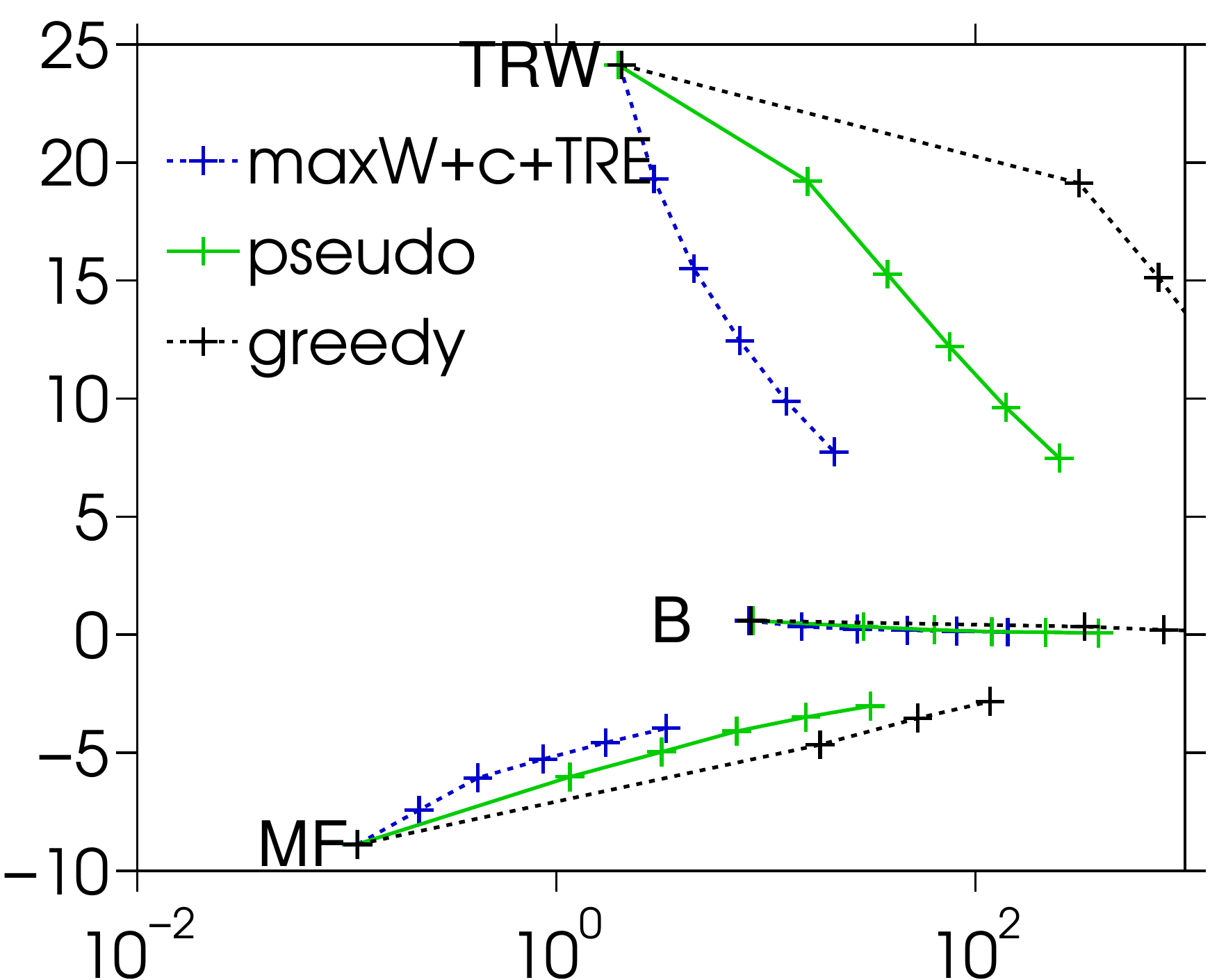} 
\\
\begin{sideways} {\small \- \; \; \quad medium (49)} \end{sideways} &
\includegraphics[width=.24\linewidth]{grid_+7_7=_+-2_2=_+-6_6=-time2-eps-converted-to.pdf} & 
\includegraphics[width=.24\linewidth]{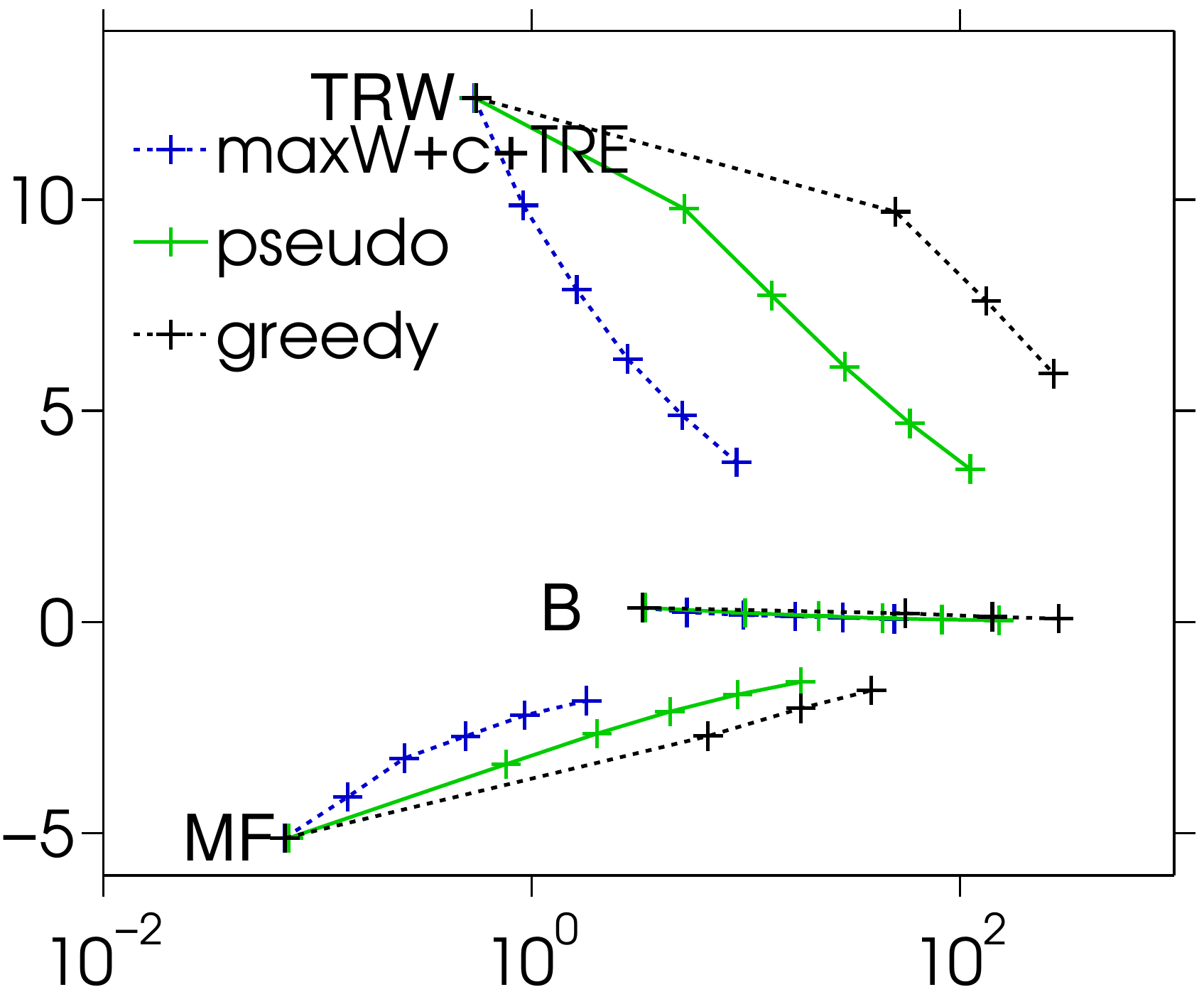} &
\includegraphics[width=.24\linewidth]{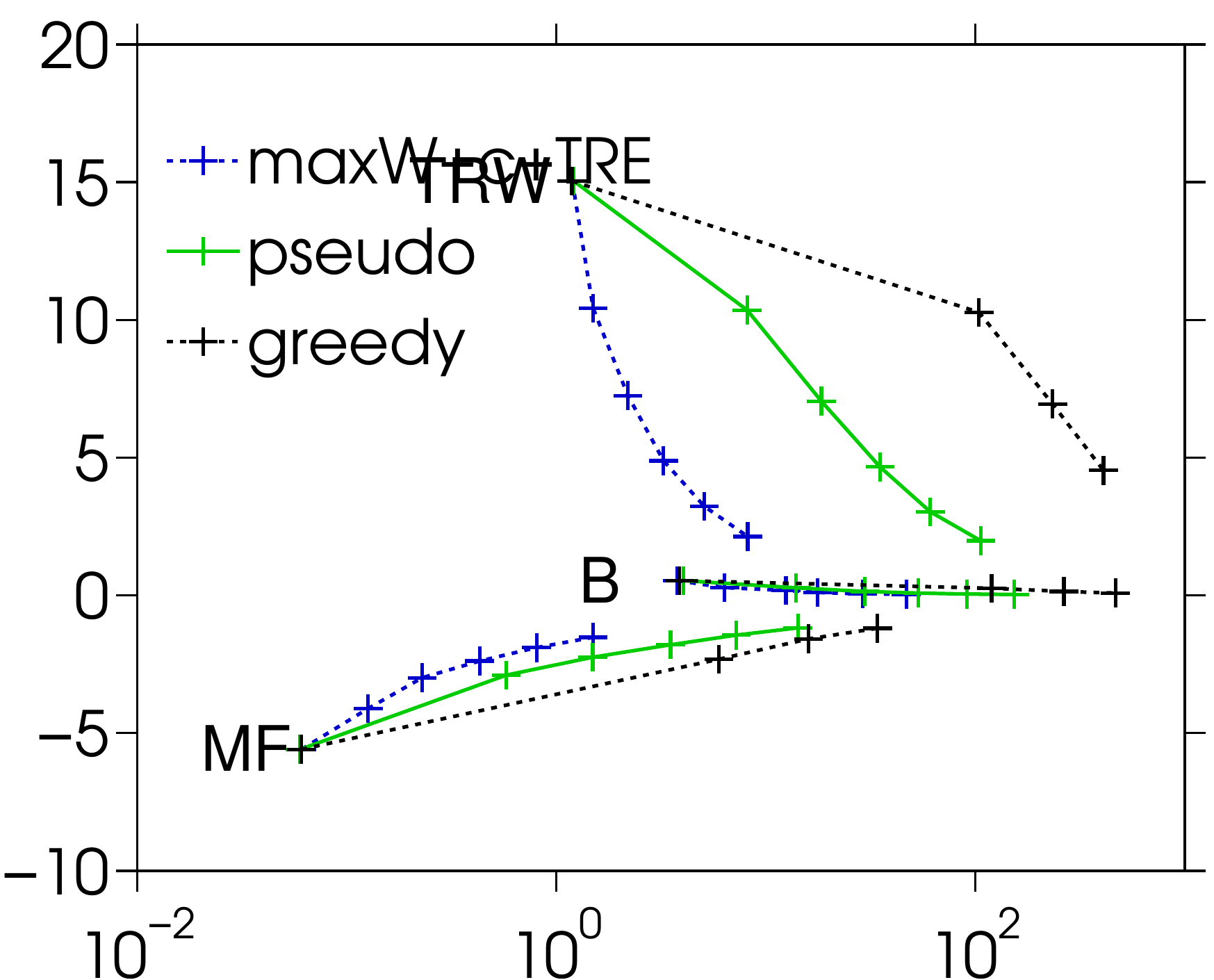} &
\- \; &
\includegraphics[width=.24\linewidth]{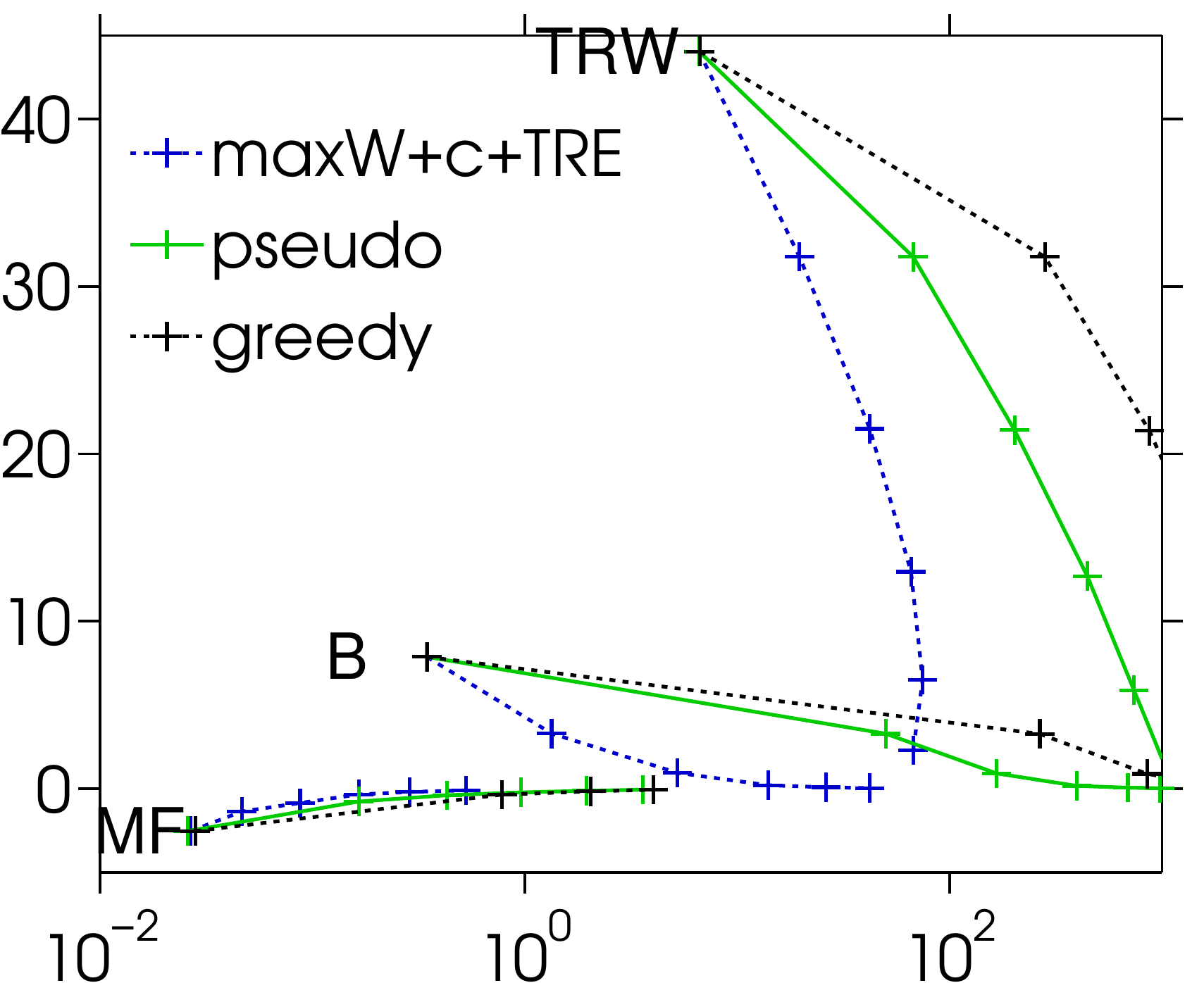} & 
\begin{sideways} {\small \- \;  \; \; complete $K_{15}$} \end{sideways} 
\\
\begin{sideways} {\small \- \; \qquad small (25)} \end{sideways} &
\includegraphics[width=.24\linewidth]{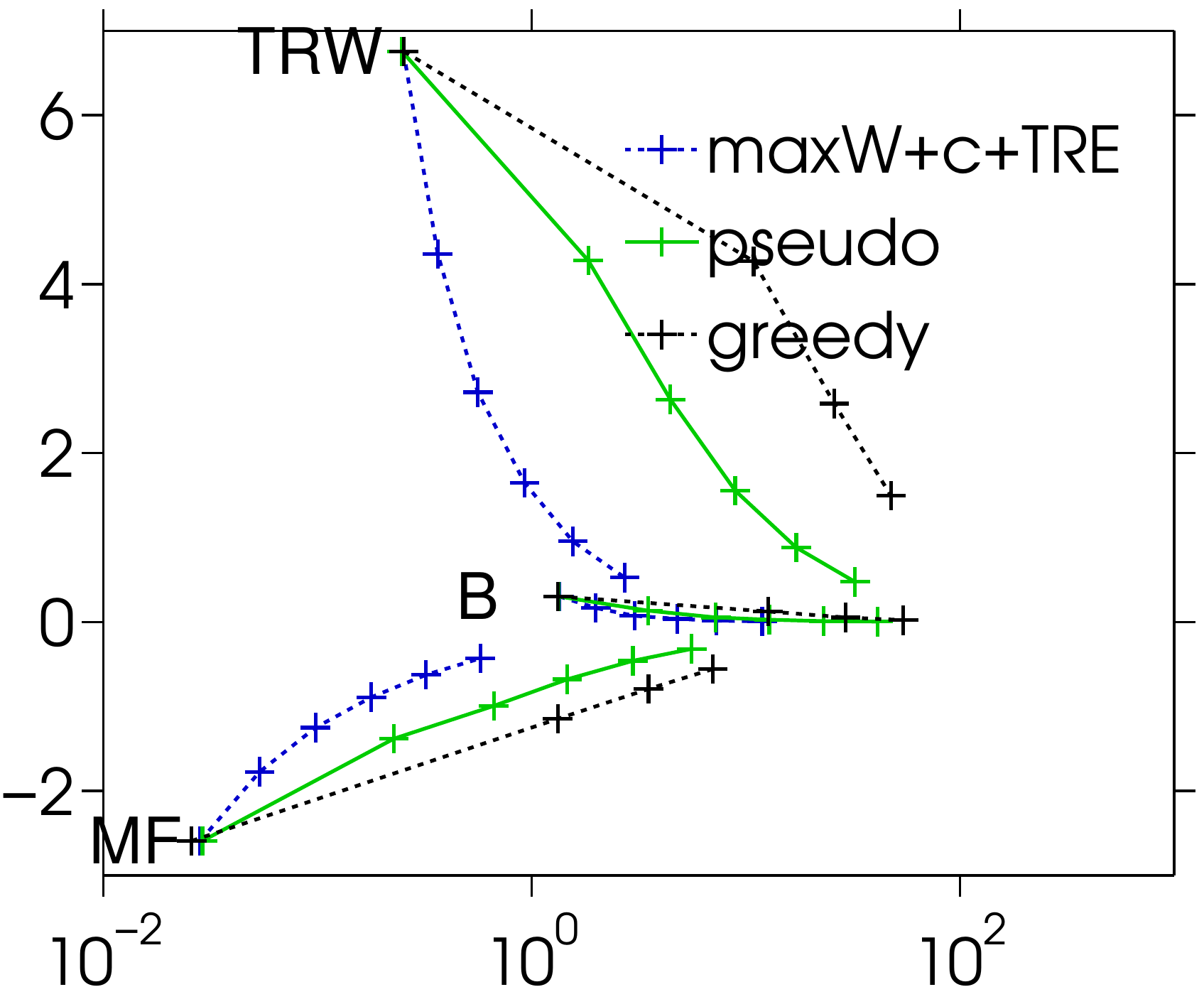} & 
\includegraphics[width=.24\linewidth]{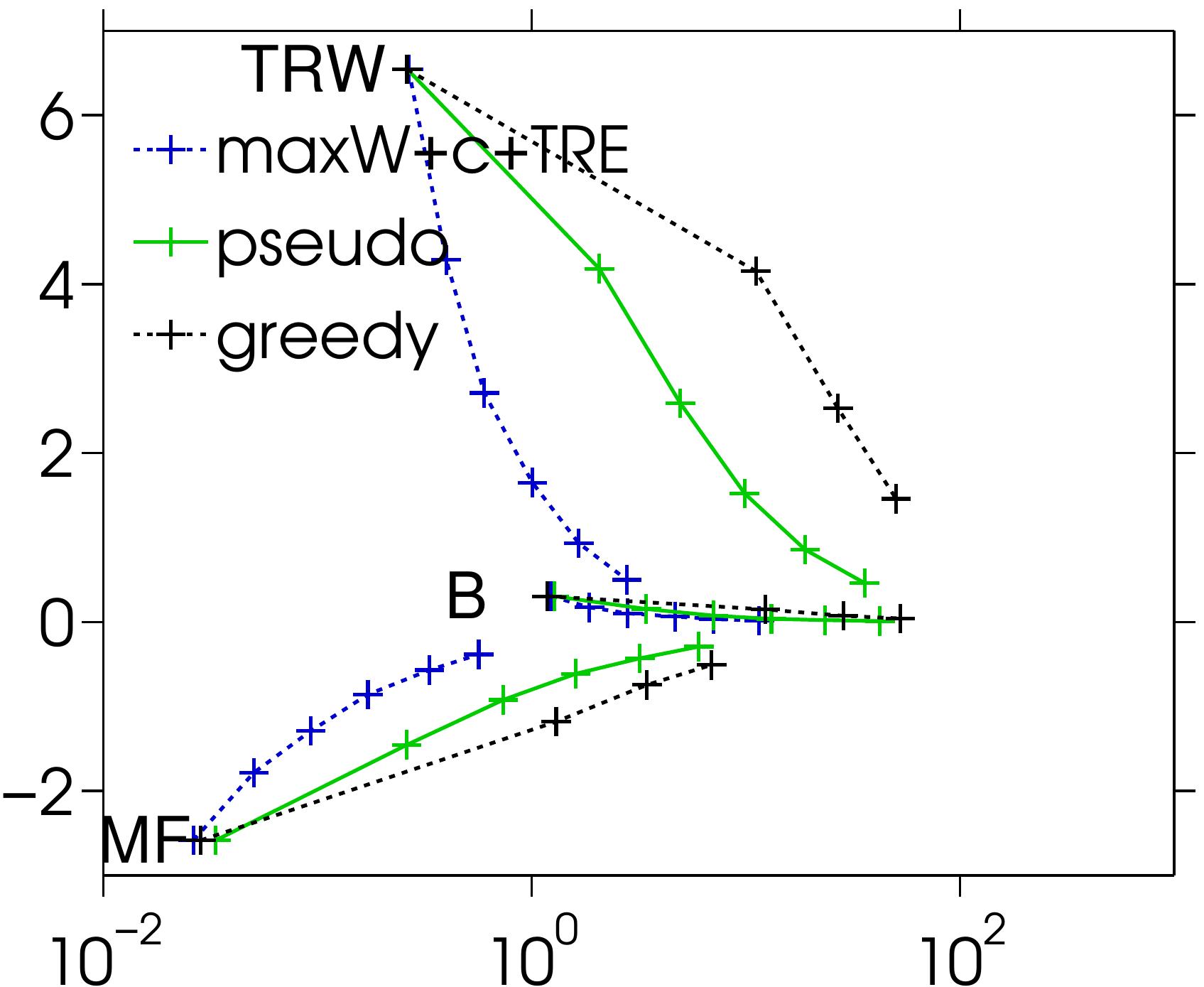} &
\includegraphics[width=.24\linewidth]{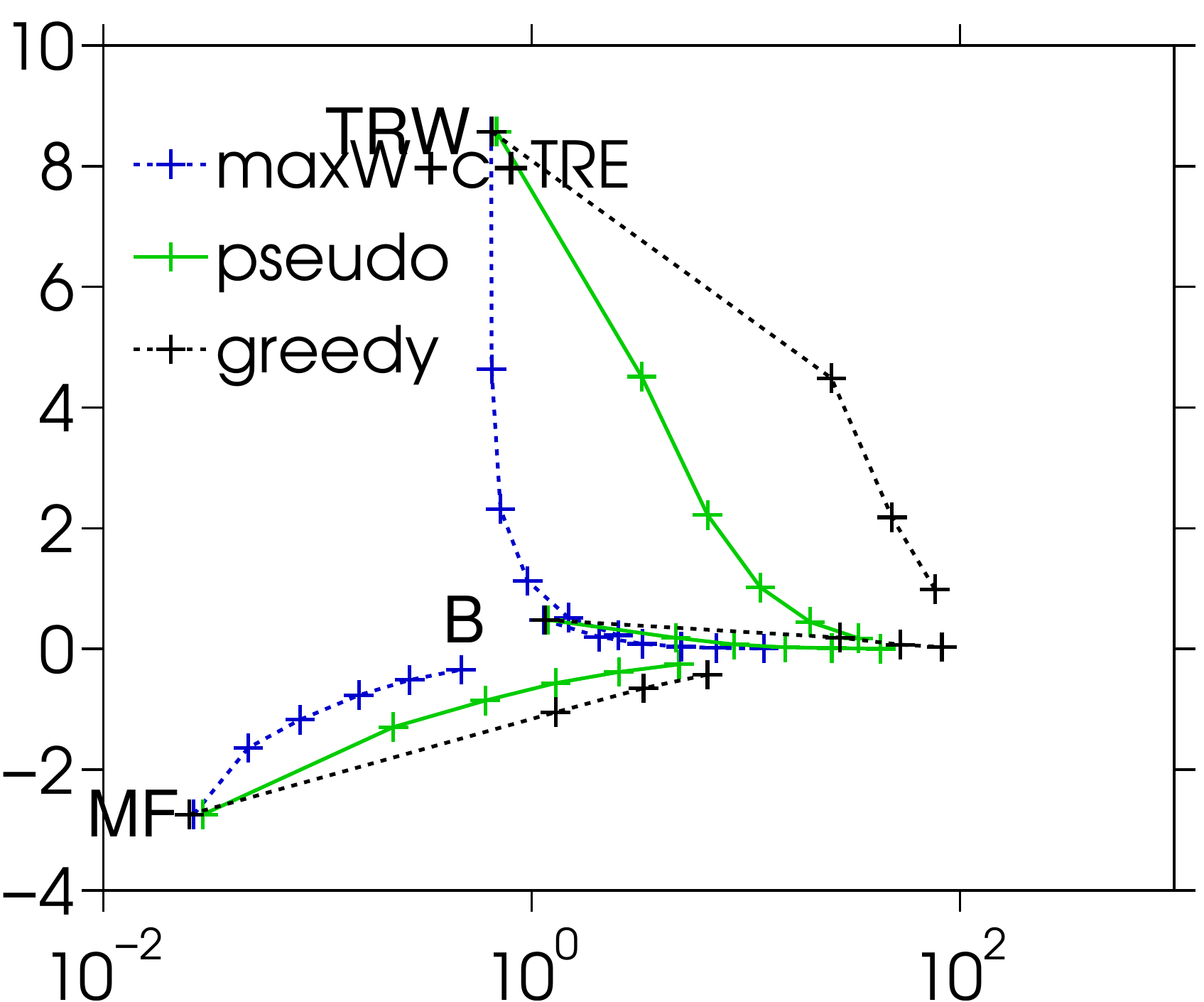} &
\- \; \; &
\includegraphics[width=.24\linewidth]{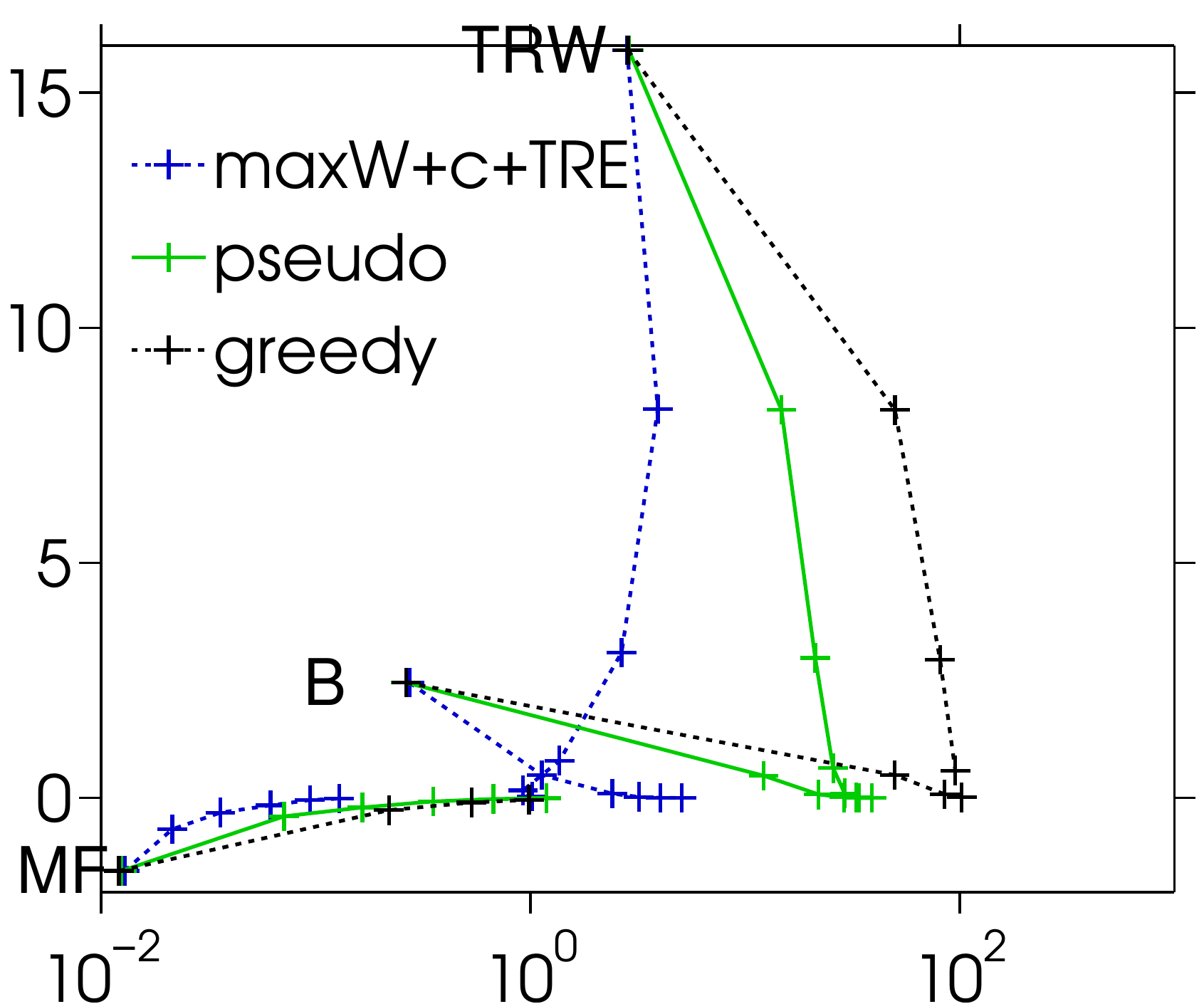} &
\begin{sideways} {\small \- \; \; \;  complete $K_{10}$} \end{sideways} 
\\
& {\small grids} & {\small random 4-regular} & {\small random Erd\"{o}s-Renyi} & & {\small complete graph} 
\end{tabular}
\end{center}
\caption{\small Mixed $[-6,6]$ timings (in secs, $\log$ scale, these give an overall sense but may be sensitive to implementation details and convergence thresholds)
}
\label{fig:weird}\label{fig:firsttime}
\end{figure}

\begin{figure}
\begin{center}
\setlength\tabcolsep{1pt}
\begin{tabular}{ccccccc}
\begin{sideways} {\small \- \; \qquad large (81)} \end{sideways} &
\includegraphics[width=.24\linewidth]{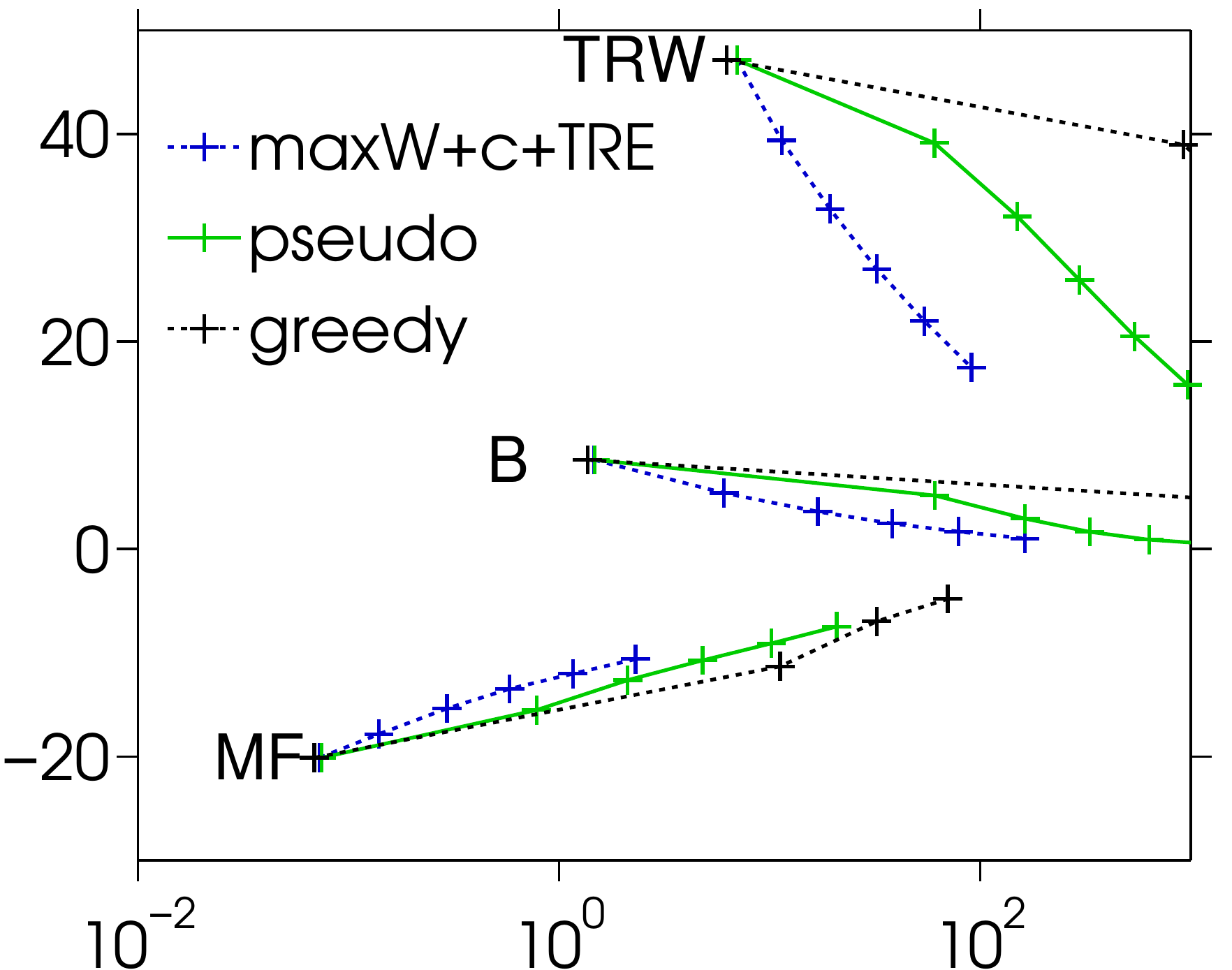} & 
\includegraphics[width=.24\linewidth]{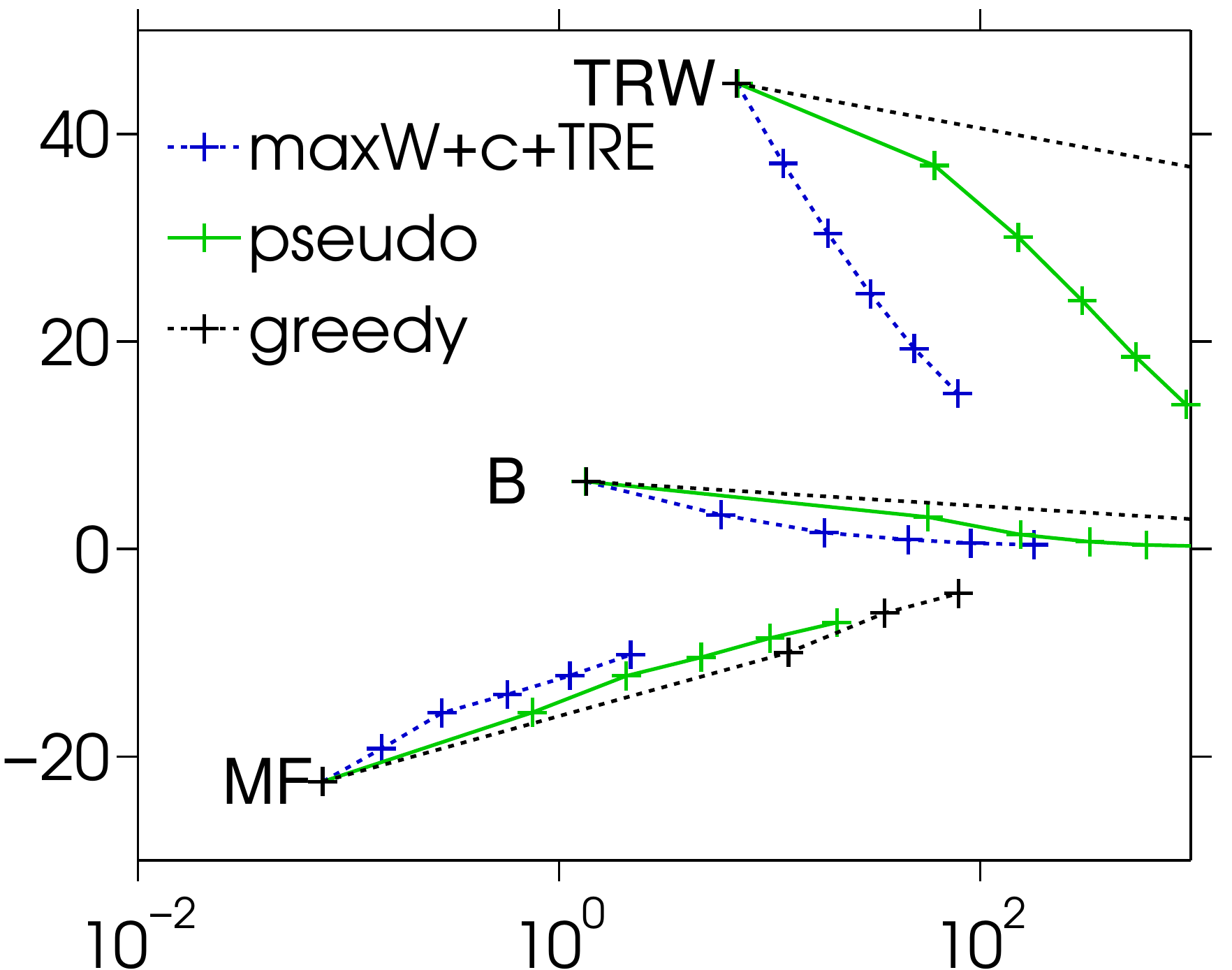} &
\includegraphics[width=.24\linewidth]{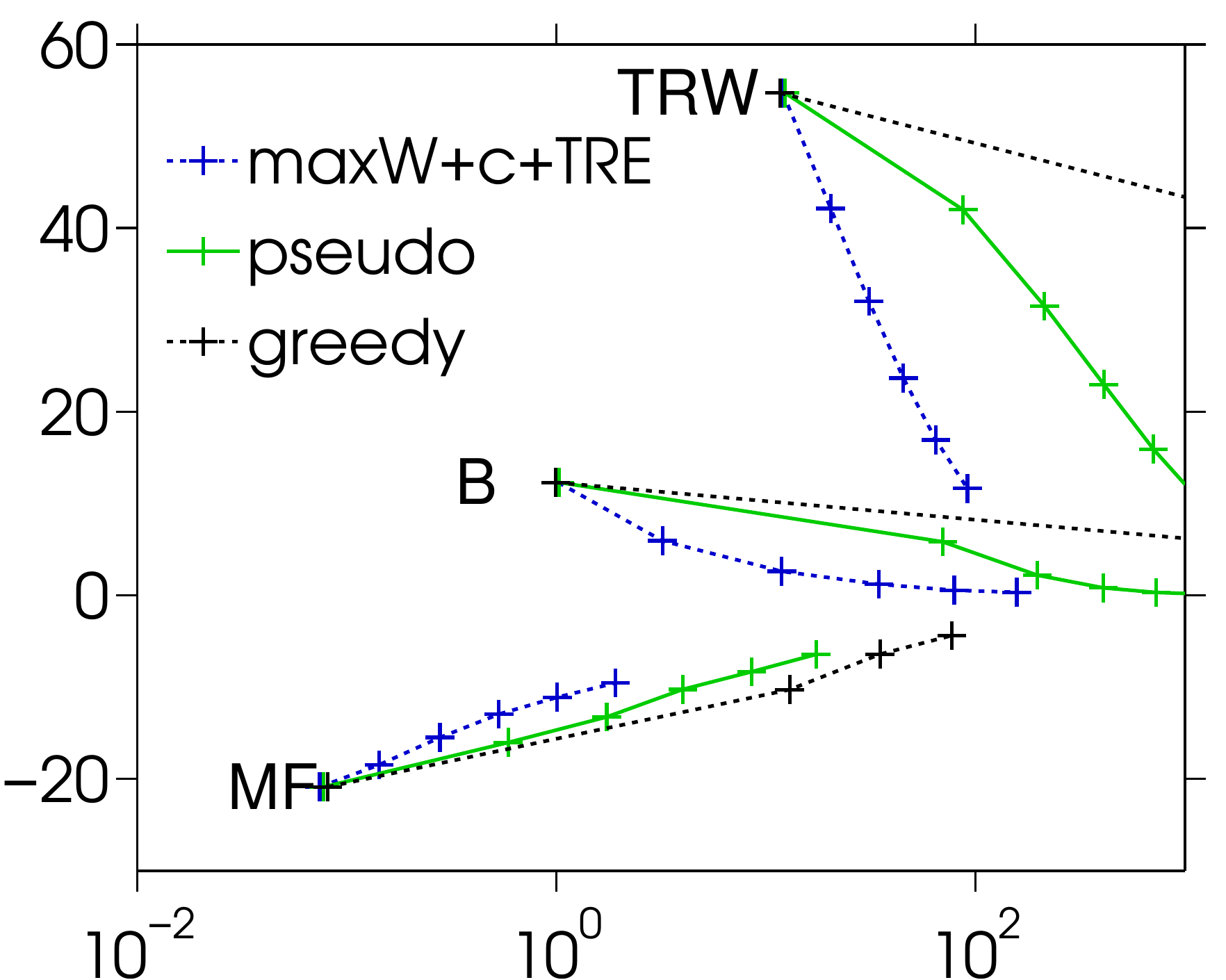} 
\\
\begin{sideways} {\small \- \; \; \quad medium (49)} \end{sideways} &
\includegraphics[width=.24\linewidth]{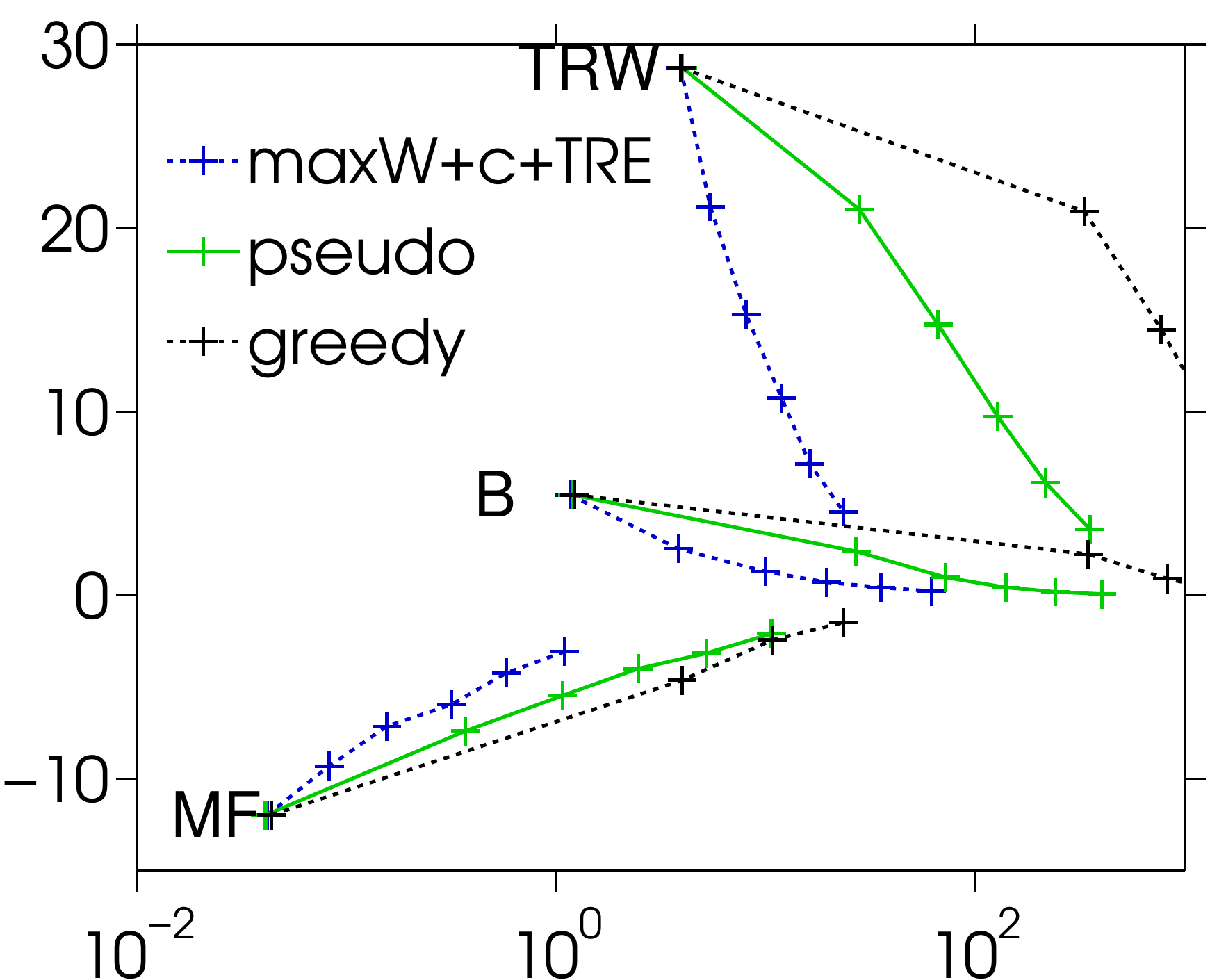} & 
\includegraphics[width=.24\linewidth]{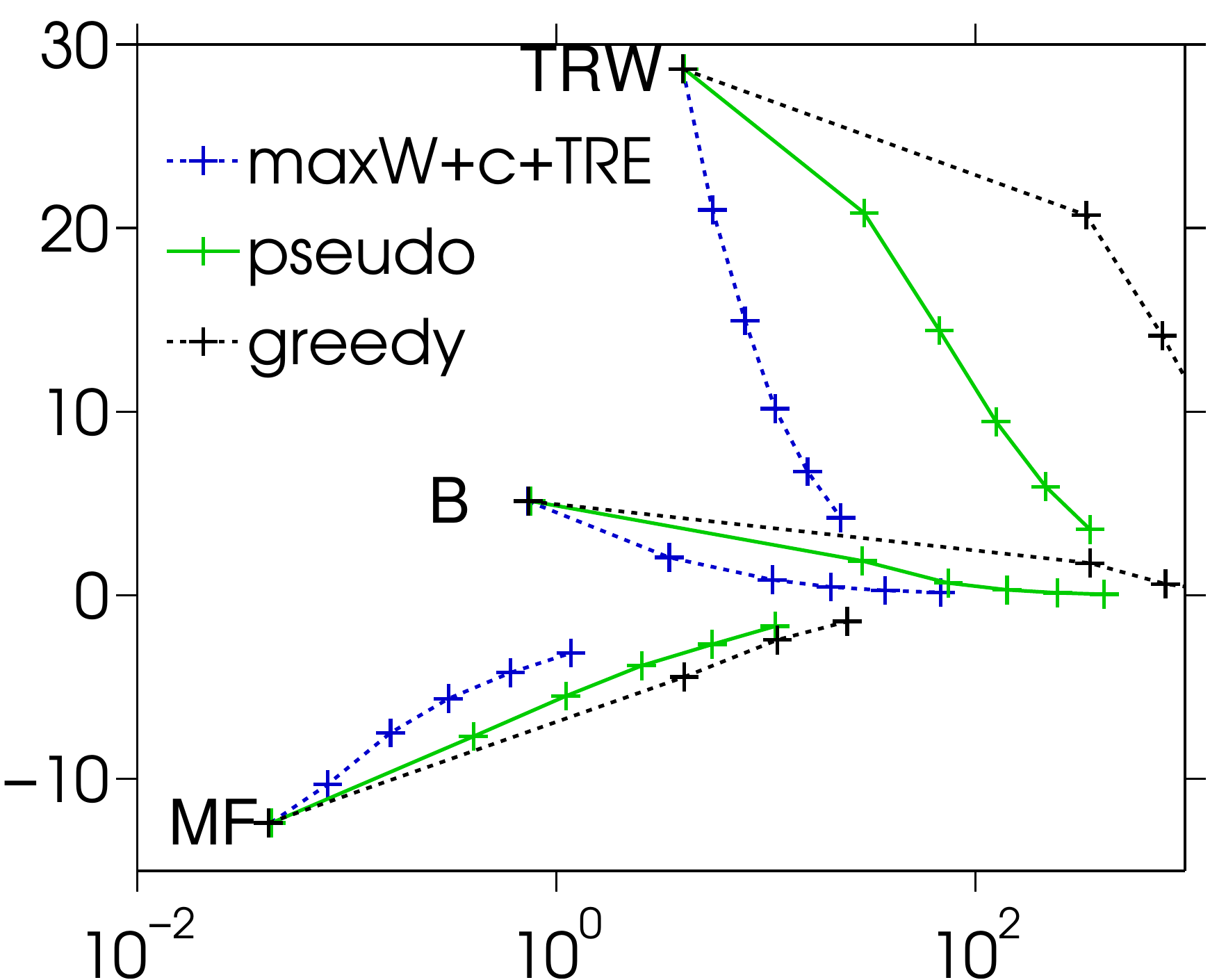} &
\includegraphics[width=.24\linewidth]{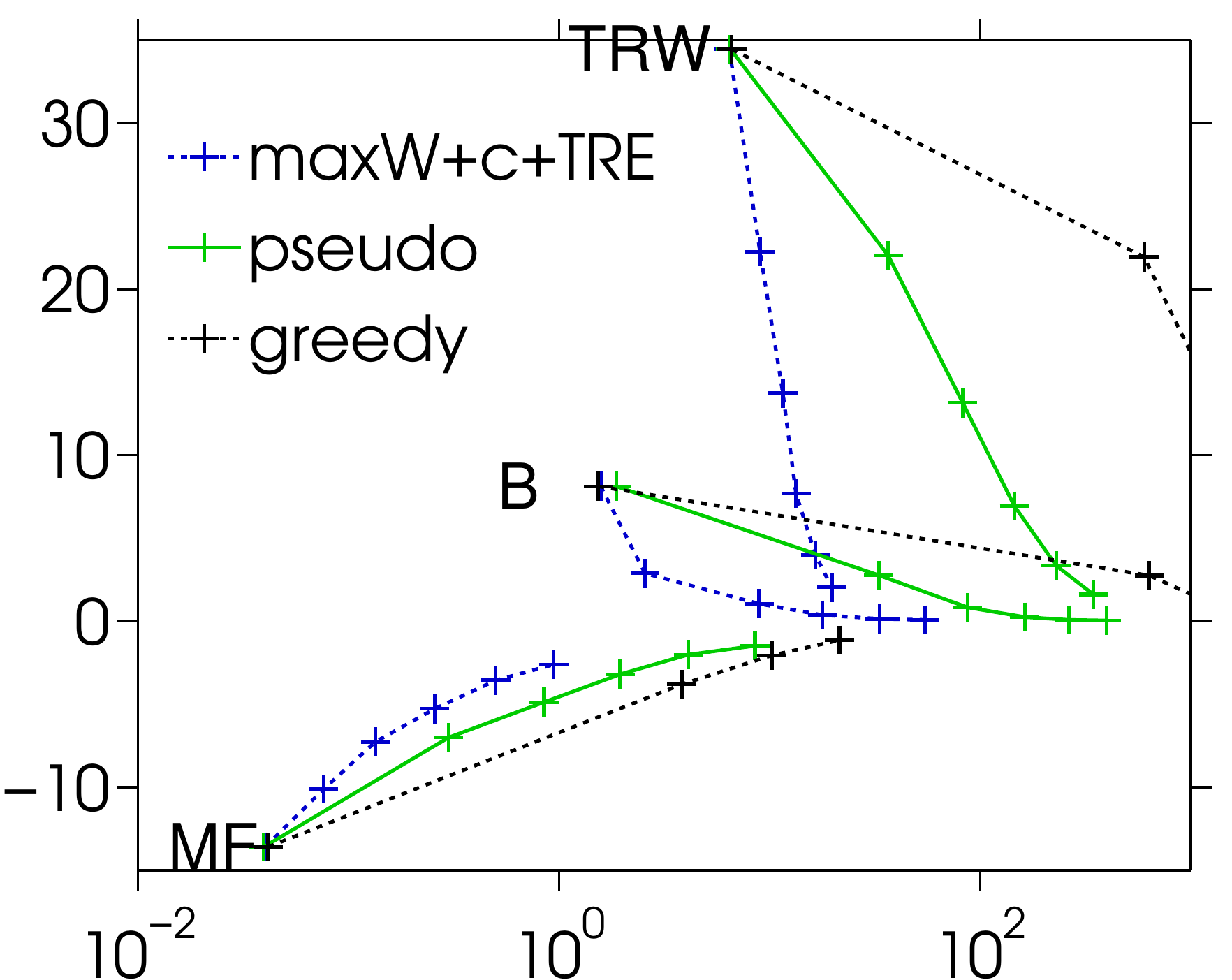} &
\- \; &
\includegraphics[width=.24\linewidth]{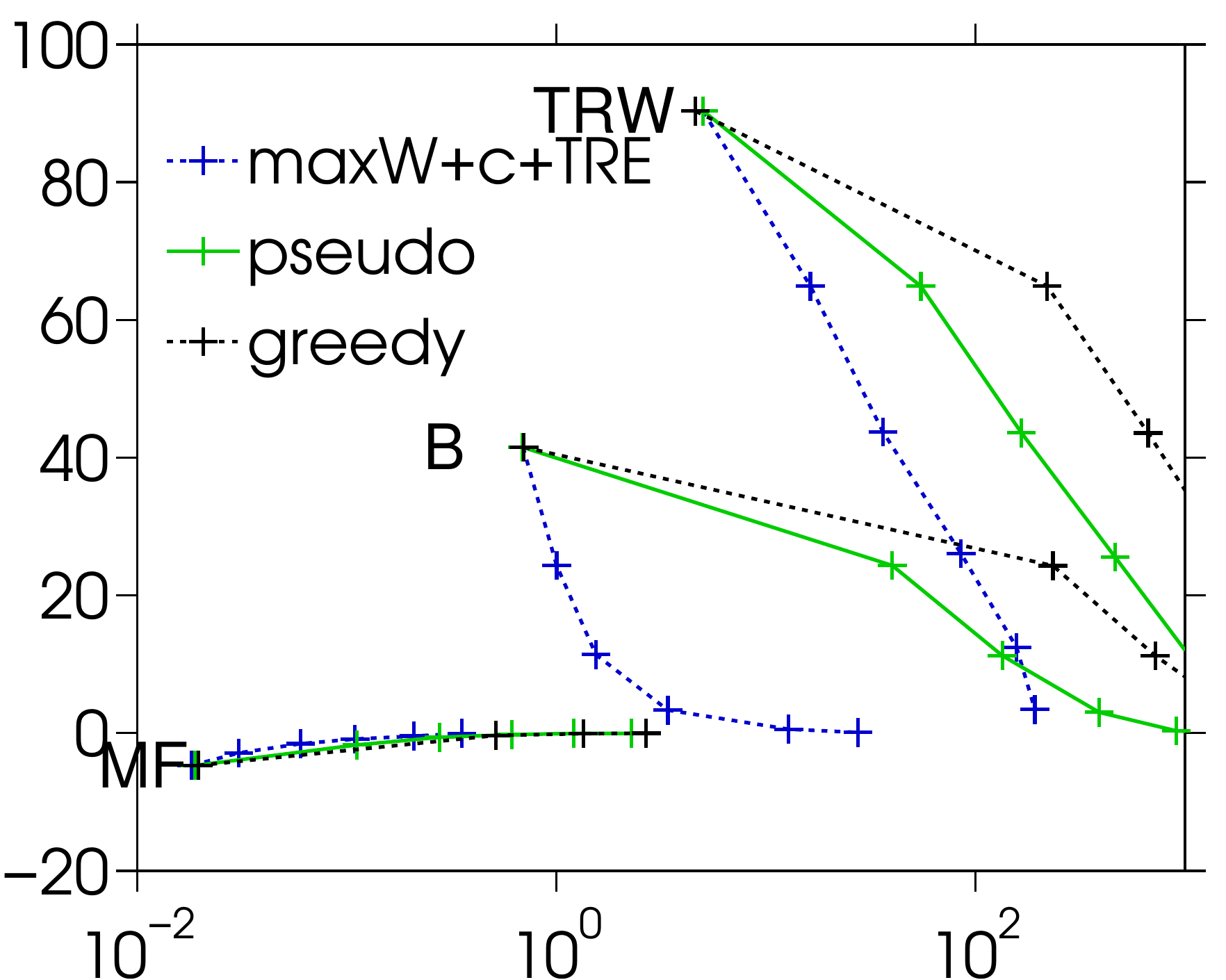} &
\begin{sideways} {\small \- \;  \; \; complete $K_{15}$} \end{sideways} \\
\begin{sideways} {\small \- \; \qquad small (25)} \end{sideways} &
\includegraphics[width=.24\linewidth]{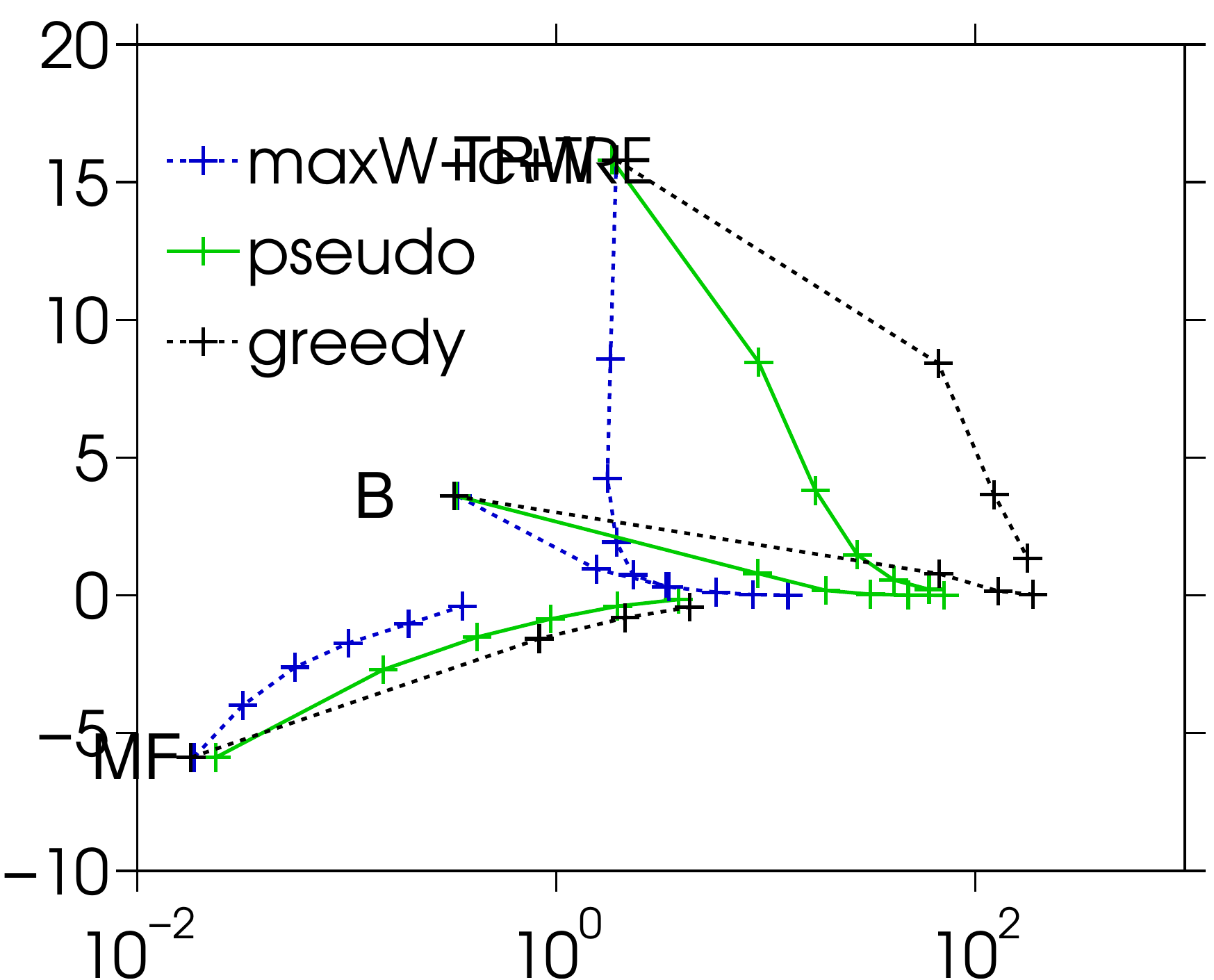} & 
\includegraphics[width=.24\linewidth]{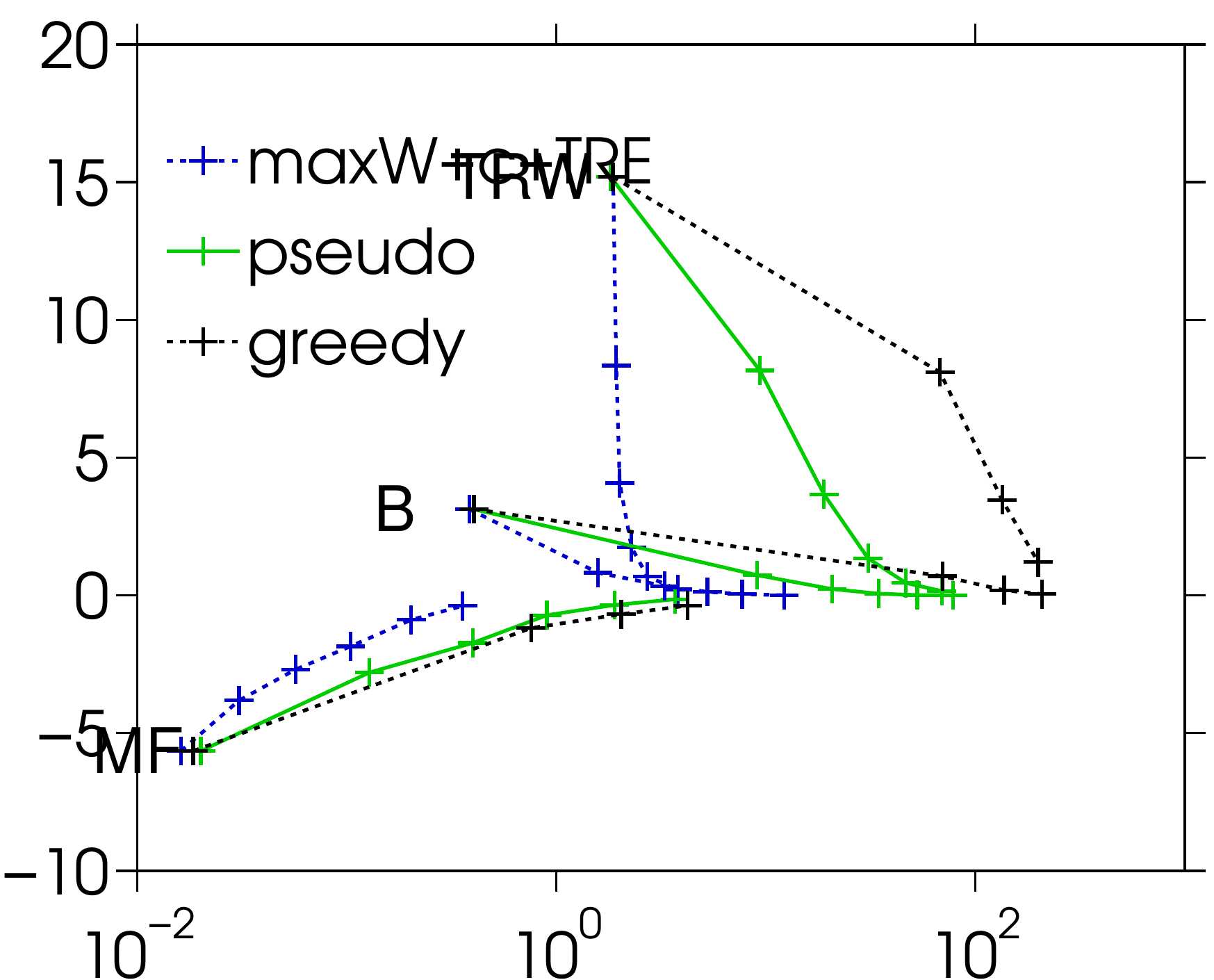} &
\includegraphics[width=.24\linewidth]{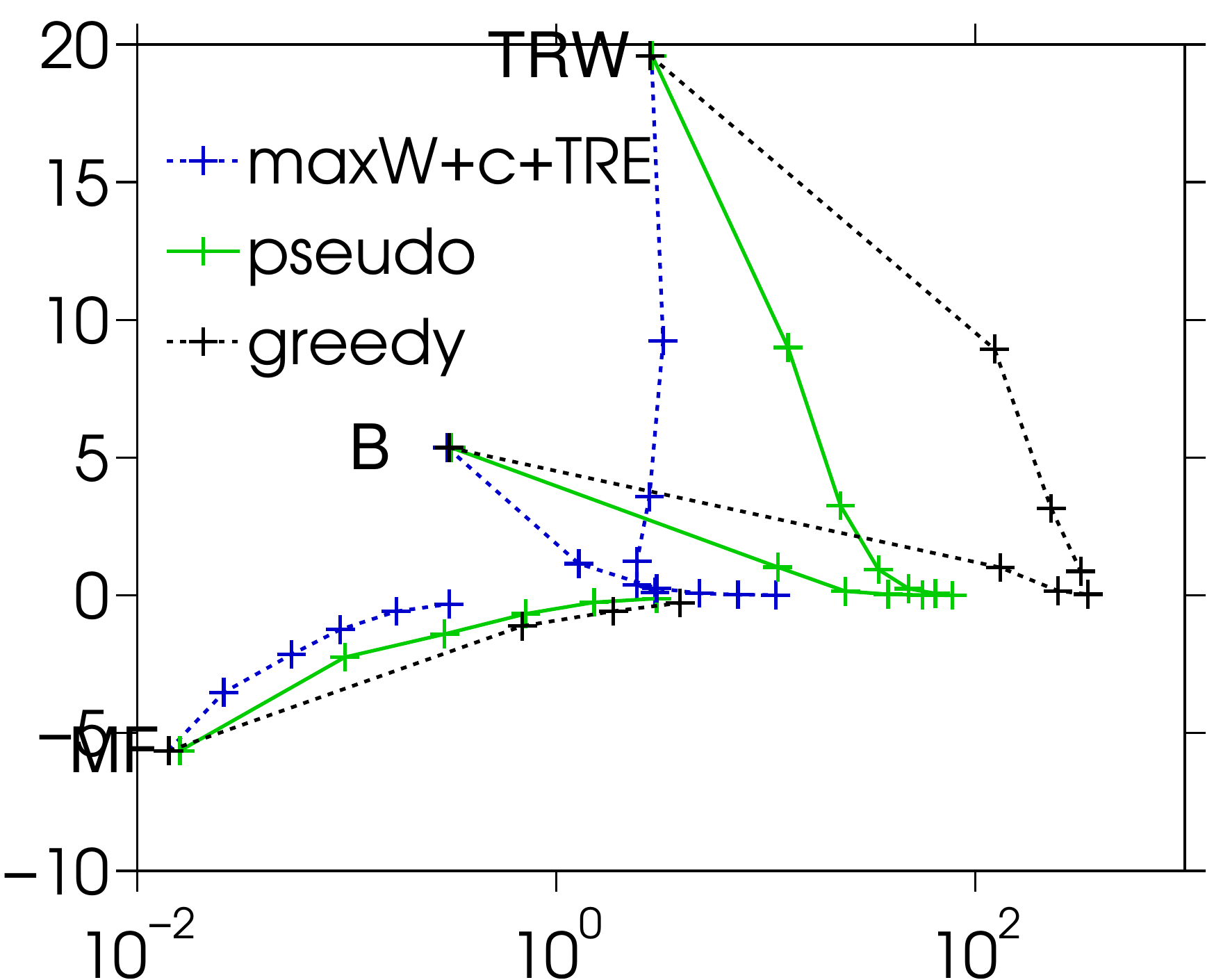} &
\- \; &
\includegraphics[width=.24\linewidth]{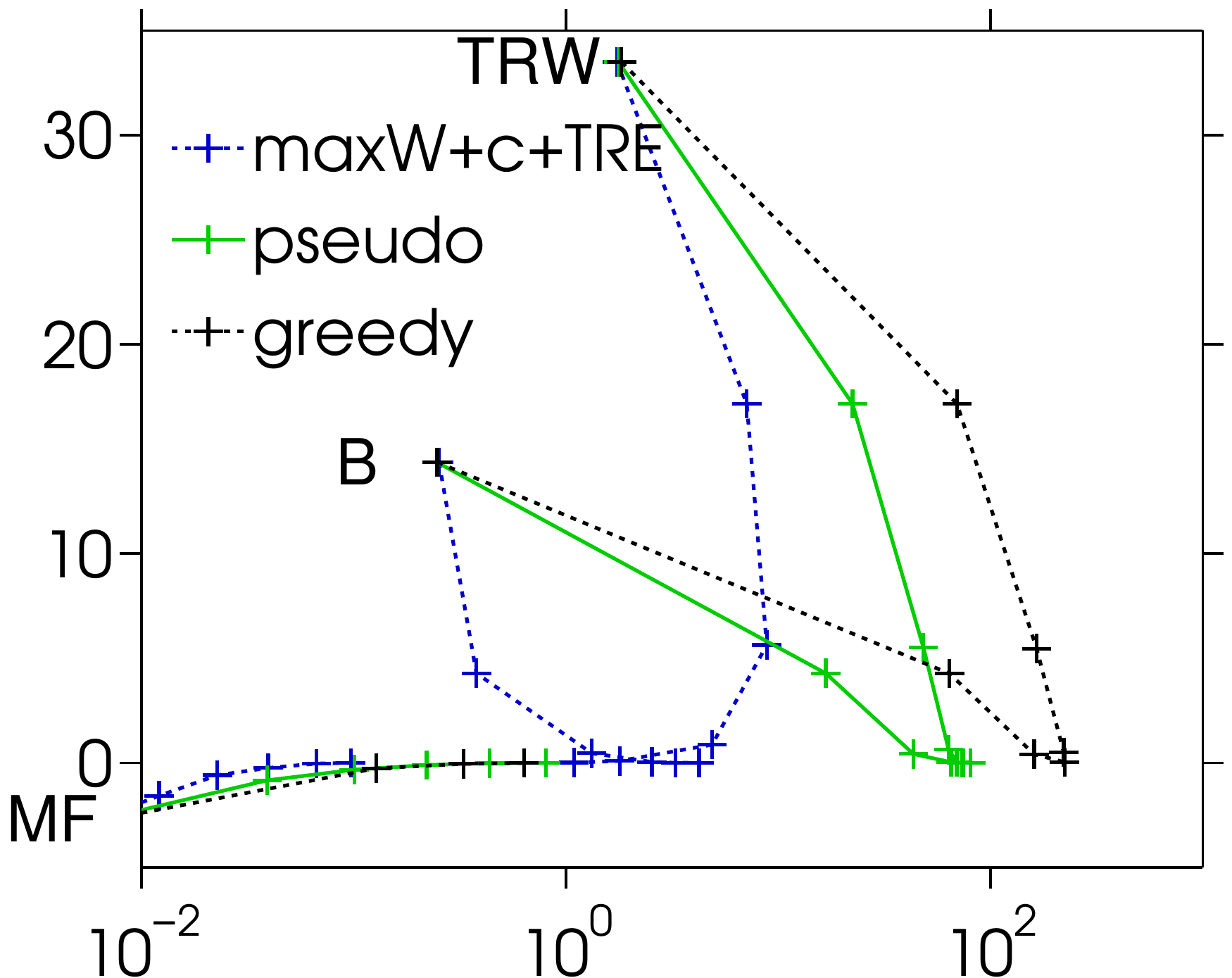} &
\begin{sideways} {\small \- \; \; \;  complete $K_{10}$} \end{sideways} \\
& {\small grids} & {\small random 4-regular} & {\small random Erd\"{o}s-Renyi} & & {\small complete graph} 
\end{tabular}
\end{center}
\caption{\small Mixed $[-12,12]$ timings (in secs, $\log$ scale, these give an overall sense but may be sensitive to implementation details and convergence thresholds)
}
\label{fig:weird2}
\end{figure}

\begin{figure}
\begin{center}
\setlength\tabcolsep{1pt}
\begin{tabular}{ccccccc}
\begin{sideways} {\small \- \; \qquad large (81)} \end{sideways} &
\includegraphics[width=.24\linewidth]{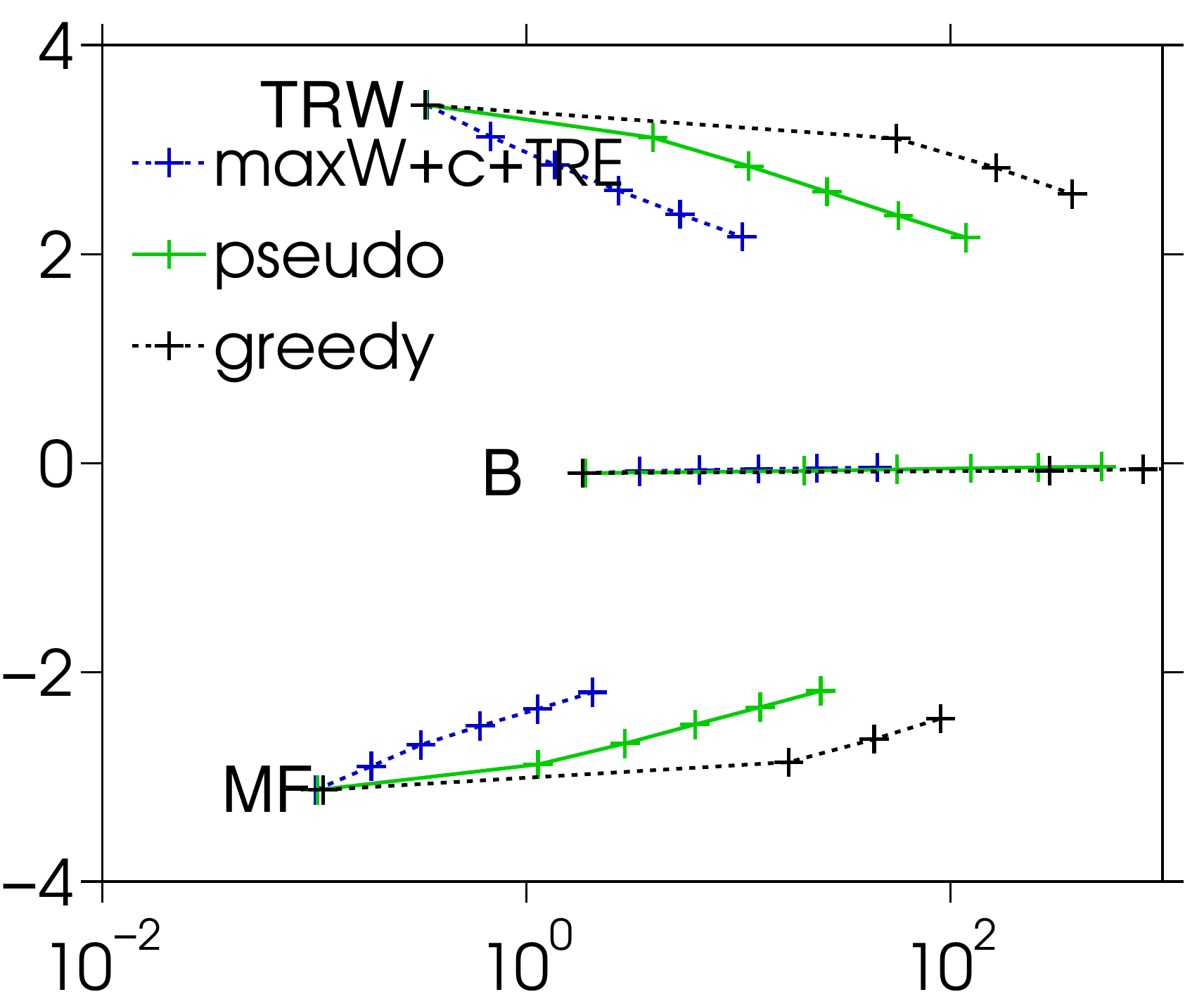} & 
\includegraphics[width=.24\linewidth]{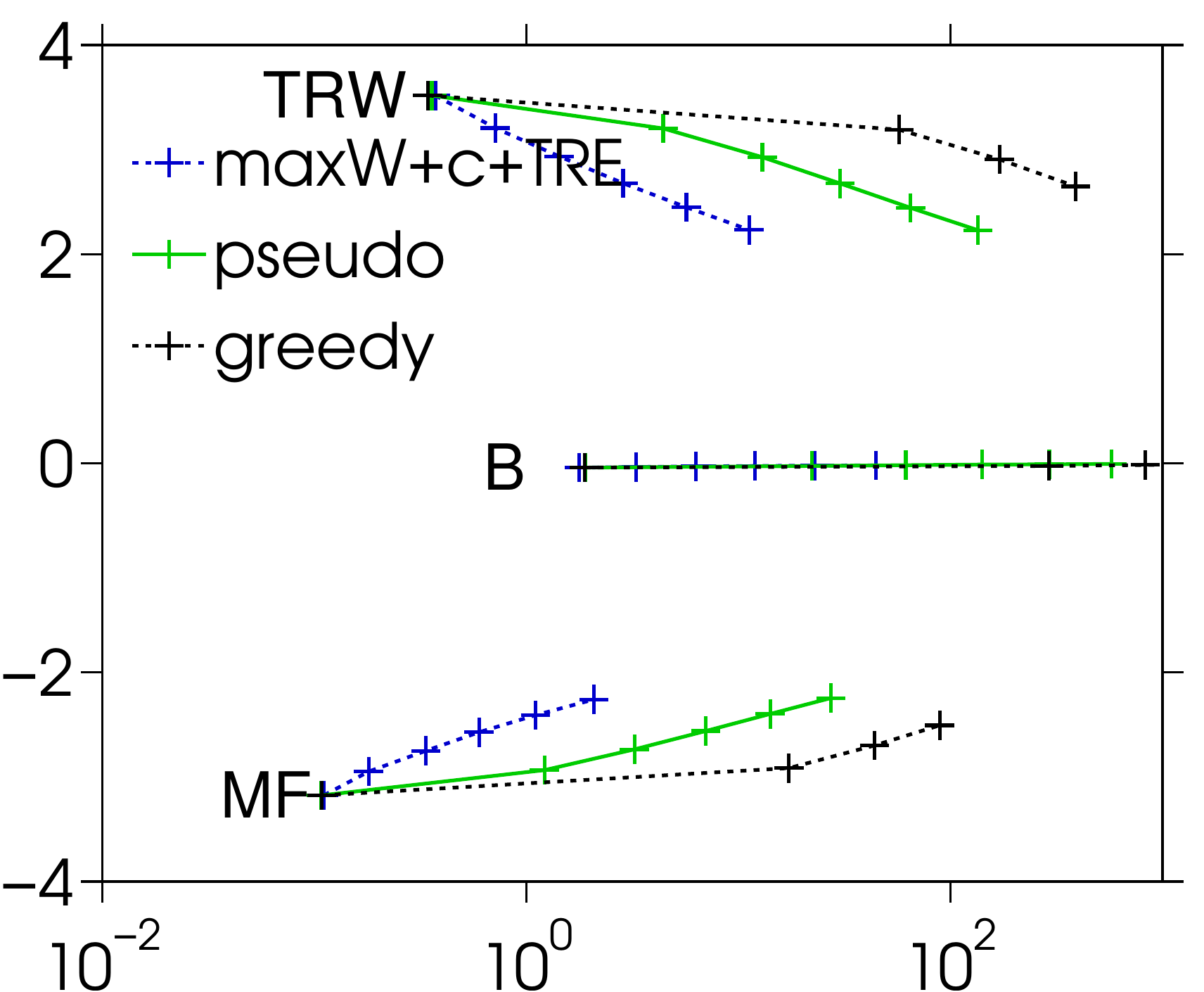} &
\includegraphics[width=.24\linewidth]{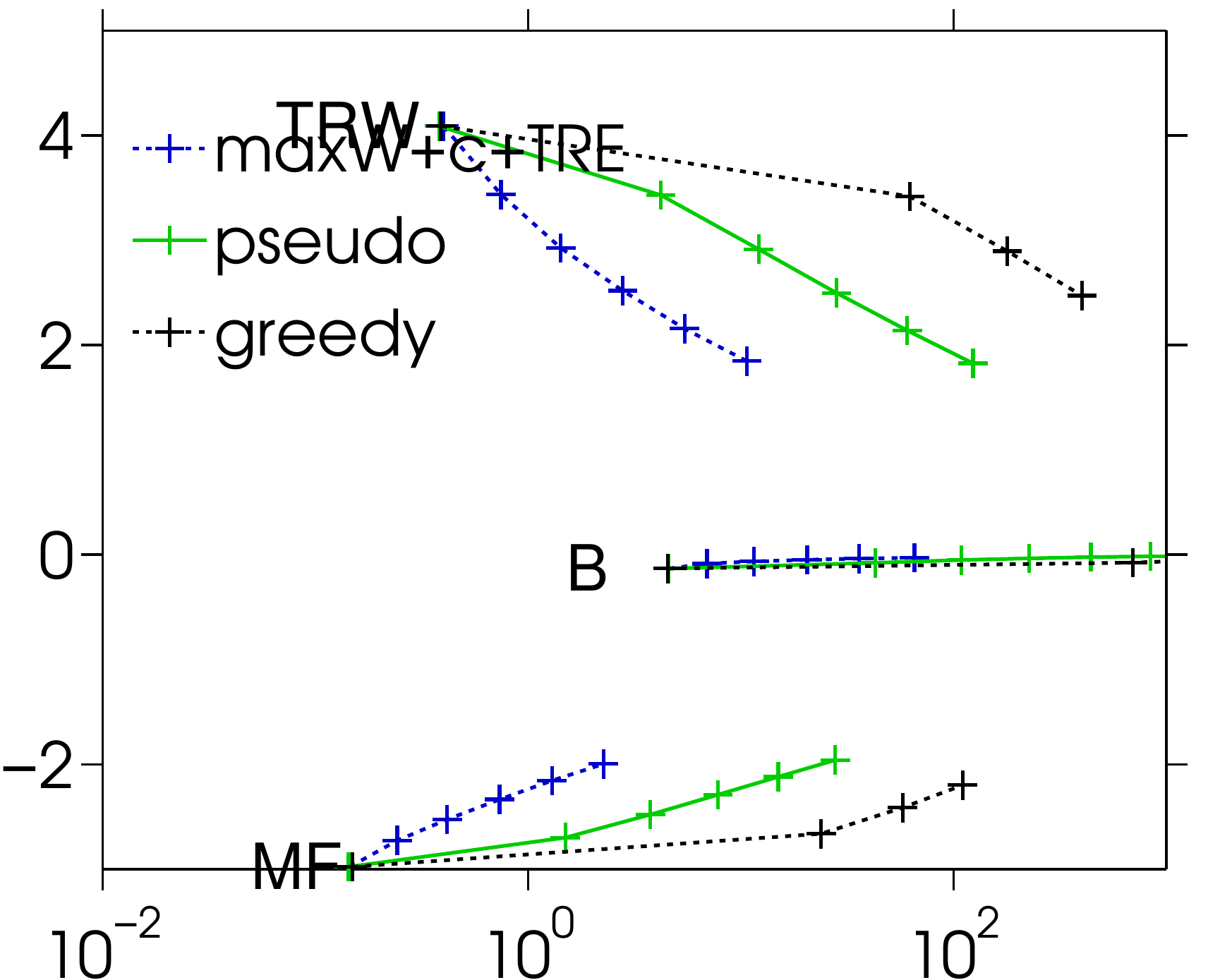} 
\\
\begin{sideways} {\small \- \; \; \quad medium (49)} \end{sideways} &
\includegraphics[width=.24\linewidth]{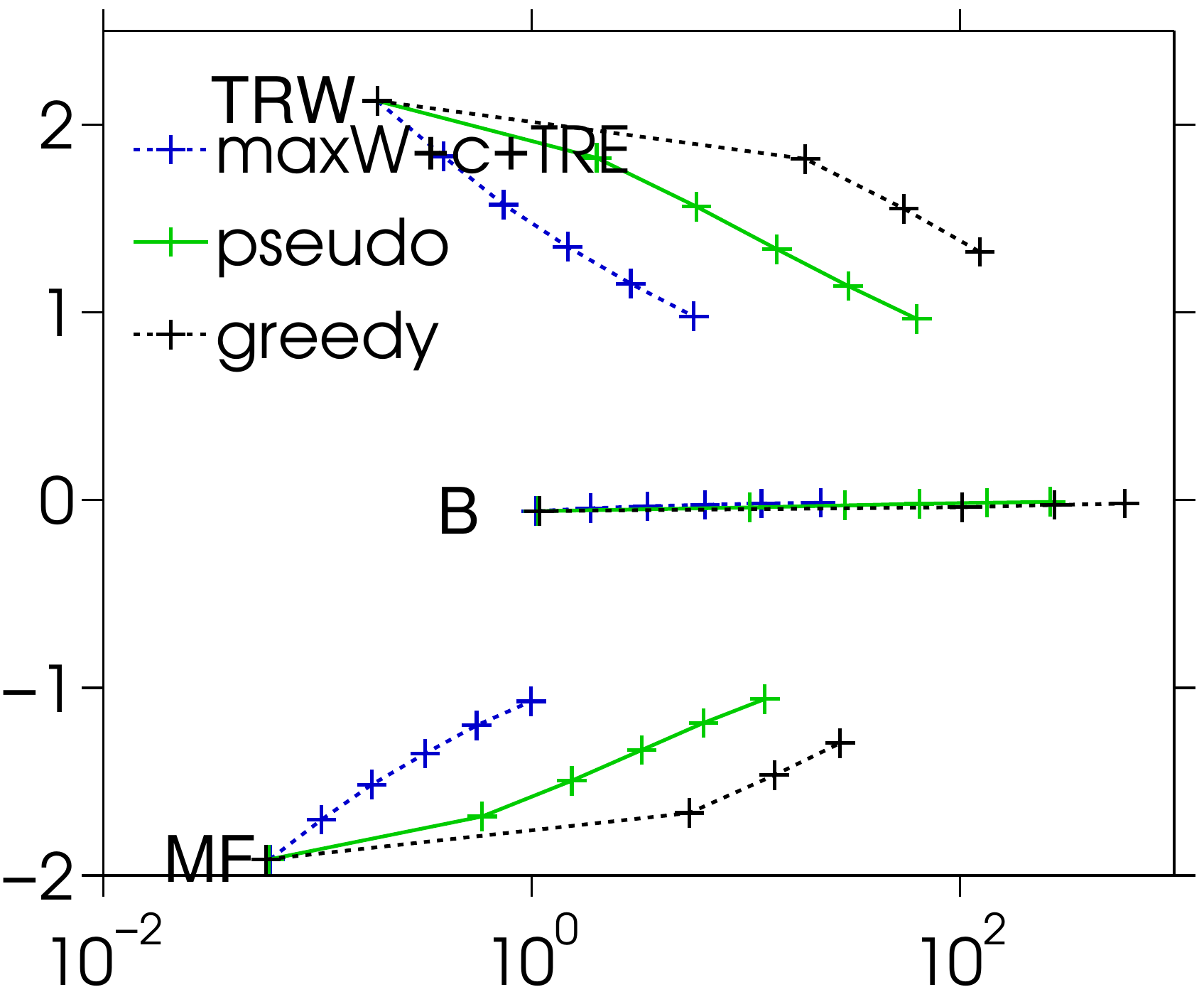} & 
\includegraphics[width=.24\linewidth]{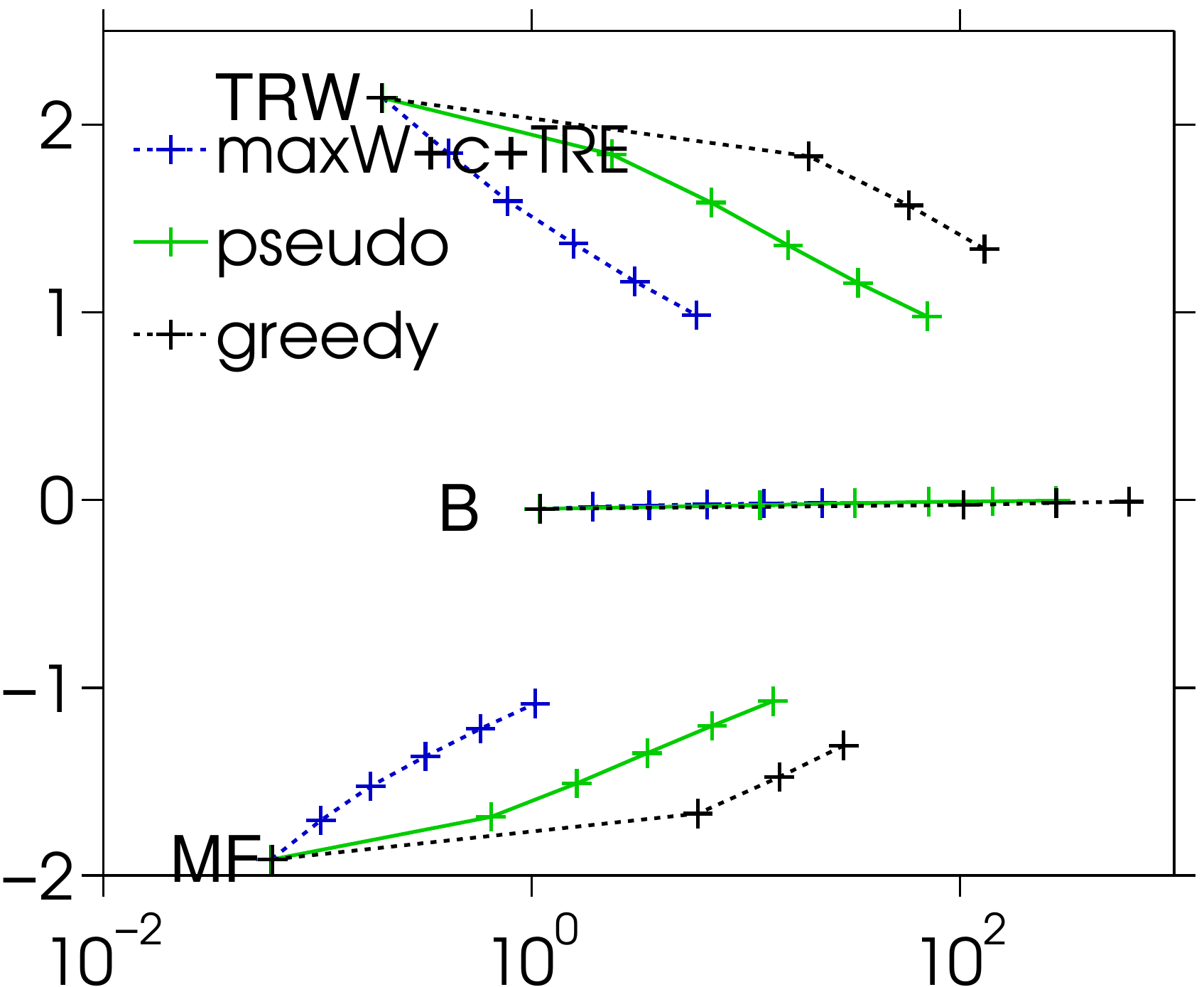} &
\includegraphics[width=.24\linewidth]{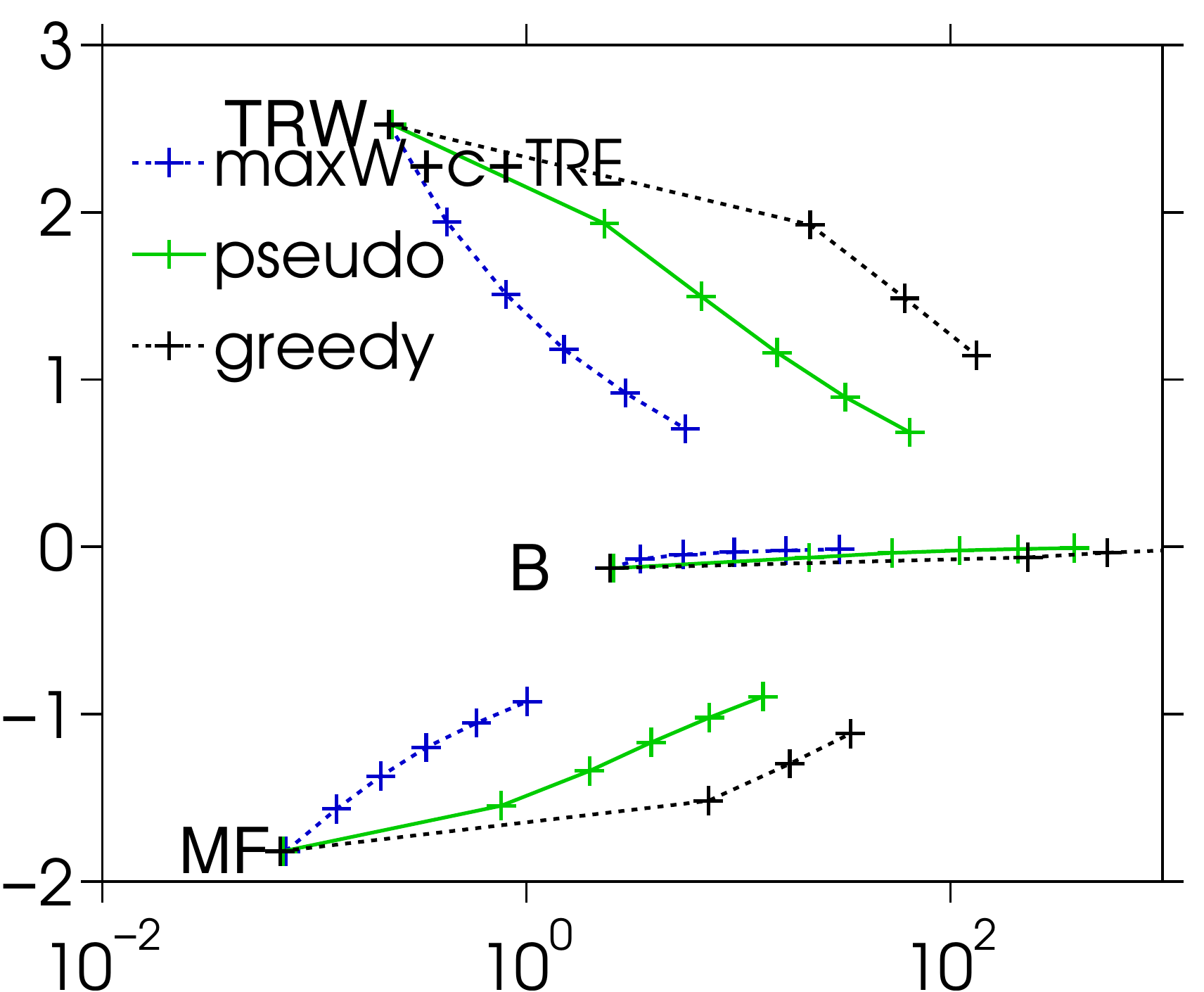} &
\- \; &
\includegraphics[width=.24\linewidth]{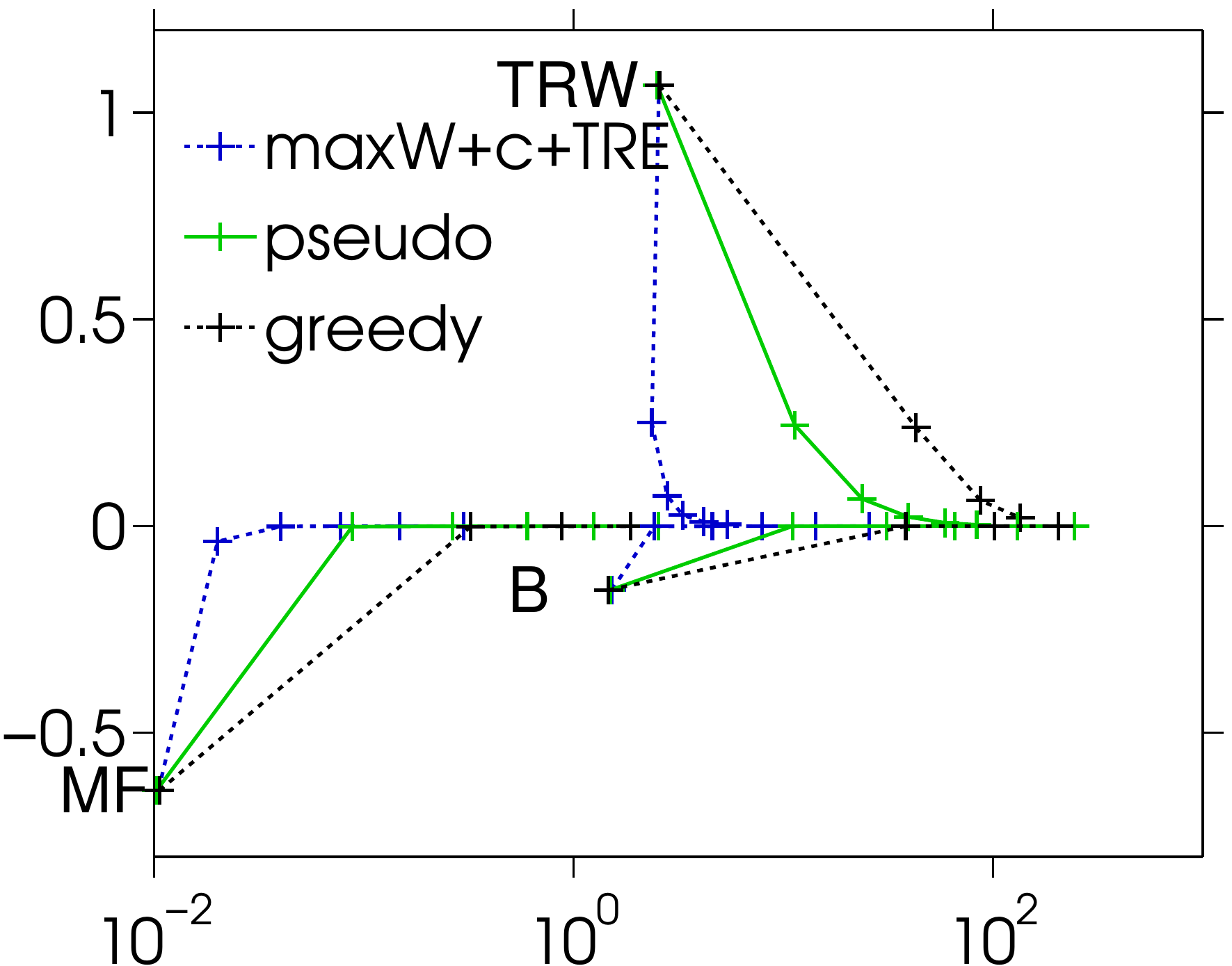} &
\begin{sideways} {\small \- \; \; \;  complete $K_{15}$} \end{sideways} \\
\begin{sideways} {\small \- \; \qquad small (25)} \end{sideways} &
\includegraphics[width=.24\linewidth]{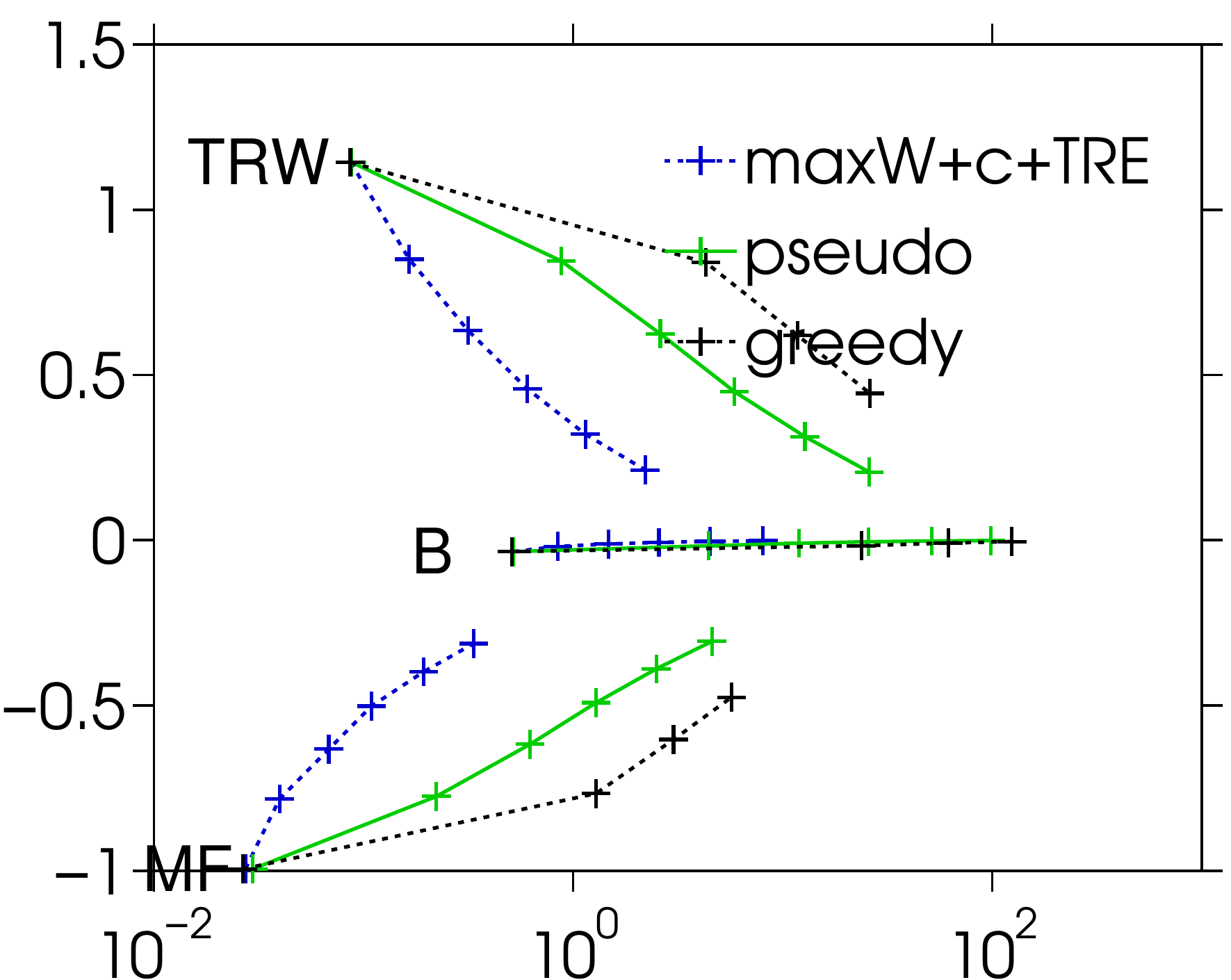} & 
\includegraphics[width=.24\linewidth]{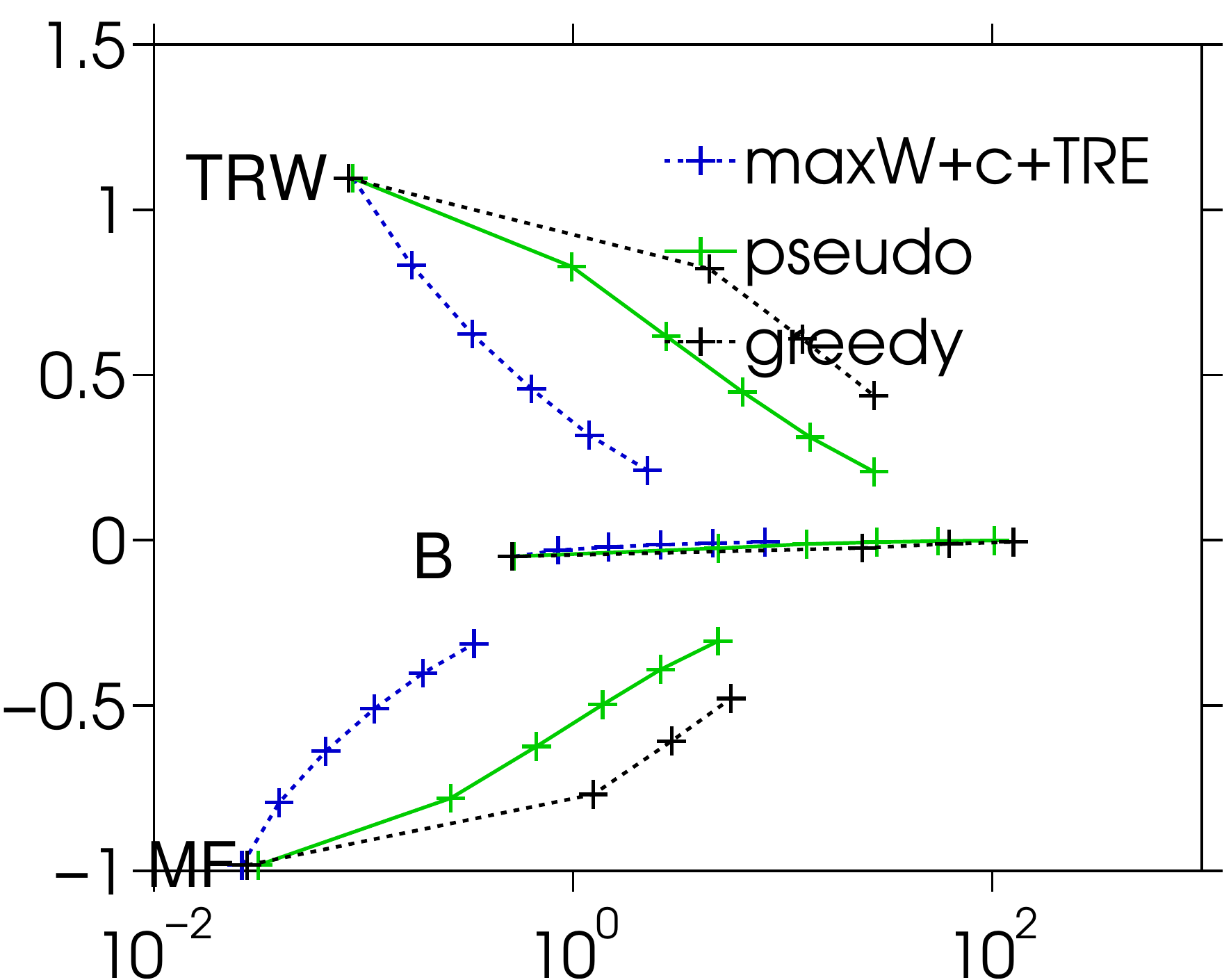} &
\includegraphics[width=.24\linewidth]{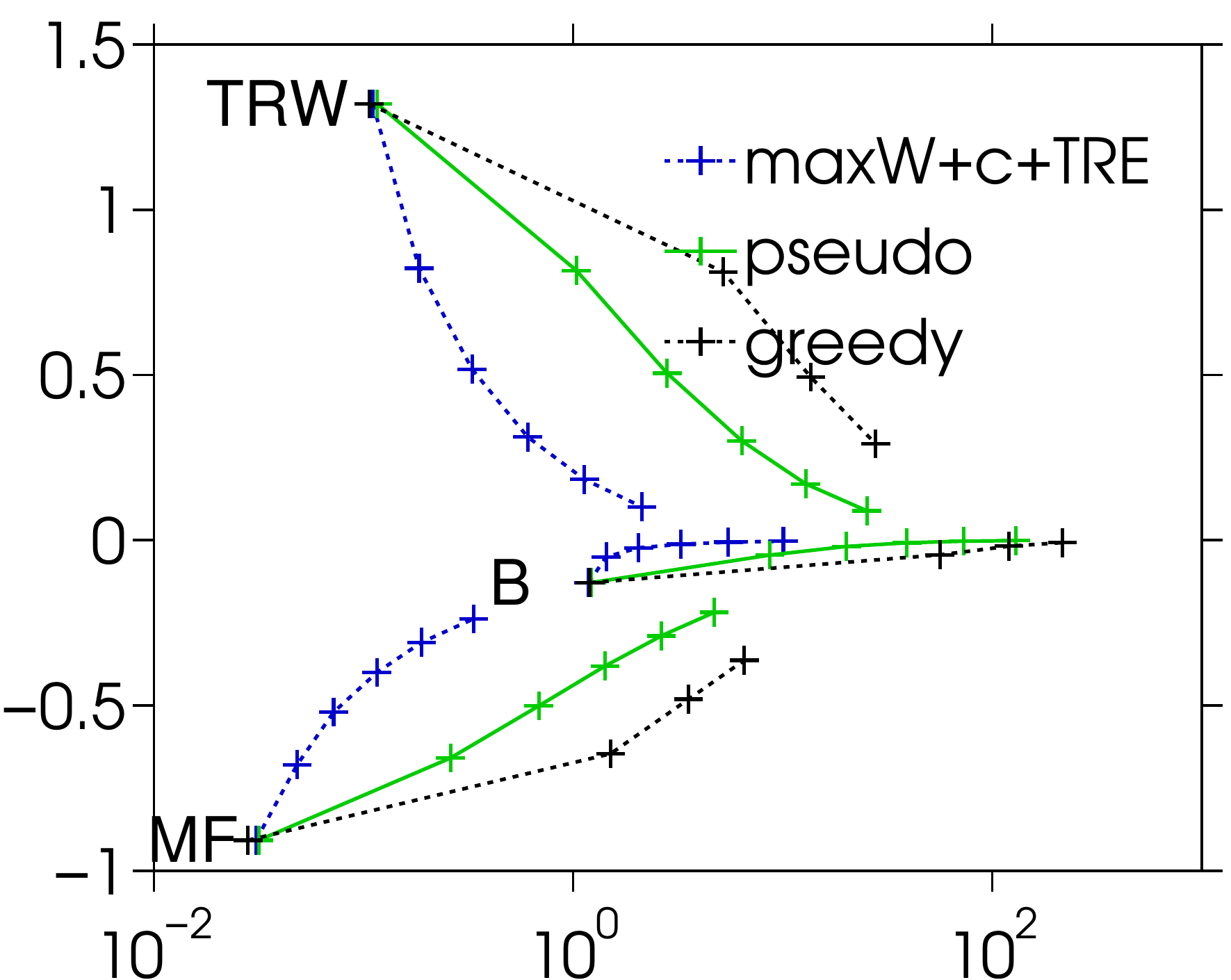} &
\- \; &
\includegraphics[width=.24\linewidth]{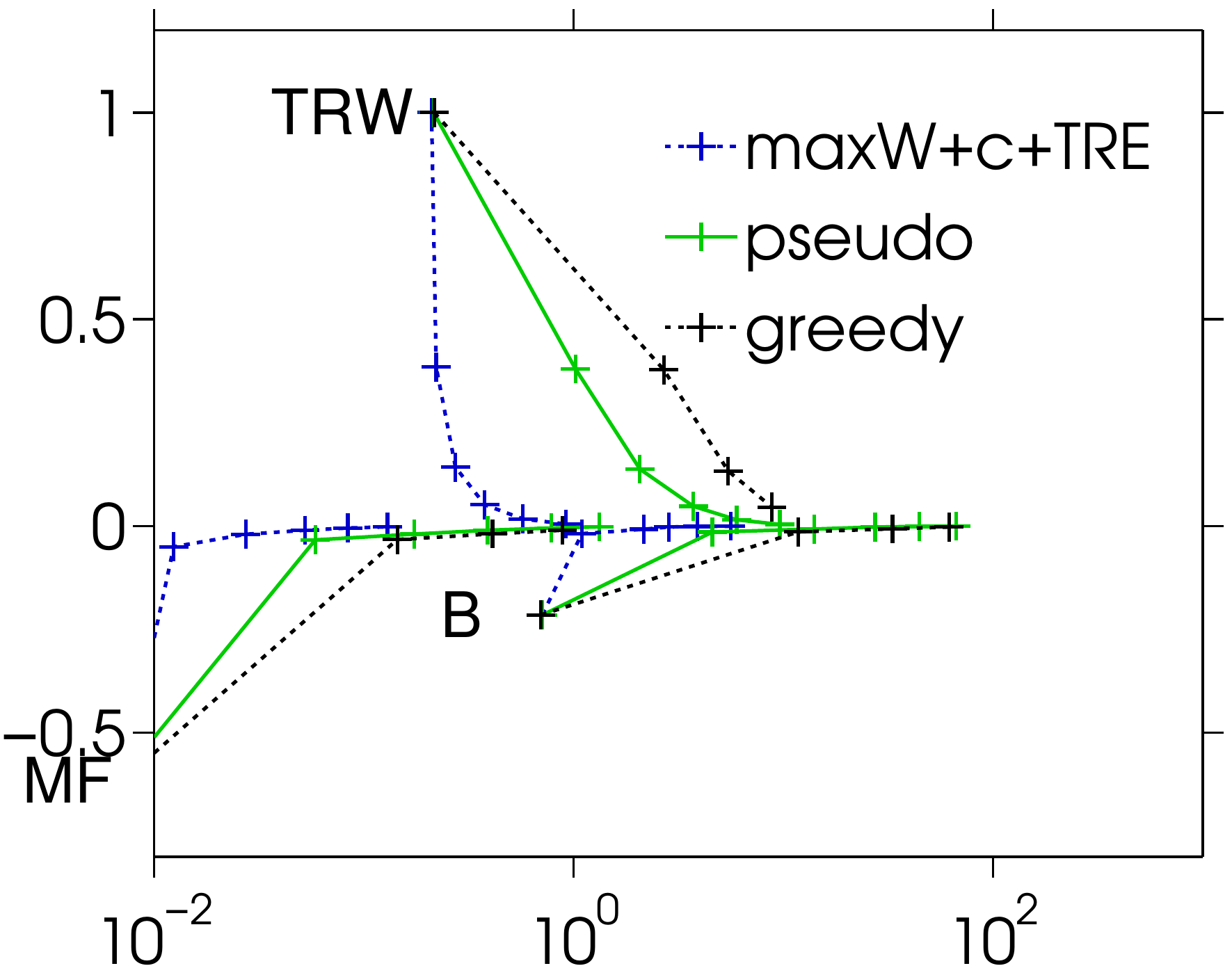} &
\begin{sideways} {\small \- \; \; \;  complete $K_{10}$} \end{sideways} \\
& {\small grids} & {\small random 4-regular} & {\small random Erd\"{o}s-Renyi} & & {\small complete graph} 
\end{tabular}
\end{center}
\caption{\small Attractive $[0,2]$ timings (in secs, $\log$ scale, these give an overall sense but may be sensitive to implementation details and convergence thresholds)
}
%\label{fig:exp3}
\end{figure}

\begin{figure}
\begin{center}
\setlength\tabcolsep{1pt}
\begin{tabular}{ccccccc}
\begin{sideways} {\small \- \; \qquad large (81)} \end{sideways} &
\includegraphics[width=.24\linewidth]{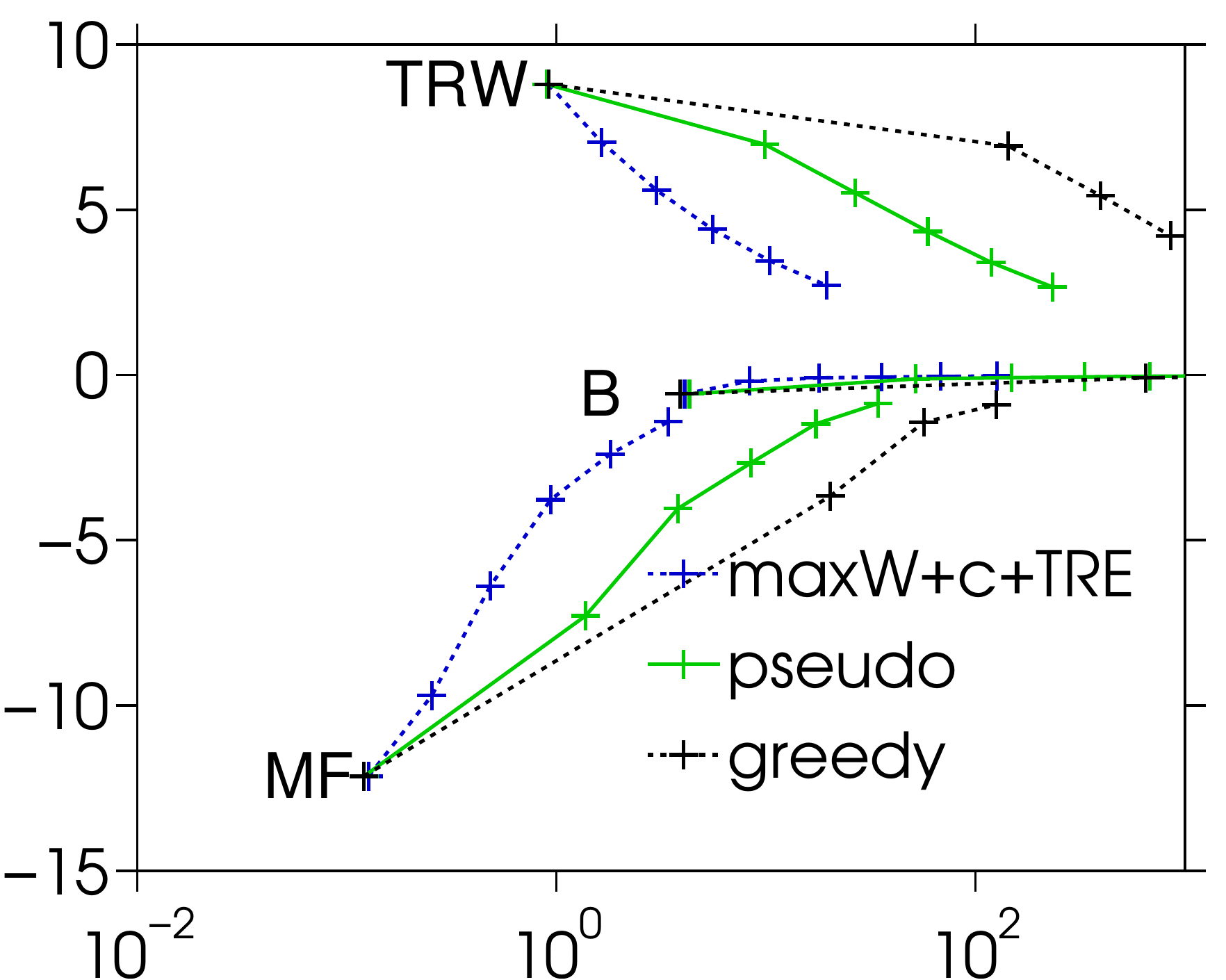} & 
\includegraphics[width=.24\linewidth]{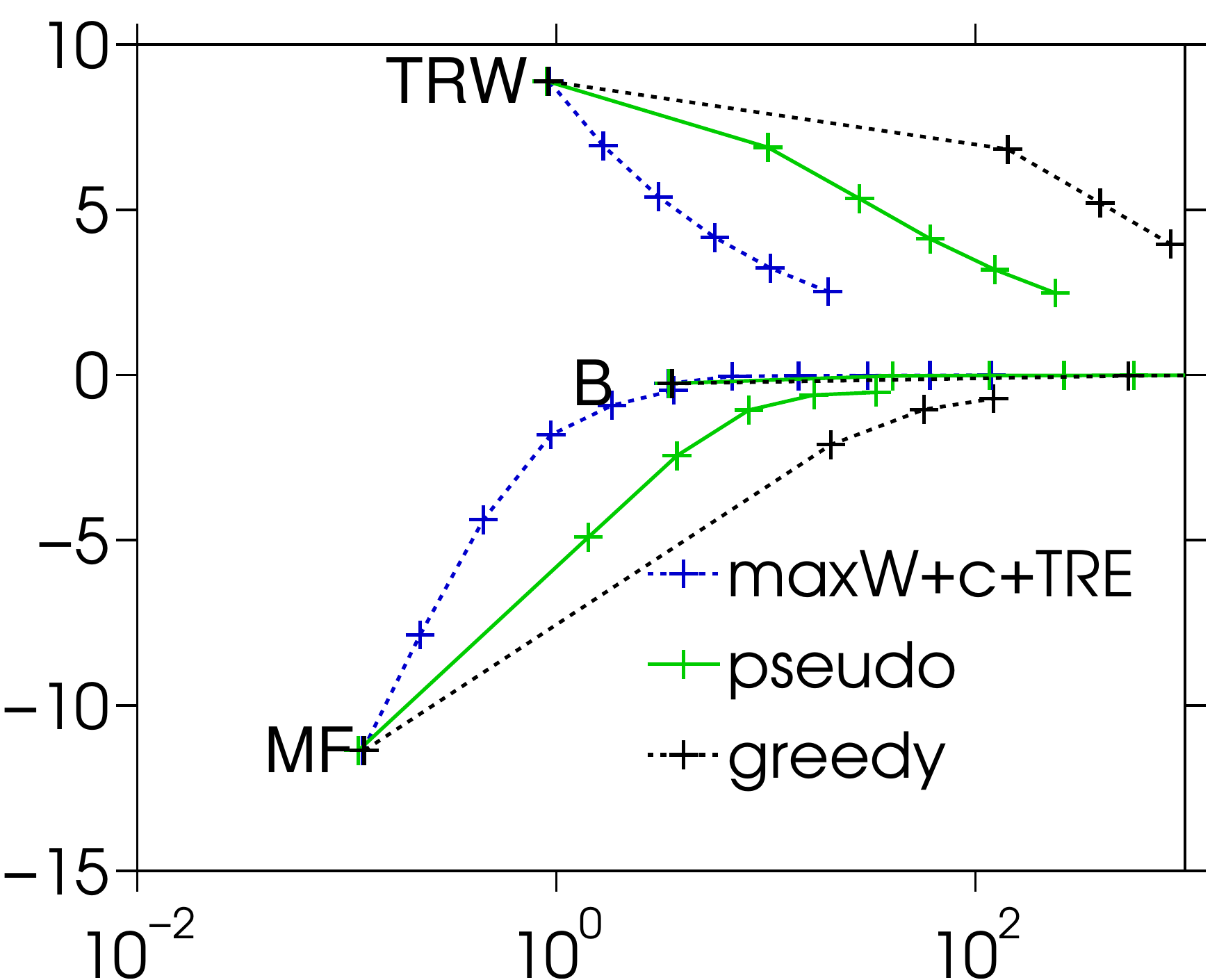} &
\includegraphics[width=.24\linewidth]{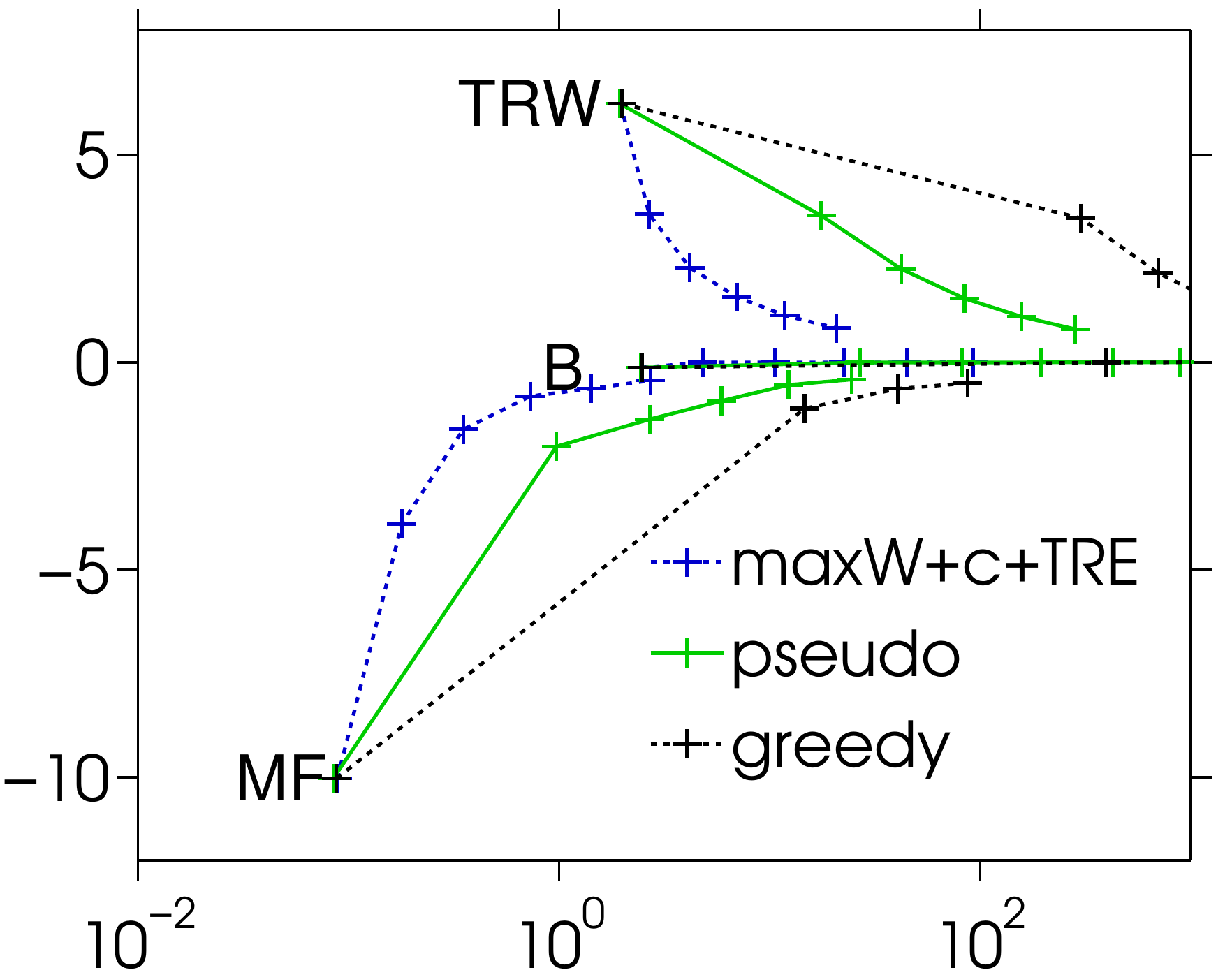} 
\\
\begin{sideways} {\small \- \; \; \quad medium (49)} \end{sideways} &
\includegraphics[width=.24\linewidth]{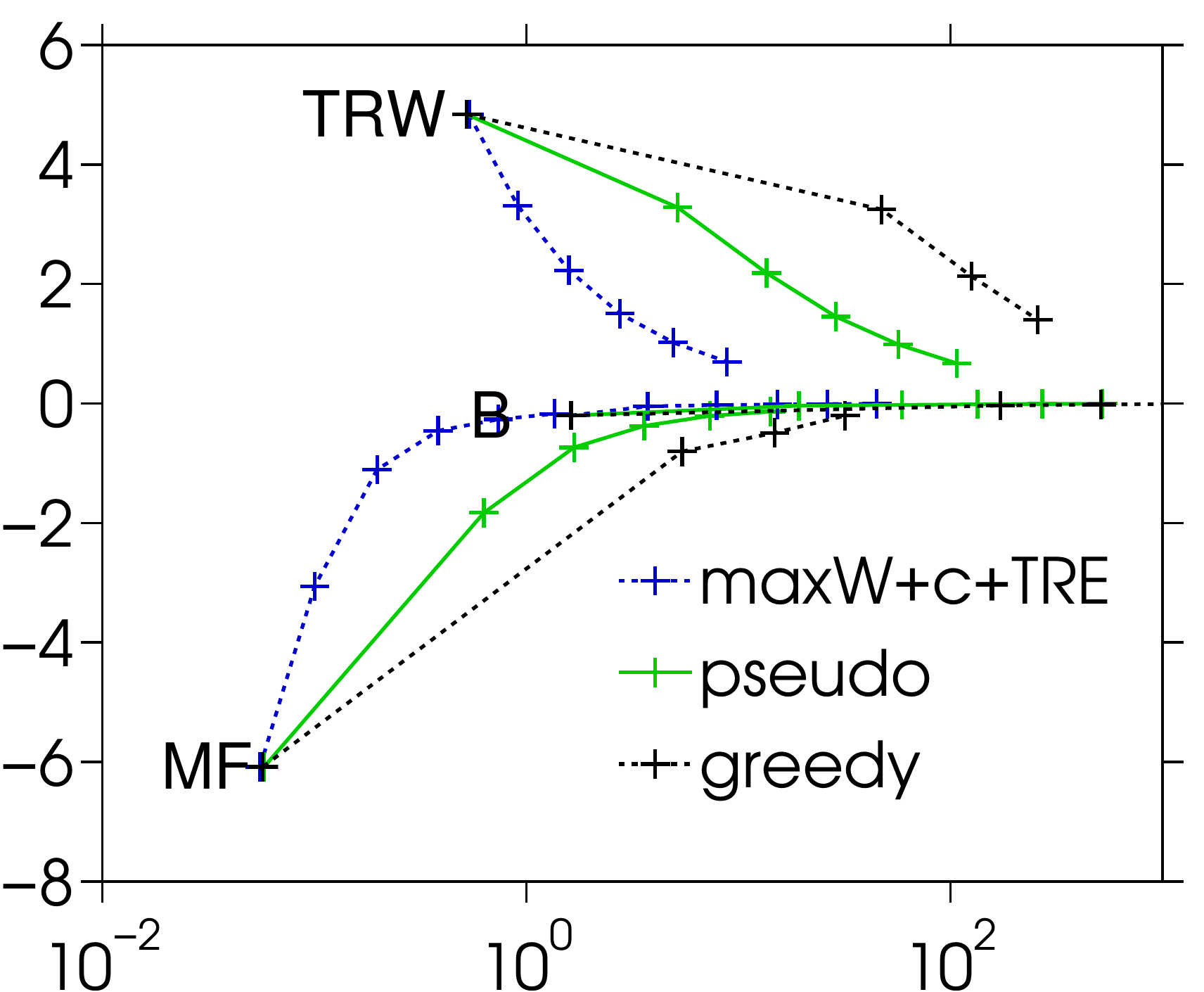} & 
\includegraphics[width=.24\linewidth]{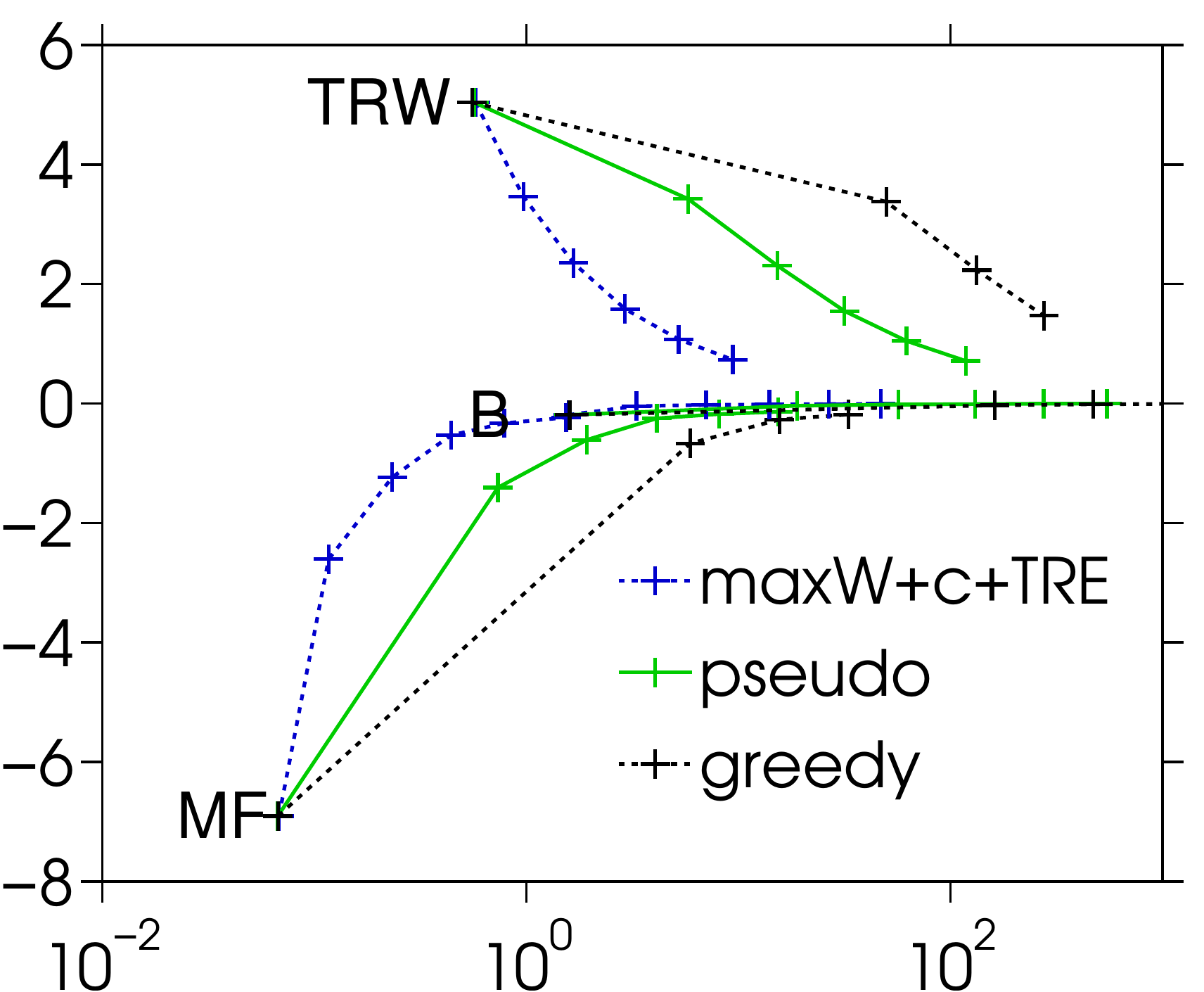} &
\includegraphics[width=.24\linewidth]{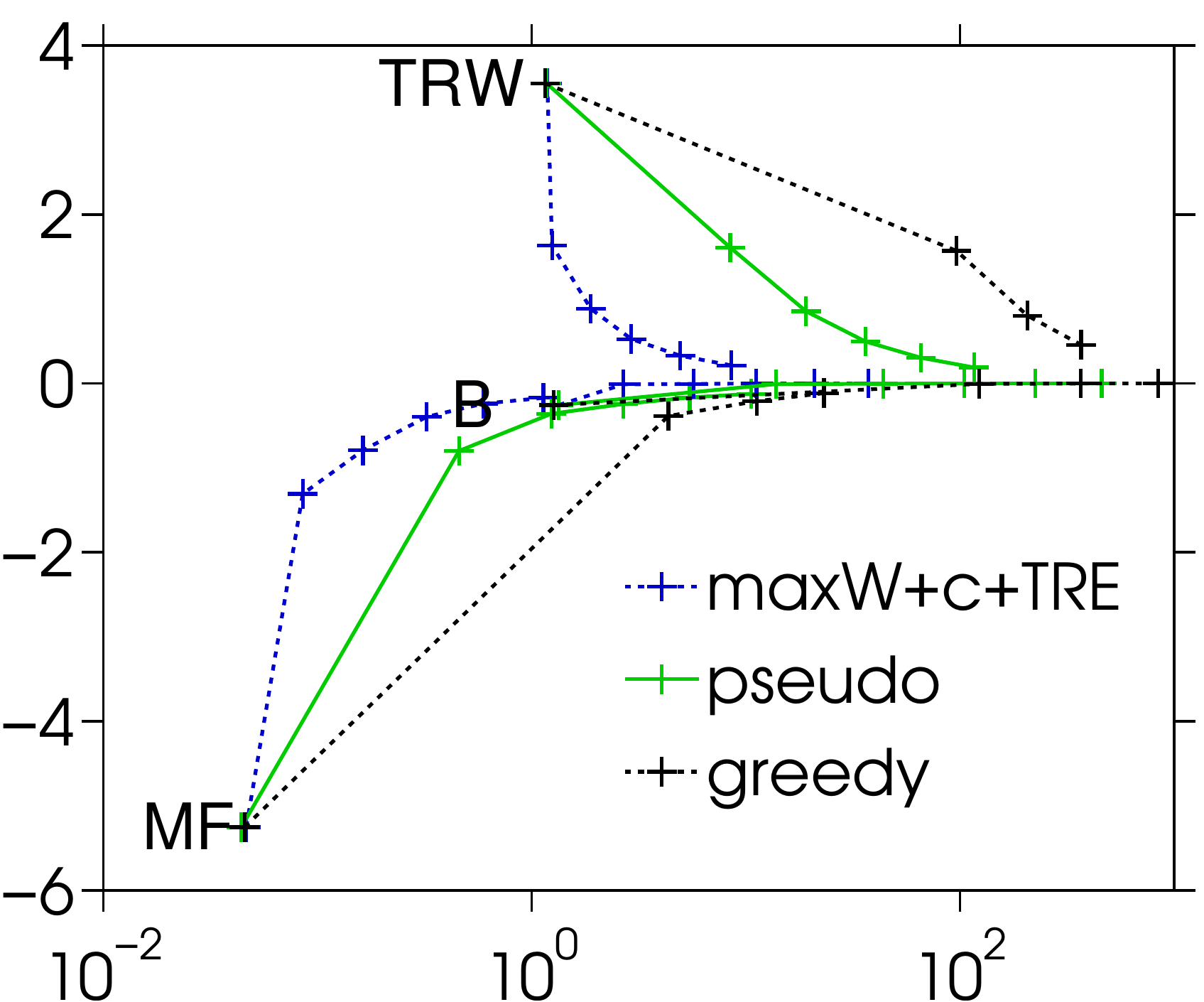} &
\- \; &
\includegraphics[width=.24\linewidth]{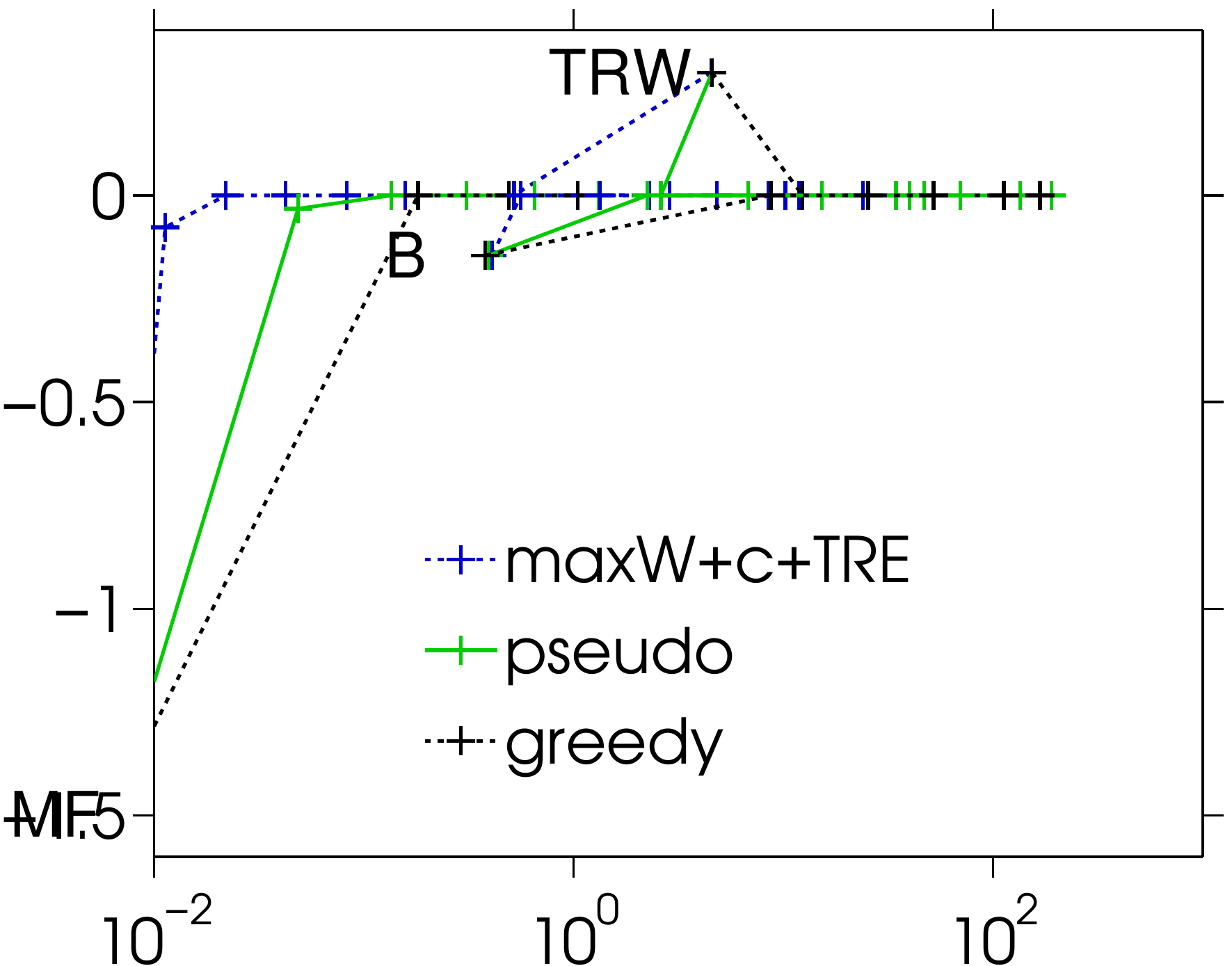} &
\begin{sideways} {\small \- \; \; \;  complete $K_{15}$} \end{sideways} \\
\begin{sideways} {\small \- \; \qquad small (25)} \end{sideways} &
\includegraphics[width=.24\linewidth]{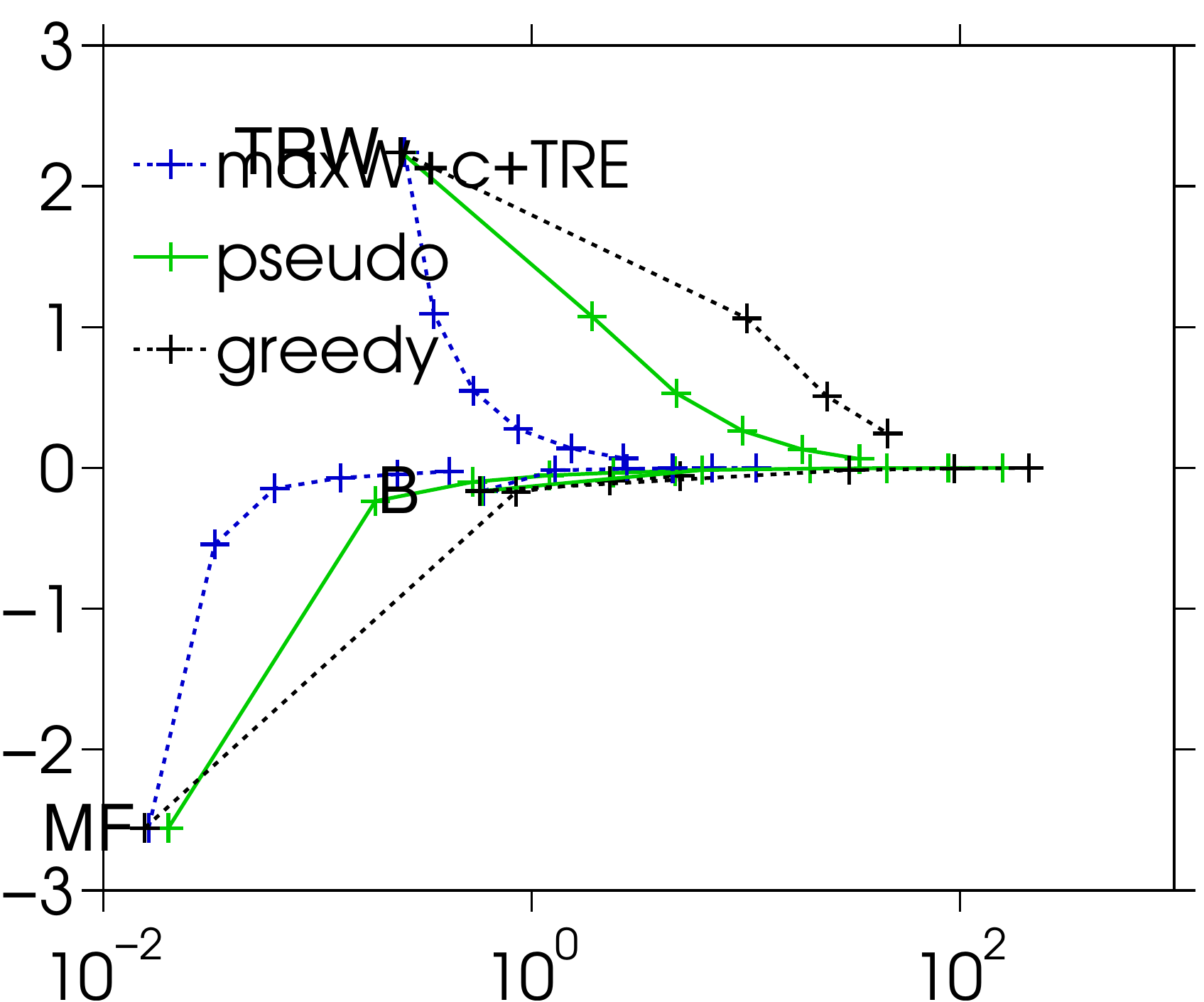} & 
\includegraphics[width=.24\linewidth]{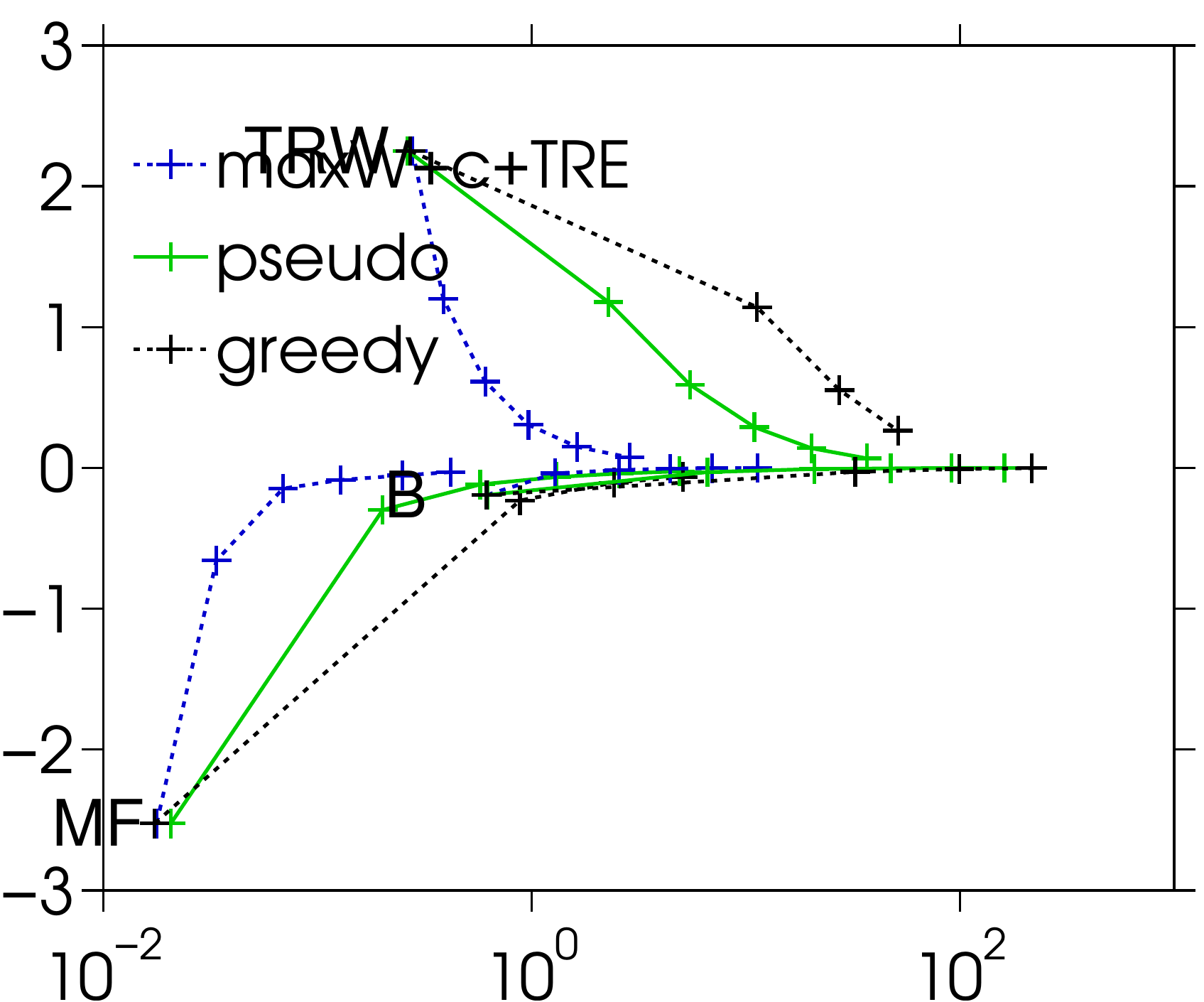} &
\includegraphics[width=.24\linewidth]{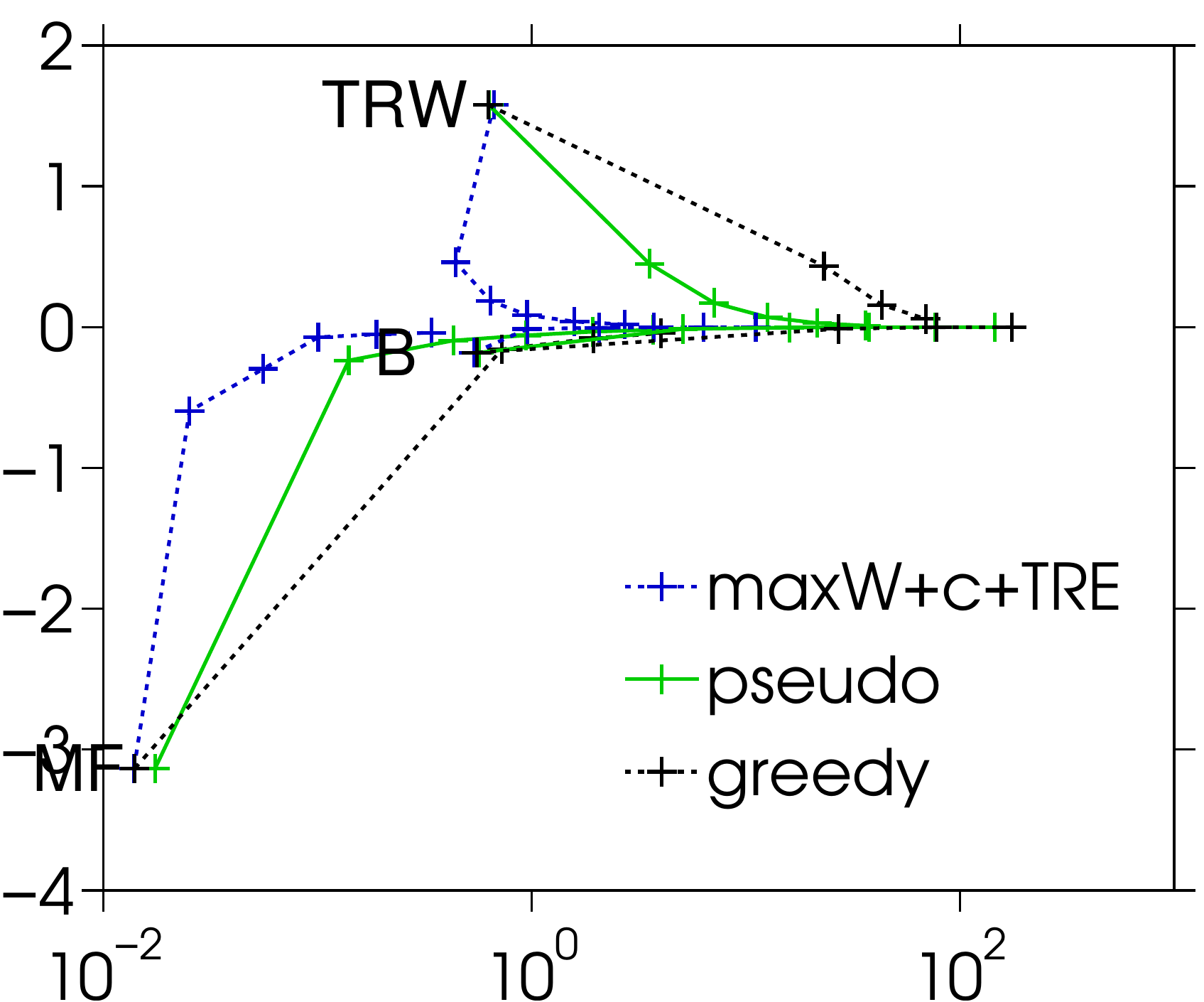} &
\- \; &
\includegraphics[width=.24\linewidth]{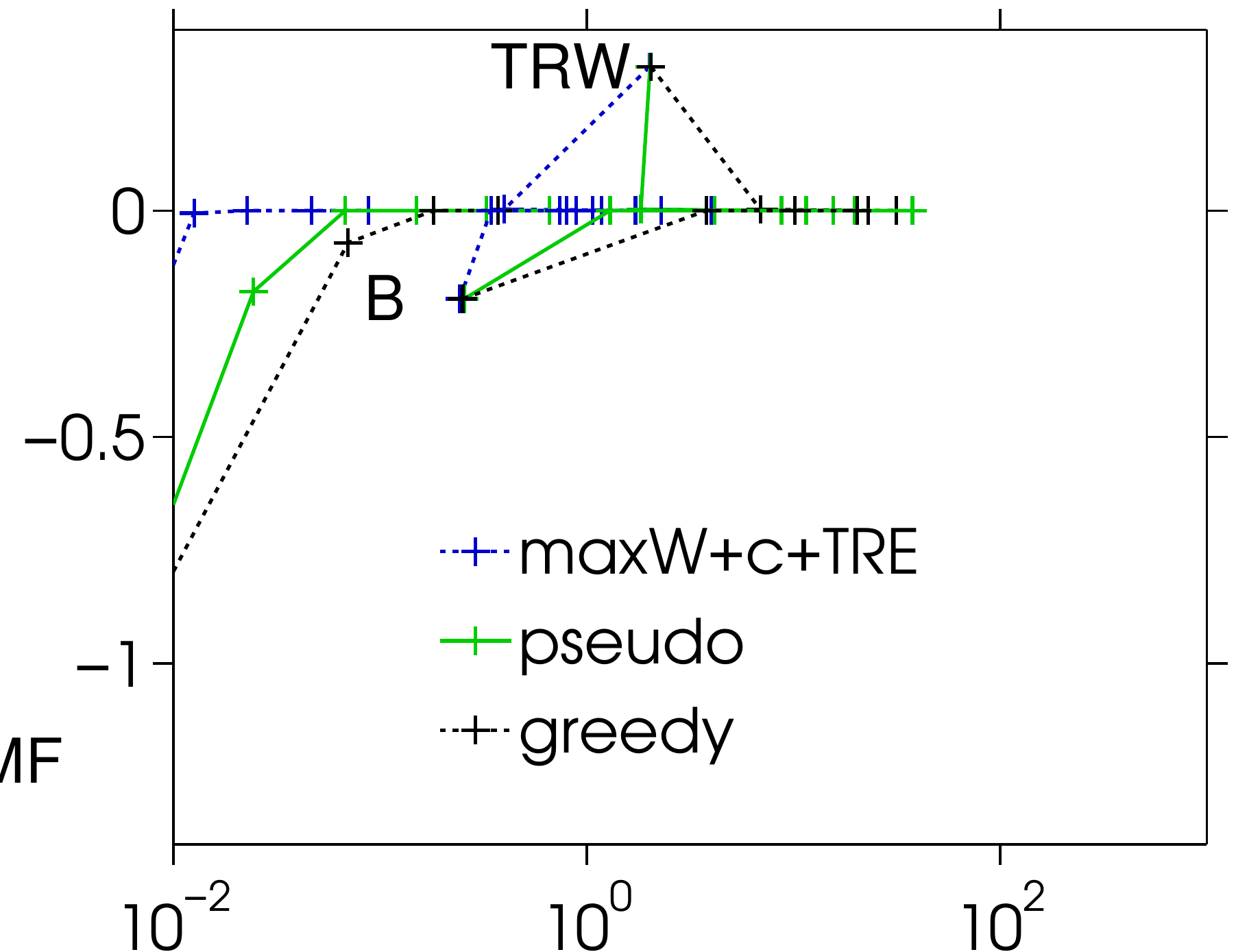} &
\begin{sideways} {\small \- \; \; \;  complete $K_{10}$} \end{sideways} \\
& {\small grids} & {\small random 4-regular} & {\small random Erd\"{o}s-Renyi} & & {\small complete graph} 
\end{tabular}
\end{center}
\caption{\small Attractive $[0,6]$ timings (in secs, $\log$ scale, these give an overall sense but may be sensitive to implementation details and convergence thresholds)
}
\label{fig:lasttime}
\end{figure}

%%%%%%%%%%%%%%%%%%%%%%%%%%%%%%%%%%%%%%%%%%%%%%% end timings

\newpage

\begin{figure}
\begin{center}
\begin{tabular}{ccc}
\includegraphics[width=.31\linewidth]{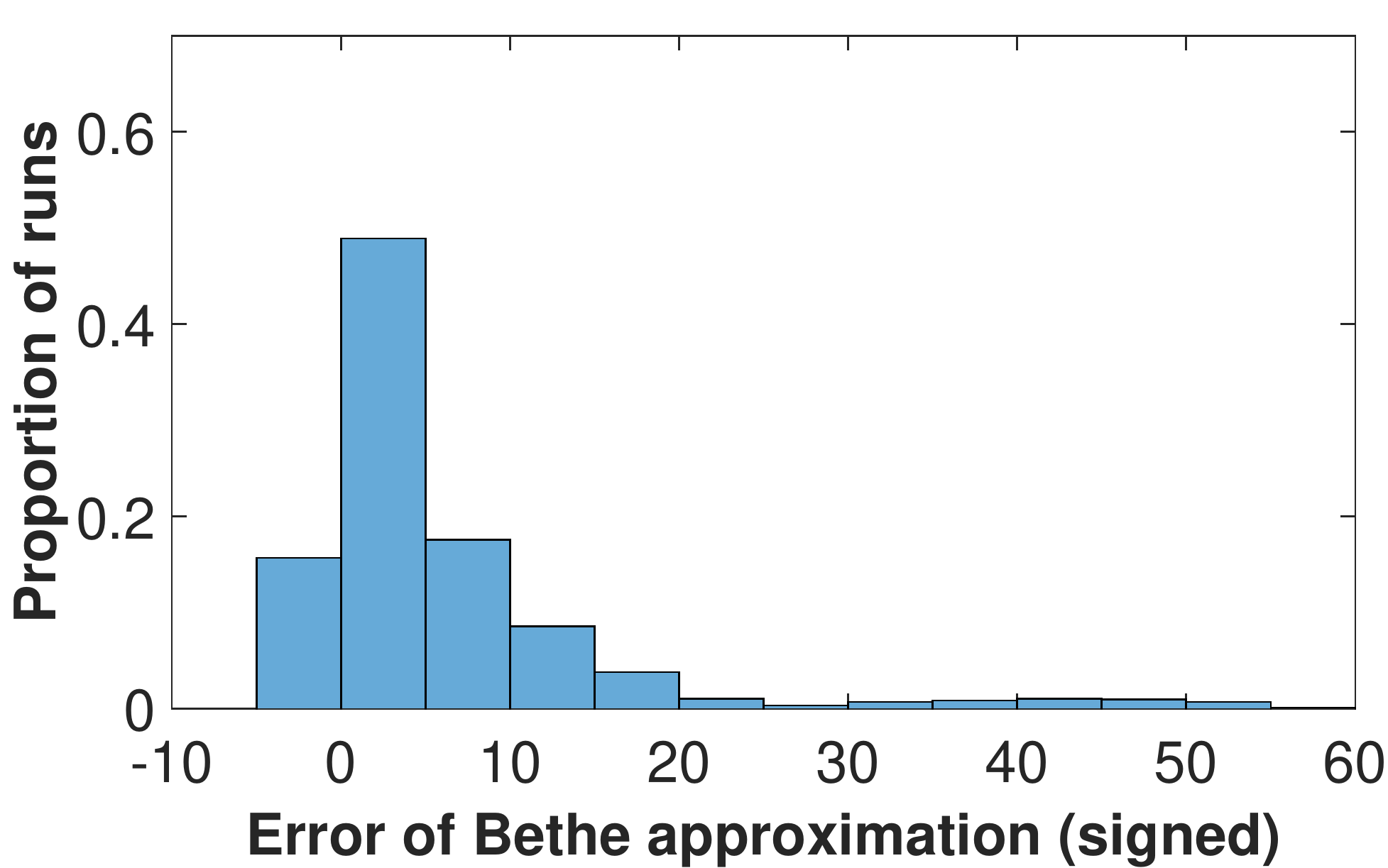} & 
\includegraphics[width=.31\linewidth]{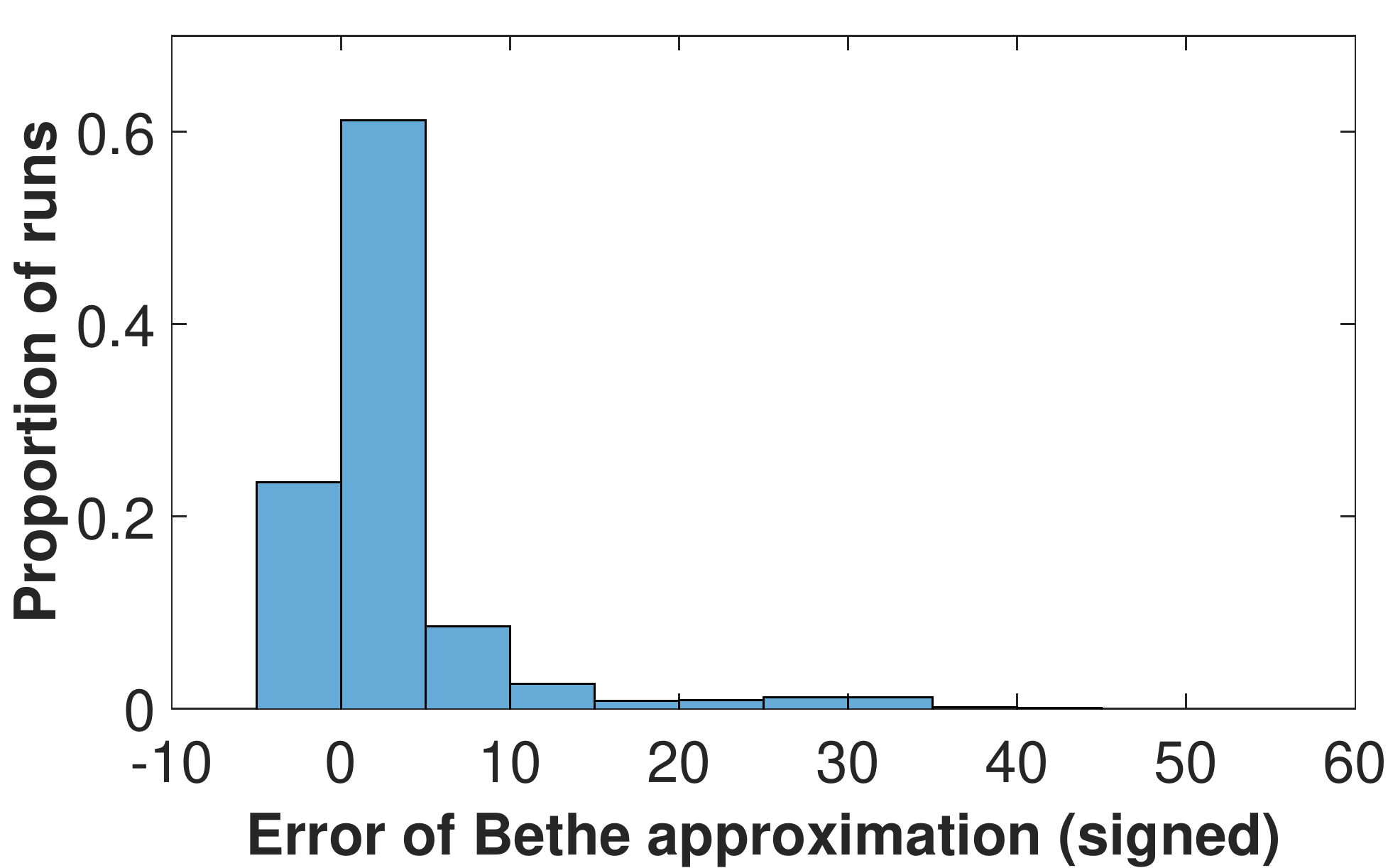} &
\includegraphics[width=.31\linewidth]{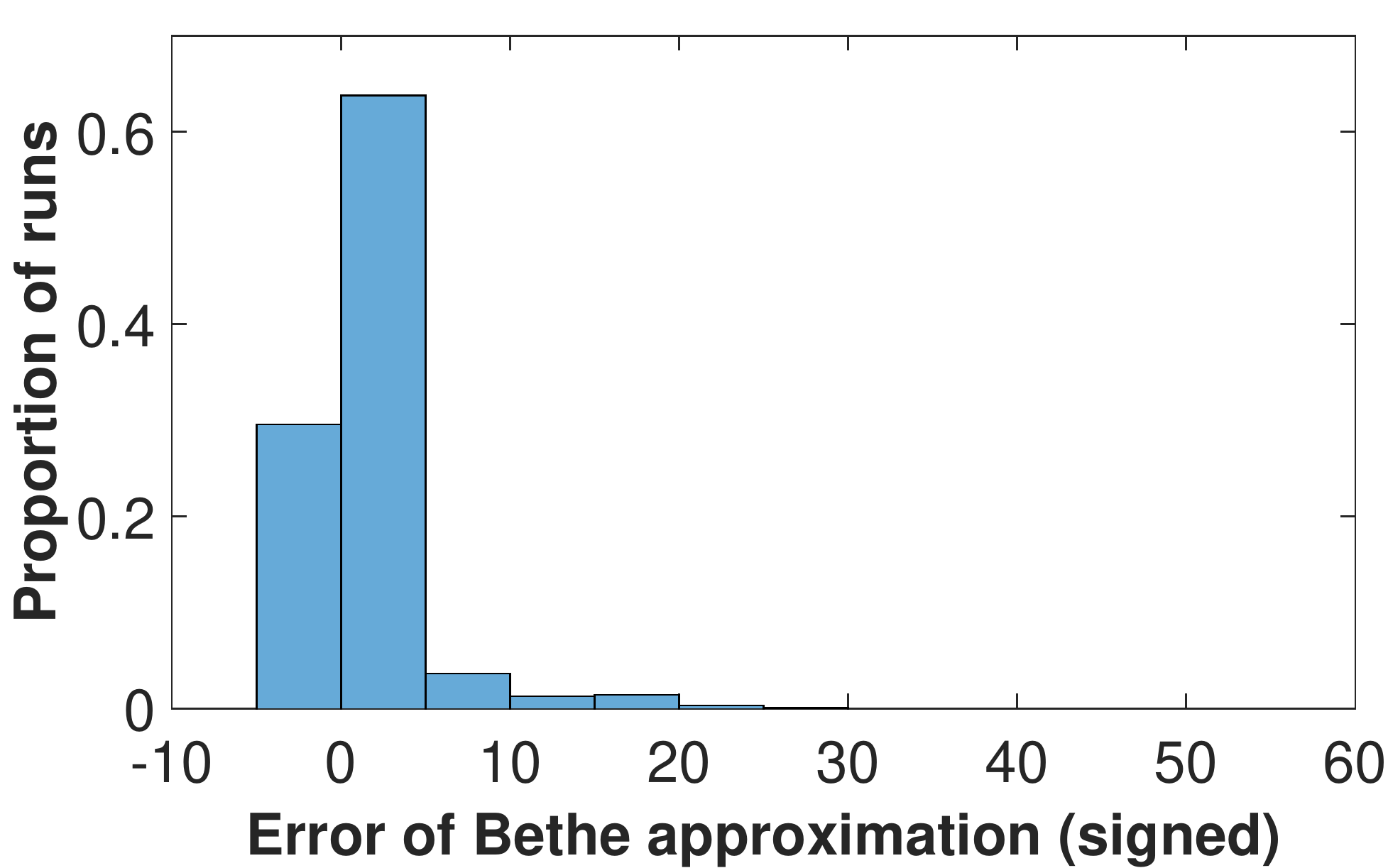} \\
{\small original model} & {\small 1 clamp} & {\small 2 clamps} 
\end{tabular}
\end{center}
\caption{\small Histograms of occurrences of \emph{signed} error bins of $\tilde A_B(\theta) - A(\theta)$ for Bethe, across all runs of \emph{mixed} models. This shows that the error is dominated by being too \emph{high}, particularly before clamping, as would be expected from the reasoning in \S \ref{sec:bal}.}
\label{fig:histB}
\end{figure}

%\subsubsection{Plots of $\log$ runtime vs error for all runs}

%%%%%%%%%%%%%%%%%%%%%%%%%%%%%%%%% success plots
\begin{figure}
\begin{center}
\setlength\tabcolsep{2pt}
\begin{tabular}{cccc}
\begin{sideways} {\small \- \; \; \qquad mixed} \end{sideways} &
\includegraphics[width=.32\linewidth]{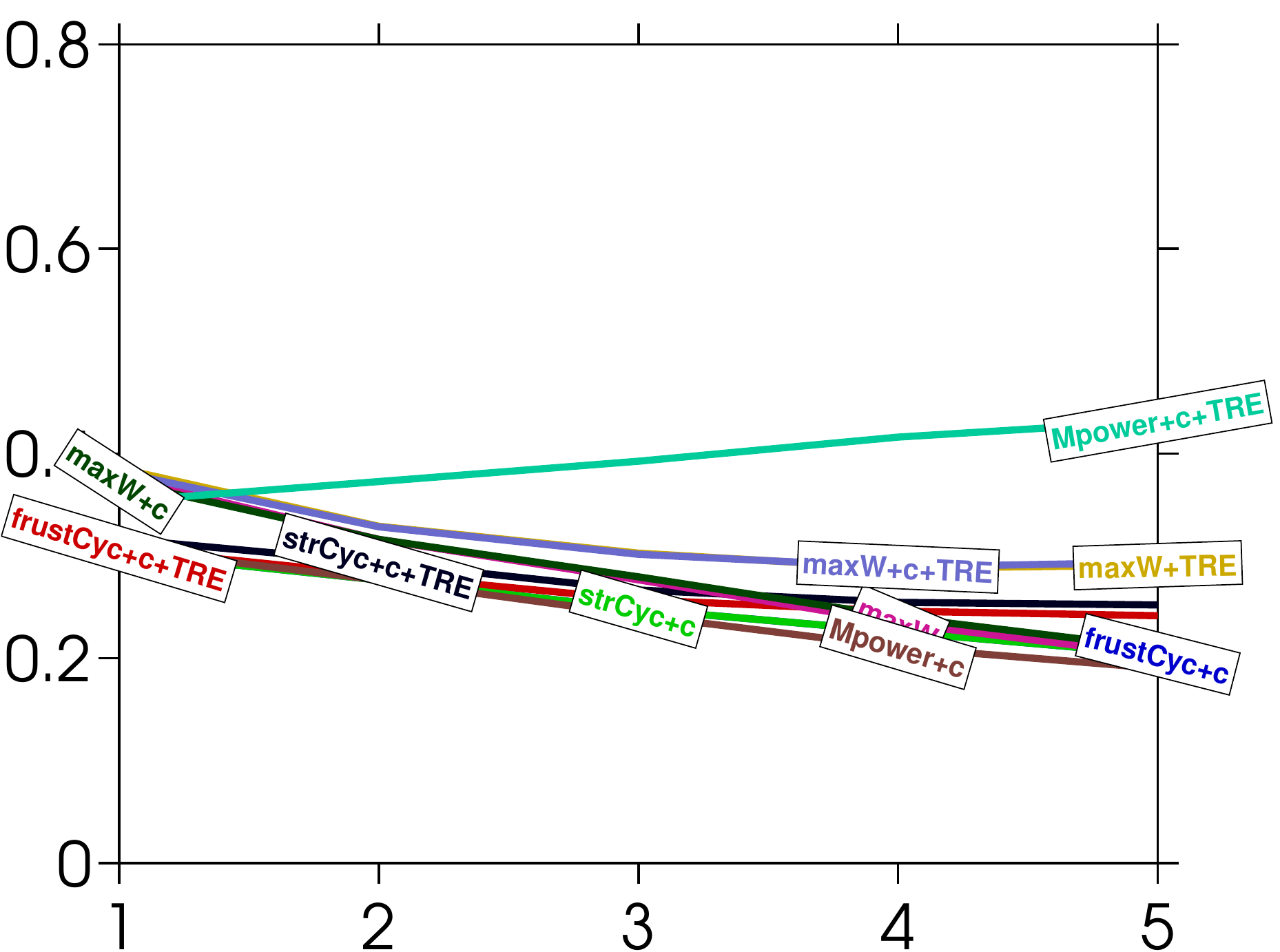} & 
\includegraphics[width=.32\linewidth]{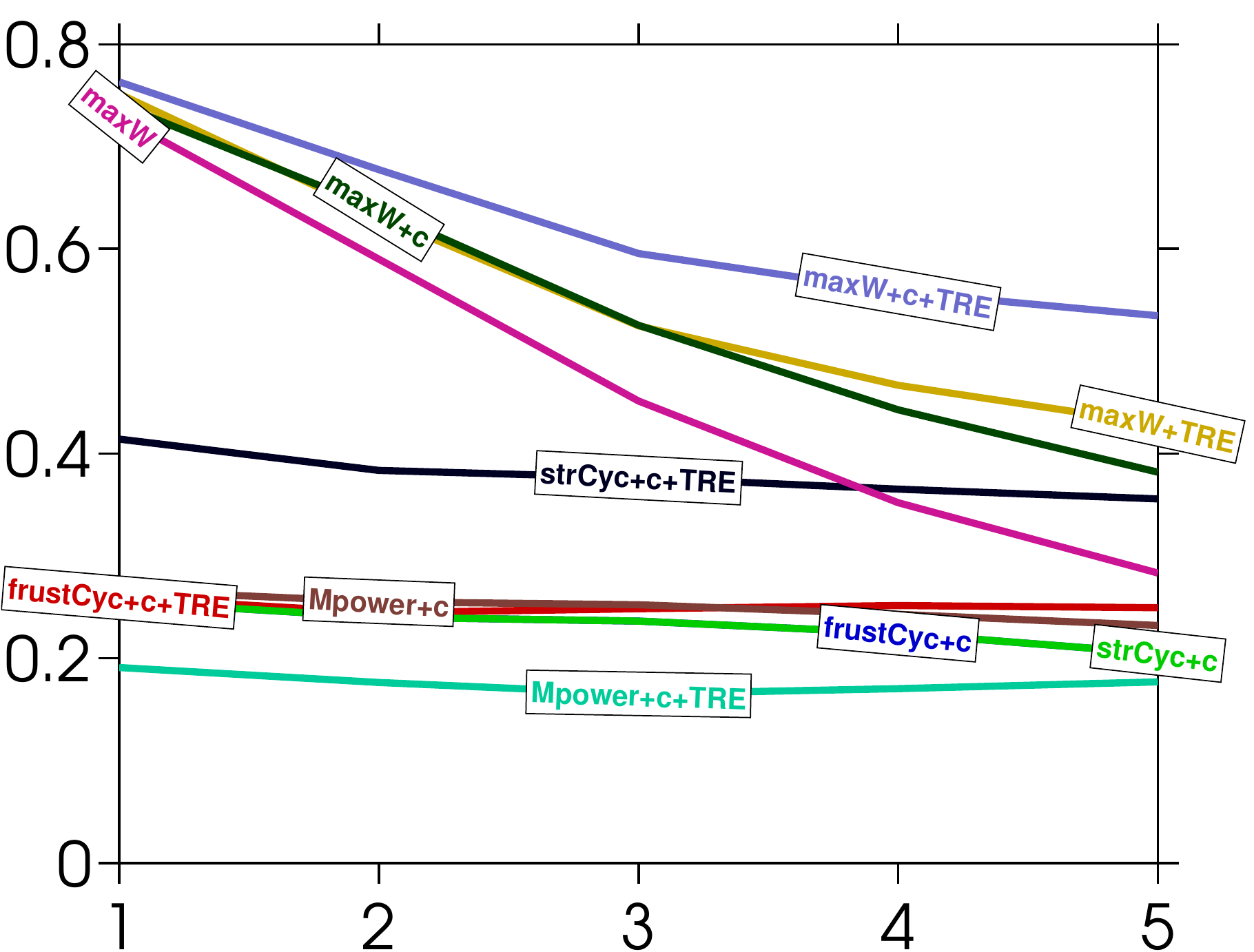} &
\includegraphics[width=.32\linewidth]{nsuccess_1_1-eps-converted-to.pdf} \\
\begin{sideways} {\small \- \; \; \qquad attractive} \end{sideways} &
\includegraphics[width=.32\linewidth]{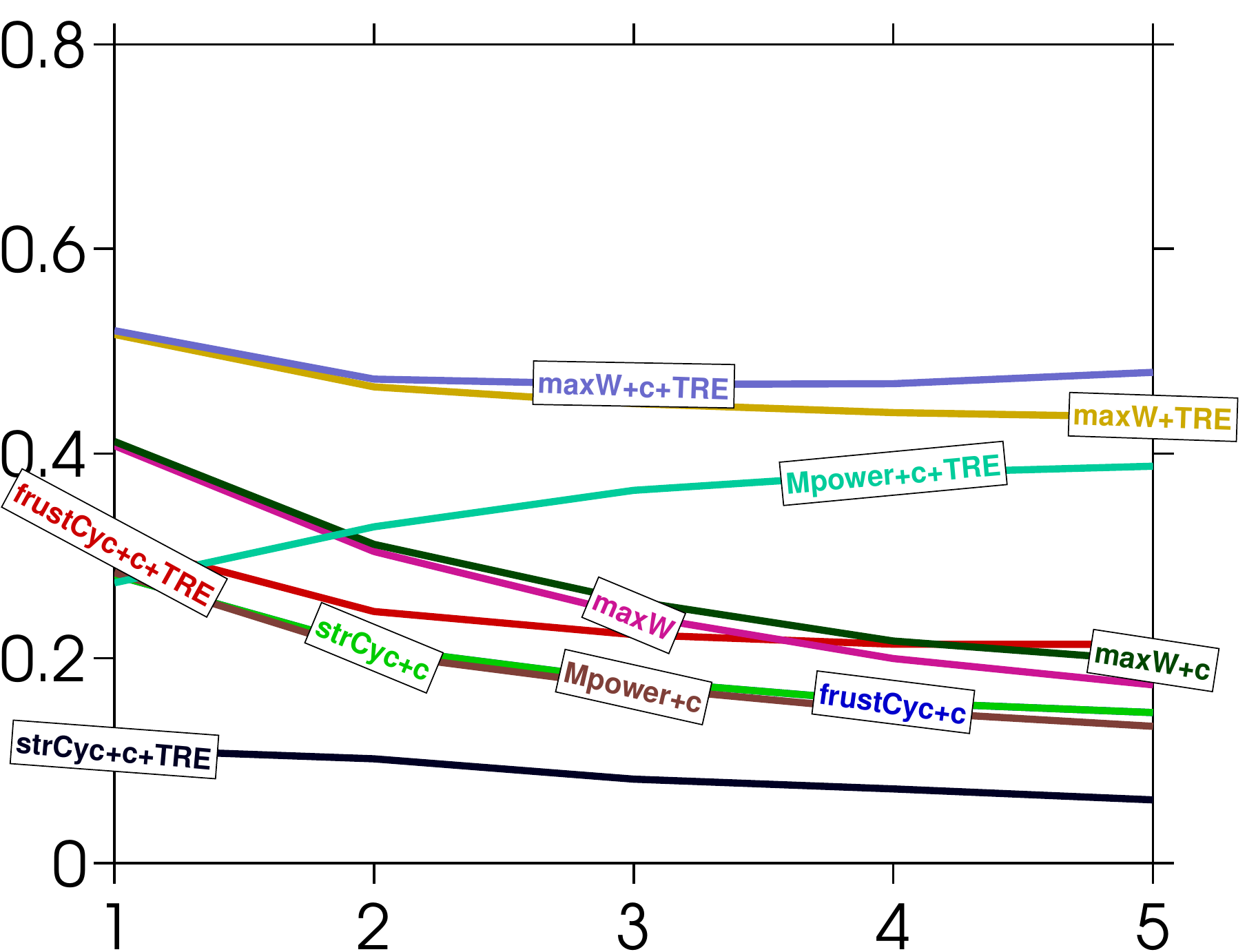} & 
\includegraphics[width=.32\linewidth]{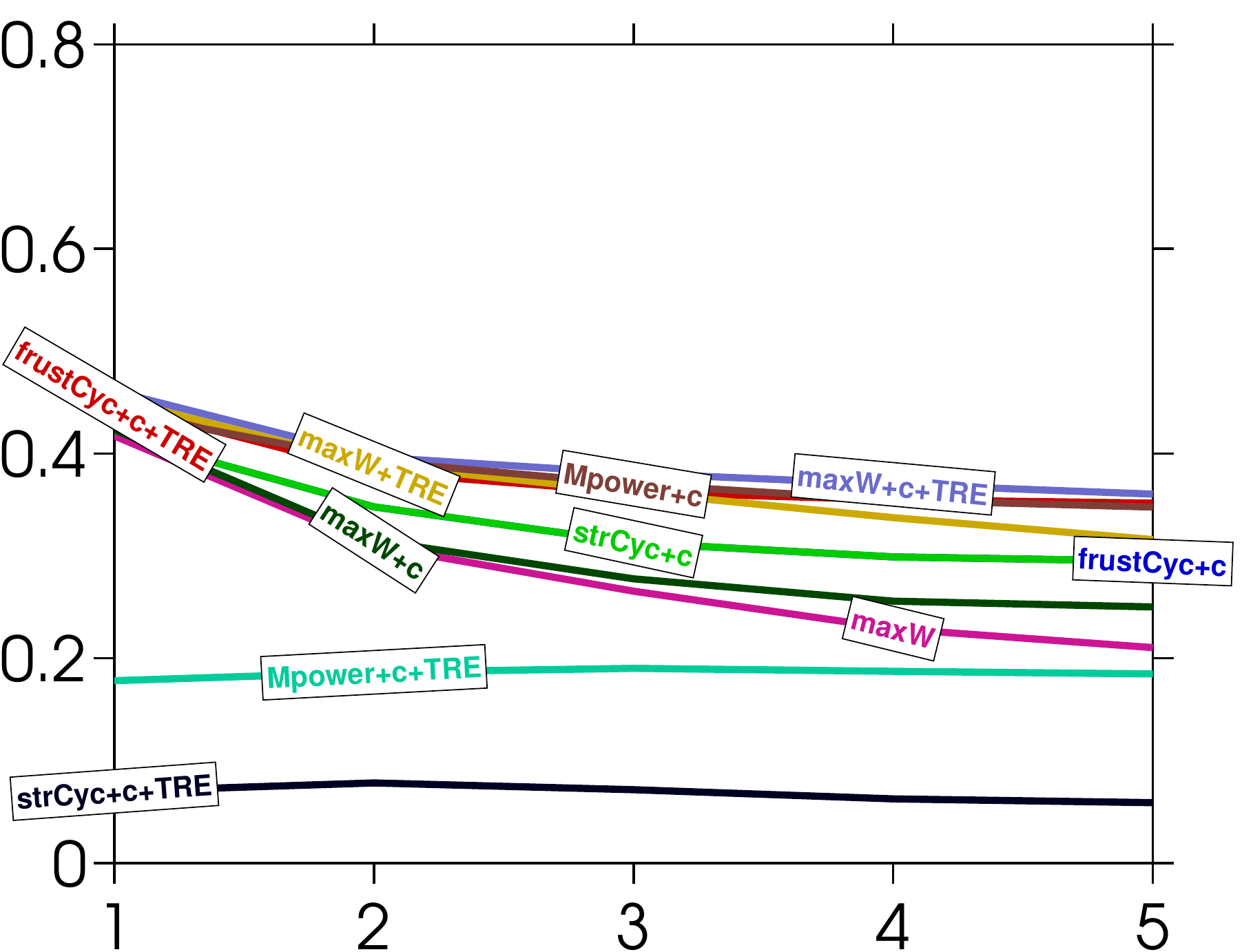} &
\includegraphics[width=.32\linewidth]{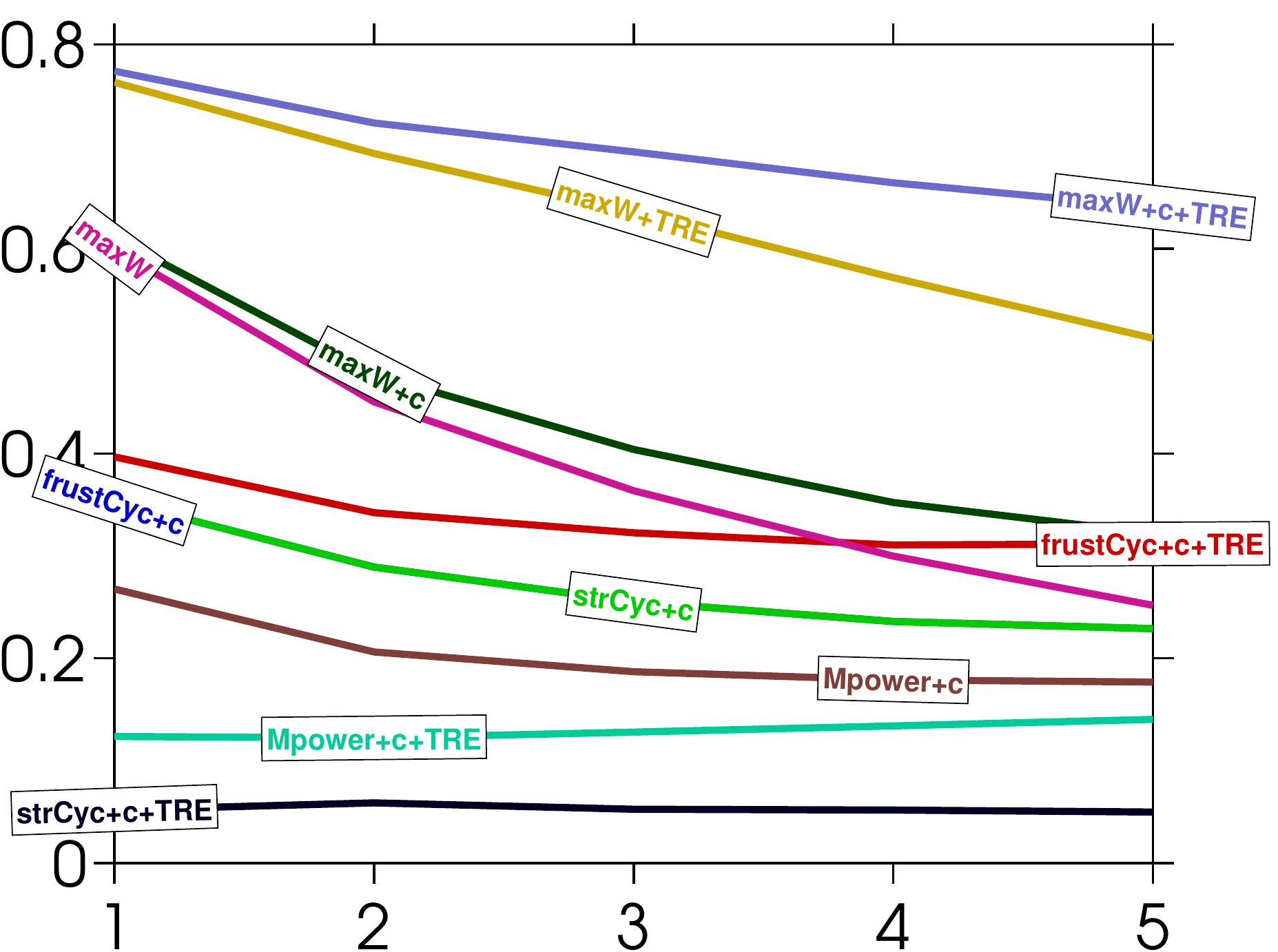} \\
& {\small MF} & {\small Bethe} & {\small TRW} 
\end{tabular}
\end{center}
\caption{\small Fraction of the time each heuristic picks the same variable to clamp as pseudo-greedy at that specific clamp step. Note that for mixed models, by our choice Bethe mimics TRW (empirically the best option: Bethe is not a bound in this case).}
\label{fig:succ}
\end{figure}

\section{Additional discussion on greedily selecting a variable to clamp}\label{sec:addl}

An interesting question is whether greedily picking the one variable that gives best error improvement and repeating say $k$ times is optimal, i.e. will it result in error as low as if instead, we try all possible \emph{sequences} of clampings up to $k$ long. It becomes computationally expensive to try this but we ran experiments out to 3 clampings. We observed that iterating a greedy search is \emph{not} optimal, in that the full optimization does perform better, but only by a very slight margin on the models we tried.

\end{document}